\documentclass{sig-alternate}
\usepackage{amsmath}
\usepackage{mathrsfs}
\usepackage[tight,footnotesize]{subfigure}
\usepackage{graphicx}
\usepackage{multirow}
\usepackage{array}
\usepackage{multirow}
\usepackage{mdwmath}
\usepackage{mdwtab}
\usepackage{amsfonts}
\usepackage{algorithm}
\usepackage{algorithmic}
\usepackage{url}
\usepackage{enumerate}

\newtheorem{definition}{Definition}
\newtheorem{theorem}{Theorem}
\newtheorem{lemma}{Lemma}
\newtheorem{corollary}{Corollary}
\newtheorem{remark}{Remark}

\newtheorem{proposition}{Proposition}

\begin{document}

\CopyrightYear{2007} 

\title{Fast Stochastic Variance Reduced Gradient Method with Momentum Acceleration for Machine Learning}

\numberofauthors{4} 
%
\author{
Fanhua Shang, Yuanyuan Liu, James Cheng, Jiacheng Zhuo\\
\affaddr{Department of Computer Science and Engineering, The Chinese University of Hong Kong}\\
\email{\{fhshang, yyliu, jcheng, jczhuo4\}@cse.cuhk.edu.hk}\\
}

\maketitle
\begin{abstract}
Recently, research on accelerated stochastic gradient descent methods (e.g., SVRG) has made exciting progress (e.g., linear convergence for strongly convex problems). However, the best-known methods (e.g., Katyusha) requires at least two auxiliary variables and two momentum parameters. In this paper, we propose a fast stochastic variance reduction gradient (FSVRG) method, in which we design a novel update rule with the Nesterov's momentum and incorporate the technique of growing epoch size. FSVRG has only one auxiliary variable and one momentum weight, and thus it is much simpler and has much lower per-iteration complexity. We prove that FSVRG achieves linear convergence for strongly convex problems and the optimal $\mathcal{O}(1/T^2)$ convergence rate for non-strongly convex problems, where $T$ is the number of outer-iterations. We also extend FSVRG to directly solve the problems with non-smooth component functions, such as SVM. Finally, we empirically study the performance of FSVRG for solving various machine learning problems such as logistic regression, ridge regression, Lasso and SVM. Our results show that FSVRG outperforms the state-of-the-art stochastic methods, including Katyusha.
\end{abstract}


\section{Introduction}
\noindent In this paper, we consider the following finite-sum composite convex optimization problem:
\vspace{-2mm}
\begin{equation}\label{equ1}
\min_{x\in\mathbb{R}^{d}} \phi(x):=f(x)+g(x)=\frac{1}{n}\!\sum\nolimits_{i=1}^{n}\!f_{i}(x)+g(x),
\vspace{-2mm}
\end{equation}
where $f(x)\!:=\!\frac{1}{n}\!\sum^{n}_{i=1}f_{i}(x)$ is a convex function that is a finite average of $n$ convex functions $f_{i}(x)\!:\!\mathbb{R}^{d}\!\rightarrow\!\mathbb{R}$, and $g(x)$ is a ``simple" possibly non-smooth convex function (referred to as a regularizer, e.g.\ $\lambda_{1}\|x\|^{2}$, the $\ell_{1}$-norm regularizer $\lambda_{2}\|x\|_{1}$, and the elastic net regularizer $\lambda_{1}\|x\|^{2}\!+\!\lambda_{2}\|x\|_{1}$, where $\lambda_{1},\lambda_{2}\!\geq\!0$ are the regularization parameters). Such a composite problem~\eqref{equ1} naturally arises in many applications of machine learning and data mining, such as regularized empirical risk minimization (ERM) and eigenvector computation~\cite{shamir:pca,garber:svd}. As summarized in~\cite{zhu:Katyusha,zhu:box}, there are mainly four interesting categories of Problem~\eqref{equ1} as follows:

\vspace{-2mm}
\begin{itemize}
\item Case 1: Each $f_{i}(x)$ is $L$-smooth and $\phi(x)$ is $\mu$-strongly convex ($\mu$-SC). Examples: ridge regression and elastic net regularized logistic regression.
\vspace{-2mm}
\item Case 2: Each $f_{i}(x)$ is $L$-smooth and $\phi(x)$ is non-strongly convex (NSC). Examples: Lasso and $\ell_{1}$-norm regularized logistic regression.
\vspace{-2mm}
\item Case 3: Each $f_{i}(x)$ is non-smooth (but Lipschitz continuous) and $\phi(x)$ is $\mu$-SC. Examples: linear support vector machine (SVM).
\vspace{-2mm}
\item Case 4: Each $f_{i}(x)$ is non-smooth (but Lipschitz continuous) and $\phi(x)$ is NSC. Examples: $\ell_{1}$-norm regularized SVM.
\vspace{-2mm}
\end{itemize}

To solve Problem~\eqref{equ1} with a large sum of $n$ component functions, computing the full (sub)gradient of $f(x)$ (e.g.\ $\nabla\!f(x)\!=\!\frac{1}{n}\!\sum^{n}_{i=1}\!\nabla\!f_{i}(x)$ for the smooth case) in first-order methods is expensive, and hence stochastic (sub)gradient descent (SGD), also known as incremental gradient descent, has been widely used in many large-scale problems~\cite{sutskever:sgd,zhang:sgd}. SGD approximates the gradient from just one example or a mini-batch, and thus it enjoys a low per-iteration computational complexity. Moreover, SGD is extremely simple and highly scalable, making it particularly suitable for large-scale machine learning, e.g., deep learning~\cite{sutskever:sgd}. However, the variance of the stochastic gradient estimator may be large in practice~\cite{johnson:svrg,zhao:sampling}, which leads to slow convergence and poor performance. Even for Case 1, standard SGD can only achieve a sub-linear convergence rate~\cite{rakhlin:sgd,shamir:sgd}.

Recently, the convergence speed of standard SGD has been dramatically improved with various variance reduced methods, such as SAG~\cite{roux:sag}, SDCA~\cite{shalev-Shwartz:sdca}, SVRG~\cite{johnson:svrg}, SAGA~\cite{defazio:saga}, and their proximal variants, such as~\cite{schmidt:sag}, \cite{shalev-Shwartz:acc-sdca}, \cite{xiao:prox-svrg} and \cite{koneeny:mini}. Indeed, many of those stochastic methods use past full gradients to progressively reduce the variance of stochastic gradient estimators, which leads to a revolution in the area of first-order methods. Thus, they are also called the semi-stochastic gradient descent method~\cite{koneeny:mini} or hybrid gradient descent method~\cite{zhang:svrg}. In particular, these recent methods converge linearly for Case 1, and their overall complexity (total number of component gradient evaluations to find an $\epsilon$-accurate solution) is $\mathcal{O}\!\left((n\!+\!{L}/{\mu})\log({1}/{\epsilon})\right)$, where $L$ is the Lipschitz constant of the gradients of $f_{i}(\cdot)$, and $\mu$ is the strong convexity constant of $\phi(\cdot)$. The complexity bound shows that those stochastic methods always converge faster than accelerated deterministic methods (e.g.\ FISTA~\cite{beck:fista})~\cite{koneeny:mini}. Moreover, \cite{zhu:vrnc} and \cite{reddi:svrnc} proved that SVRG with minor modifications can converge asymptotically to a stationary point in the non-convex case. However, there is still a gap between the overall complexity and the theoretical bound provided in~\cite{woodworth:bound}. For Case 2, they converge much slower than accelerated deterministic algorithms, i.e., $\mathcal{O}(1/T)$ vs.\ $\mathcal{O}(1/T^2)$.

More recently, some accelerated stochastic methods were proposed. Among them, the successful techniques mainly include the Nesterov's acceleration technique~\cite{lan:rpdg,lin:vrsg,nitanda:svrg}, the choice of growing epoch length~\cite{mahdavi:sgd,zhu:univr}, and the momentum acceleration trick~\cite{zhu:Katyusha,hien:asmd}. \cite{lin:vrsg} presents an accelerating Catalyst framework and achieves a complexity of $\mathcal{O}((n\!+\!\!\sqrt{n{L}/{\mu}})\log({L}/{\mu})\log({1}/{\epsilon}))$ for Case 1. However, adding a dummy regularizer hurts the performance of the algorithm both in theory and in practice~\cite{zhu:univr}. The methods~\cite{zhu:Katyusha,hien:asmd} attain the best-known complexity of $\mathcal{O}(n\sqrt{1/\epsilon}+\!\sqrt{nL/\epsilon})$ for Case 2. Unfortunately, they require at least two auxiliary variables and two momentum parameters, which lead to complicated algorithm design and high per-iteration complexity.

\textbf{Contributions:} To address the aforementioned weaknesses of existing methods, we propose a fast stochastic variance reduced gradient (FSVRG) method, in which we design a novel update rule with the Nesterov's momentum~\cite{nesterov:fast}. The key update rule has only one auxiliary variable and one momentum weight. Thus, FSVRG is much simpler and more efficient than~\cite{zhu:Katyusha,hien:asmd}. FSVRG is a direct accelerated method without using any dummy regularizer, and also works for non-smooth and proximal settings. Unlike most variance reduced methods such as SVRG, which only have convergence guarantee for Case 1, FSVRG has convergence guarantees for both Cases 1 and 2. In particular, FSVRG uses a flexible growing epoch size strategy as in~\cite{mahdavi:sgd} to speed up its convergence. Impressively, FSVRG converges much faster than the state-of-the-art stochastic methods. We summarize our main contributions as follows.

\vspace{-2mm}
\begin{itemize}
\item We design a new momentum accelerating update rule, and present two selecting schemes of momentum weights for Cases 1 and 2, respectively.
\vspace{-2mm}
\item We prove that FSVRG attains linear convergence for Case 1, and achieves the convergence rate of $\mathcal{O}(1/T^2)$ and a complexity of $\mathcal{O}(n\sqrt{1/\epsilon}\!+\!\sqrt{nL/\epsilon})$ for Case 2, which is the same as the best known result in~\cite{zhu:Katyusha}.
\vspace{-2mm}
\item Finally, we also extend FSVRG to mini-batch settings and non-smooth settings (i.e., Cases 3 and 4), and provide an empirical study on the performance of FSVRG for solving various machine learning problems.
\end{itemize}

\section{Preliminaries}
Throughout this paper, the norm $\|\!\cdot\!\|$ is the standard Euclidean norm, and $\|\!\cdot\!\|_{1}$ is the $\ell_{1}$-norm, i.e., $\|x\|_{1}\!=\!\sum_{i}\!|x_{i}|$. We denote by $\nabla\!f(x)$ the full gradient of $f(x)$ if it is differentiable, or $\partial\!f(x)$ a sub-gradient of $f(x)$ if $f(x)$ is only Lipschitz continuous. We mostly focus on the case of Problem~\eqref{equ1} when each $f_{i}(x)$ is $L$-smooth\footnote{In fact, we can extend all our theoretical results below for this case (i.e., when the gradients of all component functions have the same Lipschitz constant $L$) to the more general case, when some $f_{i}(x)$ have different degrees of smoothness.}. For non-smooth component functions, we can use the proximal operator oracle~\cite{zhu:box} or the Nesterov's smoothing~\cite{nesterov:smooth} and homotopy smoothing~\cite{xu:hs} techniques to smoothen them, and then obtain the smoothed approximations of all functions $f_{i}(\cdot)$.

When the regularizer $g(\cdot)$ is non-smooth (e.g., $\!g(\cdot)\!=\!\lambda\|\cdot\|_{1}\!$), the update rule of general SGD is formulated as follows:
\vspace{-1mm}
\begin{equation}\label{equ2}
x_{k}=\mathop{\arg\min}_{y\in\mathbb{R}^{d}}\, g(y)\!+\!y^{T}\nabla\! f_{i_{k}}\!(x_{k-\!1})\!+\!({1}/{2\eta_{k}})\!\cdot\!\|y\!-\!x_{k-\!1}\|^2,
\vspace{-2mm}
\end{equation}
where $\eta_{k}\!\propto\!1/k$ is the step size (or learning rate), and $i_{k}$ is chosen uniformly at random from $\{1,\ldots,n\}$. When $g(x)\!\equiv\!0$, the update rule in~\eqref{equ2} becomes $x_{k}\!=\!x_{k-\!1}\!-\!\eta_{k}\nabla\!f_{i_{k}}\!(x_{k-\!1})$. If each $f_{i}(\cdot)$ is non-smooth (e.g., the hinge loss), we need to replace $\nabla\! f_{i_{k}}\!(x_{k-\!1})$ in~\eqref{equ2} with $\partial\!f_{i_{k}}\!(x_{k-\!1})$.

As the representative methods of stochastic variance reduced optimization, SVRG~\cite{johnson:svrg} and its proximal variant, Prox-SVRG~\cite{xiao:prox-svrg}, are particularly attractive because of their low storage requirement compared with~\cite{roux:sag,shalev-Shwartz:sdca,defazio:saga,shalev-Shwartz:acc-sdca}, which need to store all the gradients of the $n$ component functions $f_{i}(\cdot)$ (or dual variables), so that $O(nd)$ storage is required in general problems. At the beginning of each epoch of SVRG, the full gradient $\nabla\! f(\widetilde{x})$ is computed at the snapshot point $\widetilde{x}$. With a constant step size $\eta$, the update rules for the special case of Problem~\eqref{equ1} (i.e., $g(x)\!\equiv\!0$) are given by
\vspace{-1mm}
\begin{equation}\label{equ3}
\begin{split}
\widetilde{\nabla}\! f_{i_{k}}\!(x_{k-1})&=\nabla\! f_{i_{k}}\!(x_{k-1})-\nabla\! f_{i_{k}}\!(\widetilde{x})+\nabla\! f(\widetilde{x}),\\
x_{k}&=x_{k-1}-\eta\widetilde{\nabla}\! f_{i_{k}}\!(x_{k-1}).
\end{split}
\vspace{-2mm}
\end{equation}
\cite{zhu:univr} proposed an accelerated SVRG method, SVRG++~, with doubling-epoch techniques. Moreover, Katyusha~\cite{zhu:Katyusha} is a direct accelerated stochastic variance reduction method, and its main update rules are formulated as follows:
\vspace{-1mm}
\begin{equation}\label{equ4}
\begin{split}
&x_{k}=\theta_{1}y_{k-\!1}+\theta_{2}\widetilde{x}+(1-\theta_{1}-\theta_{2})z_{k-\!1},\\
&y_{k}=\mathop{\arg\min}_{y\in\mathbb{R}^{d}}\, g(y)\!+\!y^{T}\widetilde{\nabla}\! f_{i_{k}}\!(x_{k})\!+\!({1}/{2\eta})\!\cdot\!\|y\!-\!y_{k-\!1}\|^2,\\
\vspace{-2mm}
&z_{k}=\mathop{\arg\min}_{z\in\mathbb{R}^{d}}\, g(z)\!+\!z^{T}\widetilde{\nabla}\! f_{i_{k}}\!(x_{k})\!+\!({3L}/{2})\!\cdot\!\|z\!-\!x_{k}\|^2,
\end{split}
\vspace{-2mm}
\end{equation}
where $\theta_{1},\theta_{2}\!\in\![0,1]$ are two parameters, and $\theta_{2}$ is fixed to $0.5$ in~\cite{zhu:Katyusha} to eliminate the need for parameter tuning.

\section{Fast SVRG with Momentum Acceleration}
In this paper, we propose a fast stochastic variance reduction gradient (FSVRG) method with momentum acceleration for Cases 1 and 2 (e.g., logistic regression) and Cases 3 and 4 (e.g., SVM). The acceleration techniques of the classical Nesterov's momentum and the Katyusha momentum in~\cite{zhu:Katyusha} are incorporated explicitly into the well-known SVRG method~\cite{johnson:svrg}. Moreover, FSVRG also uses a growing epoch size strategy as in~\cite{mahdavi:sgd} to speed up its convergence.

\subsection{Smooth Component Functions}
In this part, we consider the case of Problem (\ref{equ1}) when each $f_{i}(\cdot)$ is smooth, and $\phi(\cdot)$ is SC or NSC (i.e., Case 1 or 2). Similar to existing stochastic variance reduced methods such as SVRG~\cite{zhu:Katyusha} and Prox-SVRG~\cite{xiao:prox-svrg}, we design a simple fast stochastic variance reduction algorithm with momentum acceleration for solving smooth objective functions, as outlined in Algorithm~\ref{alg1}. It is clear that Algorithm~\ref{alg1} is divided into $S$ epochs (which is the same as most variance reduced methods, e.g., SVRG and Katyusha), and each epoch consists of $m_{s}$ stochastic updates, where $m_{s}$ is set to $m_{s}\!=\!\rho^{s-\!1}\!\cdot m_{1}$ as in~\cite{mahdavi:sgd}, where $m_{1}$ is a given initial value, and $\rho\!>\!1$ is a constant. Within each epoch, a full gradient $\nabla\! f(\widetilde{x}^{s})$ is calculated at the snapshot point $\widetilde{x}^{s}$. Note that we choose $\widetilde{x}^{s}$ to be the average of the past $m_{s}$ stochastic iterates rather than the last iterate because it has been reported to work better in practice~\cite{xiao:prox-svrg,zhu:univr,zhu:Katyusha}. Although our convergence guarantee for the SC case depends on the initialization of $x^{s}_{0}\!=\!y^{s}_{0}\!=\!\widetilde{x}^{s-\!1}$, the choices of $x^{s+\!1}_{0}\!=\!x^{s}_{m_{s}}$ and $y^{s+\!1}_{0}\!=\!y^{s}_{m_{s}}$ also work well in practice, especially for the case when the regularization parameter is relatively small (e.g., $10^{-7}$), as suggested in~\cite{shang:vrsgd}.

\begin{algorithm}[t]
\caption{FSVRG for smooth component functions}
\label{alg1}
\renewcommand{\algorithmicrequire}{\textbf{Input:}}
\renewcommand{\algorithmicensure}{\textbf{Initialize:}}
\renewcommand{\algorithmicoutput}{\textbf{Output:}}
\begin{algorithmic}[1]
\REQUIRE the number of epochs $S$ and the step size $\eta$.\\
\ENSURE $\widetilde{x}^{0}$\!, $m_{1}$, $\theta_{1}$, and $\rho>1$.\\
\FOR{$s=1,2,\ldots,S$}
\STATE {$\widetilde{\mu}=\frac{1}{n}\!\sum^{n}_{i=1}\!\nabla\!f_{i}(\widetilde{x}^{s-\!1})$, $x^{s}_{0}=y^{s}_{0}=\widetilde{x}^{s-\!1}$;}
\FOR{$k=1,2,\ldots,m_{s}$}
\STATE {Pick $i^{s}_{k}$ uniformly at random from $\{1,\ldots,n\}$;}
\STATE {$\widetilde{\nabla} f_{i^{s}_{k}}(x^{s}_{k-\!1})=\nabla f_{i^{s}_{k}}(x^{s}_{k-\!1})-\nabla f_{i^{s}_{k}}(\widetilde{x}^{s-\!1})+\widetilde{\mu}$;}
\STATE {$y^{s}_{k}=y^{s}_{k-\!1}-\eta\;\![\widetilde{\nabla}f_{i^{s}_{k}}(x^{s}_{k-\!1})+\nabla g(x^{s}_{k-\!1})]$;}
\STATE {$x^{s}_{k}=\widetilde{x}^{s-1}+\theta_{s}(y^{s}_{k}-\widetilde{x}^{s-1})$;}
\ENDFOR
\STATE {$\widetilde{x}^{s}=\frac{1}{m_{s}}\!\sum^{m_{s}}_{k=1}\!x^{s}_{k}$, $\,m_{s+1}=\lceil\rho^{s}\!\cdot m_{1}\rceil$;}
\ENDFOR
\OUTPUT {$\widetilde{x}^{S}$}
\end{algorithmic}
\end{algorithm}

\subsubsection{Momentum Acceleration}
When the regularizer $g(\cdot)$ is smooth, e.g., the $\ell_{2}$-norm regularizer, the update rule of the auxiliary variable $y$ is
\begin{equation}\label{equ5}
y^{s}_{k}=y^{s}_{k-\!1}-\eta[\widetilde{\nabla}f_{i^{s}_{k}}(x^{s}_{k-\!1})+\nabla g(x^{s}_{k-\!1})].
\end{equation}
When $g(\cdot)$ is non-smooth, e.g., the $\ell_{1}$-norm regularizer, the update rule of $y$ is given as follows:
\vspace{-1mm}
\begin{equation}\label{equ6}
y^{s}_{k}=\textup{prox}_{\,\eta,g}\!\left(y^{s}_{k-\!1}-\eta\widetilde{\nabla}\!f_{i^{s}_{k}}(x^{s}_{k-\!1})\right)\!,
\vspace{-1mm}
\end{equation}
and the proximal operator $\textup{prox}_{\,\eta,g}(\cdot)$ is defined as
\vspace{-1mm}
\begin{equation}\label{equ7}
\textup{prox}_{\,\eta,g}(y):=\mathop{\arg\min}_{x}({1}/{2\eta})\!\cdot\!\|x-y\|^{2}+g(x).
\vspace{-1mm}
\end{equation}
That is, we only need to replace the update rule (\ref{equ5}) in Algorithm~\ref{alg1} with (\ref{equ7}) for the case of non-smooth regularizers.

Inspired by the momentum acceleration trick for accelerating first-order optimization methods~\cite{nesterov:fast,nitanda:svrg,zhu:Katyusha}, we design the following update rule for $x$:
\vspace{-1mm}
\begin{equation}\label{equ8}
x^{s}_{k}=\widetilde{x}^{s-1}+\theta_{s}(y^{s}_{k}-\widetilde{x}^{s-1}),
\end{equation}
where $\theta_{s}\!\in\![0,1]$ is the weight for the key momentum term. The first term of the right-hand side of (\ref{equ8}) is the snapshot point of the last epoch (also called as the Katyusha momentum in~\cite{zhu:Katyusha}), and the second term plays a key role as the Nesterov's momentum in deterministic optimization.

When $\theta_{s}\!\equiv\!1$ and $\rho\!=\!2$, Algorithm~\ref{alg1} degenerates to the accelerated SVRG method, SVRG++~\cite{zhu:univr}. In other words, SVRG++ can be viewed as a special case of our FSVRG method. As shown above, FSVRG only has one additional variable $y$, while existing accelerated stochastic variance reduction methods, e.g., Katyusha~\cite{zhu:Katyusha}, require two additional variables $y$ and $z$, as shown in (\ref{equ4}). In addition, FSVRG only has one momentum weight $\theta_{s}$, compared with the two weights $\theta_{1}$ and $\theta_{2}$ in Katyusha~\cite{zhu:Katyusha}. Therefore, FSVRG is much simpler than existing accelerated methods~\cite{zhu:Katyusha,hien:asmd}.

\subsubsection{Momentum Weight}
For the case of SC objectives, we give a selecting scheme for the momentum weight $\theta_{s}$. As shown in Theorem~\ref{theo1} below, it is desirable to have a small convergence factor $\alpha$, implying fast convergence. The following proposition obtains the optimal $\theta_{\star}$, which can yield the smallest $\alpha$ value.

\begin{proposition}
Given the appropriate learning rate $\eta$, the optimal weight $\theta_{\star}$ is given by
\vspace{-2mm}
\begin{equation}\label{equ10}
\theta_{\star}=\mu\eta m_{s}/2.
\end{equation}
\end{proposition}

\begin{proof}
Using Theorem 1 below, we have
\begin{equation*}
\alpha(\theta)=1-\theta+{\theta^{2}}/({\mu\eta m_{s}}).
\end{equation*}
To minimize $\alpha(\theta)$ with given $\eta$, we have $\theta_{\star}\!=\!\mu\eta m_{s}/2$.
\end{proof}
\vspace{-2mm}

In fact, we can fix $\theta_{s}$ to a constant for the case of SC objectives, e.g., $\theta_{s}\!\equiv\!0.9$ as in accelerated SGD~\cite{ruder:sgd}, which works well in practice. Indeed, larger values of $\theta_{s}$ can result in better performance for the case when the regularization parameter is relatively large (e.g., $10^{-4}$).

Unlike the SC case, we initialize $y^{s+\!1}_{0}\!=\!y^{s}_{m_{s}}$ in each epoch for the case of NSC objectives. And the update rule of $\theta_{s}$ is defined as follows: $\theta_{1}\!=\!1\!-\!{L\eta}/({1\!-\!L\eta})$, and for any $s\!>\!1$,
\vspace{-1mm}
\begin{equation}\label{equ11}
\theta_{s}=(\sqrt{\theta^{4}_{s-\!1}+4\theta^{2}_{s-\!1}}-\theta^{2}_{s-\!1})/{2}.
\end{equation}
The above rule is the same as that in some accelerated optimization methods~\cite{nesterov:co,su:nag,liu:sadmm}.

\subsubsection{Complexity Analysis}
The per-iteration cost of FSVRG is dominated by the computation of $\nabla\! f_{i^{s}_{k}}\!(x^{s}_{k-\!1})$, $\nabla\! f_{i^{s}_{k}}\!(\widetilde{x}^{s})$, and $\nabla\!g(x^{s}_{k-\!1})$ or the proximal update~\eqref{equ6}, which is as low as that of SVRG~\cite{johnson:svrg} and SVRG++~\cite{zhu:univr}. For some ERM problems, we can save the intermediate gradients $\nabla\! f_{i}(\widetilde{x}^{s})$ in the computation of $\widetilde{\mu}$, which requires $O(n)$ additional storage in general. In addition, FSVRG has a much lower per-iteration complexity than other accelerated methods such as Katyusha~\cite{zhu:Katyusha}, which have at least one more variable, as analyzed above.

\subsection{Non-Smooth Component Functions}
In this part, we consider the case of Problem (\ref{equ1}) when each $f_{i}(\cdot)$ is non-smooth (e.g., hinge loss and other loss functions listed in~\cite{yang:ssgd}), and $\phi(\cdot)$ is SC or NSC (i.e.\ Case 3 or 4). As stated in Section 2, the two classes of problems can be transformed into the smooth ones as in~\cite{nesterov:smooth,zhu:box,xu:hs}, which can be efficiently solved by Algorithm~\ref{alg1}. However, the smoothing techniques may degrade the performance of the involved algorithms, similar to the case of the reduction from NSC problems to SC problems~\cite{zhu:box}. Thus, we extend Algorithm~\ref{alg1} to the non-smooth setting, and propose a fast stochastic variance reduced sub-gradient algorithm (i.e., Algorithm~\ref{alg2}) to solve such problems directly, as well as the case of Algorithm~\ref{alg1} to directly solve the NSC problems in Case 2.

For each outer iteration $s$ and inner iteration $k$, we denote by $\widetilde{\partial}f_{i^{s}_{k}}\!(x^{s}_{k-\!1})$ the stochastic sub-gradient $\partial f_{i^{s}_{k}}\!(x^{s}_{k-\!1})\!-\!\partial f_{i^{s}_{k}}\!(\widetilde{x}^{s-\!1})\!+\!\widetilde{\xi}$, where $\widetilde{\xi}\!=\!\!\frac{1}{n}\!\!\sum^{n}_{i=1}\!\!\partial f_{i}(\widetilde{x}^{s-\!1})$, and $\partial f_{i}(\widetilde{x}^{s-\!1})$ denotes a sub-gradient of $f_{i}(\cdot)$ at $\widetilde{x}^{s-\!1}$. When the regularizer $g(\cdot)$ is smooth, the update rule of $y$ is given by
\vspace{-1mm}
\begin{equation}\label{equ12}
y^{s}_{k}=\Pi_{\mathcal{K}}\!\!\left[y^{s}_{k-\!1}-\eta\left(\widetilde{\partial}f_{i^{s}_{k}}\!(x^{s}_{k-\!1})+\nabla g(x^{s}_{k-\!1})\right)\right]\!,
\end{equation}
where $\Pi_{\mathcal{K}}$ denotes the orthogonal projection on the convex domain $\mathcal{K}$. Following the acceleration techniques for stochastic sub-gradient methods~\cite{rakhlin:sgd,julien:ssg,shamir:sgd}, a general weighted averaging scheme is formulated as follows:
\vspace{-2mm}
\begin{equation}\label{equ13}
\widetilde{x}^{s}=\frac{1}{\sum_{k}\! w_{k}}\!\sum^{m_{s}}_{k=1}w_{k}x^{s}_{k},
\vspace{-1mm}
\end{equation}
where $w_{k}$ is the given weight, e.g., $w_{k}\!=\!1/m_{s}$.

\begin{algorithm}[t]
\caption{FSVRG for non-smooth component functions}
\label{alg2}
\renewcommand{\algorithmicrequire}{\textbf{Input:}}
\renewcommand{\algorithmicensure}{\textbf{Initialize:}}
\renewcommand{\algorithmicoutput}{\textbf{Output:}}
\begin{algorithmic}[1]
\REQUIRE the number of epochs $S$ and the step size $\eta$.\\
\ENSURE $\widetilde{x}^{0}$\!, $m_{1}$, $\theta_{1}$, $\rho>1$, and $w_{1},\ldots,w_{m}$.\\
\FOR{$s=1,2,\ldots,S$}
\STATE {$\widetilde{\xi}=\frac{1}{n}\!\sum^{n}_{i=1}\!\partial f_{i}(\widetilde{x}^{s-\!1})$, $x^{s}_{0}=y^{s}_{0}=\widetilde{x}^{s-\!1}$;}
\FOR{$k=1,2,\ldots,m_{s}$}
\STATE {Pick $i^{s}_{k}$ uniformly at random from $\{1,\ldots,n\}$;}
\STATE {$\widetilde{\partial}f_{i^{s}_{k}}(x^{s}_{k-\!1})=\partial f_{i^{s}_{k}}(x^{s}_{k-\!1})-\partial f_{i^{s}_{k}}(\widetilde{x}^{s-\!1})+\widetilde{\xi}$;}
\STATE {$y^{s}_{k}=\Pi_{\mathcal{K}}\!\!\left[y^{s}_{k-\!1}-\eta\;\!(\widetilde{\partial}f_{i^{s}_{k}}\!(x^{s}_{k-\!1})+\nabla g(x^{s}_{k-\!1}))\right]$;}
\STATE {$x^{s}_{k}=\widetilde{x}^{s-1}+\theta_{s}(y^{s}_{k}-\widetilde{x}^{s-1})$;}
\ENDFOR
\STATE {$\widetilde{x}^{s}\!=\!\frac{1}{\sum_{k}\! w_{k}}\!\sum^{m_{s}}_{k=1}\!w_{k} x^{s}_{k}$, $\,m_{s+1}=\lceil\rho^{s}\!\cdot m_{1}\rceil$;}
\ENDFOR
\OUTPUT {$\widetilde{x}^{S}$}
\end{algorithmic}
\end{algorithm}

\section{Convergence Analysis}
In this section, we provide the convergence analysis of FSVRG for solving the two classes of problems in Cases 1 and 2. Before giving a key intermediate result, we first introduce the following two definitions.

\begin{definition}[Smoothness]\label{assum1}
A function $h(\cdot):\mathbb{R}^{d}\!\rightarrow\!\mathbb{R}$ is $L$-smooth if its gradient is $L$-Lipschitz, that is, $\|\nabla h(x)-\nabla h(y)\|\leq L\|x-y\|$ for all $x,y\in\mathbb{R}^{d}$.
\end{definition}

\begin{definition}[Strong Convexity]\label{assum2}
A function $\phi(\cdot):\mathbb{R}^{d}\!\rightarrow\!\mathbb{R}$ is $\mu$-strongly convex ($\mu$-SC), if there exists a constant $\mu\!>\!0$ such that for any $x,y\!\in\!\mathbb{R}^{d}$,
\begin{equation}\label{equ30}
\phi(y)\geq \phi(x)+\nabla\phi(x)^{T}(y-x)+\frac{\mu}{2}\|y-x\|^{2}.
\vspace{-2mm}
\end{equation}
\end{definition}
If $\phi(\cdot)$ is non-smooth, we can revise the inequality~\eqref{equ30} by simply replacing $\nabla\phi(x)$ with an arbitrary sub-gradient $\partial\phi(x)$.

\begin{lemma}\label{lemm1}
Suppose each component function $f_{i}(\cdot)$ is $L$-smooth. Let $x_{\star}$ be the optimal solution of Problem~\eqref{equ1}, and $\{\widetilde{x}^{s},y^{s}_{k}\}$ be the sequence generated by Algorithm~\ref{alg1}. Then the following inequality holds for all $s\!=\!1,\ldots,S$:
\begin{equation}\label{equ31}
\begin{split}
&\!\!\!\mathbb{E}\!\left[\phi(\widetilde{x}^{s})\!-\!\phi(x_{\star})\right]\leq(1\!-\!\theta_{s})\mathbb{E}\!\left[\phi(\widetilde{x}^{s-\!1})-\phi(x_{\star})\right]\\
&\quad\qquad\quad\quad\quad\;\;\;\;+\!\frac{\theta^{2}_{s}}{2\eta m_{s}}\mathbb{E}\!\left[\|y^{s}_{0}\!-\!x_{\star}\|^2\!-\!\|y^{s}_{m_{s}}\!\!-\!x_{\star}\|^2\right]\!.
\end{split}
\end{equation}
\end{lemma}

The detailed proof of Lemma~\ref{lemm1} is provided in APPENDIX. To prove Lemma 1, we first give the following lemmas, which are useful for the convergence analysis of FSVRG.

\begin{lemma}[Variance bound, \cite{zhu:Katyusha}]
\label{lemm2}
Suppose each function $f_{i}(\cdot)$ is $L$-smooth. Then the following inequality holds:
\begin{equation*}
\begin{split}
&\mathbb{E}\!\left[\left\|\widetilde{\nabla}\! f_{i^{s}_{k}}(x^{s}_{k-1})-\nabla\! f(x^{s}_{k-1})\right\|^{2}\right]\\
\leq&\,2L\!\left[f(\widetilde{x}^{s-\!1})-f(x^{s}_{k-\!1})+[\nabla\!f(x^{s}_{k-\!1})]^{T}(x^{s}_{k-\!1}-\widetilde{x}^{s-\!1})\right].
\end{split}
\end{equation*}
\end{lemma}

Lemma~\ref{lemm2} is essentially identical to Lemma 3.4 in~\cite{zhu:Katyusha}. This lemma provides a tighter upper bound on the expected variance of the variance-reduced gradient estimator $\widetilde{\nabla}\! f_{i^{s}_{k}}\!(x^{s}_{k-\!1})$ than that of~\cite{xiao:prox-svrg,zhu:univr}, e.g., Corollary 3.5 in~\cite{xiao:prox-svrg}.

\begin{lemma} [3-point property, \cite{lan:sgd}]
\label{lemm3}
Assume that $z^{*}$ is an optimal solution of the following problem,
\begin{displaymath}
\min_{z}({\tau}/{2})\!\cdot\!\|z-z_{0}\|^{2}+\psi(z),
\end{displaymath}
where $\psi(z)$ is a convex function (but possibly non-differentiable). Then for any $z\!\in\!\mathbb{R}^{d}$, we have
\begin{displaymath}
\psi(z)+\frac{\tau}{2}\|z-z_{0}\|^{2}\geq \psi(z^{*})+\frac{\tau}{2}\|z^{*}-z_{0}\|^{2}+\frac{\tau}{2}\|z-z^{*}\|^{2}.
\end{displaymath}
\end{lemma}

\subsection{Convergence Properties for Case 1}
For SC objectives with smooth component functions (i.e., Case 1), we analyze the convergence property of FSVRG.

\begin{theorem}[Strongly Convex]\label{theo1}
Suppose each $f_{i}(\cdot)$ is $L$-smooth, $\phi(\cdot)$ is $\mu$-SC, $\theta_{s}$ is a constant $\theta$ for Case 1, and $m_{s}$ is sufficiently large\footnote{If $m_{1}$ is not sufficiently large, the first epoch can be viewed as an initialization step.} so that
\vspace{-2mm}
\begin{equation*}
\alpha_{s}:=\,1-\theta+{\theta^{2}}/({\mu\eta m_{s}})< 1.
\vspace{-1mm}
\end{equation*}
Then Algorithm 1 has the convergence in expectation:
\vspace{-2mm}
\begin{equation}\label{equ32}
\mathbb{E}\!\left[\phi(\widetilde{x}^{S})-\phi(x_{\star})\right]\leq\,\left(\prod^{S}_{s=1}\alpha_{s}\right)[\phi(\widetilde{x}^{0})-\phi(x_{\star})].
\end{equation}
\end{theorem}

\begin{proof}
Since $\phi(x)$ is $\mu$-SC, then there exists a constant $\mu\!>\!0$ such that for all $x\!\in\!\mathbb{R}^{d}$
\vspace{-1mm}
\begin{equation*}
\phi(x)\geq \phi(x_{\star})+[\nabla\!\phi(x_{\star})]^{T}(x-x_{\star})+\frac{\mu}{2}\|x-x_{\star}\|^{2}.
\vspace{-1mm}
\end{equation*}

Since $x_{\star}$ is the optimal solution, we have
\vspace{-2mm}
\begin{equation}\label{equ33}
\nabla\phi(x_{\star})=0\,\footnote{When $\phi(\cdot)$ is non-smooth, there must exist a sub-graident $\partial\phi(x_{\star})$ of $\phi(\cdot)$ at $x_{\star}$ such that $\partial\phi(x_{\star})=0$.},\;\;\;\phi(x)-\phi(x_{\star})\geq \frac{\mu}{2}\|x-x_{\star}\|^{2}.
\vspace{-1mm}
\end{equation}
Using the inequality in \eqref{equ33} and $y^{s}_{0}\!=\!\widetilde{x}^{s-\!1}$, we have
\vspace{-1mm}
\begin{equation*}
\begin{split}
&\mathbb{E}\!\left[\phi(\widetilde{x}^{s})-\phi(x_{\star})\right]\\
\leq&\,(1\!-\!\theta)\mathbb{E}[\phi(\widetilde{x}^{s-\!1})\!-\!\phi(x_{\star})]\!+\!\frac{\theta^{2}\!/\eta}{2 m_{s}\!\!}\mathbb{E}\!\left[\|y^{s}_{0}\!\!-\!x_{\star}\|^2\!-\!\|y^{s}_{m_{s}}\!\!-\!x_{\star}\|^2\right]\\
\leq&\,(1\!-\!\theta)\mathbb{E}[\phi(\widetilde{x}^{s-\!1})\!-\!\phi(x_{\star})]\!+\!\frac{\theta^{2}\!/\eta}{\mu m_{s}}\mathbb{E}\!\left[\phi(\widetilde{x}^{s-\!1})\!-\!\phi(x_{\star})\right]\\
=&\,\left(1\!-\!\theta+\frac{\theta^{2}\!/\eta}{\mu m_{s}}\right)\mathbb{E}\!\left[\phi(\widetilde{x}^{s-\!1})-\phi(x_{\star})\right]\!,
\end{split}
\end{equation*}
where the first inequality holds due to Lemma 1, and the second inequality follows from the inequality in \eqref{equ33}.
\end{proof}

From Theorem~\ref{theo1}, it is clear that $\alpha_{s}$ decreases as $s$ increases, i.e., $1\!>\!\alpha_{1}\!>\!\alpha_{2}\!>\!\ldots\!>\!\alpha_{S}$. Therefore, there exists a positive constant $\gamma\!<\!1$ such that $\alpha_{s}\!\leq\!\alpha_{1}\gamma^{s-\!1}$ for all $s\!=\!1,\ldots,S$. Then the inequality in \eqref{equ32} can be rewritten as $\mathbb{E}\!\left[\phi(\widetilde{x}^{S})\!-\!\phi(x_{\star})\right]\!\leq\!(\alpha_{1}\sqrt{\gamma^{S-\!1}})^{S}[\phi(\widetilde{x}^{0})\!-\!\phi(x_{\star})]$, which implies that FSVRG attains linear (geometric) convergence.

\subsection{Convergence Properties for Case 2}
\label{sect42}
For NSC objectives with smooth component functions (i.e., Case 2), the following theorem gives the convergence rate and overall complexity of FSVRG.

\begin{theorem}[Non-Strongly Convex]\label{theo2}
Suppose each $f_{i}(\cdot)$ is $L$-smooth. Then the following inequality holds:
\vspace{-1mm}
\begin{equation*}
\begin{split}
&\mathbb{E}\!\left[\phi(\widetilde{x}^{S})-\phi(x_{\star})\right]\\
\leq&\,\frac{4(1\!-\!\theta_{1})}{\theta^{2}_{1}(S\!+\!2)^{2}}\![\phi(\widetilde{x}^{0})\!-\!\phi(x_{\star})]\!+\!\frac{2/\eta}{m_{1}(S\!+\!2)^{2}}\|x_{\star}\!-\!\widetilde{x}^{0}\|^2.
\end{split}
\end{equation*}
In particular, choosing $m_{1}\!=\!\Theta(n)$, Algorithm~\ref{alg1} achieves an $\varepsilon$-accurate solution, i.e., $\mathbb{E}[\phi(\widetilde{x}^{S})]\!-\!\phi(x_{\star})\leq \varepsilon$ using at most $\mathcal{O}(\frac{n\sqrt{\phi(\widetilde{x}^{0})\!-\!\phi(x_{\star})}}{\sqrt{\varepsilon}}\!+\!\frac{\sqrt{nL}\|\widetilde{x}^{0}\!-\!x_{\star}\|}{\sqrt{\varepsilon}})$ iterations.
\end{theorem}

\begin{proof}
Using the update rule of $\theta_{s}$ in (\ref{equ11}), it is easy to verify that
\begin{equation}\label{equ34}
({1-\theta_{s+1}})/{\theta^{2}_{s+1}}\leq{1}/{\theta^{2}_{s}},\;\;\theta_{s}\leq2/(s+2).
\end{equation}
Dividing both sides of the inequality in (\ref{equ31}) by $\theta^{2}_{s}$, we have
\vspace{-3mm}

\begin{equation*}
\begin{split}
&\mathbb{E}[\phi(\widetilde{x}^{s})-\phi(x_{\star})]/\theta^{2}_{s}\\
\leq&\frac{1\!-\!\theta_{s}\!}{\theta^{2}_{s}}\mathbb{E}[\phi(\widetilde{x}^{s-\!1})\!-\!\phi(x_{\star})]\!+\!\frac{1/\eta}{2m_{s}\!}\mathbb{E}[\|x_{\star}\!\!-\!y^{s}_{0}\|^2\!-\!\|x_{\star}\!\!-\!y^{s}_{m_{s}}\!\|^2],
\end{split}
\end{equation*}
for all $s\!=\!1,\ldots,S$. By $y^{s+\!1}_{0}\!=\!y^{s}_{m_{s}}$ and the inequality in (\ref{equ34}), and summing the above inequality over $s\!=\!1,\ldots,S$, we have
\vspace{-1mm}
\begin{equation*}
\begin{split}
&\mathbb{E}\!\left[\phi(\widetilde{x}^{S})-\phi(x_{\star})\right]/\theta^{2}_{S}\\
\leq&\frac{1\!-\!\theta_{1}}{\theta^{2}_{1}}[\phi(\widetilde{x}^{0})\!-\!\phi(x_{\star}\!)]\!+\!\frac{1/\eta}{2m_{1}}\mathbb{E}\!\!\left[\|x_{\star}\!-\!y^{1}_{0}\|^2\!-\!\|x_{\star}\!-\!y^{S}_{m_{S}}\!\|^2\right]\!.
\end{split}
\end{equation*}
\vspace{-2mm}

Then
\vspace{-1mm}
\begin{equation*}
\begin{split}
&\mathbb{E}\!\left[\phi(\widetilde{x}^{S})-\phi(x_{\star}\!)\right]\\
\leq&\frac{4(1\!-\!\theta_{1})}{\theta^{2}_{1}(S\!\!+\!\!2)^{2}}\![\phi(\widetilde{x}^{0})\!-\!\phi(x_{\star}\!)]\!+\!\!\frac{2/\eta}{m_{1}\!(S\!\!+\!\!2)^{2}}\!\mathbb{E}\!\!\left[\!\|x_{\star}\!\!-\!y^{1}_{0}\|^2\!\!-\!\|x_{\star}\!\!-\!y^{S}_{m_{\!S}}\!\|^2\!\right]\\
\leq&\frac{4(1\!-\!\theta_{1})}{\theta^{2}_{1}(S\!+\!2)^{2}}\![\phi(\widetilde{x}^{0})\!-\!\phi(x_{\star}\!)]\!+\!\frac{2/\eta}{m_{1}(S\!+\!2)^{2}}\!\left[\|x_{\star}\!-\!\widetilde{x}^{0}\|^2\right].
\end{split}
\end{equation*}
This completes the proof.
\end{proof}

\begin{figure*}[t]
\centering
\includegraphics[width=0.496\columnwidth]{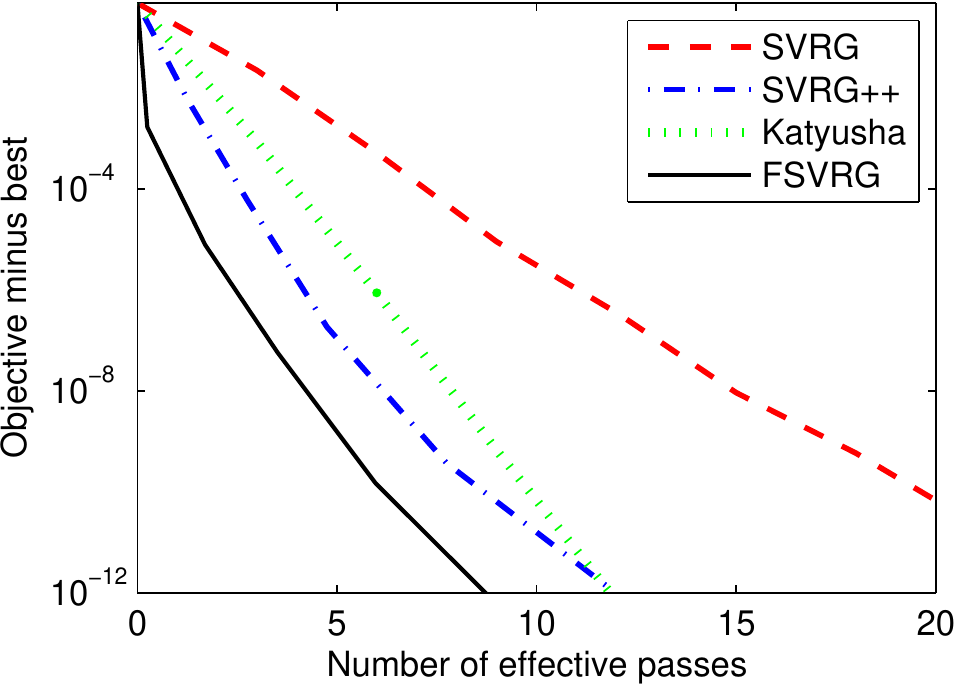}
\includegraphics[width=0.496\columnwidth]{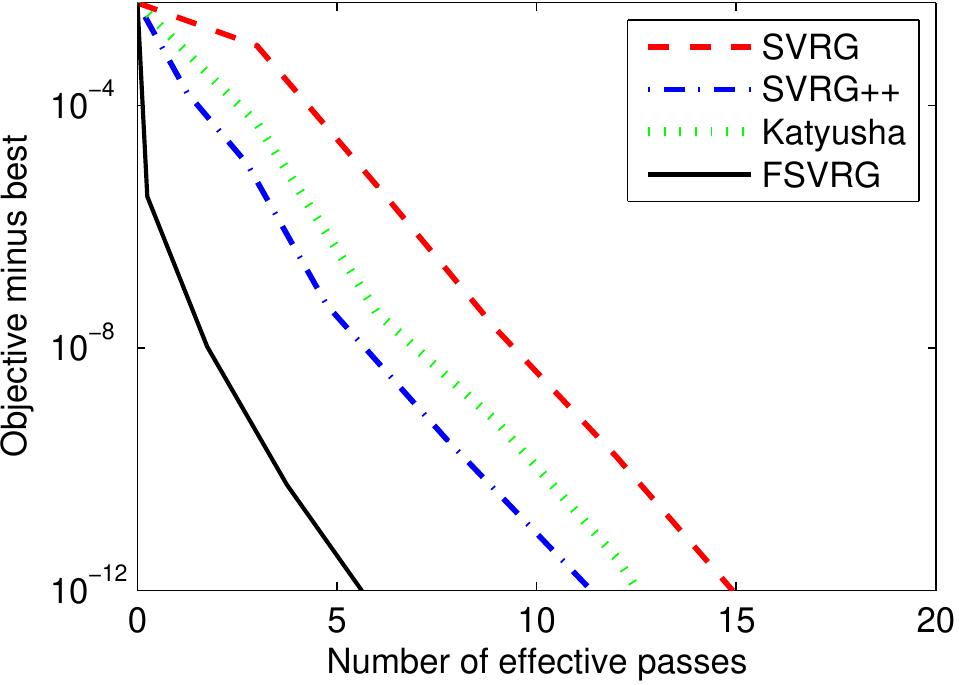}
\includegraphics[width=0.496\columnwidth]{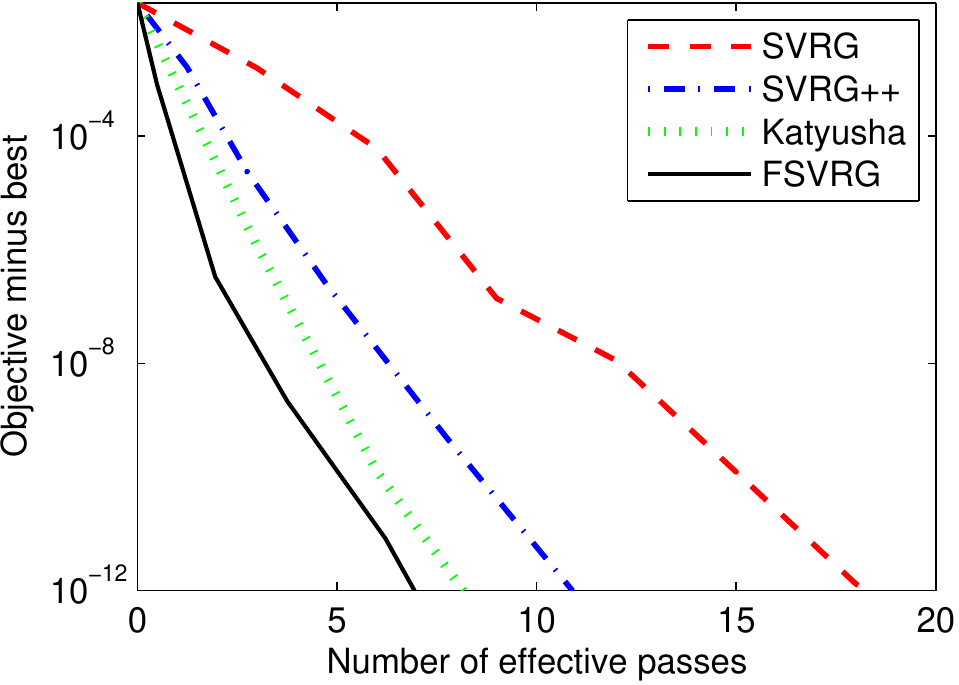}
\includegraphics[width=0.496\columnwidth]{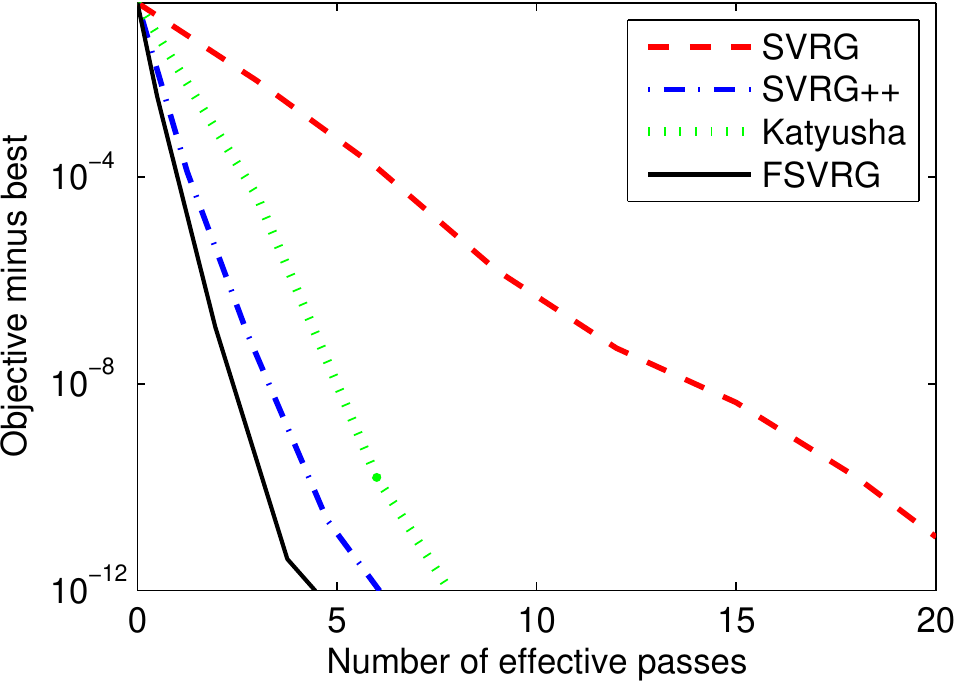}
\vspace{-1.6mm}

\subfigure[IJCNN: $\lambda\!=\!10^{-4}$]{\includegraphics[width=0.496\columnwidth]{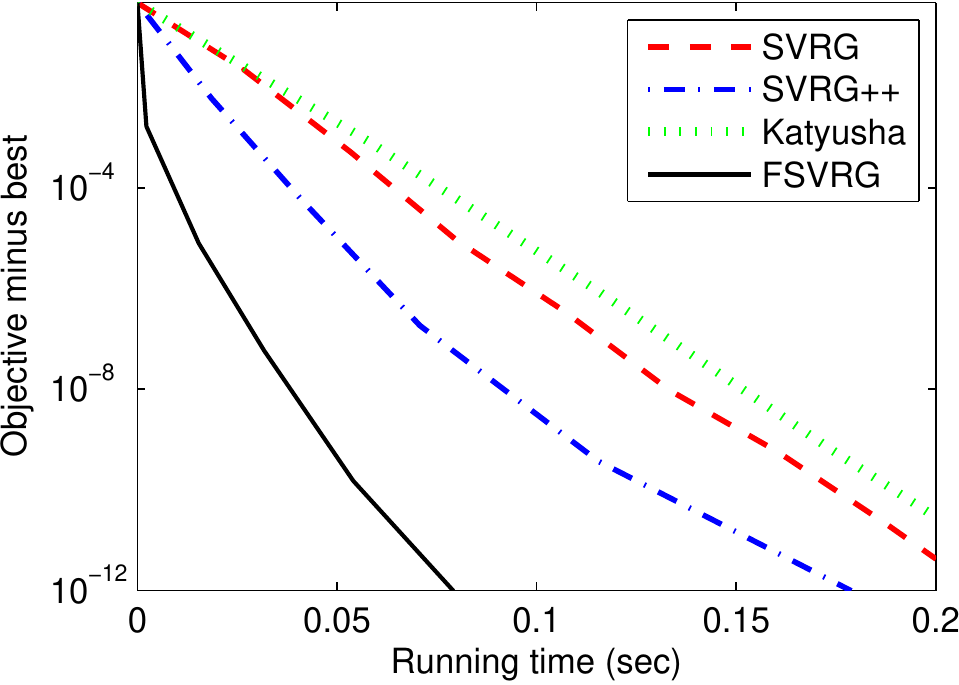}}
\subfigure[Protein: $\lambda\!=\!10^{-4}$]{\includegraphics[width=0.496\columnwidth]{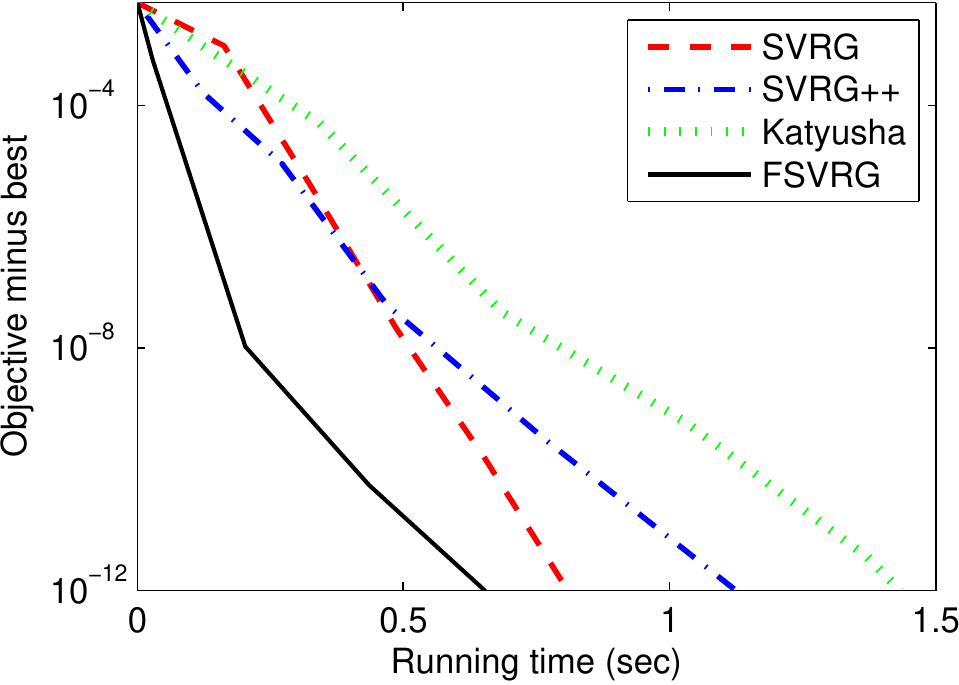}}
\subfigure[Covtype: $\lambda\!=\!10^{-5}$]{\includegraphics[width=0.496\columnwidth]{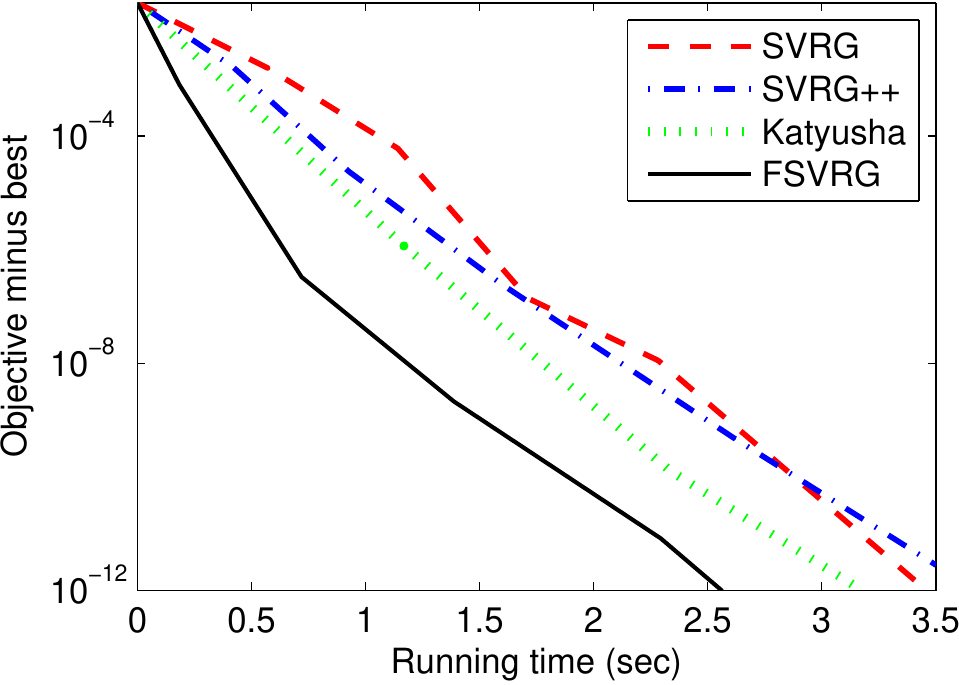}}
\subfigure[SUSY: $\lambda\!=\!10^{-6}$]{\includegraphics[width=0.496\columnwidth]{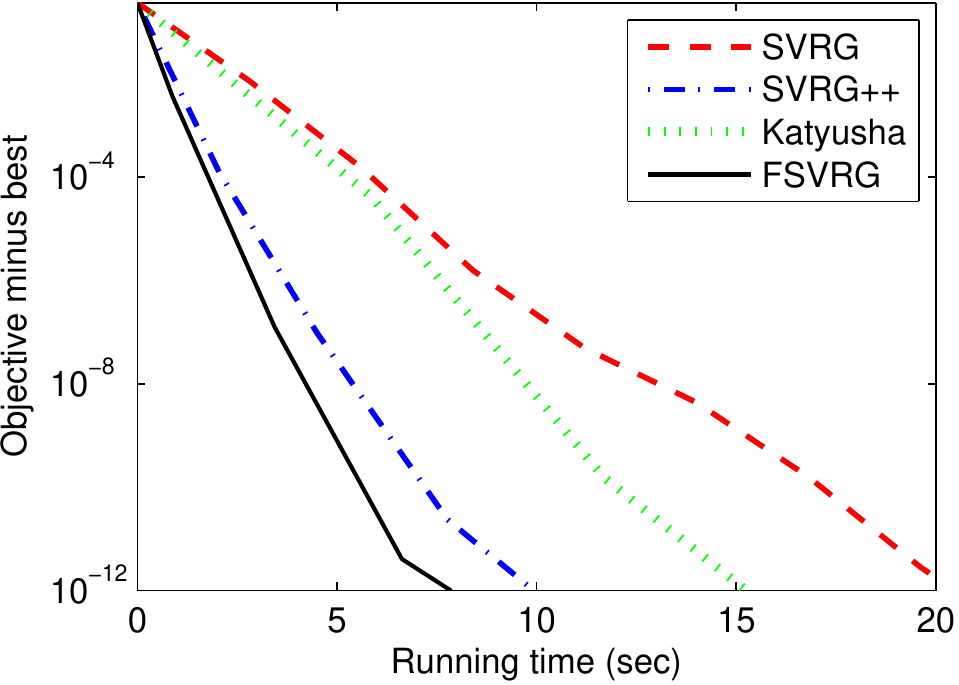}}
\vspace{-3mm}

\caption{Comparison of SVRG~\cite{johnson:svrg}, SVRG++~\cite{zhu:vrnc}, Katyusha~\cite{zhu:Katyusha}, and our FSVRG method for $\ell_{2}$-norm (i.e., $\lambda\|x\|^{2}$) regularized logistic regression problems. The $y$-axis represents the objective value minus the minimum, and the $x$-axis corresponds to the number of effective passes (top) or running time (bottom).}
\label{figs1}
\end{figure*}

From Theorem 2, we can see that FSVRG achieves the optimal convergence rate of $\mathcal{O}(1/T^2)$ and the complexity of $\mathcal{O}(n\sqrt{1/\epsilon}\!+\!\sqrt{nL/\epsilon})$ for NSC problems, which is consistent with the best known result in~\cite{zhu:Katyusha,hien:asmd}. By adding a proximal term into the problem of Case 2 as in~\cite{lin:vrsg,zhu:box}, one can achieve faster convergence. However, this hurts the performance of the algorithm both in theory and in practice~\cite{zhu:univr}.

\begin{figure*}[!th]
\centering
\subfigure[IJCNN: $\lambda_{1}\!=\!\lambda_{2}\!=\!10^{-4}$]{\includegraphics[width=0.496\columnwidth]{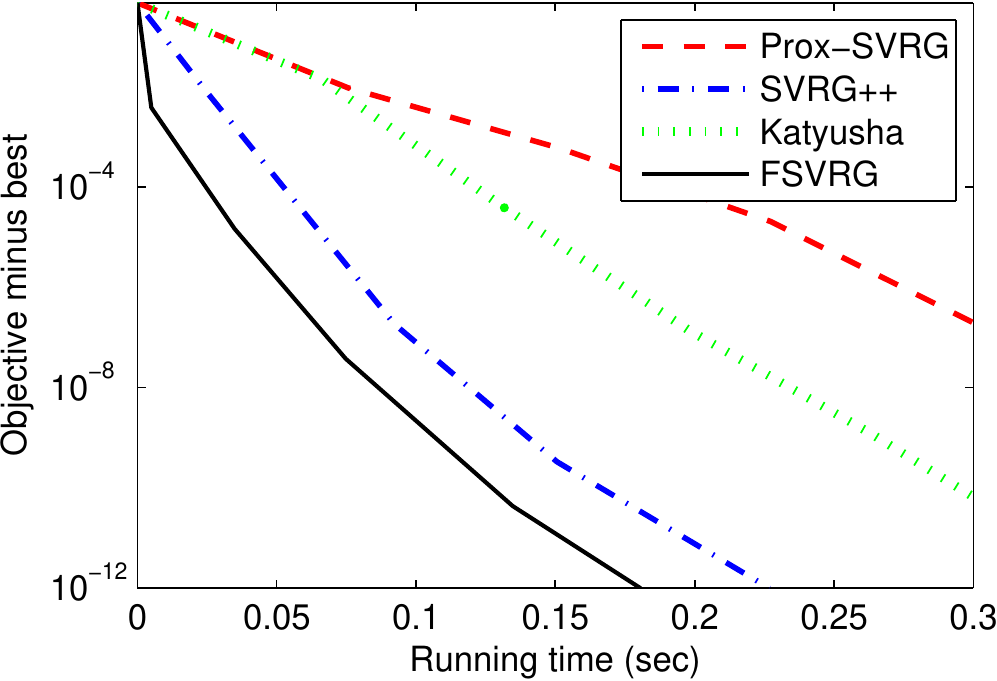}}
\subfigure[Protein: $\lambda_{1}\!=\!\lambda_{2}\!=\!10^{-4}$]{\includegraphics[width=0.496\columnwidth]{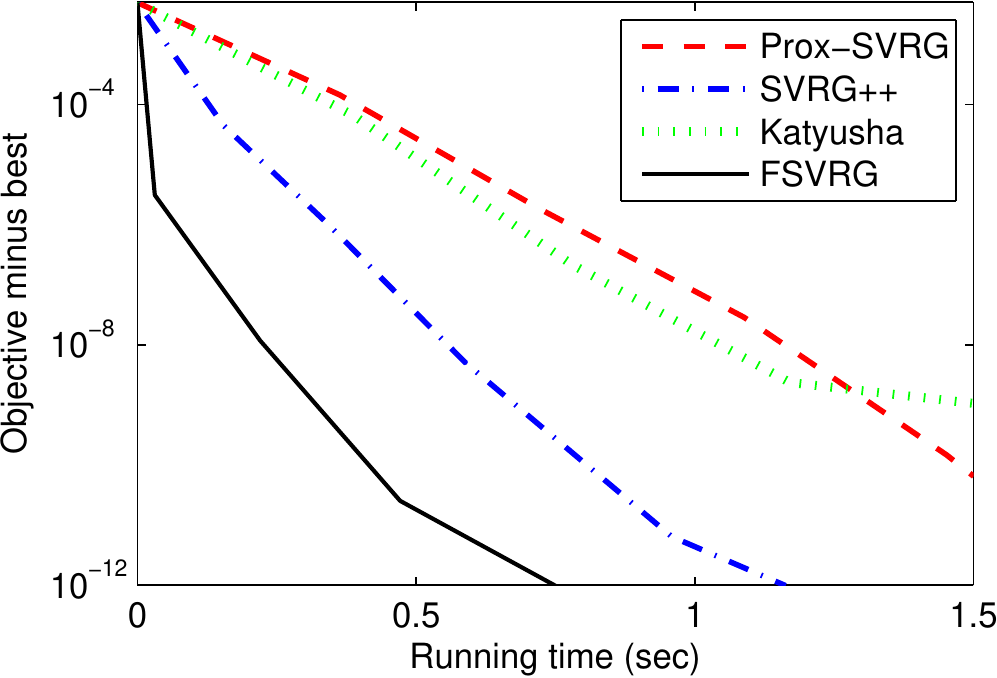}}
\subfigure[Covtype: $\lambda_{1}\!=\!\lambda_{2}\!=\!10^{-5}$]{\includegraphics[width=0.496\columnwidth]{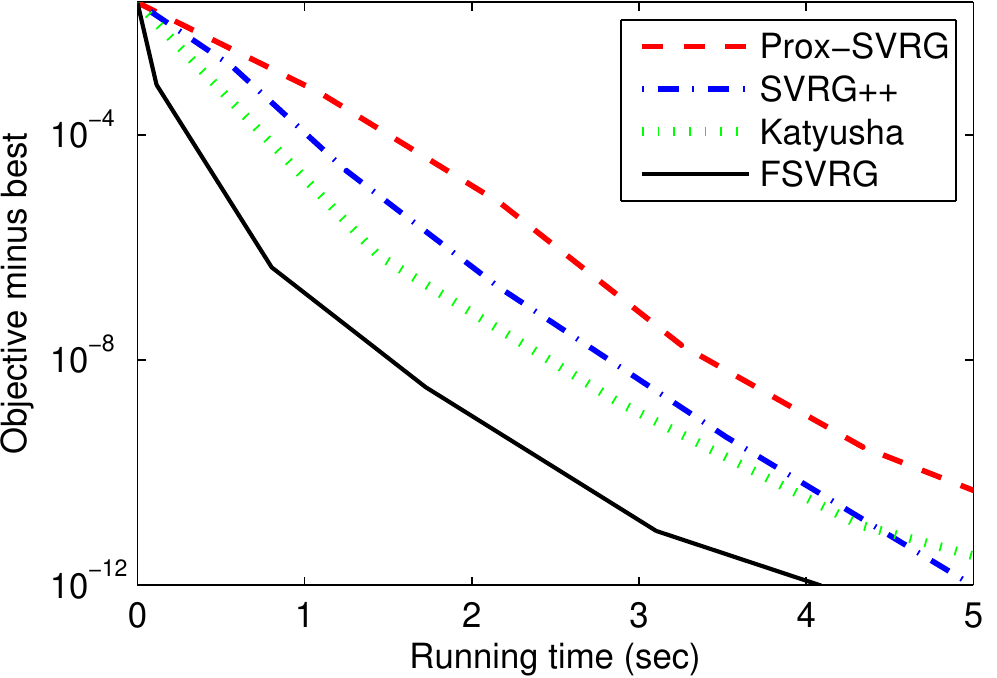}}
\subfigure[SUSY: $\lambda_{1}\!=\!\lambda_{2}\!=\!10^{-6}$]{\includegraphics[width=0.496\columnwidth]{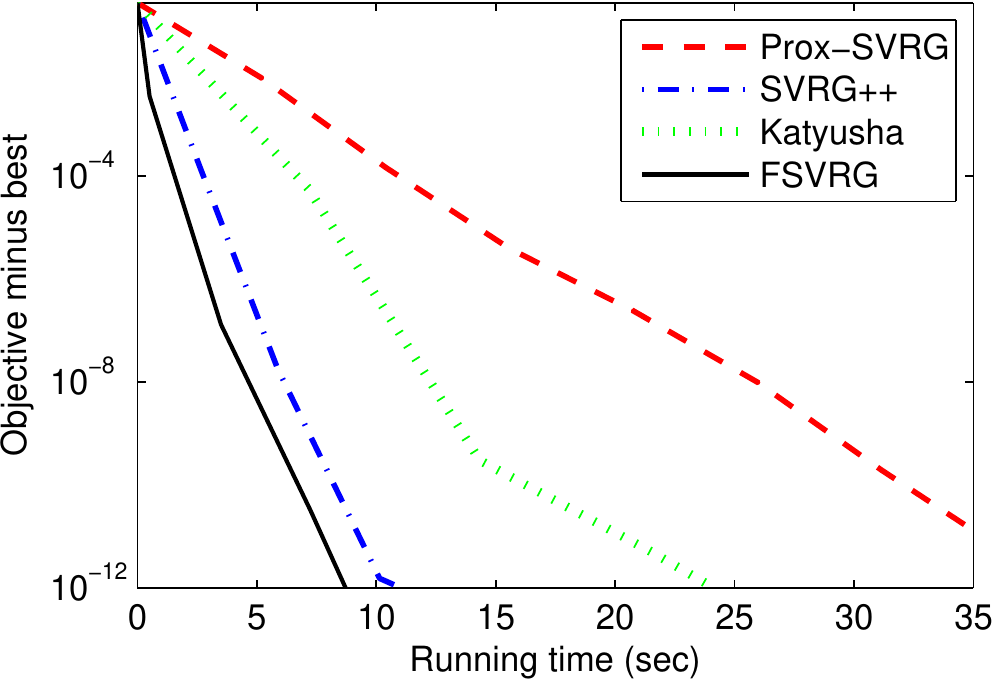}}
\vspace{-3mm}

\caption{Comparison of Prox-SVRG~\cite{xiao:prox-svrg}, SVRG++~\cite{zhu:vrnc}, Katyusha~\cite{zhu:Katyusha}, and our FSVRG method for elastic net (i.e., $\lambda_{1}\|x\|^{2}\!+\!\lambda_{2}\|x\|_{1}$) regularized logistic regression problems.}
\label{figs2}
\end{figure*}

\subsection{Convergence Properties for Mini-Batch Settings}
It has been shown in~\cite{shwartz:svm,nitanda:svrg,koneeny:mini} that mini-batching can effectively decrease the variance of stochastic gradient estimates. So, we extend FSVRG and its convergence results to the mini-batch setting. Here, we denote by $b$ the mini-batch size and $I^{s}_{k}$ the selected random index set $I_{k}\!\subset\![n]$ for each outer-iteration $s\!\in\![S]$ and inner-iteration $k\!\in\!\{1,\ldots,m_{s}\}$. Then the stochastic gradient estimator $\widetilde{\nabla}\!f_{i^{s}_{k}}\!(x^{s}_{k-\!1})$ becomes
\begin{equation}\label{equ36}
\widetilde{\nabla}\! f_{I^{s}_{k}}(x^{s}_{k-\!1})\!=\!\frac{1}{b}\!\sum_{i\in I^{s}_{k}}\!\!\left[\nabla\!f_{i}(x^{s}_{k-\!1})\!-\!\nabla\! f_{i}(\widetilde{x}^{s-\!1})\right]\!+\!{\nabla}\! f(\widetilde{x}^{s-\!1}).
\vspace{-2mm}
\end{equation}
And the momentum weight $\theta_{s}$ is required to satisfy $\theta_{s}\!\leq\! 1\!-\!\frac{\rho(b)L\eta}{1-L\eta}$ for SC and NSC cases, where $\rho(b)\!=\!(n\!-\!b)/[(n\!-\!1)b]$. The upper bound on the variance of $\widetilde{\nabla}\!f_{i^{s}_{k}}(x^{s}_{k-\!1})$ in Lemma~\ref{lemm2} is extended to the mini-batch setting as follows~\cite{liu:sadmm}.

\begin{corollary}[Variance bound of Mini-Batch]
\label{cor11}
\begin{displaymath}
\begin{split}
&\mathbb{E}\!\left[\left\|\widetilde{\nabla}\! f_{I^{s}_{k}}(x^{s}_{k-1})-\nabla\! f(x^{s}_{k-1})\right\|^{2}\right]\\
\leq&2L\rho(b)\!\left[\nabla\! f(x^{s}_{k-\!1})^{T}\!(x^{s}_{k-\!1}\!-\!\widetilde{x}^{s-\!1})\!-\!f(x^{s}_{k-\!1})\!+\!f(\widetilde{x}^{s-\!1})\right]\!.
\end{split}
\end{displaymath}
\end{corollary}
\vspace{-2mm}

It is easy to verify that $0\!\leq\!\rho(b)\!\leq\!1$, which implies that mini-batching is able to reduce the variance upper bound in Lemma~\ref{lemm2}. Based on the variance upper bound in Corollary~\ref{cor11}, we further analyze the convergence properties of our algorithms for the mini-batch setting. Obviously, the number of stochastic iterations in each epoch is reduced from $m_{s}$ to $\lfloor m_{s}/b\rfloor$. For the case of SC objective functions, the mini-batch variant of FSVRG has almost identical convergence properties to those in Theorem~\ref{theo1}. In contrast, we need to initialize $\theta_{1}\!=\!1\!-\!\frac{\rho(b){L}\eta}{1-{L}\eta}$ and update $\theta_{s}$ by the procedure in (10) for the case of NSC objective functions. Theorem~\ref{theo2} is also extended to the mini-batch setting as follows.

\begin{corollary}\label{cor12}
Suppose $f_{i}(\cdot)$ is $L$-smooth, and let $\theta_{1}\!=\!1\!-\!\frac{\rho(b){L}\eta}{1-{L}\eta}$ and $\beta\!=\!{1}/{L\eta}$, then the following inequality holds:
\vspace{-1mm}
\begin{equation}\label{equ37}
\begin{split}
&\mathbb{E}\!\left[\phi(\widetilde{x}^{S})\!-\phi(x_{\star})\right]\leq\frac{2/\eta}{m_{1}(S\!+\!2)^{2}}\!\left[\|x_{\star}\!-\!\widetilde{x}^{0}\|^2\right]\\
&\quad\;\;+\!\frac{4(\beta-1)\rho(b)}{(\beta\!-\!1\!-\!\rho(b))^{2}(S\!+\!2)^{2}}[\phi(\widetilde{x}^{0})\!-\!\phi(x_{\star}\!)].
\end{split}
\end{equation}
\end{corollary}

\begin{proof}
Since
\begin{displaymath}
\theta_{1}=1-({\rho(b){L}\eta})/({1-{L}\eta})=1-{\rho(b)}/({\beta-1}),
\end{displaymath}
then we have
\begin{equation*}
\begin{split}
&\mathbb{E}\!\left[\phi(\widetilde{x}^{S})-\phi(x_{\star})\right]\\
\leq&\frac{4(\beta-1)\rho(b)}{(\beta\!-\!1\!-\!\rho(b))^{2}(S\!\!+\!\!2)^{2}}\![\phi(\widetilde{x}^{0})\!-\!\phi(x_{\star}\!)]\!+\!\!\frac{2/\eta}{ m_{1}\!(S\!\!+\!\!2)^{2}}\!\left[\|x_{\star}\!\!-\!\widetilde{x}^{0}\!\|^2\right]\!\!.
\end{split}
\end{equation*}
This completes the proof.
\end{proof}

\begin{remark}
When $b\!=\!1$, we have $\rho(1)\!=\!1$, and then Corollary~\ref{cor12} degenerates to Theorem~\ref{theo2}. If $b\!=\!n$ (i.e., the batch setting), we have $\rho(n)\!=\!0$, and the second term on the right-hand side of \eqref{equ37} diminishes. In other words, FSVRG degenerates to the accelerated deterministic method with the optimal convergence rate of $\mathcal{O}(1/T^{2})$.
\end{remark}

\section{Experimental Results}
\label{sec5}
In this section, we evaluate the performance of our FSVRG method for solving various machine learning problems, such as logistic regression, ridge regression, Lasso and SVM. All the codes of FSVRG and related methods can be downloaded from the first author's website.

\subsection{Experimental Setup}
For fair comparison, FSVRG and related stochastic variance reduced methods, including SVRG~\cite{johnson:svrg}, Prox-SVRG~\cite{xiao:prox-svrg}, SVRG++~\cite{zhu:univr} and Katyusha~\cite{zhu:Katyusha}, were implemented in C++, and the experiments were performed on a PC with an Intel i5-2400 CPU and 16GB RAM. As suggested in~\cite{johnson:svrg,xiao:prox-svrg,zhu:Katyusha}, the epoch size is set to $m\!=\!2n$ for SVRG, Prox-SVRG, and Katyusha. FSVRG and SVRG++ have the similar strategy of growing epoch size, e.g., $m_{1}\!\!=\!n/2$ and $\rho\!=\!1.6$ for FSVRG, and $m_{1}\!\!=\!n/4$ and $\rho\!=\!2$ for SVRG++.  Then for all these methods, there is only one parameter to tune, i.e., the learning rate. Note that we compare their performance in terms of both the number of effective passes (evaluating $n$ component gradients or computing a single full gradient is considered as one effective pass) and running time (seconds). Moreover, we do not compare with other stochastic algorithms such as SAGA~\cite{defazio:saga} and Catalyst~\cite{lin:vrsg}, as they have been shown to be comparable or inferior to Katyusha~\cite{zhu:Katyusha}.

\subsection{Logistic Regression}
In this part, we conducted experiments for both the $\ell_{2}$-norm and elastic net regularized logistic regression problems on the four popular data sets: IJCNN, Covtype, SUSY, and Protein, all of which were obtained from the LIBSVM Data website{\footnote{\url{https://www.csie.ntu.edu.tw/~cjlin/libsvm/}}} and the KDD Cup 2004 website{\footnote{\url{http://osmot.cs.cornell.edu/kddcup}}}. Each example of these date sets was normalized so that they have unit length as in~\cite{xiao:prox-svrg}, which leads to the same upper bound on the Lipschitz constants $L_{i}$ of functions $f_{i}(\cdot)$.

Figures~\ref{figs1} and~\ref{figs2} show the performance of different methods for solving the two classes of logistic regression problems, respectively. It can be seen that SVRG++ and FSVRG consistently converge much faster than the other methods in terms of both running time (seconds) and number of effective passes. The accelerated stochastic variance reduction method, Katyusha, has much better performance than the standard SVRG method in terms of number of effective passes, while it sometimes performs worse in terms of running time. FSVRG achieves consistent speedups for all the data sets, and outperforms the other methods in all the settings. The main reason is that FSVRG not only takes advantage of the momentum acceleration trick, but also can use much larger step sizes, e.g., 1/(3$L$) for FSVRG vs.\ 1/(7$L$) for SVRG++ vs.\ 1/(10$L$) for SVRG. This also confirms that FSVRG has much lower per-iteration cost than Katyusha.

\begin{figure*}[!th]
\centering
\includegraphics[width=0.496\columnwidth]{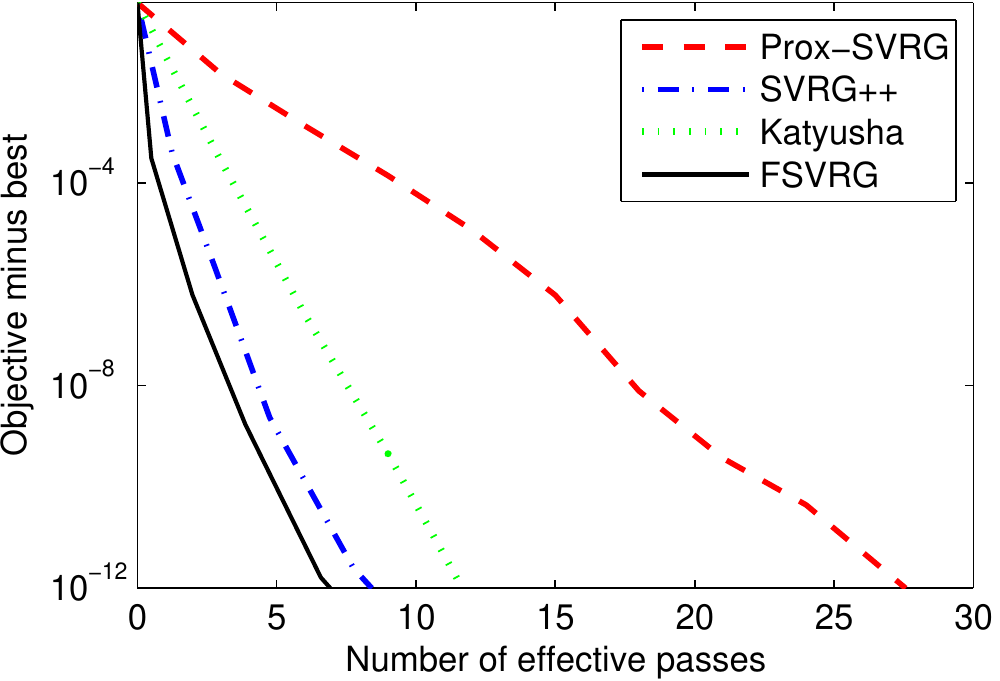}
\includegraphics[width=0.496\columnwidth]{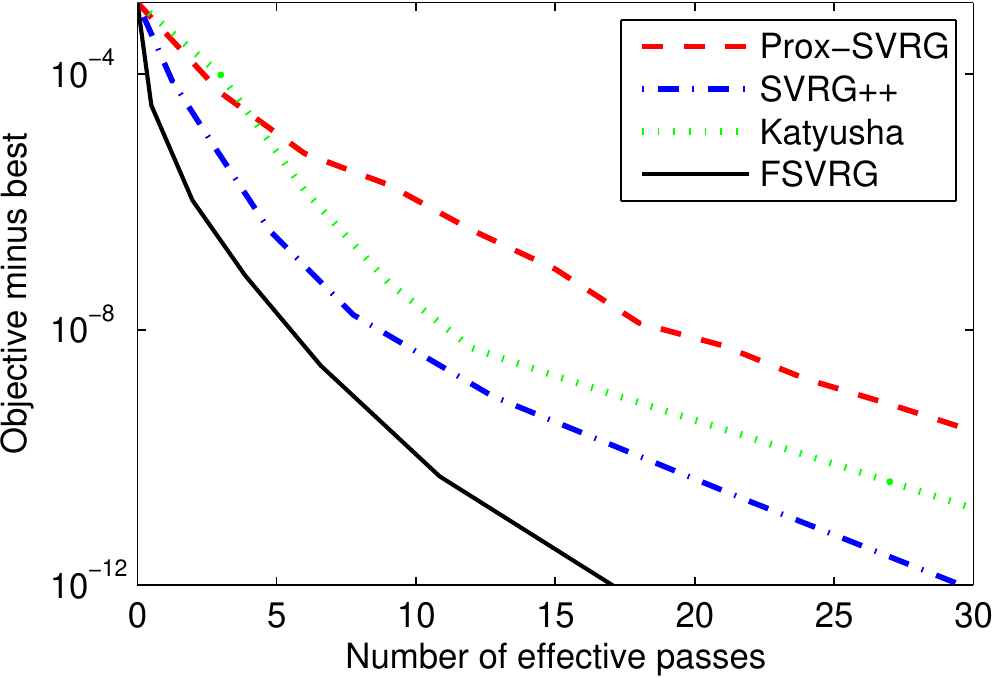}
\includegraphics[width=0.496\columnwidth]{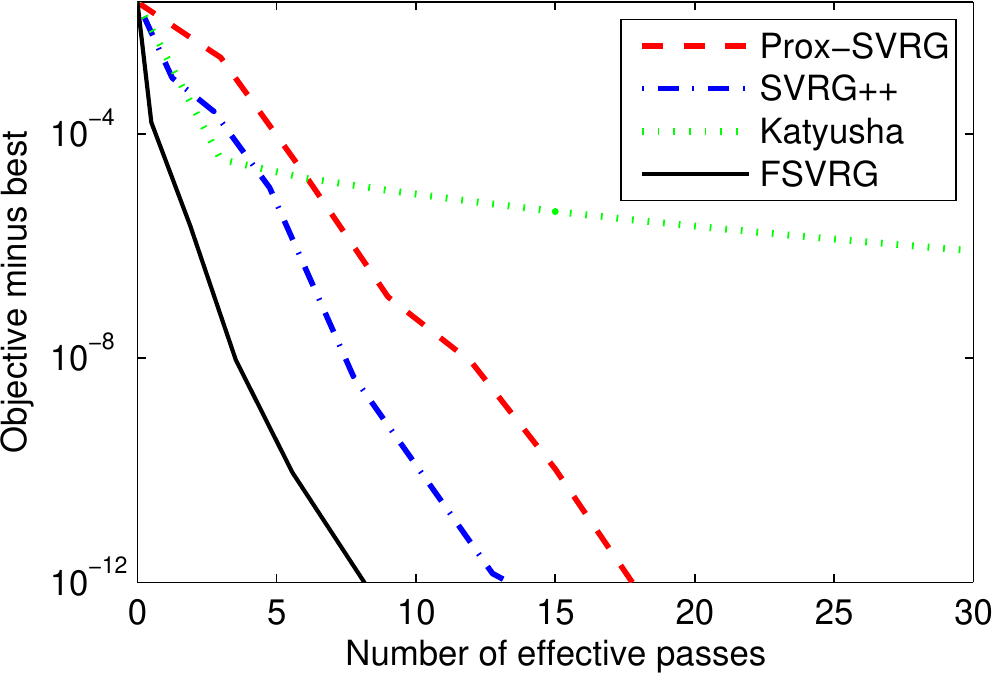}
\includegraphics[width=0.496\columnwidth]{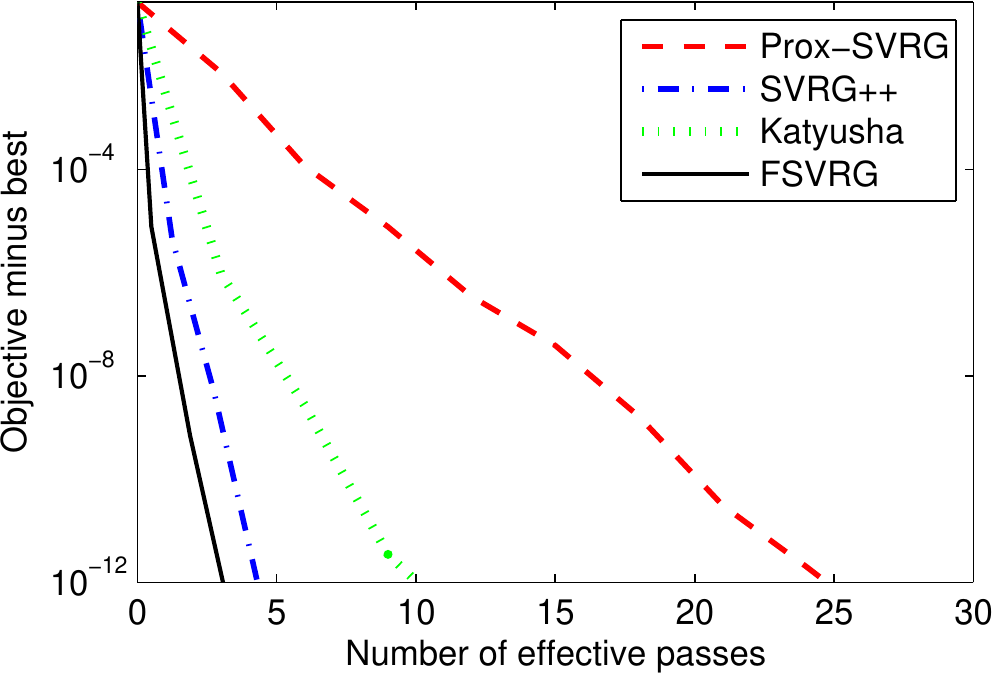}
\vspace{-1.6mm}

\subfigure[IJCNN: $\lambda\!=\!10^{-4}$]{\includegraphics[width=0.496\columnwidth]{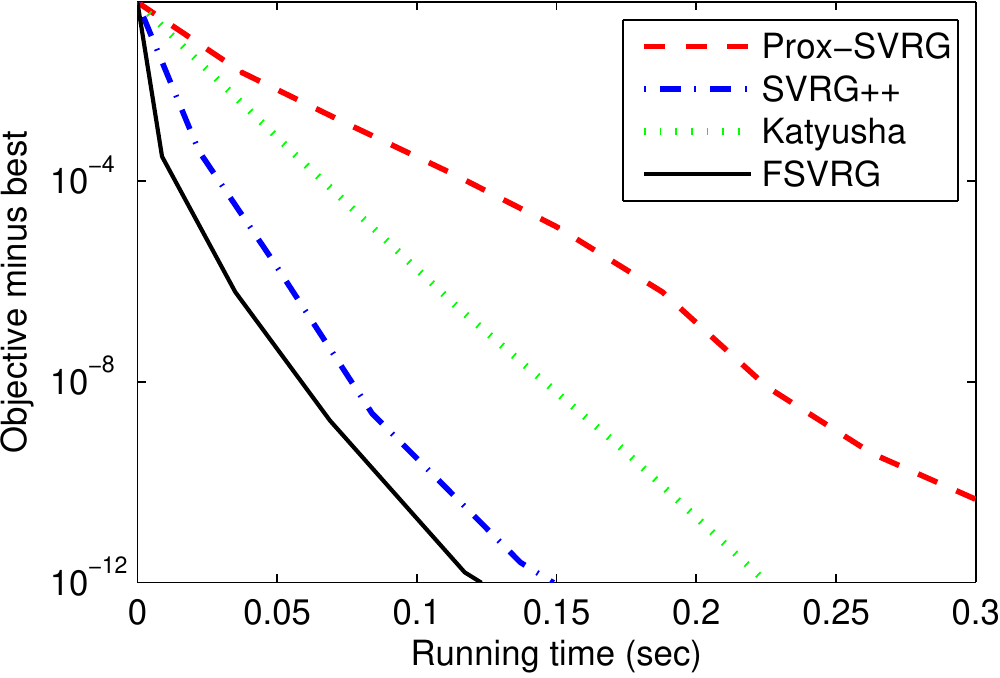}}
\subfigure[Protein: $\lambda\!=\!10^{-4}$]{\includegraphics[width=0.496\columnwidth]{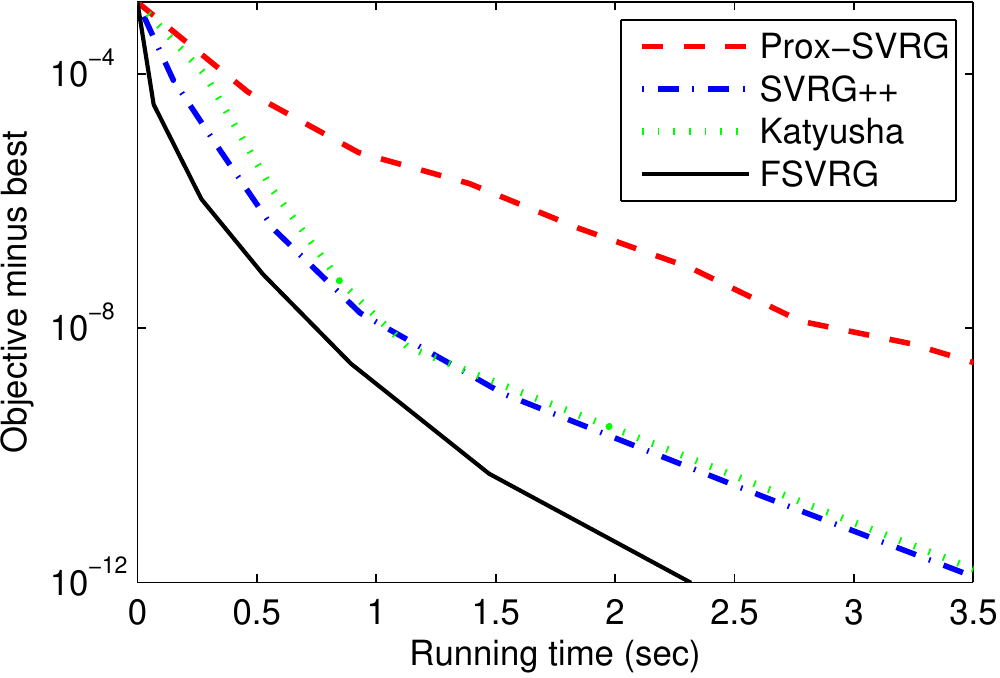}}
\subfigure[Covtype: $\lambda\!=\!10^{-5}$]{\includegraphics[width=0.496\columnwidth]{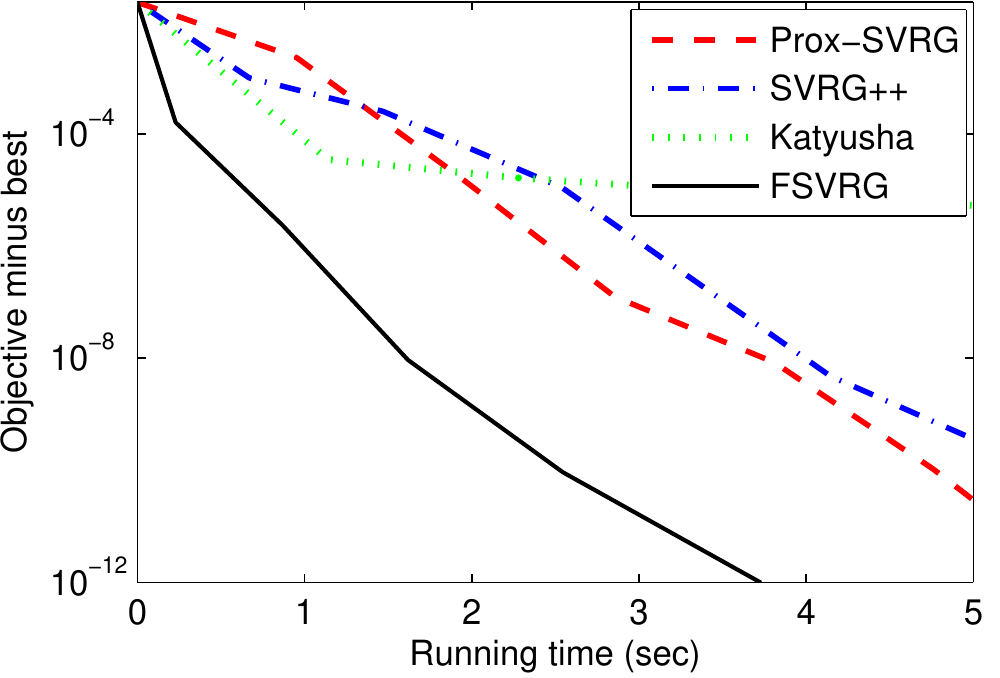}}
\subfigure[SUSY: $\lambda\!=\!10^{-6}$]{\includegraphics[width=0.496\columnwidth]{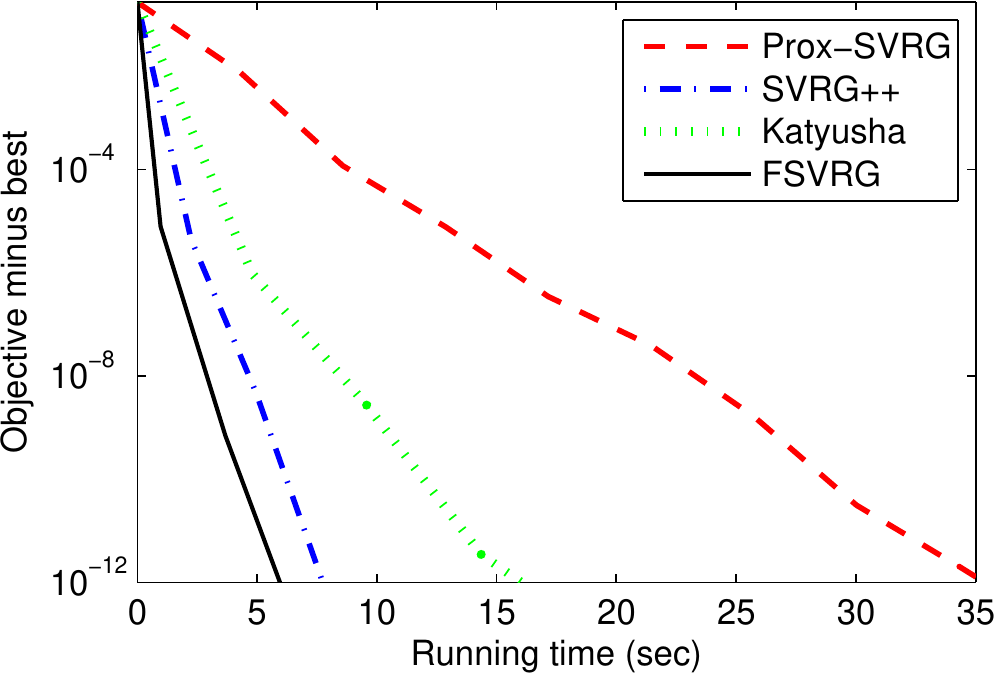}}
\vspace{-3mm}

\caption{Comparison of Prox-SVRG~\cite{xiao:prox-svrg}, SVRG++~\cite{zhu:vrnc}, Katyusha~\cite{zhu:Katyusha}, and FSVRG for Lasso problems.}
\label{figs3}
\end{figure*}

\subsection{Lasso}
In this part, we conducted experiments for the Lasso problem with the regularizer $\lambda\|x\|_{1}$ on the four data sets. We report the experimental results of different methods in Figure~\ref{figs3}, where the regularization parameter is varied from $\lambda\!=\!10^{-4}$ to $\lambda\!=\!10^{-6}$. From all the results, we can observe that FSVRG converges much faster than the other methods, and also outperforms Katyusha in terms of number of effective passes, which matches the optimal convergence rate for the NSC problem. SVRG++ achieves comparable and sometimes even better performance than SVRG and Katyusha. This further verifies that the efficiency of the growing epoch size strategy in SVRG++ and FSVRG.

\subsection{Ridge Regression}
In this part, we implemented all the algorithms mentioned above in C++ for high-dimensional sparse data, and compared their performance for solving ridge regression problems on the two very sparse data sets, Rcv1 and Sido0 (whose sparsity is 99.84\% and 90.16\%), as shown in Figure~\ref{figs4}. The two data sets can be downloaded from the LIBSVM Data website and the Causality Workbench website{\footnote{\url{http://www.causality.inf.ethz.ch/home.php}}}. From the results, one can observe that although Katyusha outperforms SVRG in terms of number of effective passes, both of them usually have similar convergence speed. SVRG++ has relatively inferior performance (maybe due to large value for $\rho$, i.e., $\rho\!=\!2$) than the other methods. It can be seen that the objective value of FSVRG is much lower than those of the other methods, suggesting faster convergence.

\begin{figure}[t]
\centering
\includegraphics[width=0.486\columnwidth]{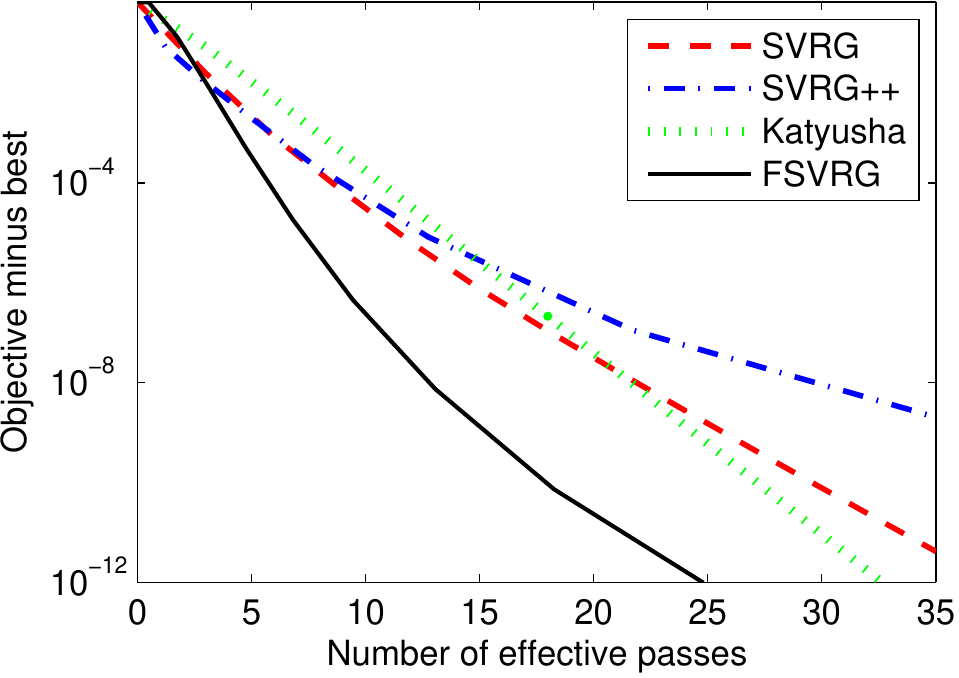}
\includegraphics[width=0.486\columnwidth]{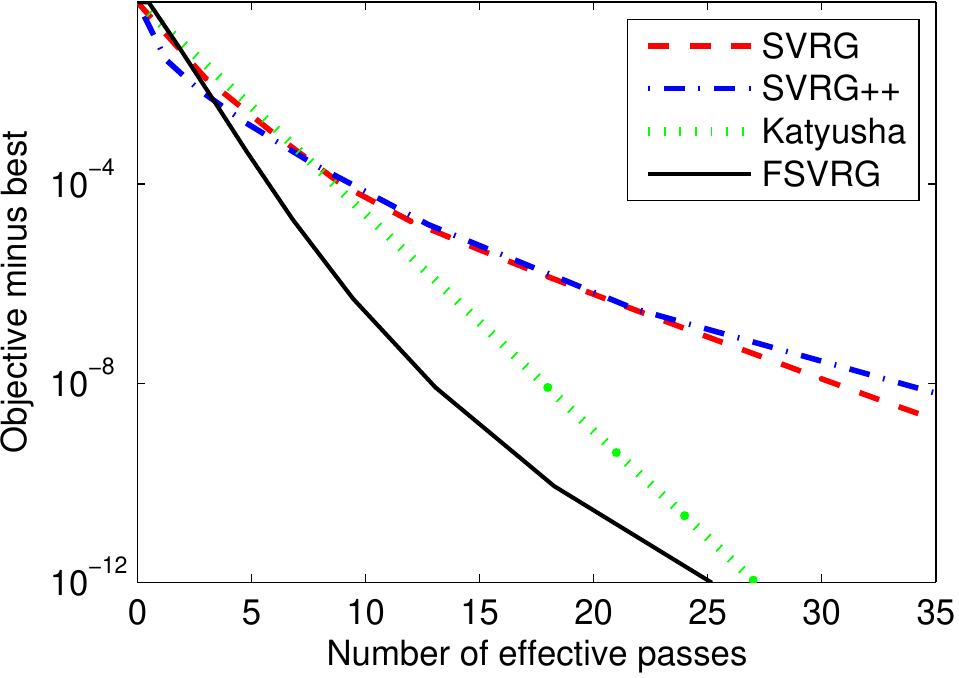}
\vspace{-1.6mm}

\subfigure[Rcv1]{\includegraphics[width=0.486\columnwidth]{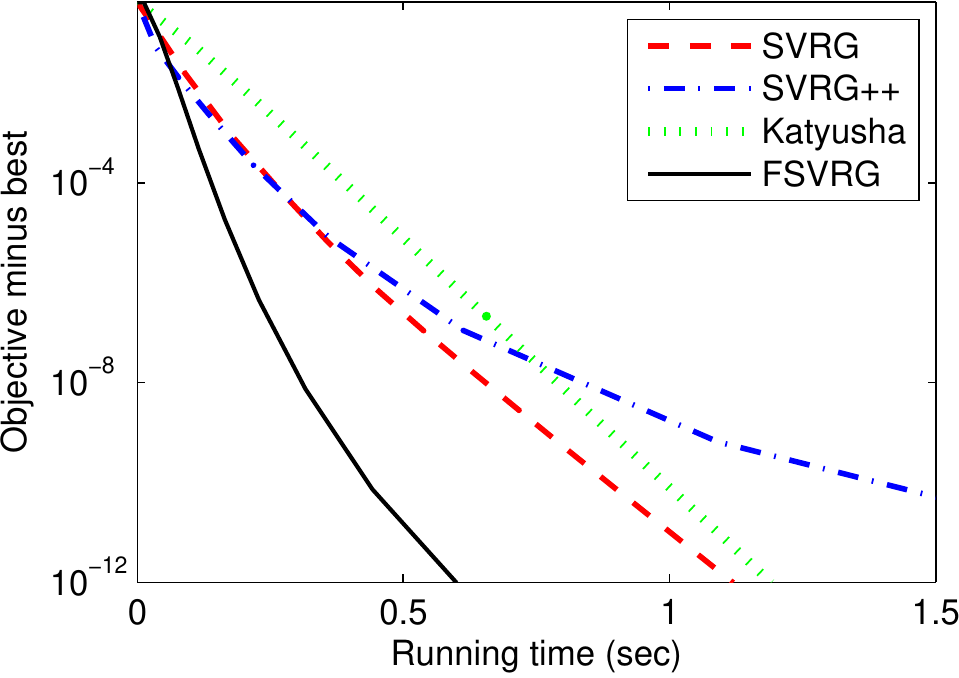}}
\subfigure[Sido0]{\includegraphics[width=0.486\columnwidth]{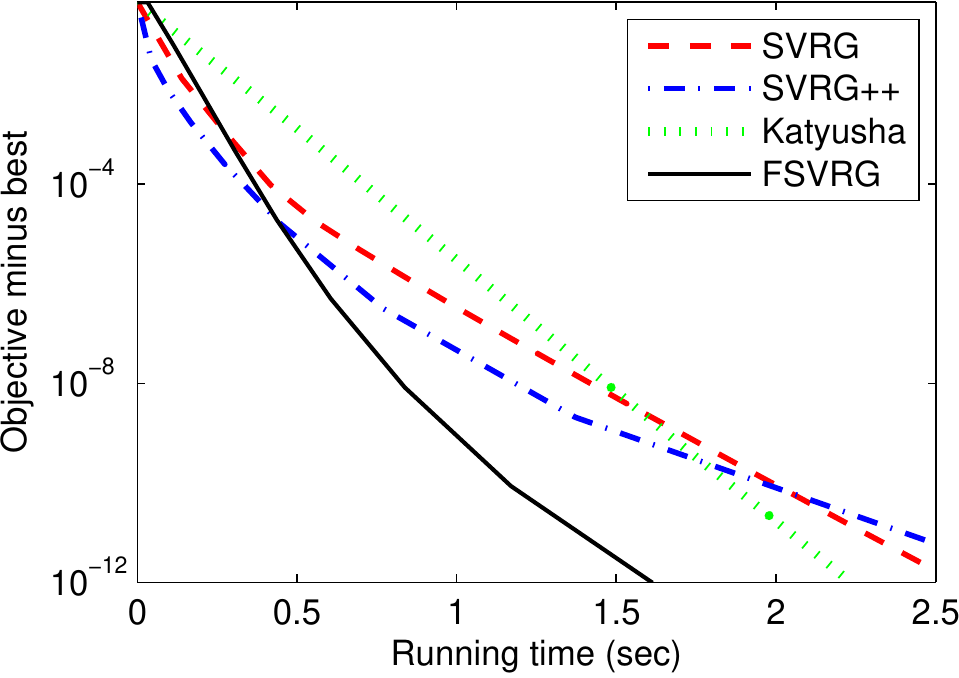}}
\vspace{-3mm}

\caption{Comparison of SVRG~\cite{johnson:svrg}, SVRG++~\cite{zhu:vrnc}, Katyusha~\cite{zhu:Katyusha}, and FSVRG for ridge regression problems with regularization parameter $\lambda\!=\!10^{-4}$.}
\label{figs4}
\end{figure}

Figure~\ref{figs5} compares the performance of our FSVRG method with different mini-batch sizes on the two data sets, IJCNN and Protein. It can be seen that by increasing the mini-batch size to $b\!=\!2,4$, FSVRG has comparable or even better performance than the case when $b\!=\!1$.

\begin{figure}[t]
\centering
\includegraphics[width=0.486\columnwidth]{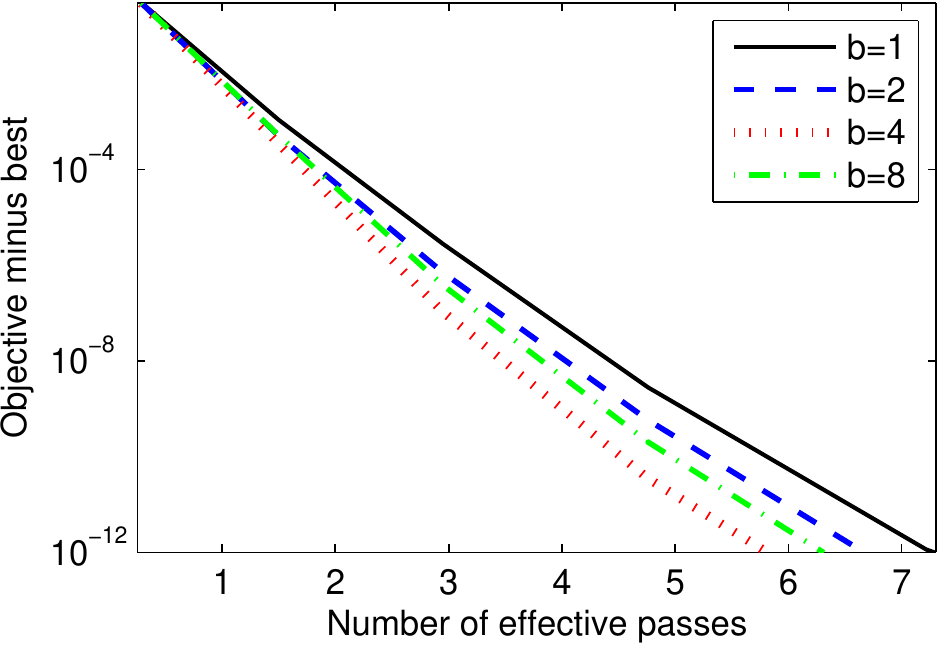}
\includegraphics[width=0.486\columnwidth]{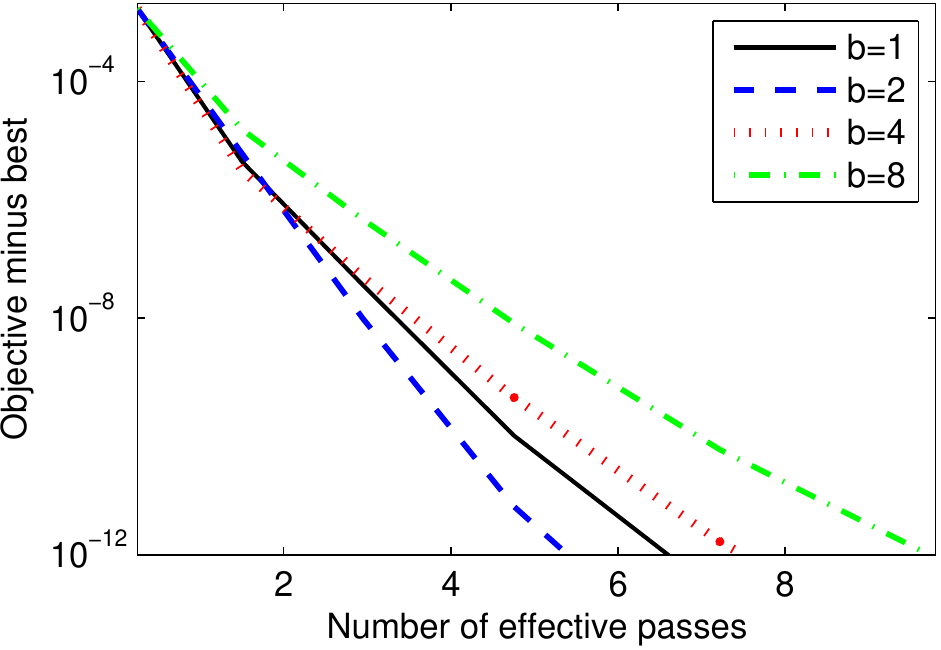}
\vspace{-3mm}

\caption{Results of FSVRG with different mini-batch sizes on IJCNN (left) and Protein (right).}
\label{figs5}
\end{figure}

\subsection{SVM}
Finally, we evaluated the empirical performance of FSVRG for solving the SVM optimization problem
\vspace{-2mm}
\begin{equation*}
\min_{x}\frac{1}{n}\sum^{n}_{i=1}\max\{0,1-b_{i}\langle a_{i},x\rangle\}+\frac{\lambda}{2}\|x\|^{2},
\vspace{-2mm}
\end{equation*}
where $(a_{i},b_{i})$ is the feature-label pair. For the binary classification data set, IJCNN, we randomly divided it into 10\% training set and 90\% test set. We used the standard one-vs-rest scheme for the multi-class data set, the MNIST data set{\footnote{\url{http://yann.lecun.com/exdb/mnist/}}}, which has a training set of 60,000 examples and a test set of 10,000 examples. The regularization parameter is set to $\lambda\!=\!10^{-5}$. Figure~\ref{figs6} shows the performance of the stochastic sub-gradient descent method (SSGD)~\cite{shamir:sgd}, SVRG and FSVRG for solving the SVM problem. Note that we also extend SVRG to non-smooth settings, and use the same scheme in (12). We can see that the variance reduced methods, SVRG and FSVRG, yield significantly better performance than SSGD. FSVRG consistently outperforms SSGD and SVRG in terms of convergence speed and testing accuracy. Intuitively, the momentum acceleration trick in (8) can lead to faster convergence. We leave the theoretical analysis of FSVRG for this case as our future research.

\begin{figure}[t]
\centering
\includegraphics[width=0.486\columnwidth]{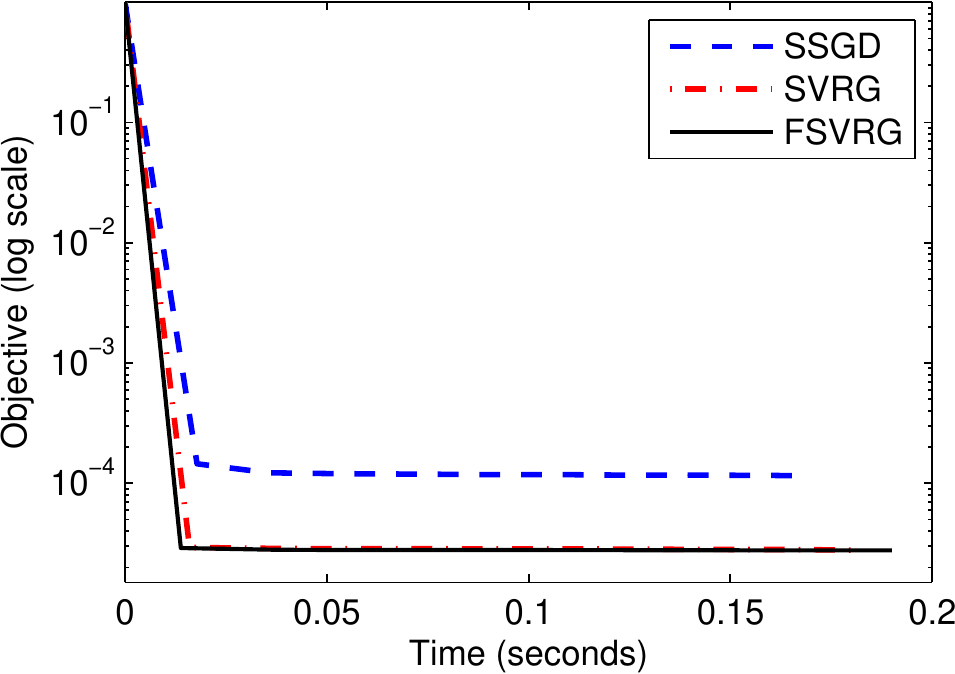}
\includegraphics[width=0.486\columnwidth]{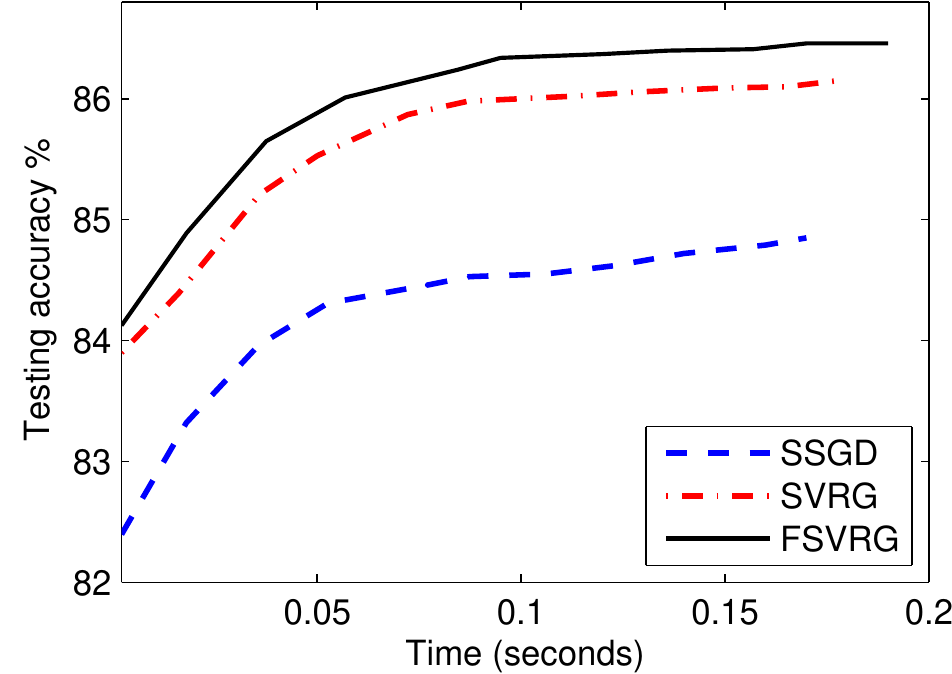}
\includegraphics[width=0.486\columnwidth]{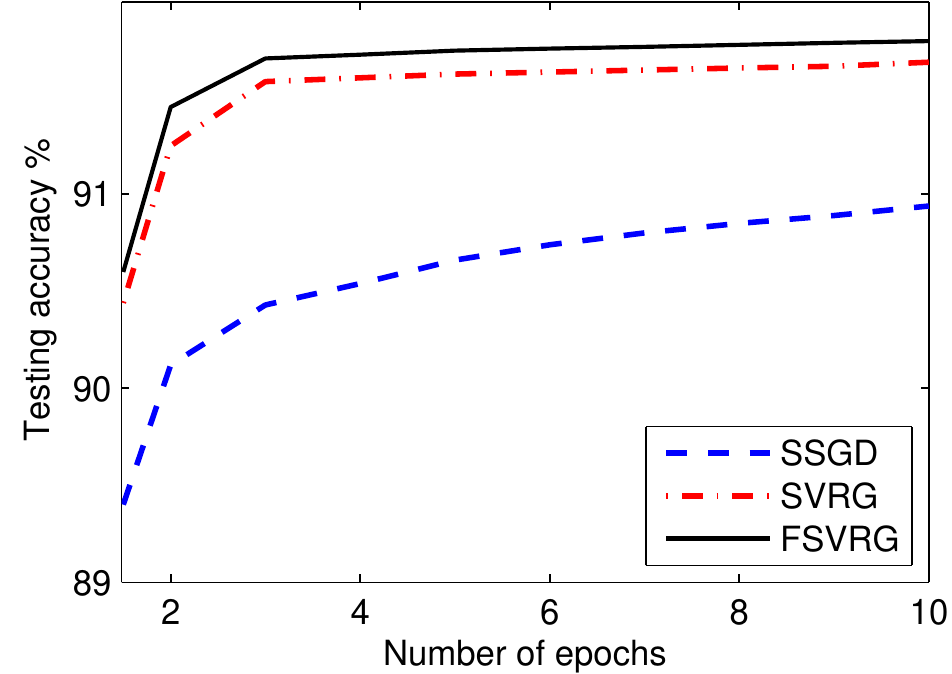}
\includegraphics[width=0.486\columnwidth]{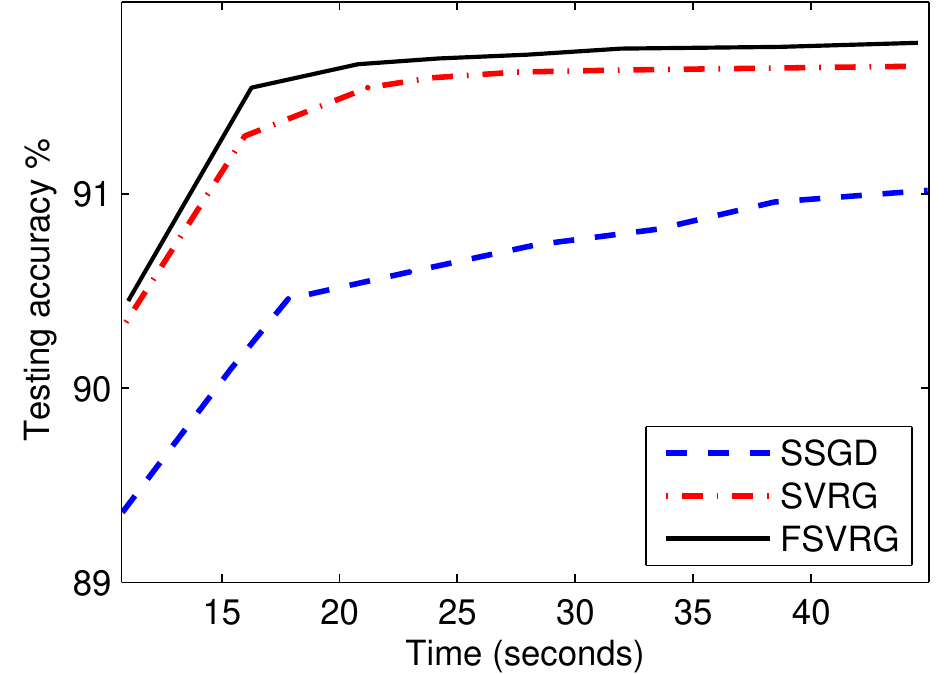}
\vspace{-3mm}

\caption{Comparison of different methods for SVM problems on IJCNN (top) and MNIST (bottom).}
\label{figs6}
\end{figure}

\section{Discussion and Conclusion}
Recently, there is a surge of interests in accelerating stochastic variance reduction gradient optimization. SVRG++~\cite{zhu:vrnc} uses the doubling-epoch technique (i.e., $m_{s+\!1}\!\!=\!2m_{s}$), which can reduce the number of gradient calculations in the early iterations, and lead to faster convergence in general. In contrast, our FSVRG method uses a more general growing epoch size strategy as in~\cite{mahdavi:sgd}, i.e., $m_{s+\!1}\!=\!\rho m_{s}$ with the factor $\rho\!>\!1$, which implies that we can be much more flexible in choosing $\rho$. Unlike SVRG++, FSVRG also enjoys the momentum acceleration trick and has the convergence guarantee for SC problems in Case 1.

The momentum acceleration technique has been used in accelerated SGD~\cite{sutskever:sgd} and variance reduced methods~\cite{lan:rpdg,lin:vrsg,nitanda:svrg,zhu:Katyusha}. Different from other methods~\cite{lan:rpdg,lin:vrsg,nitanda:svrg}, the momentum term of FSVRG involves the snapshot point, i.e., $\widetilde{x}^{s}$, which is also called as the Katyusha momentum in~\cite{zhu:Katyusha}. It was shown in~\cite{zhu:Katyusha} that Katyusha has the best known overall complexities for both SC and NSC problems. As analyzed above, FSVRG is much simpler and more efficient than Katyusha, which has also been verified in our experiments. Therefore, FSVRG is suitable for large-scale machine learning and data mining problems.

In this paper, we proposed a simple and fast stochastic variance reduction gradient (FSVRG) method, which integrates the momentum acceleration trick, and also uses a more flexible growing epoch size strategy. Moreover, we provided the convergence guarantees of our method, which show that FSVRG attains linear convergence for SC problems, and achieves the optimal $\mathcal{O}(1/T^2)$ convergence rate for NSC problems. The empirical study showed that the performance of FSVRG is much better than that of the state-of-the-art stochastic methods. Besides the ERM problems in Section~\ref{sec5}, we can apply FSVRG to other machine learning problems, e.g., deep learning and eigenvector computation~\cite{garber:svd}.

\section*{APPENDIX: Proof of Lemma 1}

\vspace{3mm}
\begin{proof}
The smoothness inequality in Definition 1 has the following equivalent form,
\begin{displaymath}
f(x_{2})\!\leq\! f(x_{1})\!+\!\langle\nabla\! f(x_{1}),x_{2}\!-\!x_{1}\rangle\!+\!0.5L\|x_{2}\!-\!x_{1}\|^{2}\!,\forall x_{1},x_{2}\!\in\!\mathbb{R}^{d}.
\end{displaymath}
Let $\beta_{1}\!\!>\!2$ be an appropriate constant, $\beta_{2}\!\!=\!\beta_{1}\!\!-\!1\!>\!1$, and $\widetilde{\nabla}_{\!i^{s}_{k}}\!:=\!\widetilde{\nabla}\!f_{i^{s}_{k}}\!(x^{s}_{k-\!1})$. Using the above inequality, we have
\begin{equation}\label{equ71}
\begin{split}
\!\!\!f(x^{s}_{k})\!\leq& f(x^{s}_{k\!-\!1})\!+\!\langle\nabla\!f(x^{s}_{k\!-\!1}),x^{s}_{k}\!\!-\!x^{s}_{k\!-\!1}\rangle\!\!+\!0.5L\|x^{s}_{k}\!\!-\!x^{s}_{k\!-\!1}\!\|^{2}\!\\
=& f(x^{s}_{k-\!1})\!+\!\left\langle\nabla\! f(x^{s}_{k-\!1}),x^{s}_{k}\!-\!x^{s}_{k-\!1}\right\rangle\\
&+\!0.5\beta_{1} L\!\left\|x^{s}_{k}\!-\!x^{s}_{k-\!1}\right\|^{2}\!-\!0.5\beta_{2}L\!\left\|x^{s}_{k}\!-\!x^{s}_{k-\!1}\right\|^{2}\\
=& f(x^{s}_{k-\!1})\!+\!\langle \widetilde{\nabla}_{\!i^{s}_{k}},x^{s}_{k}\!-\!x^{s}_{k-\!1}\rangle\!+\!0.5\beta_{1}L\|x^{s}_{k}\!-\!x^{s}_{k-\!1}\|^2\\
&\!+\!\!\langle\nabla\! f(x^{s}_{k\!-\!1})\!-\!\!\widetilde{\nabla}_{\!i^{s}_{k}}\!,x^{s}_{k}\!\!-\!x^{s}_{k\!-\!1}\rangle\!-\!0.5\beta_{2}L\|x^{s}_{k}\!\!-\!x^{s}_{k\!-\!1}\!\|^{2}\!\!.
\end{split}
\end{equation}
\vspace{-5mm}

\begin{equation}\label{equ72}
\begin{split}
&\mathbb{E}\!\left[\langle\nabla\! f(x^{s}_{k-\!1})\!-\!\widetilde{\nabla}_{i^{s}_{k}},\,x^{s}_{k}\!-\!x^{s}_{k-\!1}\rangle\right]\\
\leq\,& \mathbb{E}\!\left[\frac{1}{2\beta_{2}{L}}\|\nabla\!f(x^{s}_{k-\!1})\!-\!\widetilde{\nabla}_{i^{s}_{k}}\|^{2}+\frac{\beta_{2}{L}}{2}\|x^{s}_{k}\!-\!x^{s}_{k-\!1}\|^{2}\right]\\
\leq\,& \frac{1}{\beta_{2}}\!\left[f(\widetilde{x}^{s-\!1})\!-\!f(x^{s}_{k-\!1})\!-\!\left\langle\nabla\! f(x^{s}_{k-\!1}),\,\widetilde{x}^{s-\!1}\!-\!x^{s}_{k-\!1}\right\rangle\right]\\
&+\!{0.5\beta_{2}L}\;\!\mathbb{E}\!\left[\|x^{s}_{k}\!-\!x^{s}_{k-\!1}\|^{2}\right],
\end{split}
\end{equation}
where the first inequality follows from the Young's inequality (i.e.\ $x^{T}y\!\leq\!{\|x\|^2}\!/{(2\beta)}\!+\!{\beta\|y\|^2}\!/{2}$ for all $\beta\!>\!0$), and the second inequality holds due to Lemma~\ref{lemm2}. Note that for the mini-batch setting, Lemma~\ref{lemm2} needs to be replaced with Corollary~\ref{cor11}, and all statements in the proof of this lemma can be revised by simply replacing $1/\beta_{2}$ with $\rho(b)/\beta_{2}$.

Substituting the inequality \eqref{equ72} into the inequality \eqref{equ71}, and taking expectation over $i^{s}_{k}$, we have
\vspace{-1mm}
\begin{equation*}\label{equ73}
\begin{split}
&\;\;\;\;\,\mathbb{E}\left[\phi(x^{s}_{k})\right]-f(x^{s}_{k-\!1})\\
&\leq\!\mathbb{E}\!\left[g(x^{s}_{k})\!+\!\langle\widetilde{\nabla}_{\!i^{s}_{k}}, x^{s}_{k}\!\!-\!x^{s}_{k-\!1}\rangle\!+\!0.5\beta_{1}\!{L}\|x^{s}_{k}\!\!-\!x^{s}_{k-\!1}\|^2\right]\\
&\quad+\!\beta^{-\!1}_{2}\!\left[f(\widetilde{x}^{s-\!1})\!-\!f(x^{s}_{k-\!1})\!+\!\left\langle\nabla\! f(x^{s}_{k-\!1}),\,x^{s}_{k-\!1}\!-\!\widetilde{x}^{s-\!1}\right\rangle\right]\\
&\leq\!\mathbb{E}\!\left[\langle \theta_{\!s}\!\widetilde{\nabla}_{\!i^{s}_{k}}, y^{s}_{k}\!-\!y^{s}_{k-\!1}\rangle\!+\!0.5\beta_{1}\!{L}\theta_{s}^{2}\|y^{s}_{k}\!-\!y^{s}_{k-\!1}\|^2\!+\!\theta_{s} g(y^{s}_{k})\right]\\
&\quad\!\!\!\!\!+\!(1\!\!-\!\theta_{\!s}\!)g(\widetilde{x}^{s\!-\!1}\!)\!+\!\beta^{-\!1}_{2}\!\!\left[f(\widetilde{x}^{s\!-\!1}\!)\!-\!\!f(x^{s}_{k\!-\!1}\!)\!+\!\!\langle\nabla\! f(x^{s}_{k\!-\!1}\!),x^{s}_{k\!-\!1}\!\!-\!\widetilde{x}^{s\!-\!1}\rangle\right]\\
&\leq\!\mathbb{E}\!\!\left[\langle \theta_{\!s}\!\widetilde{\nabla}_{\!i^{s}_{k}}\!,x_{\star}\!\!-\!y^{s}_{k\!-\!1}\rangle\!+\!\frac{\!\beta_{1}\!{L} \theta_{\!s}^{2}\!\!}{2}(\|y^{s}_{k\!-\!1}\!\!-\!x_{\star}\!\|^2\!\!-\!\|y^{s}_{k}\!\!-\!x_{\star}\!\|^2)\!+\!\theta_{\!s} g(x_{\star})\!\right]\\
&\quad\!\!\!\!+\!(1\!\!-\!\theta_{\!s}\!)g(\widetilde{x}^{s\!-\!1}\!)\!+\!\beta^{-\!1}_{2}\!\!\left[f(\widetilde{x}^{s\!-\!1}\!)\!-\!\!f(x^{s}_{k\!-\!1}\!)\!+\!\langle\nabla\! f(x^{s}_{k\!-\!1}\!),x^{s}_{k\!-\!1}\!\!-\!\widetilde{x}^{s\!-\!1}\rangle\right]\\
&=\!\mathbb{E}\!\!\left[\frac{\beta_{1}\!{L} \theta_{\!s}^{2}}{2}\!\!\left(\|y^{s}_{k\!-\!1}\!-\!x_{\star}\!\|^2\!\!-\!\|y^{s}_{k}\!-\!x_{\star}\!\|^2\right)\!+\!\theta_{\!s} g(x_{\star})\right]\!\!+\!(1\!\!-\!\theta_{\!s}\!)g(\widetilde{x}^{s\!-\!1}\!)\\
&\quad\!\!\!\!\!\!+\!\left\langle\!\nabla\! f(x^{s}_{k\!-\!1}\!),\theta_{\!s} x_{\star}\!\!+\!(1\!\!-\!\theta_{\!s}\!)\widetilde{x}^{s\!-\!1}\!\!\!-\!x^{s}_{k\!-\!1}\!\!+\!\!\beta^{-\!1}_{2}\!(x^{s}_{k\!-\!1}\!\!-\!\widetilde{x}^{s\!-\!1}\!)\!\right\rangle\!+\!\!\beta^{-\!1}_{2}\!f(\widetilde{x}^{s\!-\!1}\!)\\
&\quad\!\!\!\!\!\!+\!\mathbb{E}\!\!\left[\left\langle\!-\!\nabla\!f_{i^{s}_{k}}\!(\widetilde{x}^{s\!-\!1}\!)\!+\!\!\nabla\! f(\widetilde{x}^{s\!-\!1}\!),\theta_{\!s} x_{\star}\!\!+\!(1\!-\!\theta_{\!s}\!)\widetilde{x}^{s\!-\!1}\!\!-\!x^{s}_{k\!-\!1}\!\right\rangle\!\right]\!\!-\!\!\beta^{-\!1}_{2}\!f(x^{s}_{\!k\!-\!1}\!)\\
&=\!\mathbb{E}\!\!\left[\!\frac{\beta_{1}\!{L} \theta_{\!s}^{2}}{2}\!\!\left(\|y^{s}_{k-\!1}\!\!-\!x_{\star}\!\|^2\!\!-\!\|y^{s}_{k}\!\!-\!x_{\star}\!\|^2\right)\!+\!\theta_{\!s} g(x_{\star})\!\right]\!\!+\!(1\!-\!\theta_{\!s}\!)g(\widetilde{x}^{s-\!1})\\
&\quad\!+\!\!\left\langle\nabla\! f(x^{s}_{k-\!1}),\theta_{s} x_{\star}\!+\!(1\!-\!\theta_{s})\widetilde{x}^{s-\!1}\!-\!x^{s}_{k-\!1}\!+\!\beta^{-\!1}_{2}(x^{s}_{k-\!1}\!-\!\widetilde{x}^{s-\!1})\right\rangle\\
&\quad\!+\!\beta^{-\!1}_{2}\!\left[f(\widetilde{x}^{s-\!1})\!-\!f(x^{s}_{k-\!1})\right],
\end{split}
\vspace{-2mm}
\end{equation*}
where the second inequality follows from the facts that $x^{s}_{k}\!-\!x^{s}_{k-\!1}\!=\!\theta_{\!s}(y^{s}_{k}\!-\!y^{s}_{k-\!1})$ and
$g(\theta_{\!s}y^{s}_{k}\!+\!(1\!\!-\!\theta_{\!s})\widetilde{x}^{s-\!1})\!\leq\! \theta_{\!s} g(y^{s}_{k})\!+\!(1\!\!-\!\theta_{\!s})g(\widetilde{x}^{s-\!1})$; the third inequality holds due to Lemma~\ref{lemm3} with $z^{*}\!=\!y^{s}_{k}$, $z\!=\!x_{\star}$, $z_{0}\!=\!y^{s}_{k-\!1}$, $\rho\!=\!\beta_{1}L\theta_{\!s}$, and $\psi(z)\!:=\!g(z)+\!\langle \widetilde{\nabla}_{\!i^{s}_{k}},z\!-\!y^{s}_{k-\!1}\rangle$ (or $\psi(z)\!:=\!\langle \widetilde{\nabla}_{\!i^{s}_{k}}\!+\!\nabla\! g(x^{s}_{k-\!1}),z\!-\!y^{s}_{k-\!1}\rangle$). The first equality holds due to the facts that
\vspace{-2mm}
\begin{displaymath}
\begin{split}
&\langle \theta_{s}\!\widetilde{\nabla}_{\!i^{s}_{k}},\, x_{\star}\!\!-\!y^{s}_{k-\!1}\rangle\!=\!\langle \widetilde{\nabla}_{\!i^{s}_{k}},\, \theta_{s} x_{\star}\!\!+\!(1\!-\!\theta_{s})\widetilde{x}^{s-\!1}\!\!-\!x^{s}_{k-\!1}\rangle\\
=&\left\langle\nabla\! f_{i^{s}_{k}}\!(x^{s}_{k-\!1}),\, \theta_{s} x_{\star}\!\!+\!(1\!-\!\theta_{s})\widetilde{x}^{s-\!1}\!\!-\!x^{s}_{k-\!1}\right\rangle\\
&+\!\left\langle-\nabla\! f_{i^{s}_{k}}\!(\widetilde{x}^{s-\!1})\!+\!\nabla\! f(\widetilde{x}^{s-\!1}),\, \theta_{s} x_{\star}\!\!+\!(1\!-\!\theta_{s})\widetilde{x}^{s-\!1}\!\!-\!x^{s}_{k-\!1}\right\rangle,
\end{split}
\end{displaymath}
and $\mathbb{E}[\nabla\! f_{i^{s}_{k}}\!(x^{s}_{k-\!1})]\!=\!\nabla\! f(x^{s}_{k-\!1})$. The last equality follows from
$\mathbb{E}\!\left[\langle-\nabla\! f_{i^{s}_{k}}\!(\widetilde{x}^{s-\!1})\!+\!\nabla\! f(\widetilde{x}^{s-\!1}),\,\theta_{s} x_{\star}\!+\!(1\!-\!\theta_{s})\widetilde{x}^{s-\!1}\!\!-\!x^{s}_{k-\!1}\rangle\right]\!=\!0.$
\vspace{-4mm}

\begin{equation*}\label{equ75}
\begin{split}
&\langle\nabla\! f(x^{s}_{k\!-\!1}),\theta_{s} x_{\star}\!\!+\!(1\!-\!\theta_{s})\widetilde{x}^{s\!-\!1}\!\!-\!x^{s}_{k\!-\!1}\!\!+\!\beta^{-\!1}_{2}\!(x^{s}_{k\!-\!1}\!\!-\!\widetilde{x}^{s\!-\!1})\rangle\\
=&\langle\nabla\! f(x^{s}_{k\!-\!1}),\theta_{s} x_{\star}\!+\!(1\!-\!\theta_{s}\!-\!\beta^{-\!1}_{2})\widetilde{x}^{s\!-\!1}\!\!+\!\beta^{-\!1}_{2}\!x^{s}_{k\!-\!1}\!\!-\!x^{s}_{k\!-\!1}\rangle\\
\leq &f\!\left(\theta_{s} x_{\star}\!+\!(1\!-\!\theta_{s}\!-\!\beta^{-\!1}_{2})\widetilde{x}^{s-\!1}\!+\!\beta^{-\!1}_{2}\!x^{s}_{k-\!1}\right)\!-\!f(x^{s}_{k-\!1})\\
\leq &\theta_{s} f(x_{\star})\!+\!(1\!-\!\theta_{s}\!-\!\beta^{-\!1}_{2})f(\widetilde{x}^{s-\!1})\!+\!\beta^{-\!1}_{2}\!f(x^{s}_{k-\!1})\!-\!f(x^{s}_{k-\!1}),
\end{split}
\end{equation*}
where the first inequality holds due to the fact that $\langle \nabla\! f(x_{1}),$ $x_{2}-x_{1}\rangle\!\leq\! f(x_{2})\!-\!f(x_{1})$, and the last inequality follows from the convexity of $f(\cdot)$ and $1\!-\!\theta_{s}\!-\!\beta^{-\!1}_{2}\!\geq\!0$, which can be easily satisfied. Combining the above two inequalities, we have
\vspace{-1mm}
\begin{equation*}\label{equ76}
\begin{split}
&\mathbb{E}\!\left[\phi(x^{s}_{k})\right]\\
\leq&\theta_{s}\phi(x_{\star})\!+\!(1\!\!-\!\theta_{s})\phi(\widetilde{x}^{s-\!1})\!+\!\frac{\!\beta_{1}\!{L} \theta_{\!s}^{2}\!}{2}\mathbb{E}\!\!\left[\|y^{s}_{k-\!1}\!\!-\!x_{\star}\|^2\!-\!\|y^{s}_{k}\!\!-\!x_{\star}\|^2\right]\!\!,
\end{split}
\end{equation*}
\vspace{-4mm}
\begin{equation*}
\begin{split}
&\;\,\mathbb{E}\!\left[\phi(x^{s}_{k})\!-\!\phi(x_{\star})\right]\\
\vspace{-1mm}
\leq&\,(1\!\!-\!\theta_{s})[\phi(\widetilde{x}^{s-\!1})\!-\!\phi(x_{\star})]\!+\!\frac{\!\beta_{1}\!{L} \theta_{\!s}^{2}\!}{2}\mathbb{E}\!\!\left[\|y^{s}_{k-\!1}\!\!-\!x_{\star}\|^2\!-\!\|y^{s}_{k}\!-\!x_{\star}\|^2\right]\!\!.
\end{split}
\vspace{-1mm}
\end{equation*}
Due to the convexity of $\phi(\cdot)$ and the definition $\widetilde{x}^{s}\!\!=\!\!\frac{1}{m_{s}}\!\!\sum^{m_{s}}_{k=1}\!x^{s}_{k}$, then $\phi\!\left(\frac{1}{m_{s}}\!\!\sum^{m_{s}}_{k=1}\!x^{s}_{k}\right)\!\leq\!\frac{1}{m_{s}}\!\sum^{m_{s}}_{k=1}\!\!\phi(x^{s}_{k})$. Taking expectation over the history of $i_{1},\ldots,i_{m_{s}}$ on the above inequality, and summing it up over $k\!=\!1,\ldots,m_{s}$ at the $s$-th stage, we have
$\mathbb{E}\!\left[\phi(\widetilde{x}^{s})\!-\!\phi(x_{\star})\right]
\leq\frac{\theta_{s}^{2}}{2\eta m_{s}}\mathbb{E}\!\left[\|y^{s}_{0}\!-\!x_{\star}\|^2\!-\!\|y^{s}_{m_{s}}\!-\!x_{\star}\|^2\right]\!+(1\!-\!\theta_{s})\mathbb{E}[\phi(\widetilde{x}^{s-\!1})\!-\!\phi(x_{\star})]$. This completes the proof.
\end{proof}

\bibliographystyle{abbrv}
\bibliography{sigproc}

\onecolumn

\section*{More Experimental Results}
We first present the detailed descriptions for the seven data sets, as shown in Table~\ref{tab_sim1}. In particular, we report some additional experimental results for $\ell_{2}$-norm and elastic net regularized logistic regression, ridge regression and Lasso problems, respectively. Similar to Algorithm~\ref{alg2}, we also extend the original SVRG~\cite{johnson:svrg} to the setting of non-smooth component functions, and present a stochastic variance reduced sub-gradient descent (SVRSG) algorithm, as outlined in Algorithm~\ref{alg3}.

\begin{algorithm}[!th]
\caption{A variant of SVRG for non-smooth component functions}
\label{alg3}
\renewcommand{\algorithmicrequire}{\textbf{Input:}}
\renewcommand{\algorithmicensure}{\textbf{Initialize:}}
\renewcommand{\algorithmicoutput}{\textbf{Output:}}
\begin{algorithmic}[1]
\REQUIRE the number of epochs $S$, the number of iterations $m$ per epoch, and the step size $\eta$.\\
\ENSURE $\widetilde{x}^{0}$ and $w_{1},\ldots,w_{m}$.\\
\FOR{$s=1,2,\ldots,S$}
\STATE {$\widetilde{\xi}=\frac{1}{n}\!\sum^{n}_{i=1}\!\partial f_{i}(\widetilde{x}^{s-\!1})$, $x^{s}_{0}=\widetilde{x}^{s-\!1}$;}
\FOR{$k=1,2,\ldots,m$}
\STATE {Pick $i^{s}_{k}$ uniformly at random from $\{1,\ldots,n\}$;}
\STATE {$\widetilde{\partial}f_{i^{s}_{k}}(x^{s}_{k-\!1})=\partial f_{i^{s}_{k}}(x^{s}_{k-\!1})-\partial f_{i^{s}_{k}}(\widetilde{x}^{s-\!1})+\widetilde{\xi}$;}
\STATE {$x^{s}_{k}=\Pi_{\mathcal{K}}\!\!\left[x^{s}_{k-\!1}-\eta\;\!(\widetilde{\partial}f_{i^{s}_{k}}\!(x^{s}_{k-\!1})+\nabla g(x^{s}_{k-\!1}))\right]$;}
\ENDFOR
\STATE {$\widetilde{x}^{s}\!=\!\frac{1}{\sum_{k}\! w_{k}}\!\sum^{m}_{k=1}\!w_{k} x^{s}_{k}$;}
\ENDFOR
\OUTPUT {$\widetilde{x}^{S}$}
\end{algorithmic}
\end{algorithm}

\begin{table}[!th]
\centering
\caption{Summary of data sets and regularization parameters used in our experiments.}
\label{tab_sim1}
\setlength{\tabcolsep}{16.5pt}
\renewcommand\arraystretch{1.09}
\begin{tabular}{lcccc}
\hline
\ Data sets   & Sizes $n$    & Dimensions $d$  & $\lambda$ & $\lambda_{1},\;\lambda_{2}$\\
\hline
\ IJCNN      & 49,990         & 22             & $\{10^{-3},10^{-4},10^{-5},10^{-6},10^{-7}\}$    & $\{10^{-4},10^{-5},10^{-6}\}$  \\
\ Protein     & 145,751        & 74             & $\{10^{-3},10^{-4},10^{-5},10^{-6}\}$    & $\{10^{-4},10^{-5},10^{-6}\}$ \\
\ Covtype     & 581,012        & 54             & $\{10^{-3},10^{-4},10^{-5},10^{-6},10^{-7}\}$    & $\{10^{-4},10^{-6},10^{-6}\}$ \\
\ SUSY        & 5,000,000      & 18             & $\{10^{-4},10^{-5},10^{-6},10^{-7}\}$    & $\{10^{-4},10^{-5},10^{-6}\}$ \\
\ Rcv1        & 20,242         & 47,236         &     $10^{-4}$             & --- \\
\ Sido0       & 12,678         & 4,932          &     $10^{-4}$             & --- \\
\ MNIST       & 70,000         & 784            &     $10^{-5}$             & --- \\
\hline
\end{tabular}
\end{table}

\begin{figure}[!th]
\centering
\subfigure[IJCNN: $\lambda\!=\!10^{-3}$]{\includegraphics[width=0.243\columnwidth]{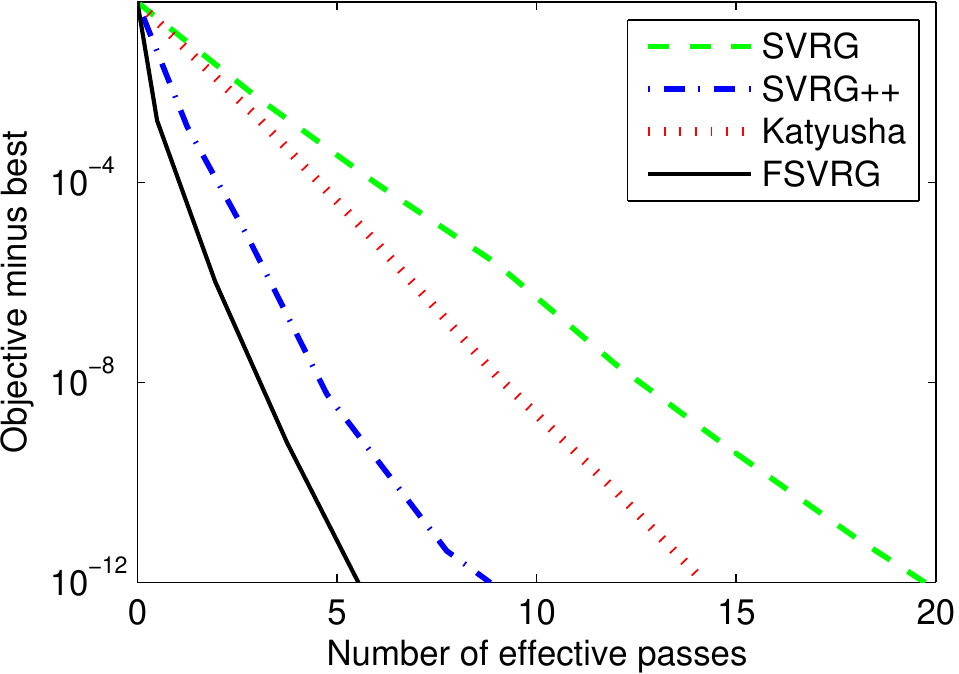}}
\subfigure[Protein: $\lambda\!=\!10^{-3}$]{\includegraphics[width=0.243\columnwidth]{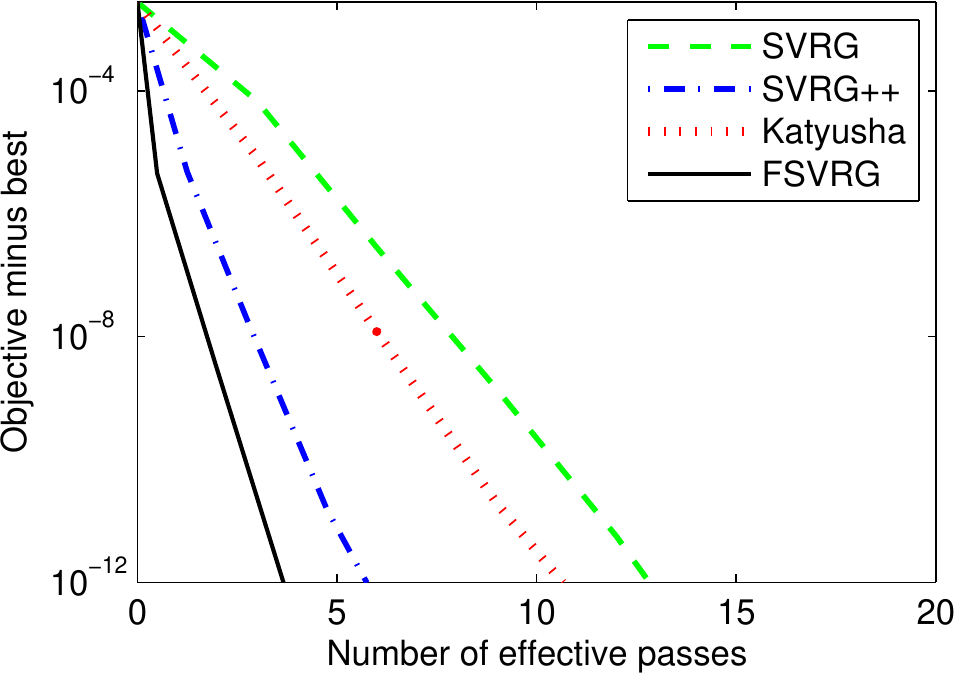}}
\subfigure[Covtype: $\lambda\!=\!10^{-4}$]{\includegraphics[width=0.243\columnwidth]{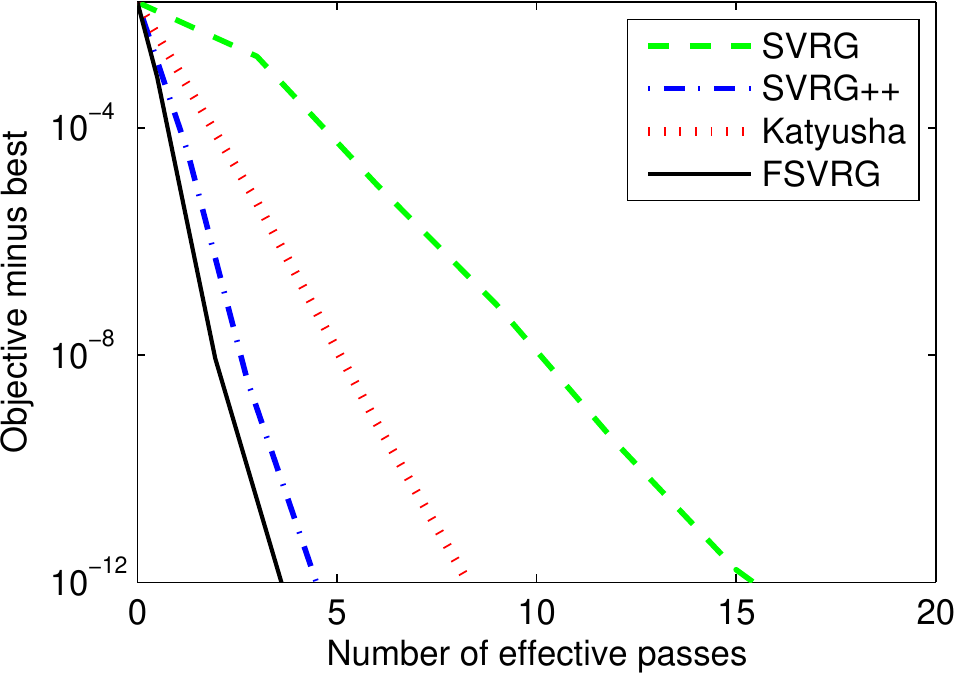}}
\subfigure[SUSY: $\lambda\!=\!10^{-4}$]{\includegraphics[width=0.243\columnwidth]{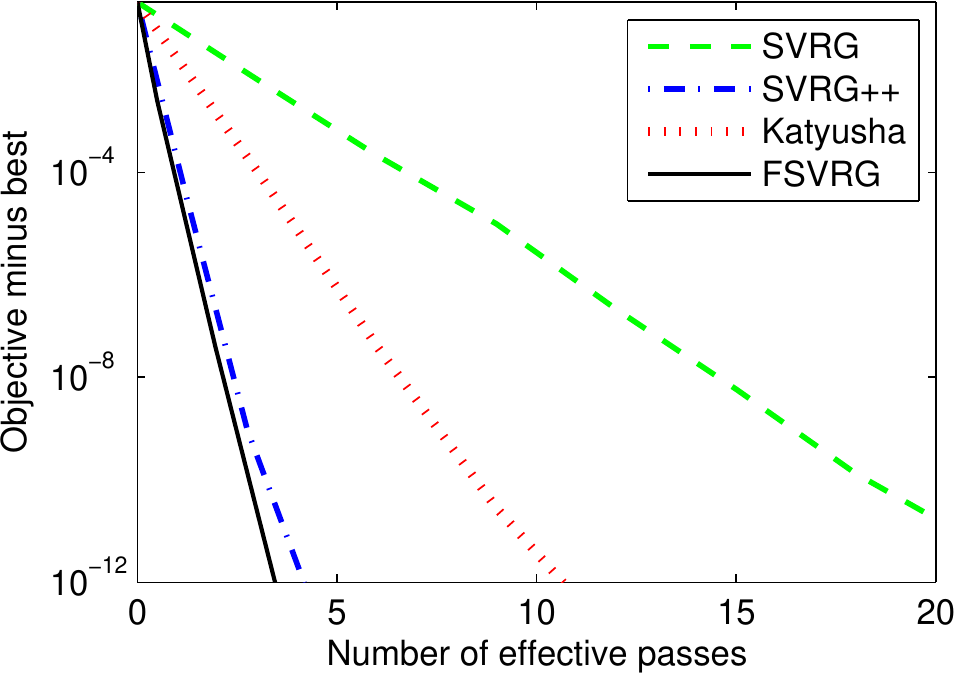}}
\subfigure[IJCNN: $\lambda\!=\!10^{-4}$]{\includegraphics[width=0.243\columnwidth]{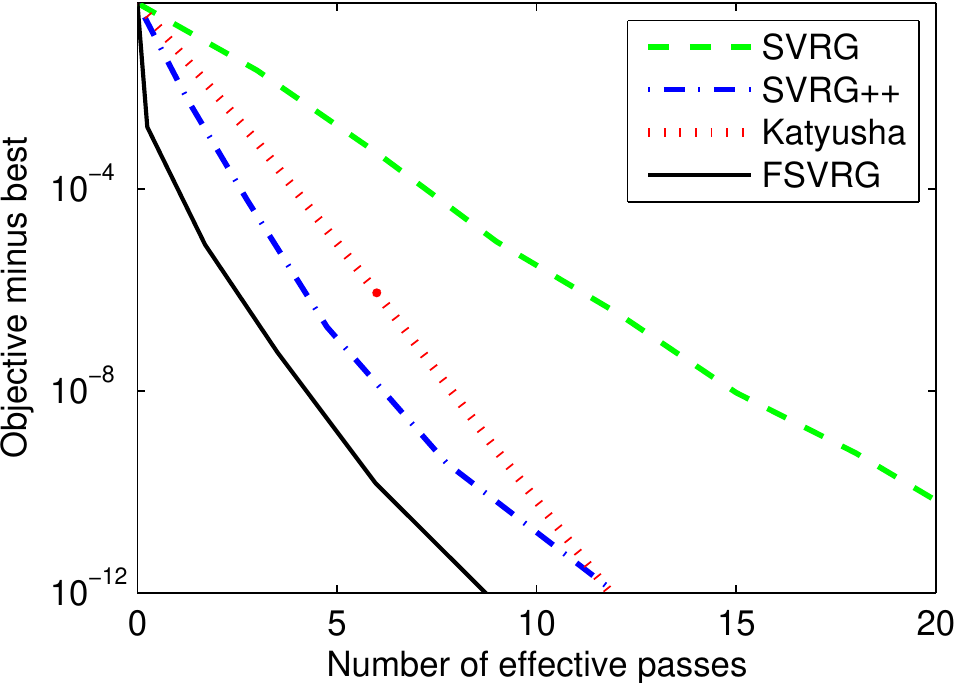}}
\subfigure[Protein: $\lambda\!=\!10^{-4}$]{\includegraphics[width=0.243\columnwidth]{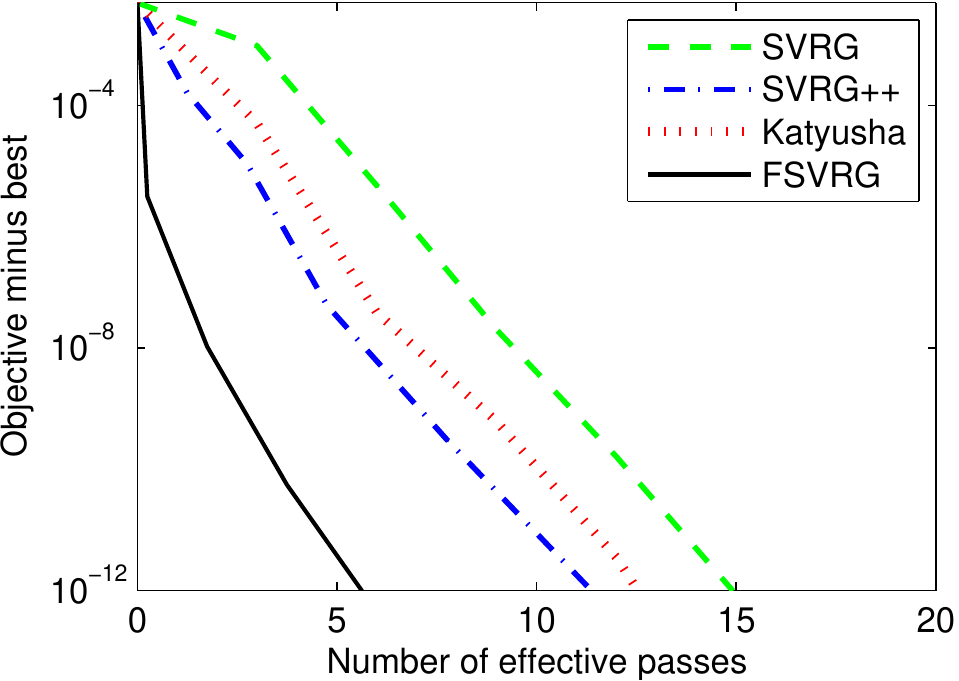}}
\subfigure[Covtype: $\lambda\!=\!10^{-5}$]{\includegraphics[width=0.243\columnwidth]{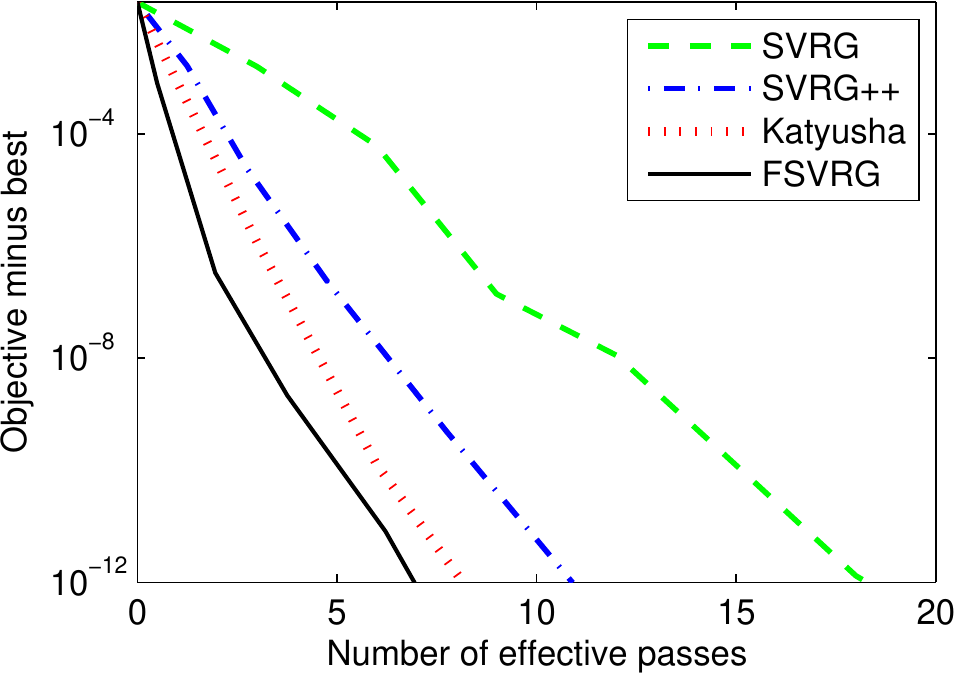}}
\subfigure[SUSY: $\lambda\!=\!10^{-5}$]{\includegraphics[width=0.243\columnwidth]{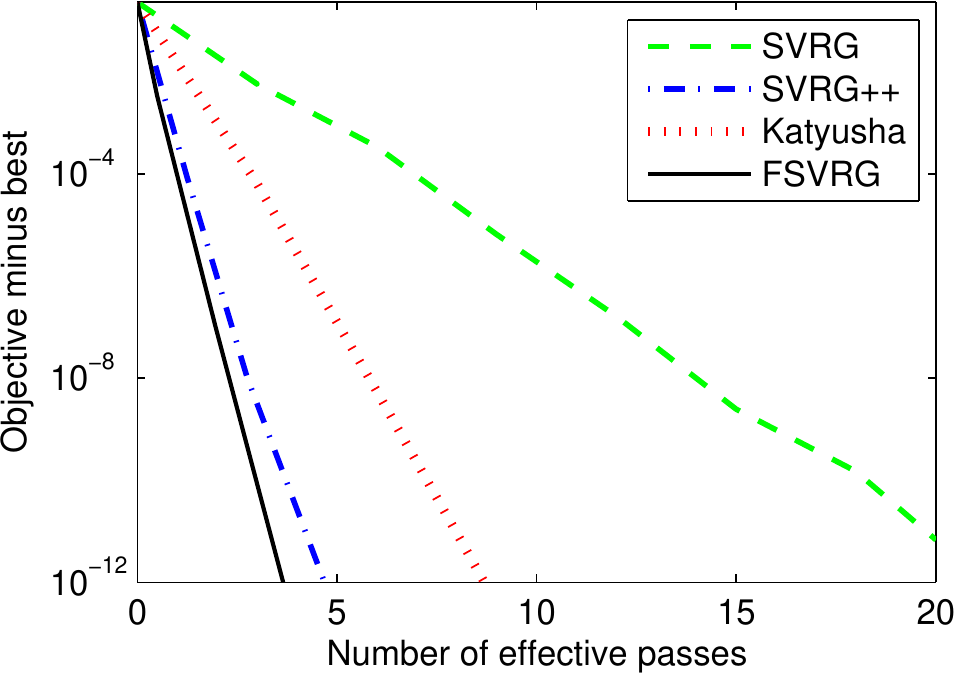}}
\subfigure[IJCNN: $\lambda\!=\!10^{-5}$]{\includegraphics[width=0.243\columnwidth]{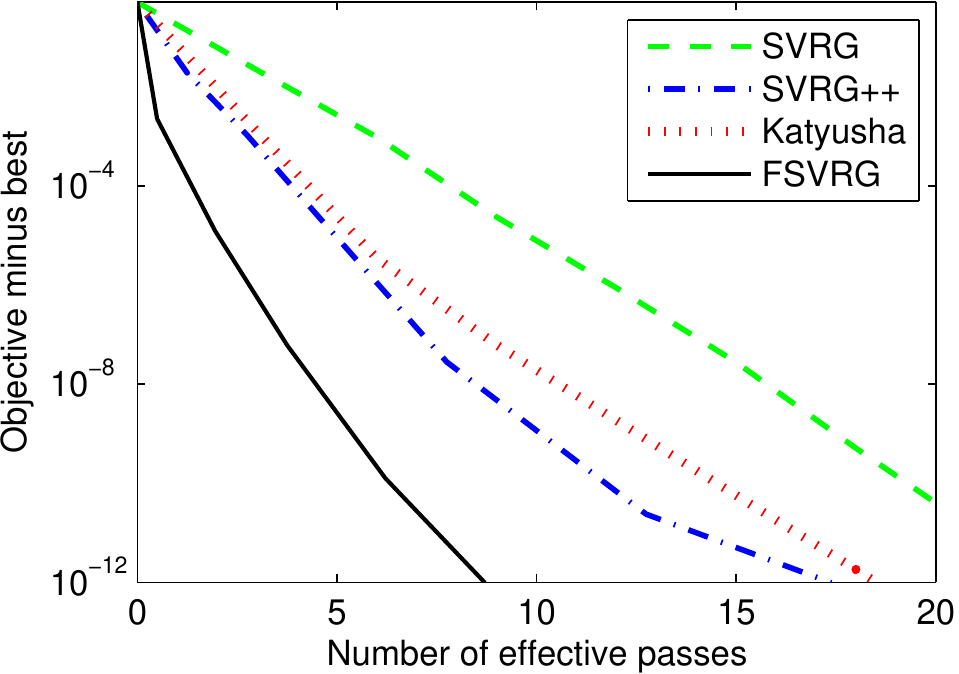}}
\subfigure[Protein: $\lambda\!=\!10^{-5}$]{\includegraphics[width=0.243\columnwidth]{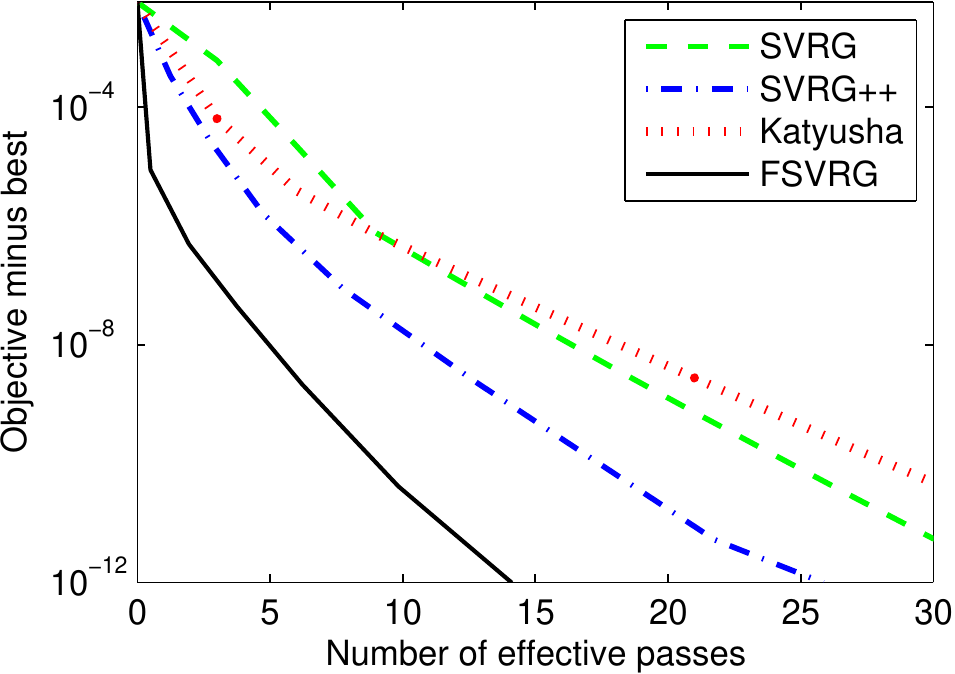}}
\subfigure[Covtype: $\lambda\!=\!10^{-6}$]{\includegraphics[width=0.243\columnwidth]{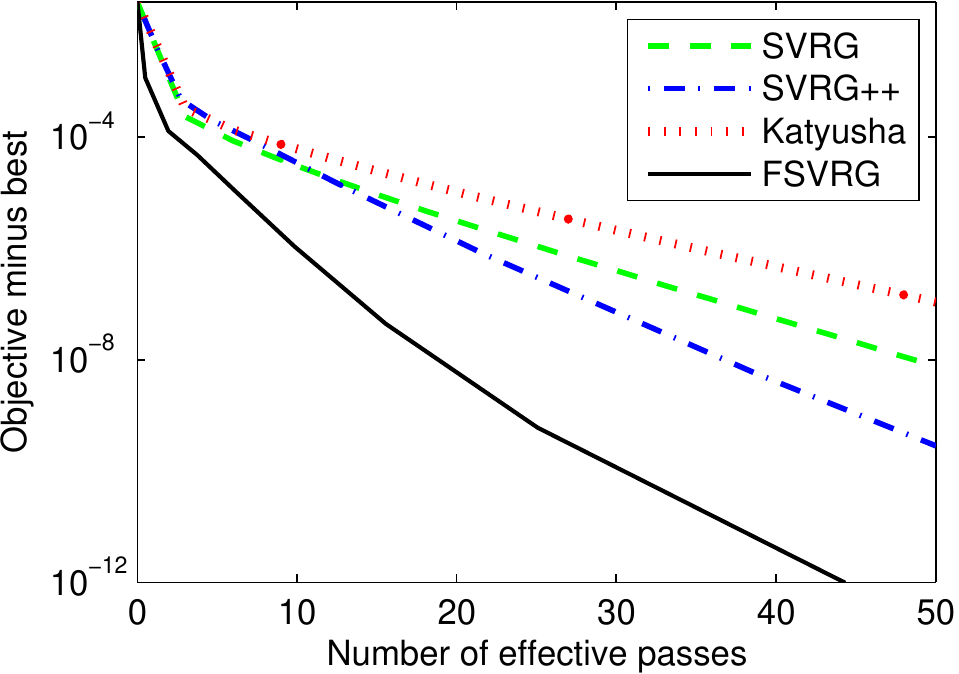}}
\subfigure[SUSY: $\lambda\!=\!10^{-6}$]{\includegraphics[width=0.243\columnwidth]{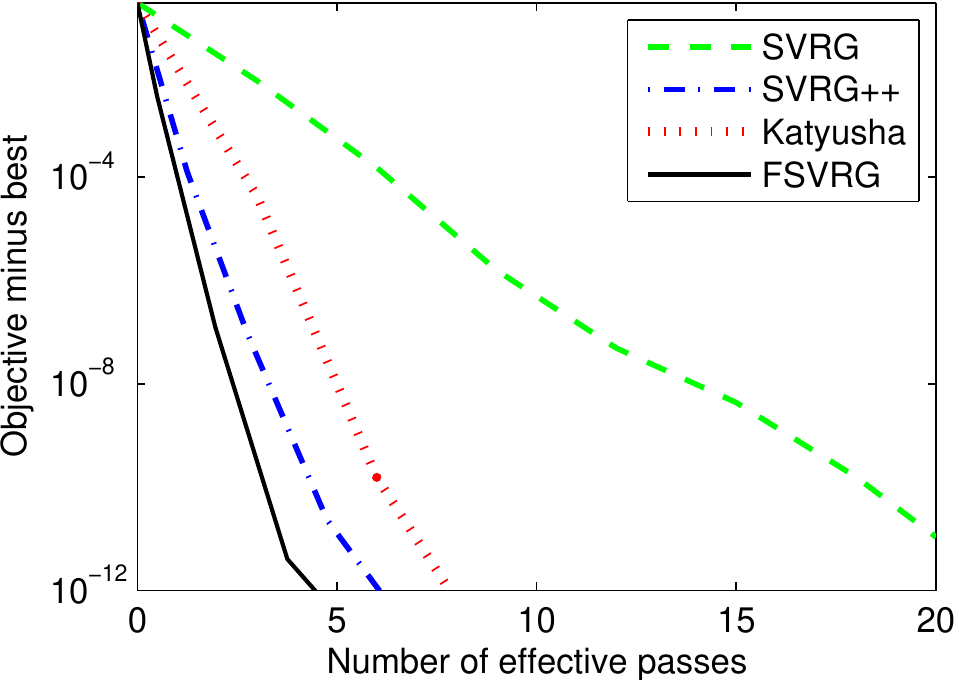}}
\subfigure[IJCNN: $\lambda\!=\!10^{-6}$]{\includegraphics[width=0.243\columnwidth]{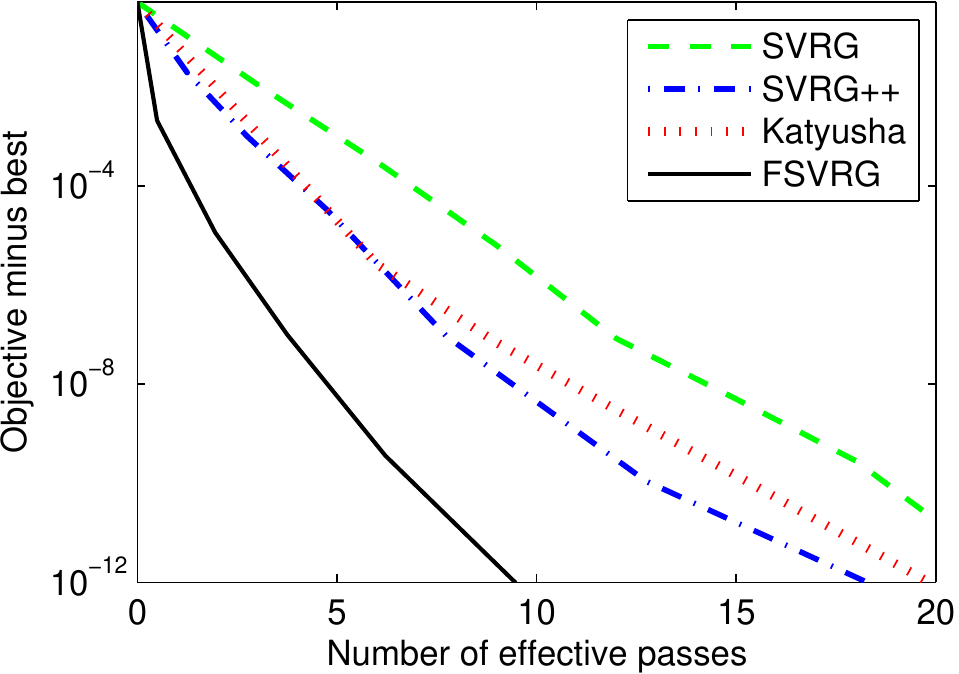}}
\subfigure[Protein: $\lambda\!=\!10^{-6}$]{\includegraphics[width=0.243\columnwidth]{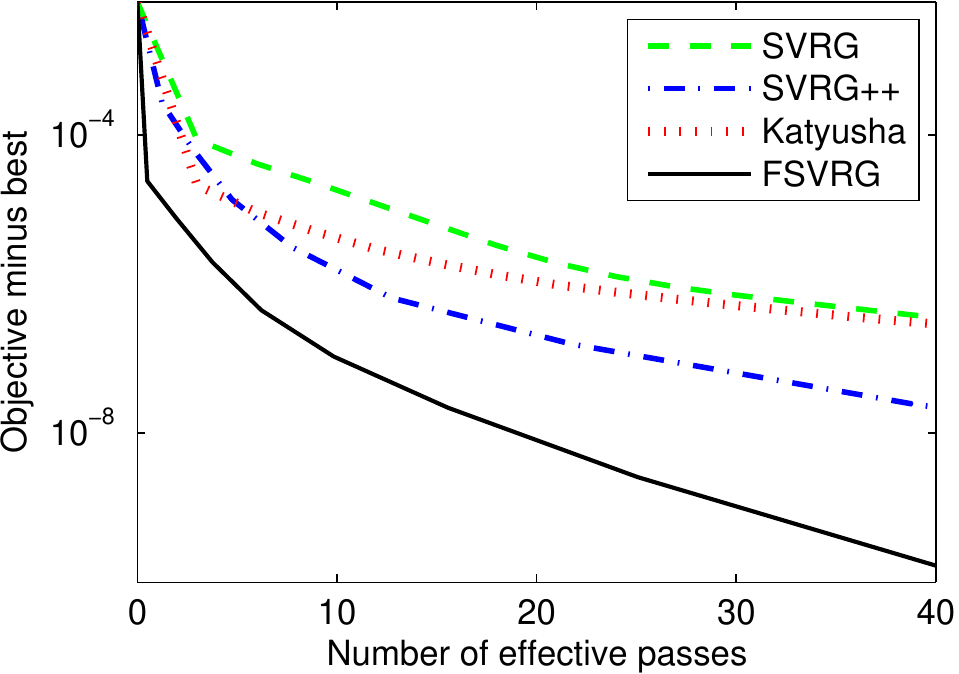}}
\subfigure[Covtype: $\lambda\!=\!10^{-7}$]{\includegraphics[width=0.243\columnwidth]{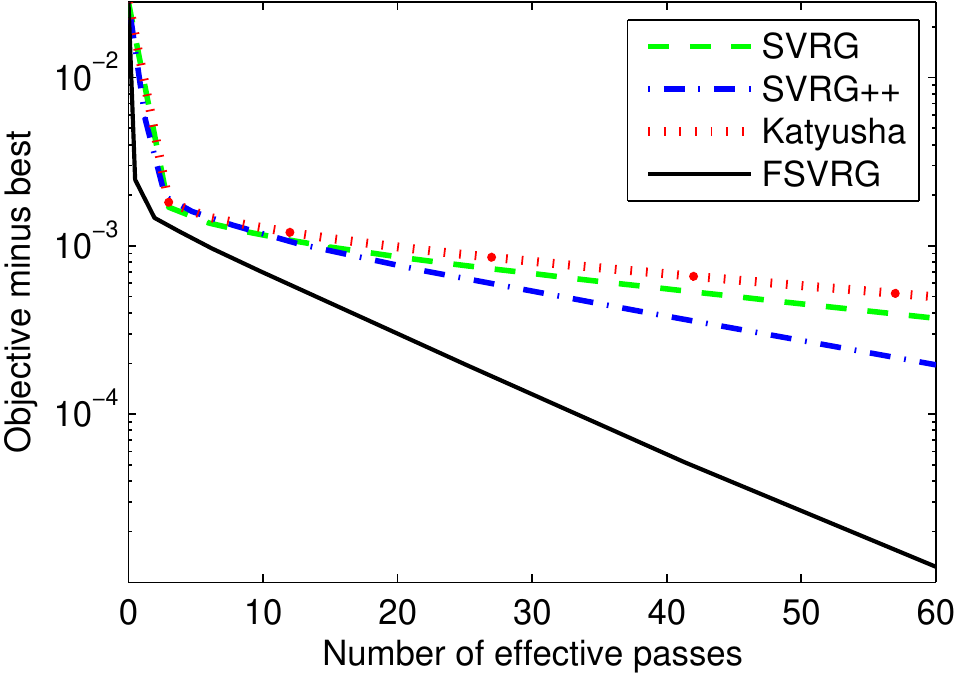}}
\subfigure[SUSY: $\lambda\!=\!10^{-7}$]{\includegraphics[width=0.243\columnwidth]{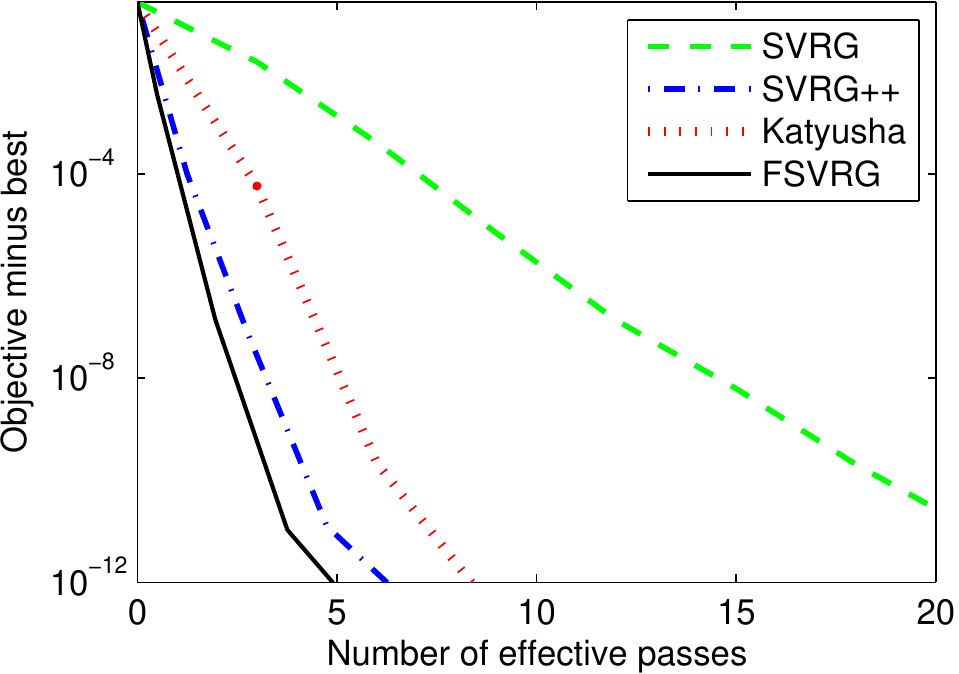}}
\caption{Comparison of SVRG~\cite{johnson:svrg}, SVRG++~\cite{zhu:vrnc}, Katyusha~\cite{zhu:Katyusha}, and FSVRG for solving $\ell_{2}$-norm regularized logistic regression problems with different regularization parameters. The $y$-axis represents the objective value minus the minimum, and the $x$-axis corresponds to the number of effective passes.}
\label{figs7}
\end{figure}

\begin{figure}[!th]
\centering
\subfigure[IJCNN: $\lambda\!=\!10^{-3}$]{\includegraphics[width=0.243\columnwidth]{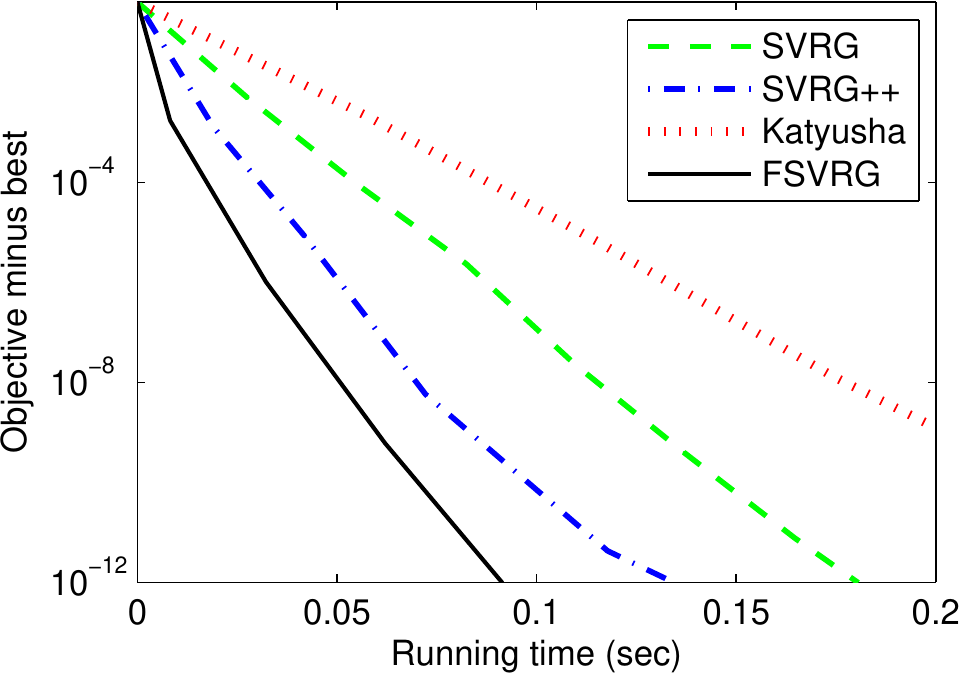}}
\subfigure[Protein: $\lambda\!=\!10^{-3}$]{\includegraphics[width=0.243\columnwidth]{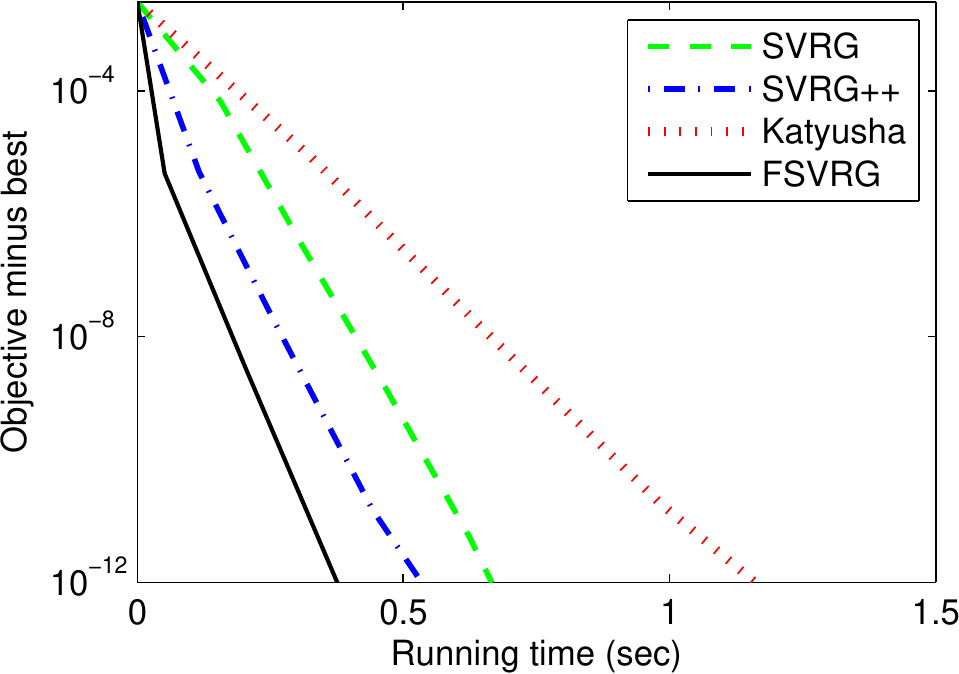}}
\subfigure[Covtype: $\lambda\!=\!10^{-4}$]{\includegraphics[width=0.243\columnwidth]{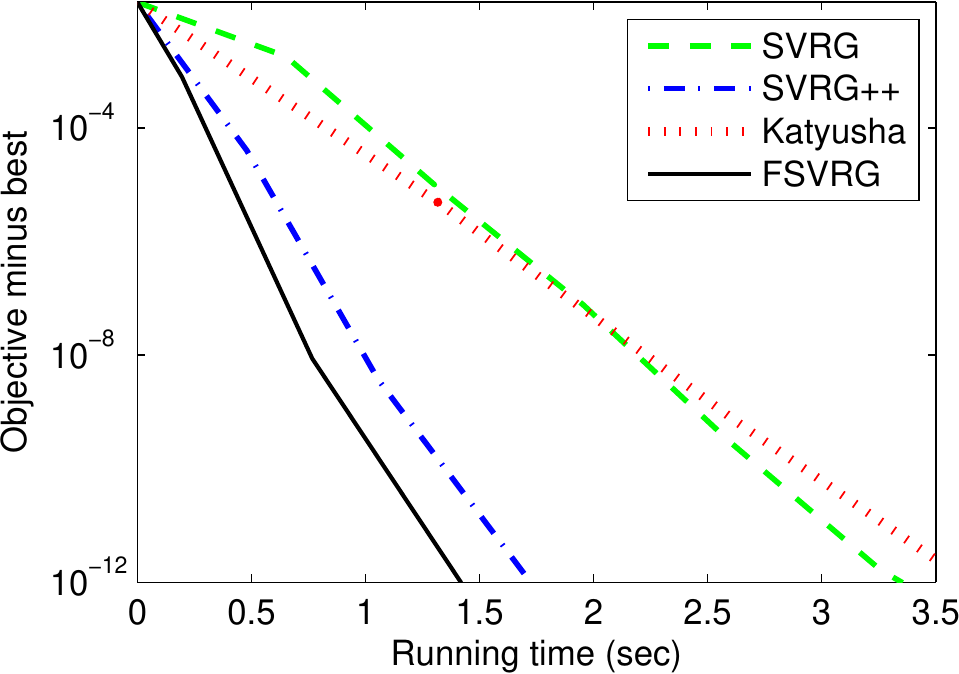}}
\subfigure[SUSY: $\lambda\!=\!10^{-4}$]{\includegraphics[width=0.243\columnwidth]{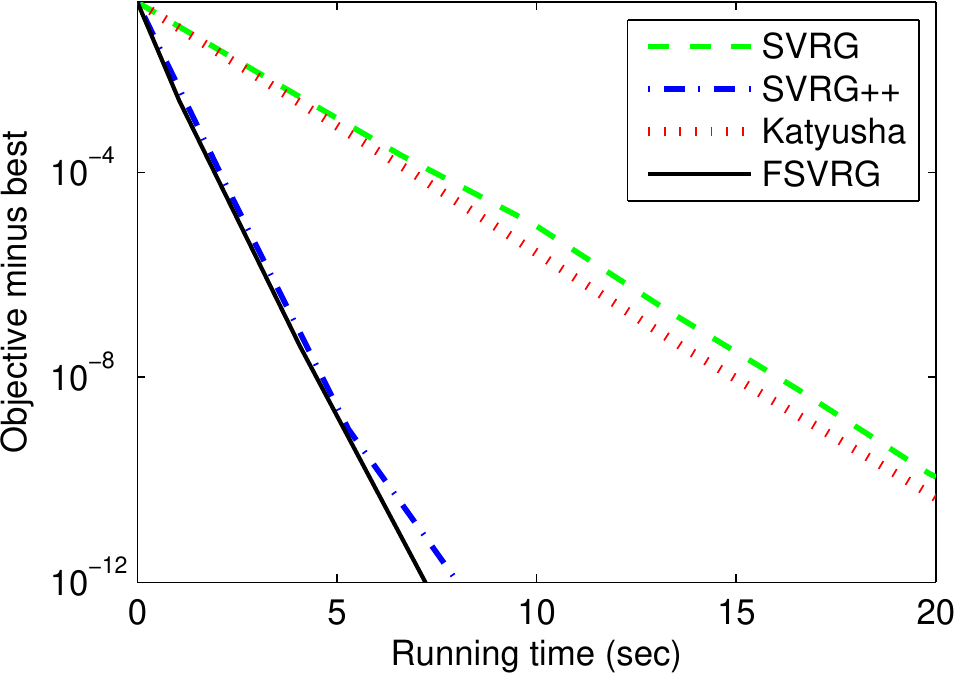}}

\subfigure[IJCNN: $\lambda\!=\!10^{-4}$]{\includegraphics[width=0.243\columnwidth]{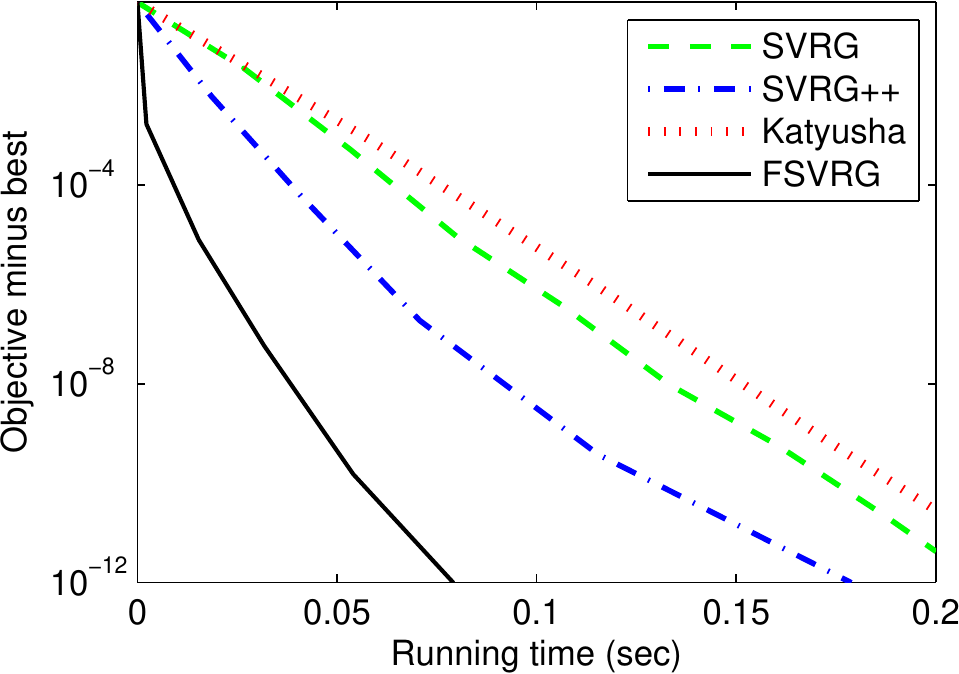}}
\subfigure[Protein: $\lambda\!=\!10^{-4}$]{\includegraphics[width=0.243\columnwidth]{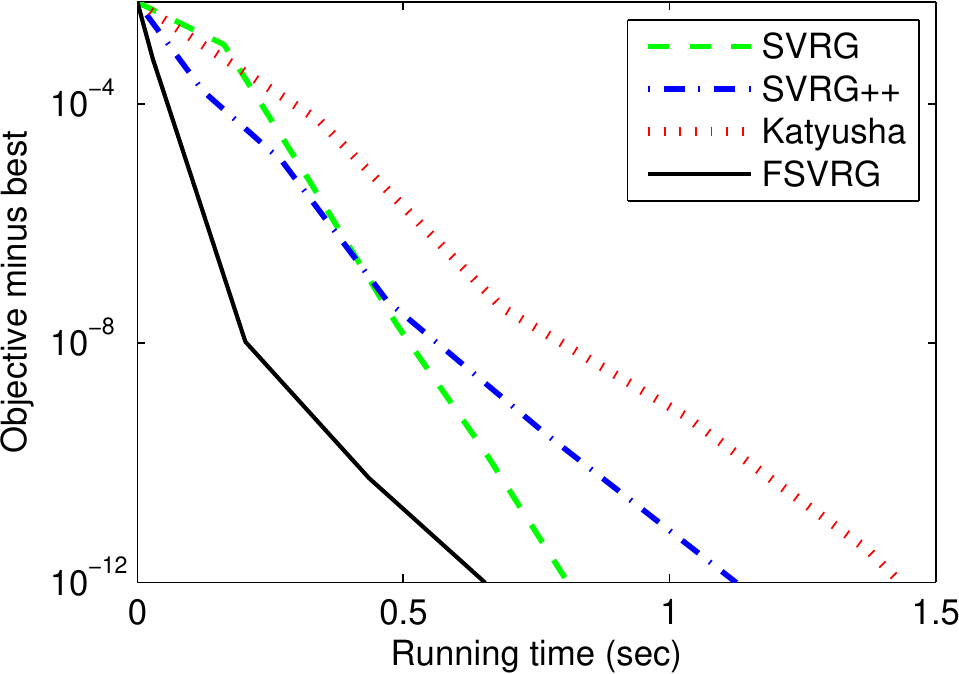}}
\subfigure[Covtype: $\lambda\!=\!10^{-5}$]{\includegraphics[width=0.243\columnwidth]{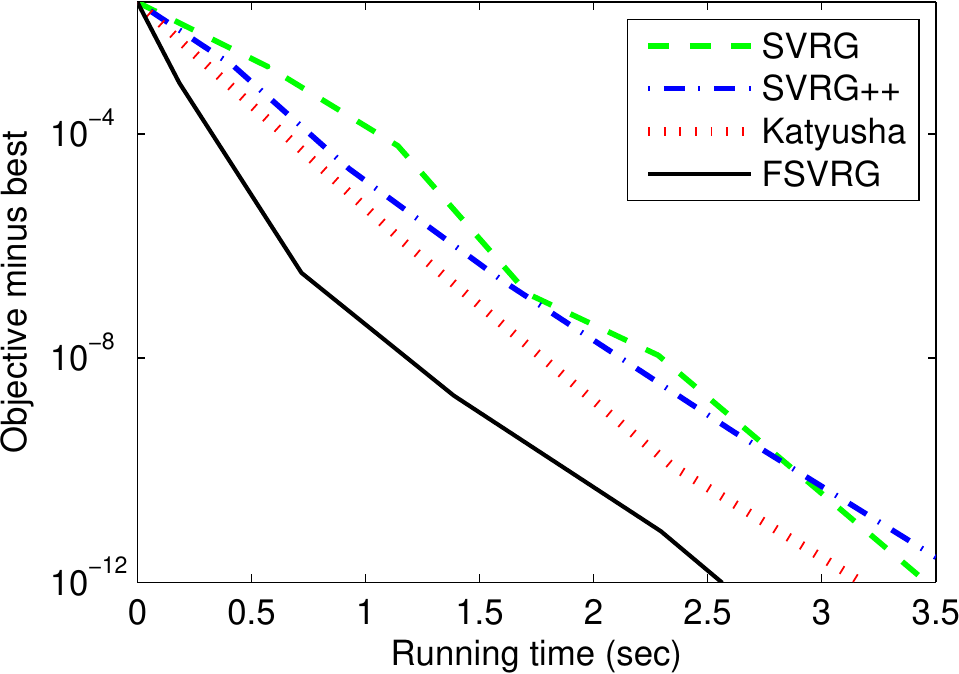}}
\subfigure[SUSY: $\lambda\!=\!10^{-5}$]{\includegraphics[width=0.243\columnwidth]{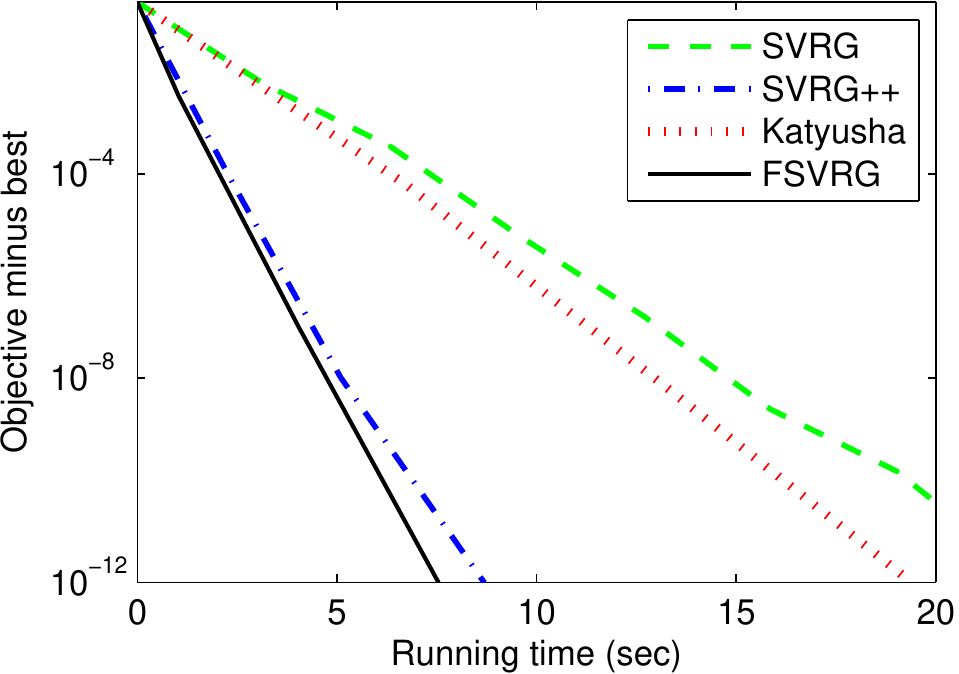}}

\subfigure[IJCNN: $\lambda\!=\!10^{-5}$]{\includegraphics[width=0.243\columnwidth]{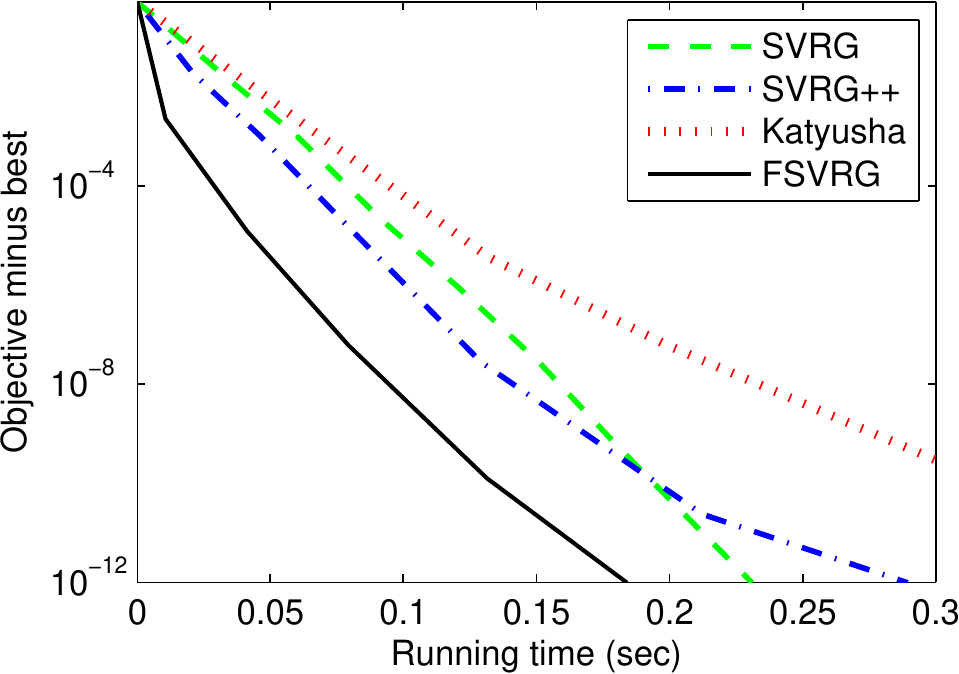}}
\subfigure[Protein: $\lambda\!=\!10^{-5}$]{\includegraphics[width=0.243\columnwidth]{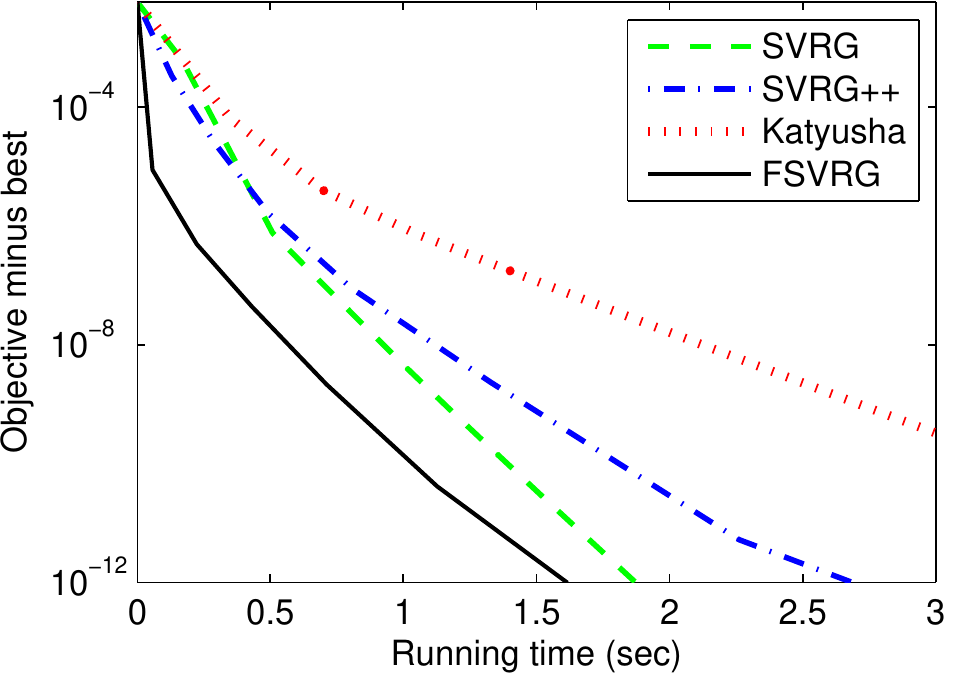}}
\subfigure[Covtype: $\lambda\!=\!10^{-6}$]{\includegraphics[width=0.243\columnwidth]{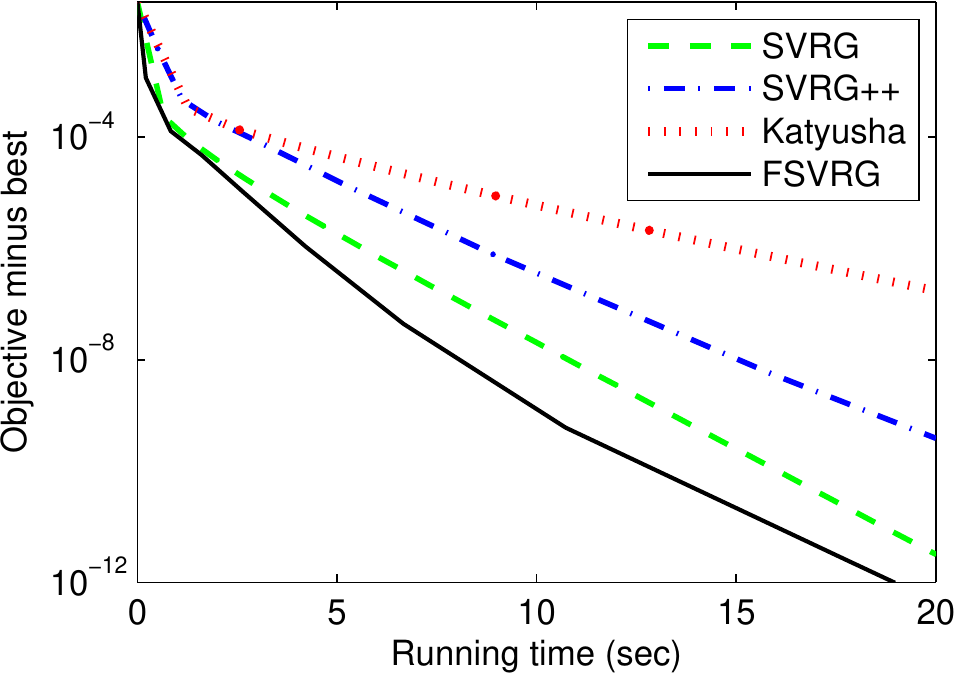}}
\subfigure[SUSY: $\lambda\!=\!10^{-6}$]{\includegraphics[width=0.243\columnwidth]{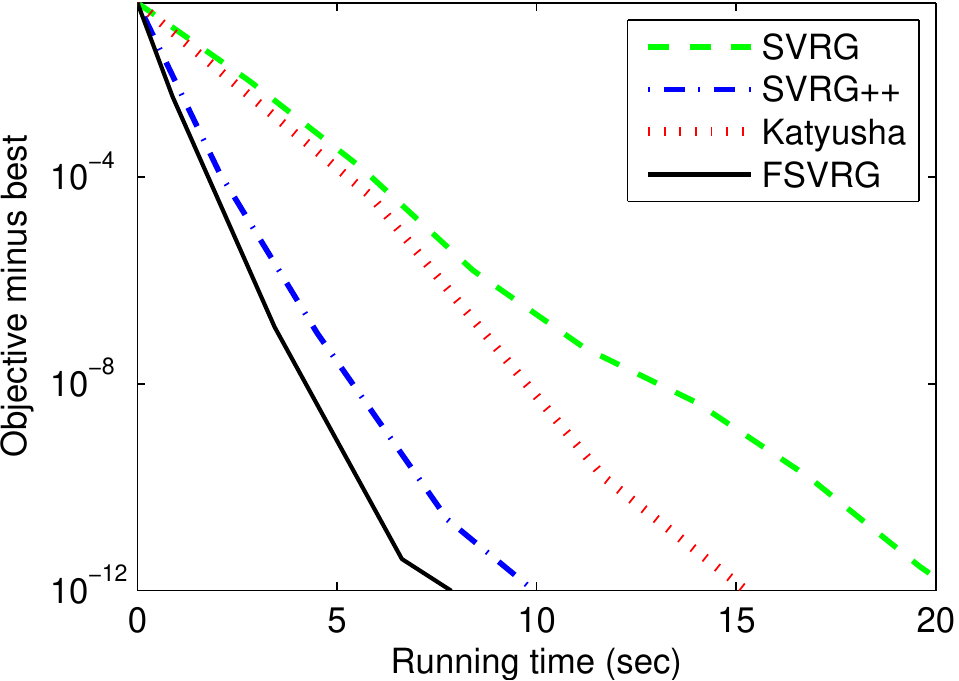}}

\subfigure[IJCNN: $\lambda\!=\!10^{-6}$]{\includegraphics[width=0.243\columnwidth]{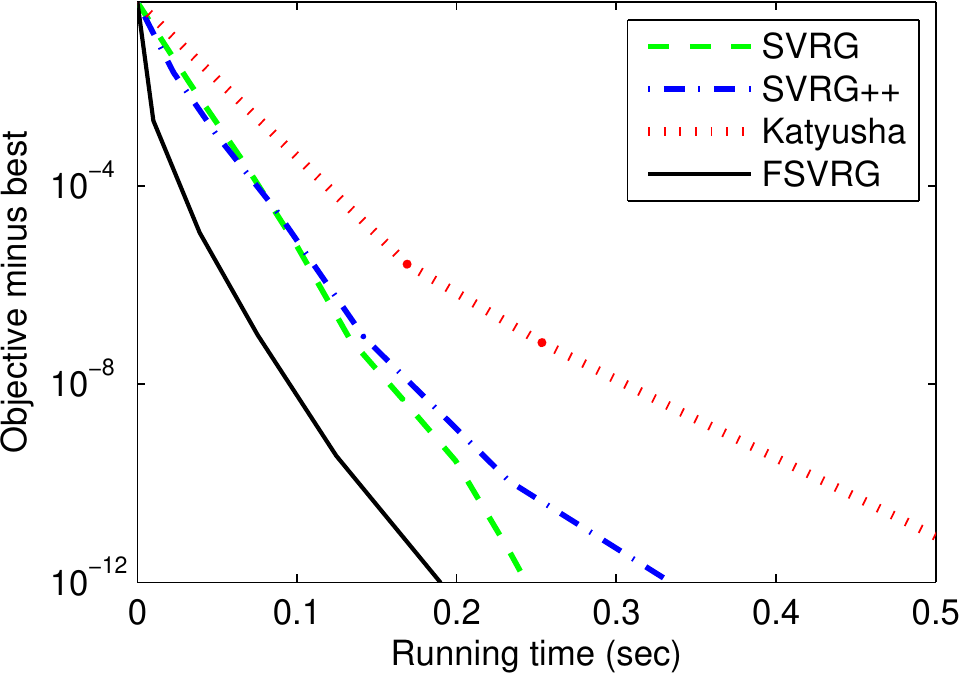}}
\subfigure[Protein: $\lambda\!=\!10^{-6}$]{\includegraphics[width=0.243\columnwidth]{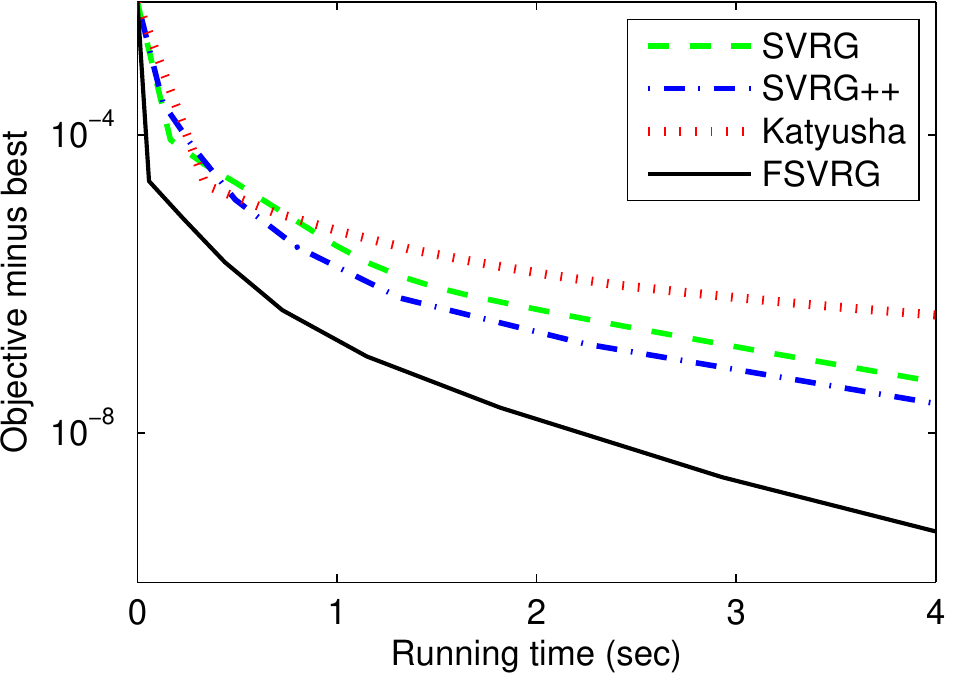}}
\subfigure[Covtype: $\lambda\!=\!10^{-7}$]{\includegraphics[width=0.243\columnwidth]{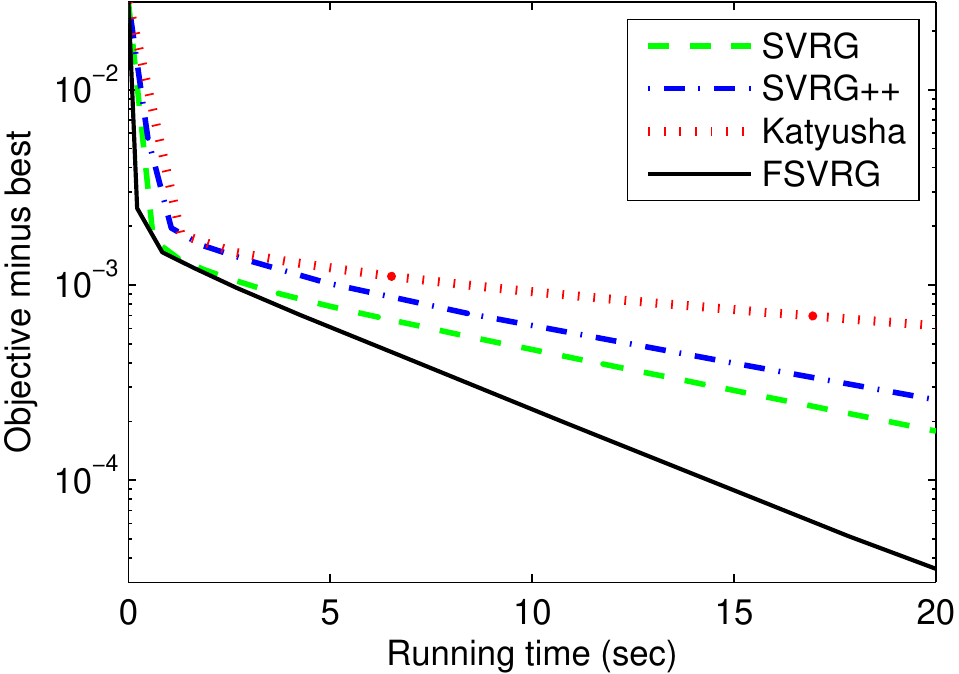}}
\subfigure[SUSY: $\lambda\!=\!10^{-7}$]{\includegraphics[width=0.243\columnwidth]{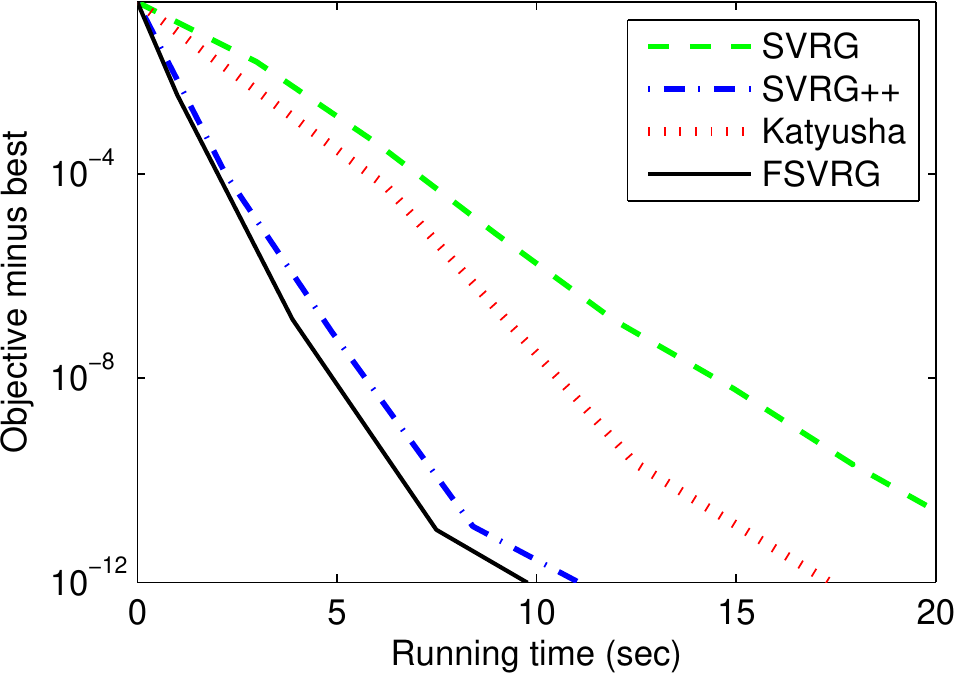}}
\caption{Comparison of SVRG~\cite{johnson:svrg}, SVRG++~\cite{zhu:vrnc}, Katyusha~\cite{zhu:Katyusha}, and FSVRG for solving $\ell_{2}$-norm regularized logistic regression problems with different regularization parameters. The $y$-axis represents the objective value minus the minimum, and the $x$-axis corresponds to the running time (seconds).}
\label{figs8}
\end{figure}

\begin{figure}[!th]
\centering
\subfigure[$\lambda_{1}\!=\!10^{-4}$ and $\lambda_{2}\!=\!10^{-4}$]{\includegraphics[width=0.243\columnwidth]{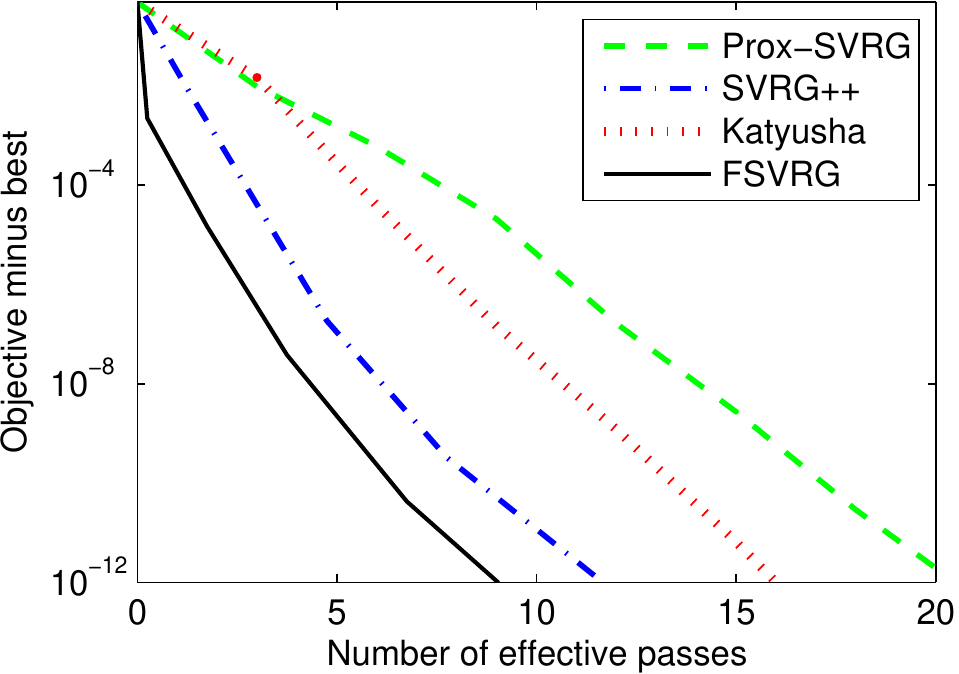}}
\subfigure[$\lambda_{1}\!=\!10^{-4}$ and $\lambda_{2}\!=\!10^{-4}$]{\includegraphics[width=0.243\columnwidth]{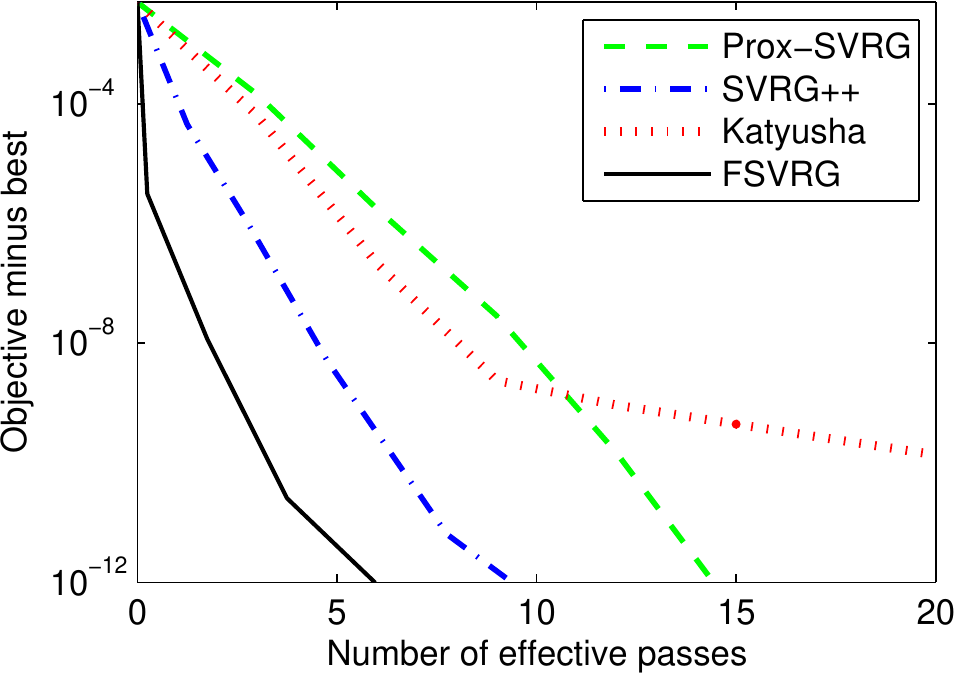}}
\subfigure[$\lambda_{1}\!=\!10^{-4}$ and $\lambda_{2}\!=\!10^{-5}$]{\includegraphics[width=0.243\columnwidth]{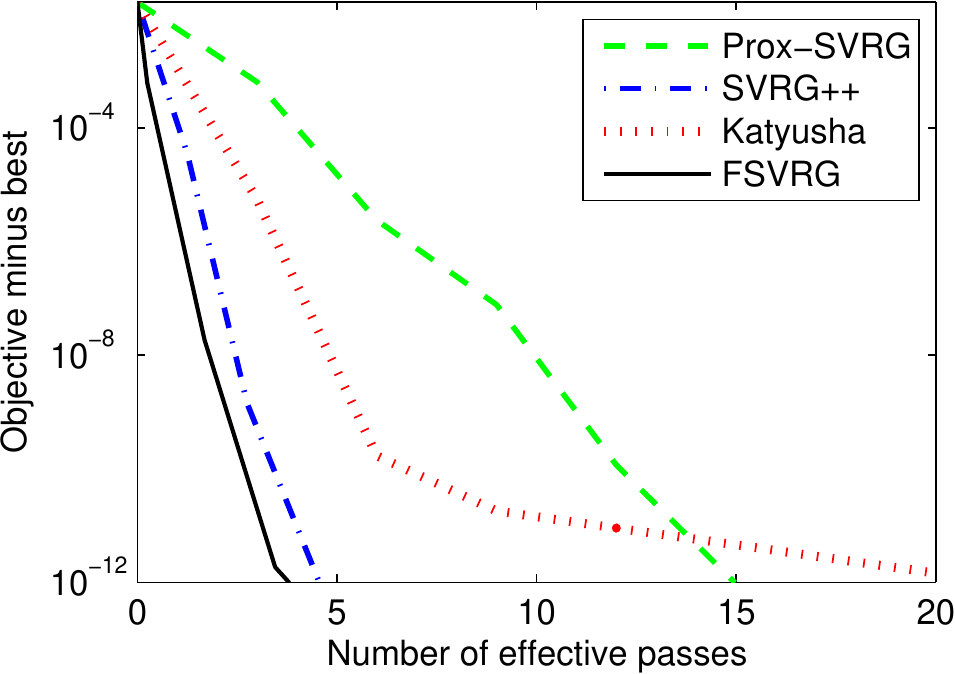}}
\subfigure[$\lambda_{1}\!=\!10^{-4}$ and $\lambda_{2}\!=\!10^{-5}$]{\includegraphics[width=0.243\columnwidth]{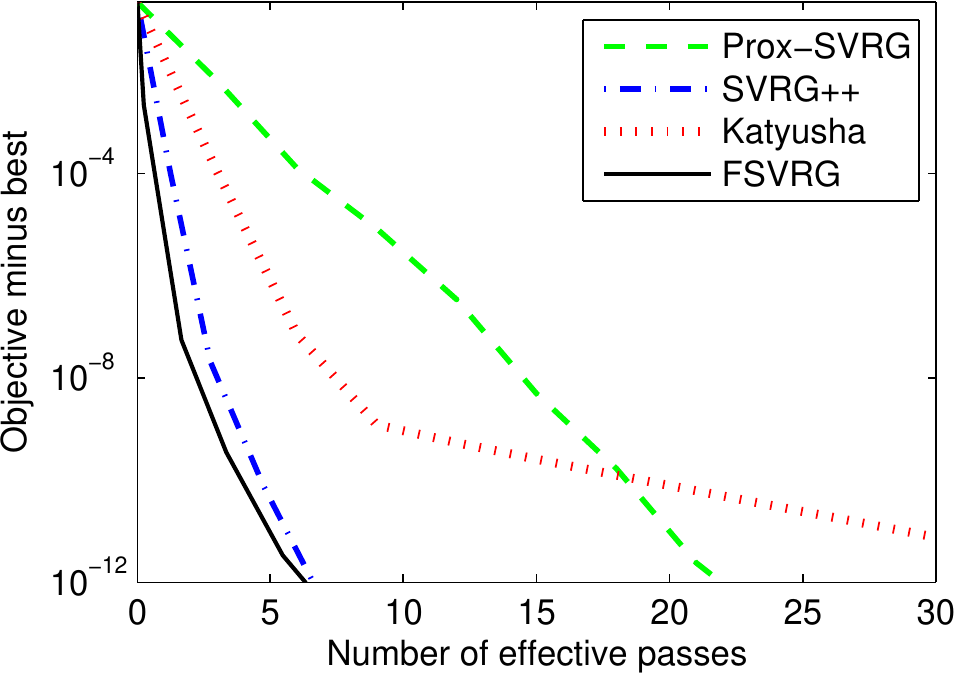}}

\subfigure[$\lambda_{1}\!=\!10^{-4}$ and $\lambda_{2}\!=\!10^{-5}$]{\includegraphics[width=0.243\columnwidth]{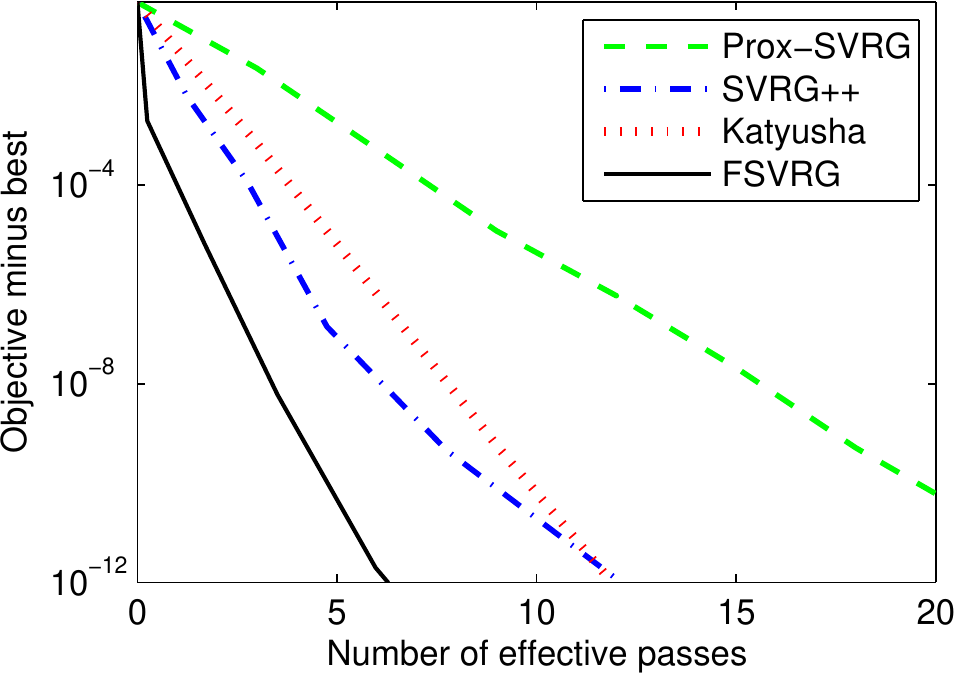}}
\subfigure[$\lambda_{1}\!=\!10^{-4}$ and $\lambda_{2}\!=\!10^{-5}$]{\includegraphics[width=0.243\columnwidth]{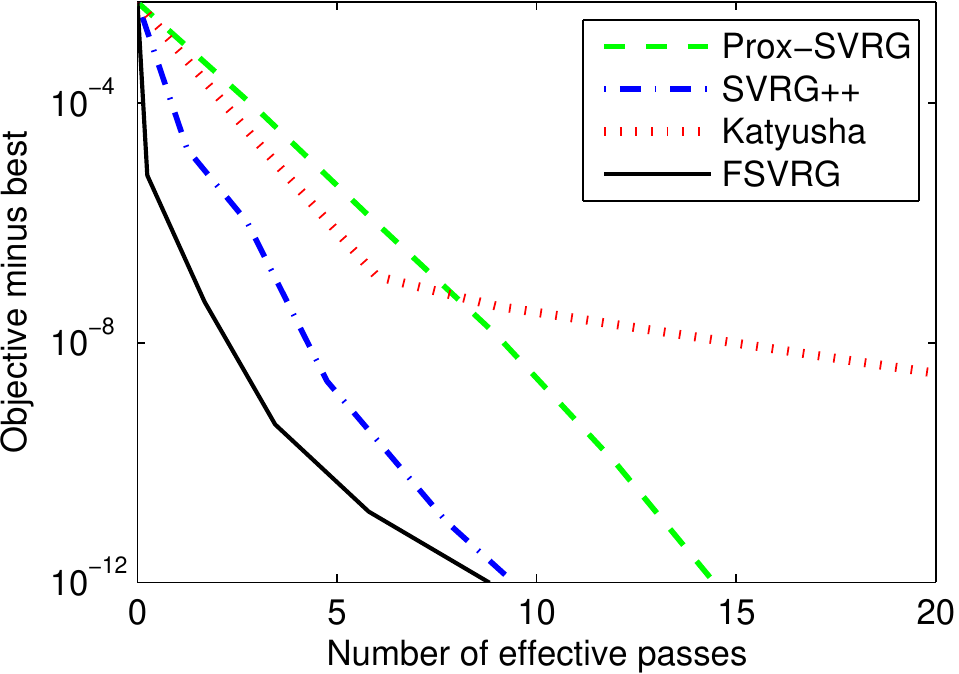}}
\subfigure[$\lambda_{1}\!=\!10^{-5}$ and $\lambda_{2}\!=\!10^{-5}$]{\includegraphics[width=0.243\columnwidth]{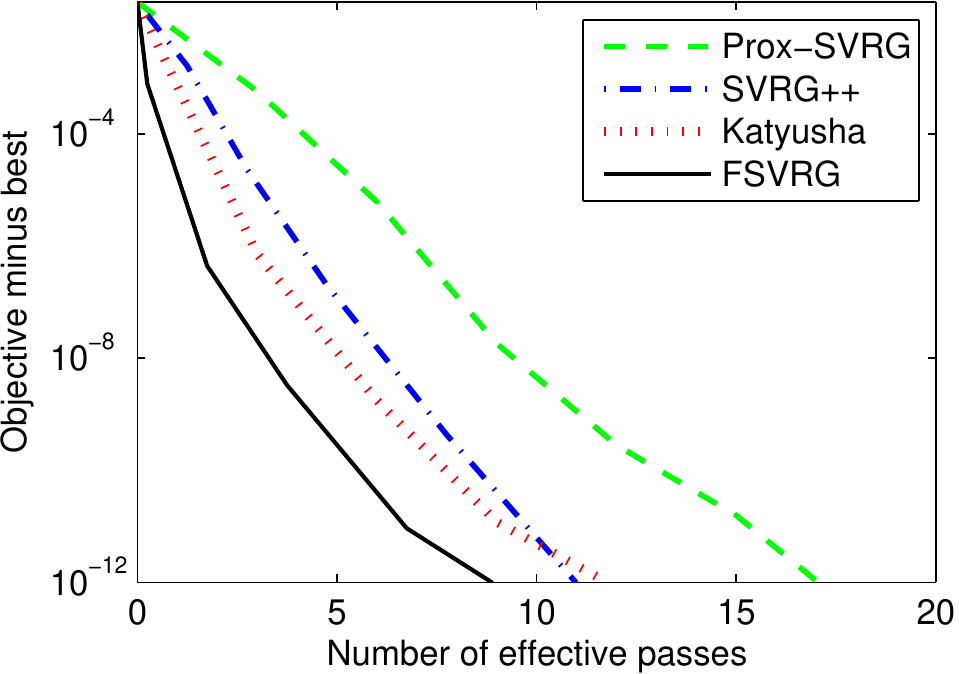}}
\subfigure[$\lambda_{1}\!=\!10^{-5}$ and $\lambda_{2}\!=\!10^{-5}$]{\includegraphics[width=0.243\columnwidth]{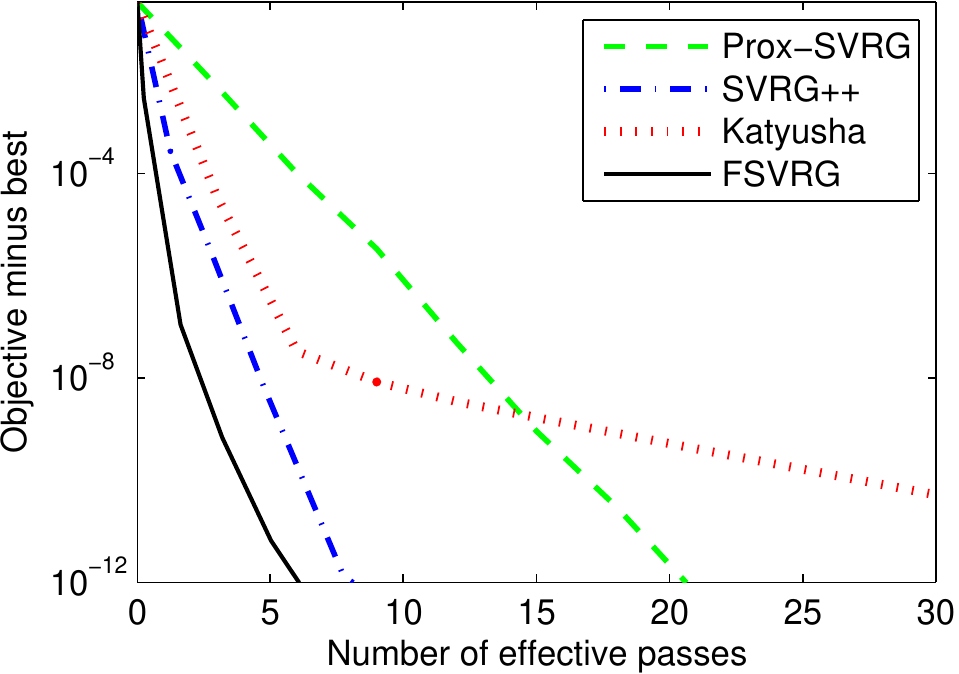}}

\subfigure[$\lambda_{1}\!=\!10^{-5}$ and $\lambda_{2}\!=\!10^{-5}$]{\includegraphics[width=0.243\columnwidth]{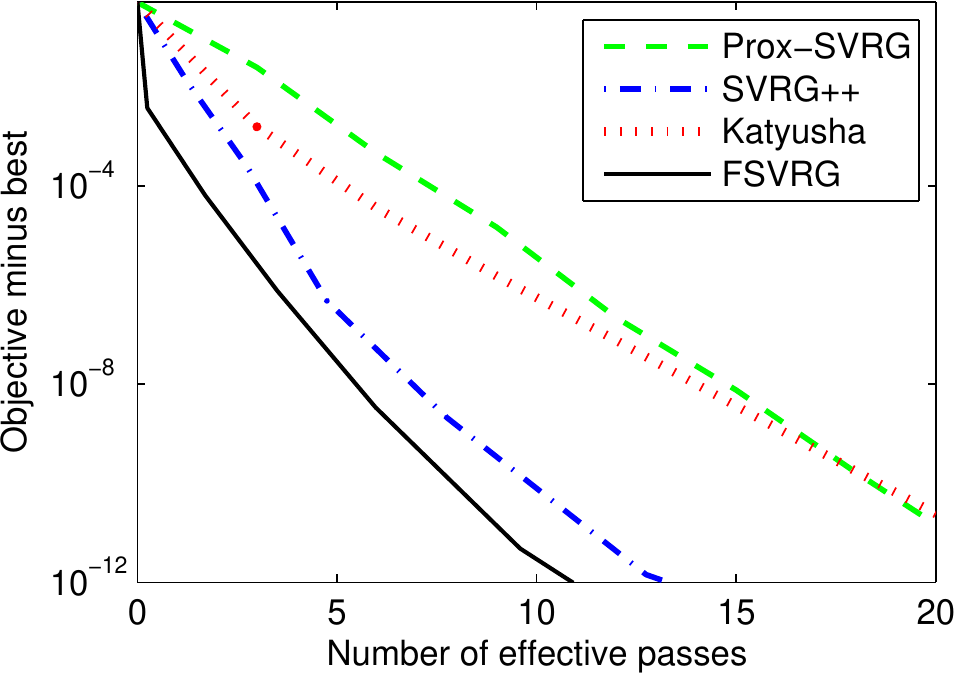}}
\subfigure[$\lambda_{1}\!=\!10^{-5}$ and $\lambda_{2}\!=\!10^{-5}$]{\includegraphics[width=0.243\columnwidth]{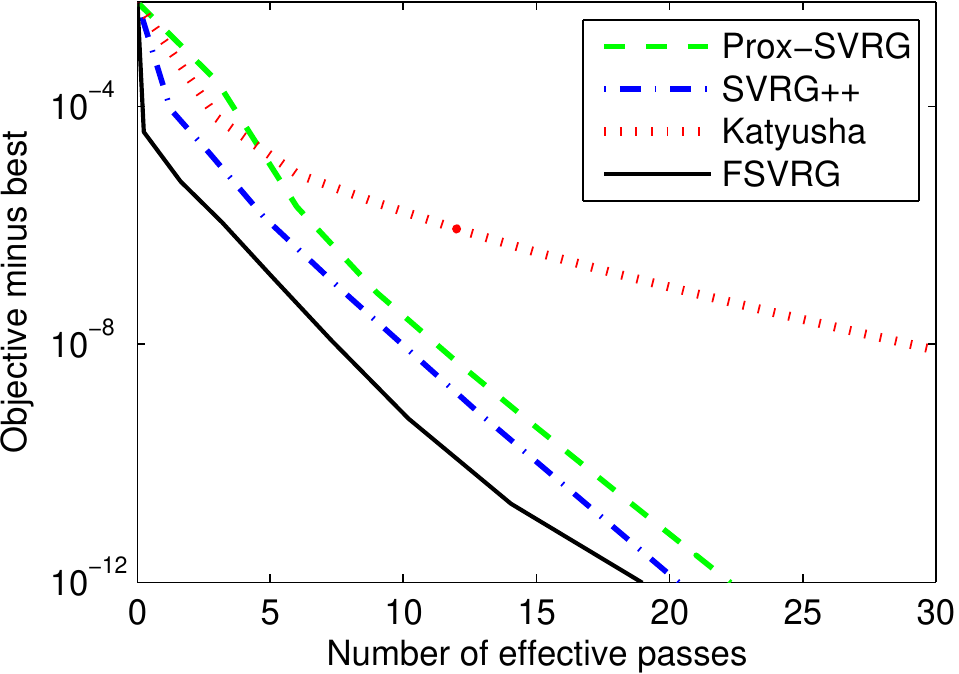}}
\subfigure[$\lambda_{1}\!=\!10^{-5}$ and $\lambda_{2}\!=\!10^{-6}$]{\includegraphics[width=0.243\columnwidth]{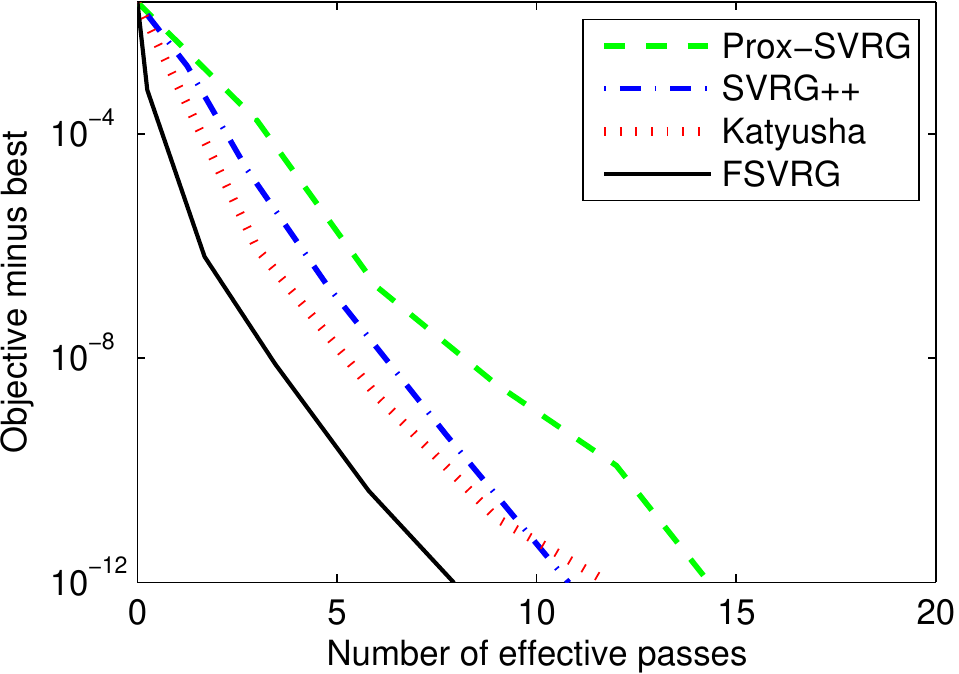}}
\subfigure[$\lambda_{1}\!=\!10^{-5}$ and $\lambda_{2}\!=\!10^{-6}$]{\includegraphics[width=0.243\columnwidth]{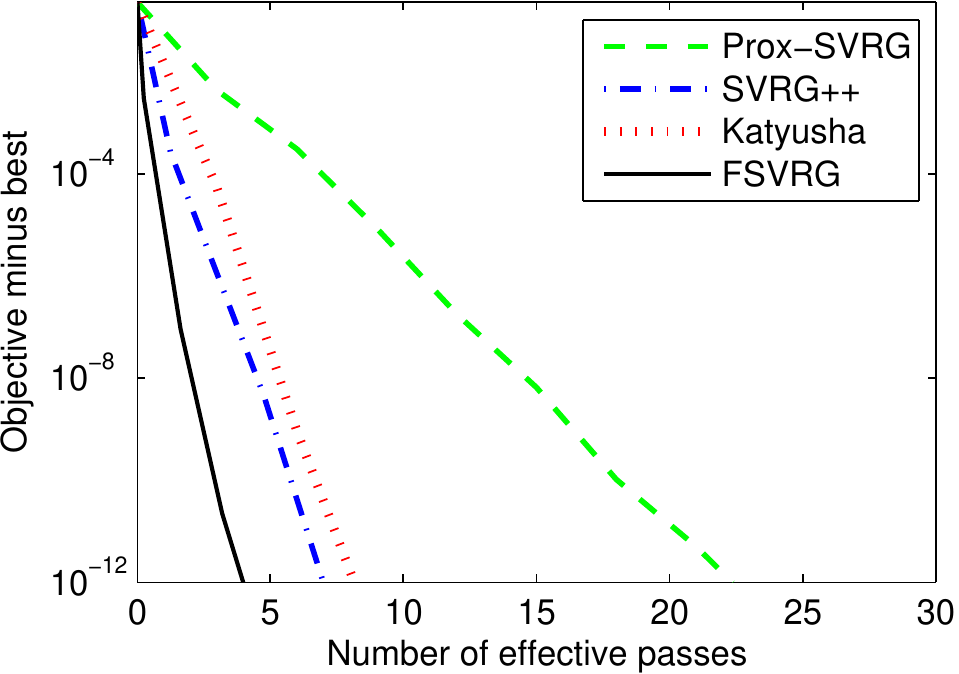}}

\subfigure[$\lambda_{1}\!=\!10^{-5}$ and $\lambda_{2}\!=\!10^{-6}$]{\includegraphics[width=0.243\columnwidth]{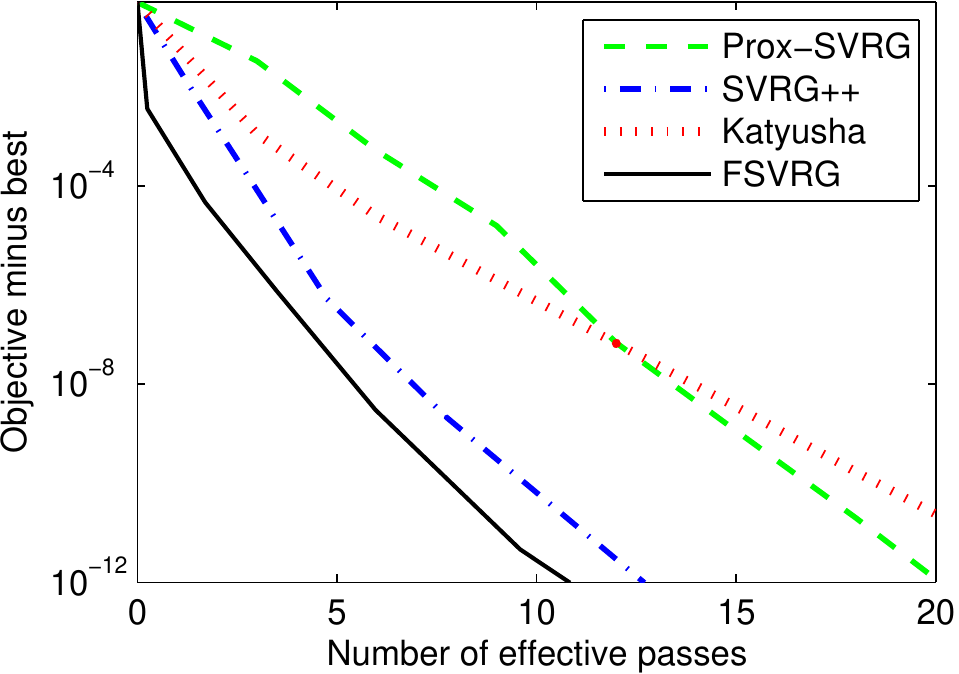}}
\subfigure[$\lambda_{1}\!=\!10^{-5}$ and $\lambda_{2}\!=\!10^{-6}$]{\includegraphics[width=0.243\columnwidth]{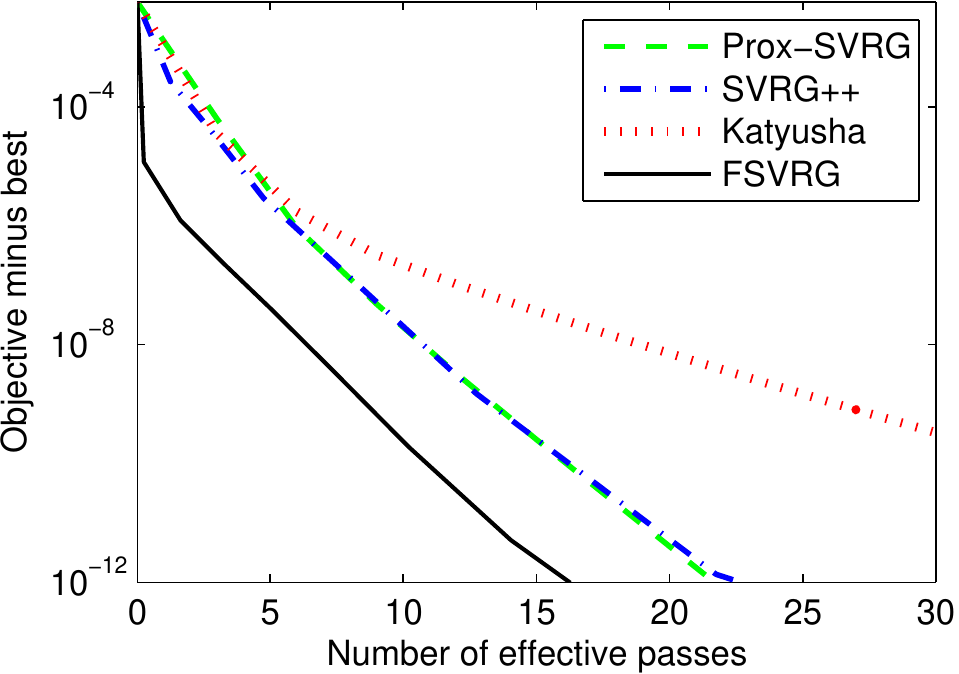}}
\subfigure[$\lambda_{1}\!=\!10^{-6}$ and $\lambda_{2}\!=\!10^{-6}$]{\includegraphics[width=0.243\columnwidth]{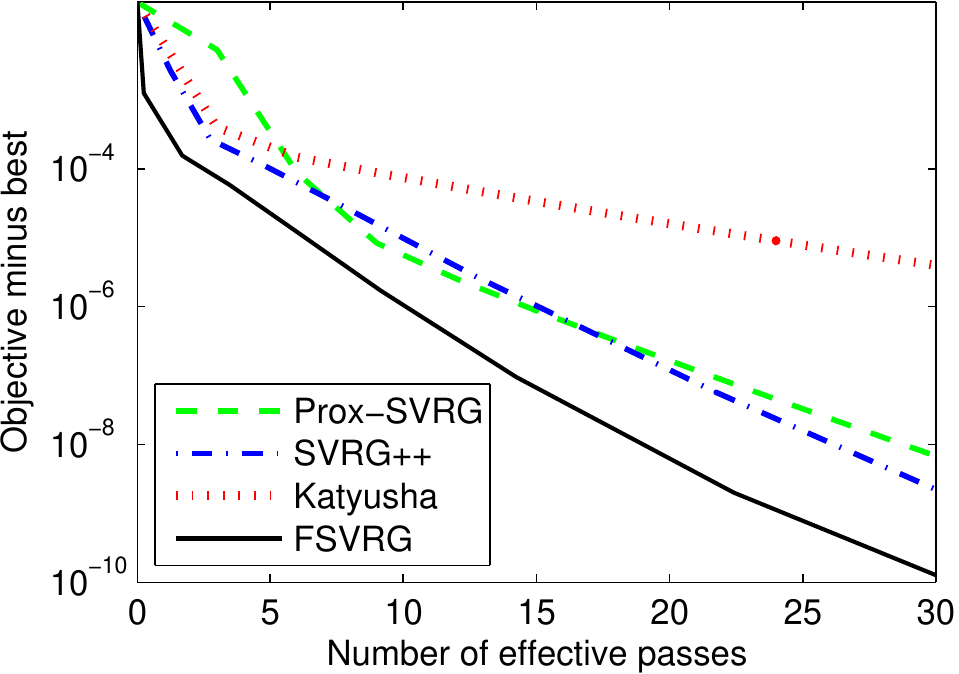}}
\subfigure[$\lambda_{1}\!=\!10^{-6}$ and $\lambda_{2}\!=\!10^{-6}$]{\includegraphics[width=0.243\columnwidth]{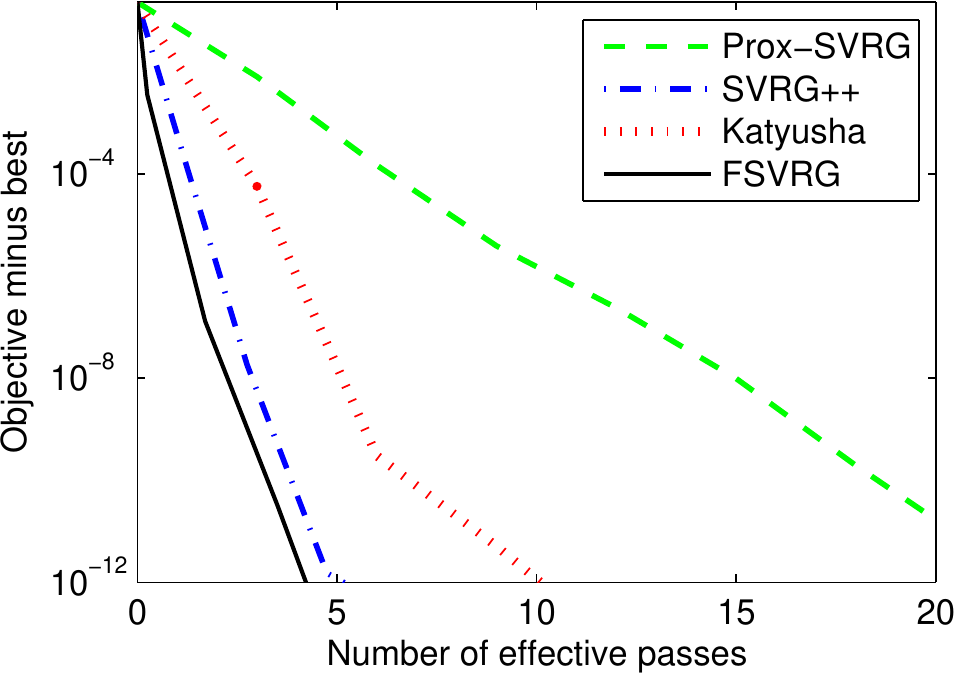}}
\caption{Comparison of Prox-SVRG~\cite{xiao:prox-svrg}, SVRG++~\cite{zhu:vrnc}, Katyusha~\cite{zhu:Katyusha}, and FSVRG for solving elastic net (i.e., $\lambda_{1}\|x\|^{2}\!+\!\lambda_{2}\|x\|_{1}$) regularized logistic regression problems on the four data sets: IJCNN (the first column), Protein (the second column), Covtype (the third column), and SUSY (the fourth column). Note that the $y$-axis represents the objective value minus the minimum, and the $x$-axis corresponds to the number of effective passes.}
\label{figs9}
\end{figure}

\begin{figure}[!th]
\centering
\subfigure[$\lambda_{1}\!=\!10^{-4}$ and $\lambda_{2}\!=\!10^{-4}$]{\includegraphics[width=0.243\columnwidth]{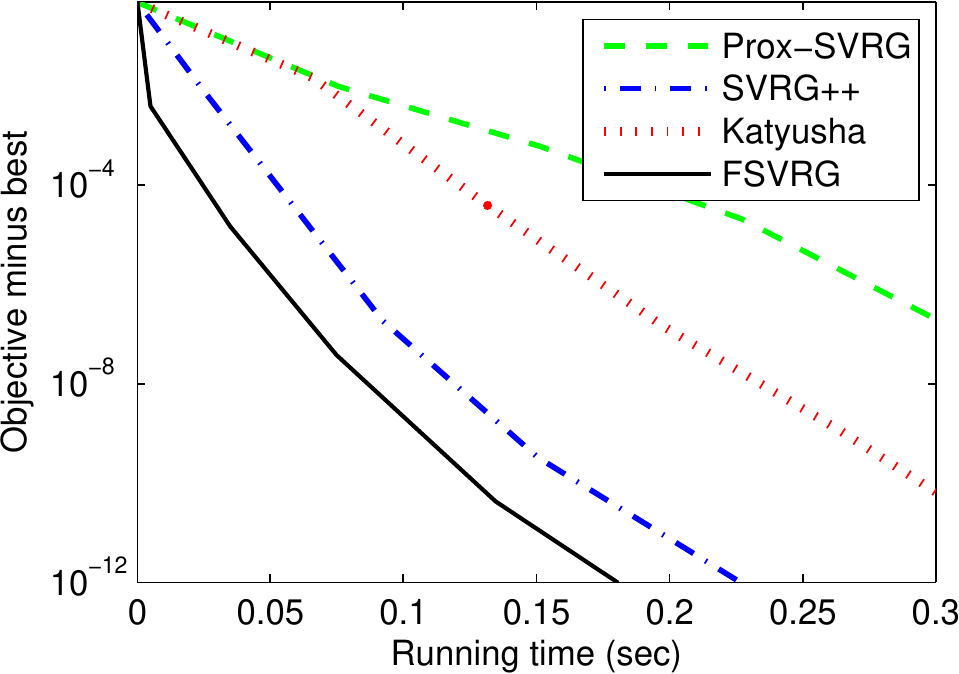}}
\subfigure[$\lambda_{1}\!=\!10^{-4}$ and $\lambda_{2}\!=\!10^{-4}$]{\includegraphics[width=0.243\columnwidth]{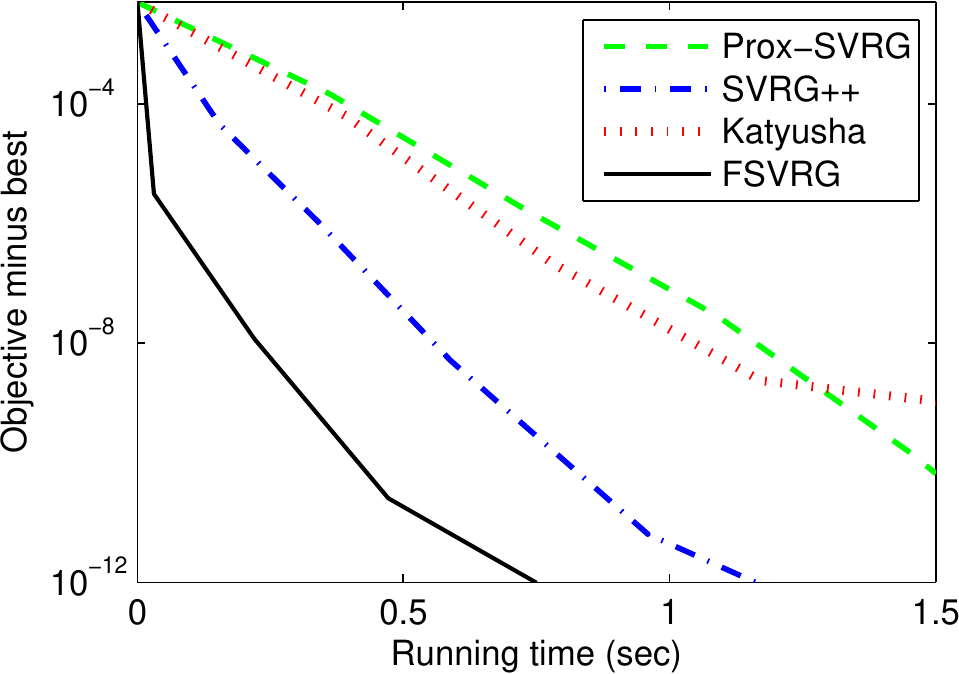}}
\subfigure[$\lambda_{1}\!=\!10^{-4}$ and $\lambda_{2}\!=\!10^{-5}$]{\includegraphics[width=0.243\columnwidth]{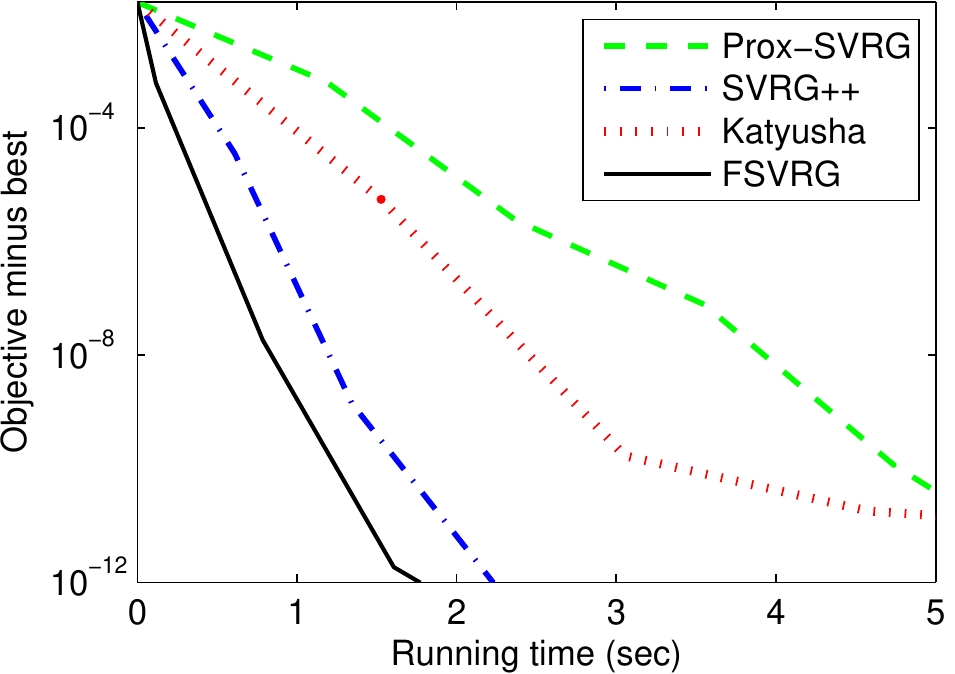}}
\subfigure[$\lambda_{1}\!=\!10^{-4}$ and $\lambda_{2}\!=\!10^{-5}$]{\includegraphics[width=0.243\columnwidth]{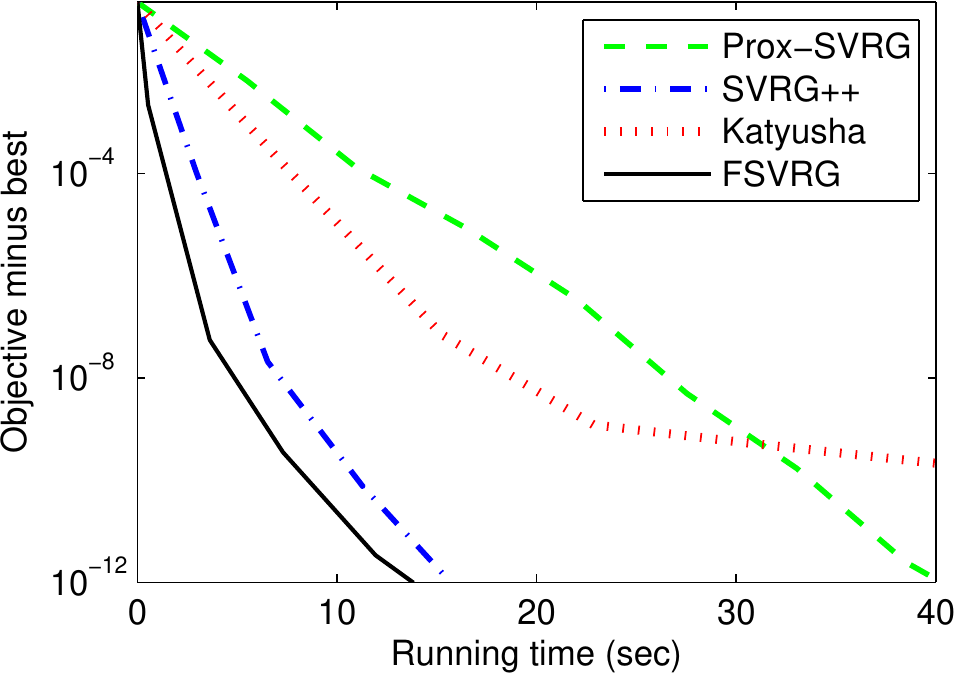}}

\subfigure[$\lambda_{1}\!=\!10^{-4}$ and $\lambda_{2}\!=\!10^{-5}$]{\includegraphics[width=0.243\columnwidth]{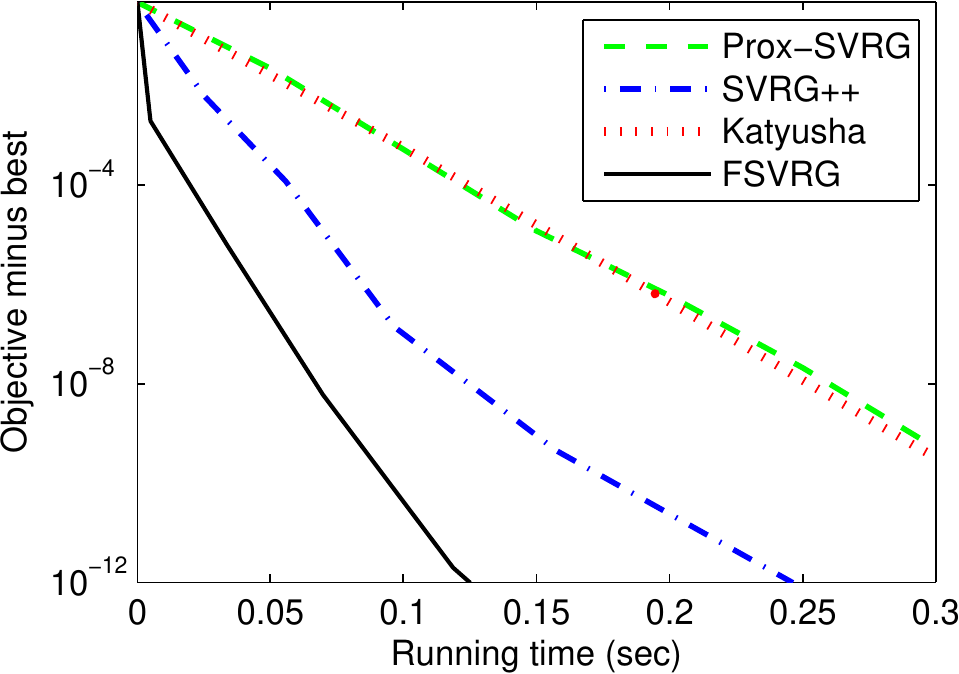}}
\subfigure[$\lambda_{1}\!=\!10^{-4}$ and $\lambda_{2}\!=\!10^{-5}$]{\includegraphics[width=0.243\columnwidth]{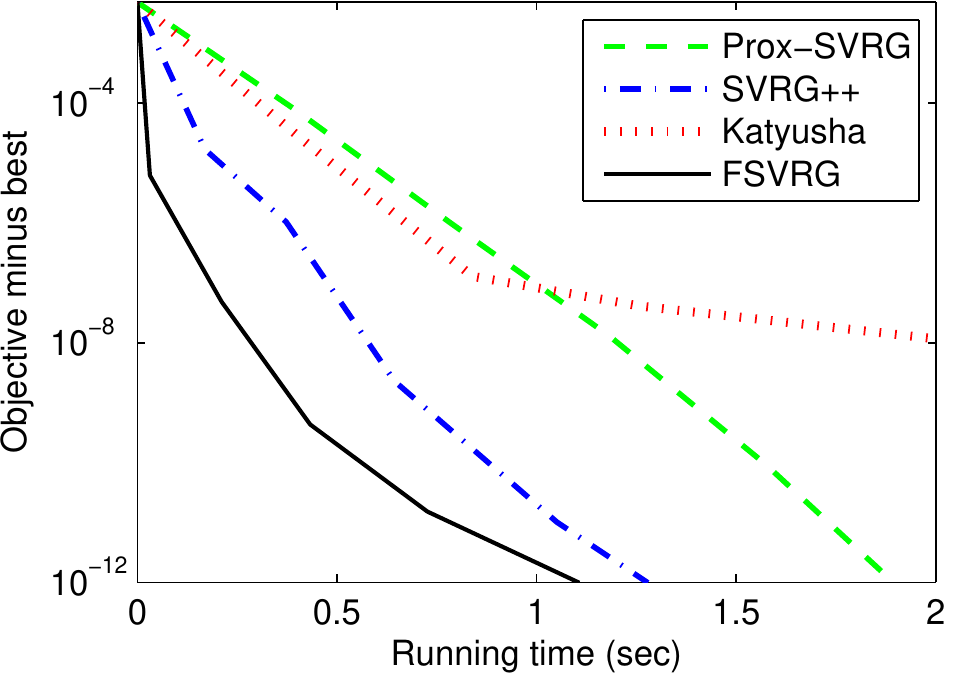}}
\subfigure[$\lambda_{1}\!=\!10^{-5}$ and $\lambda_{2}\!=\!10^{-5}$]{\includegraphics[width=0.243\columnwidth]{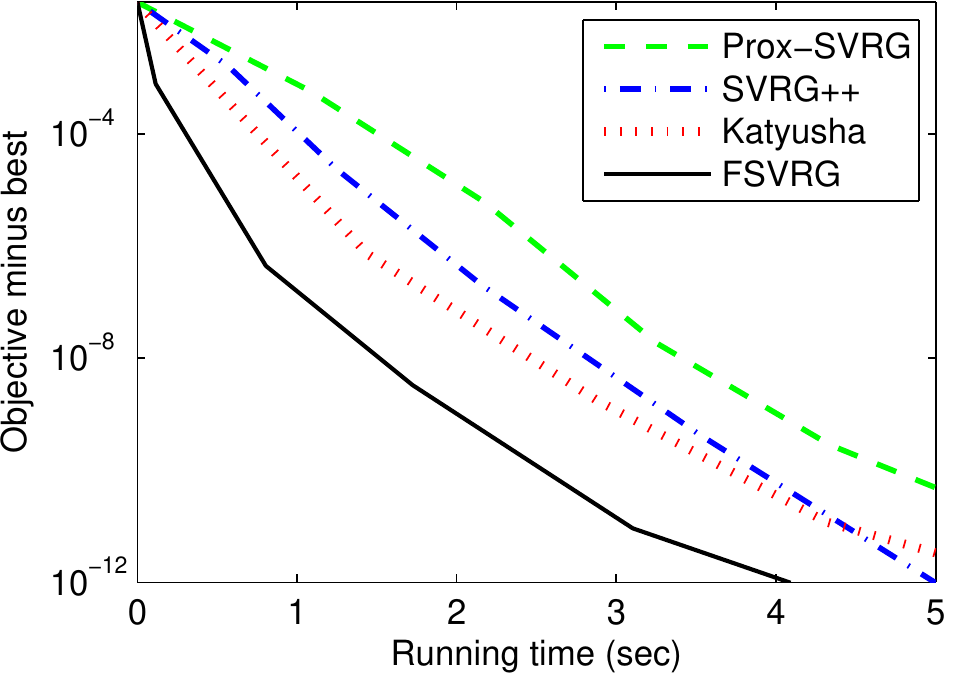}}
\subfigure[$\lambda_{1}\!=\!10^{-5}$ and $\lambda_{2}\!=\!10^{-5}$]{\includegraphics[width=0.243\columnwidth]{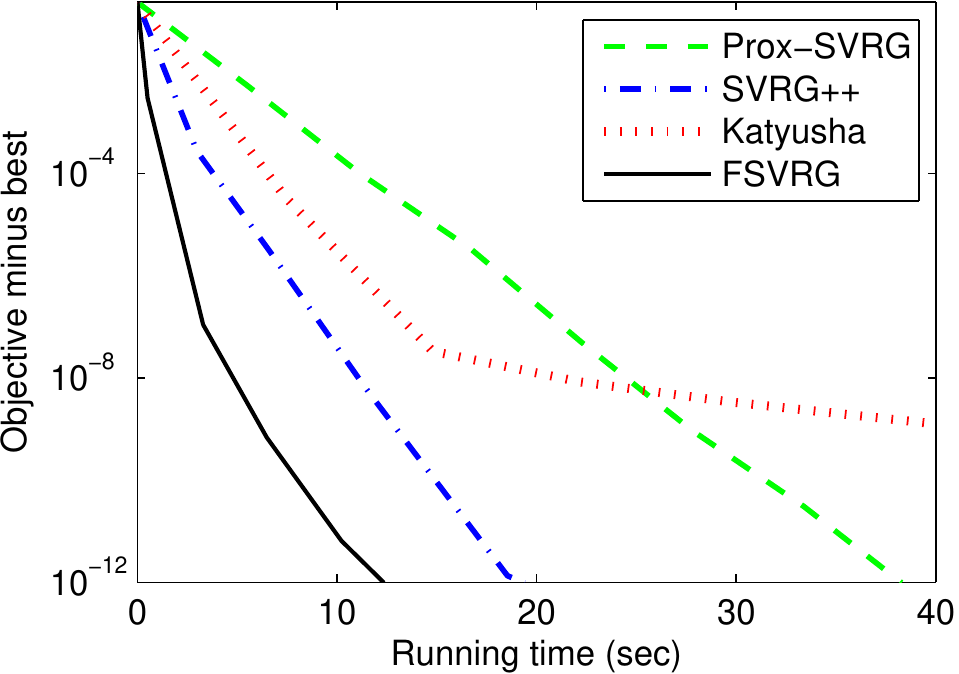}}

\subfigure[$\lambda_{1}\!=\!10^{-5}$ and $\lambda_{2}\!=\!10^{-5}$]{\includegraphics[width=0.243\columnwidth]{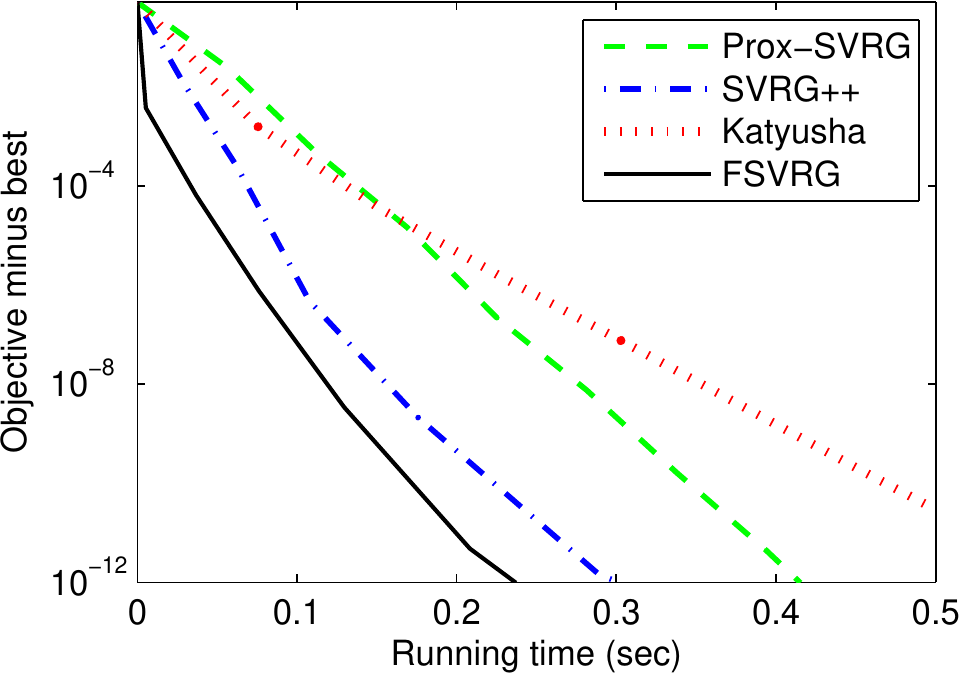}}
\subfigure[$\lambda_{1}\!=\!10^{-5}$ and $\lambda_{2}\!=\!10^{-5}$]{\includegraphics[width=0.243\columnwidth]{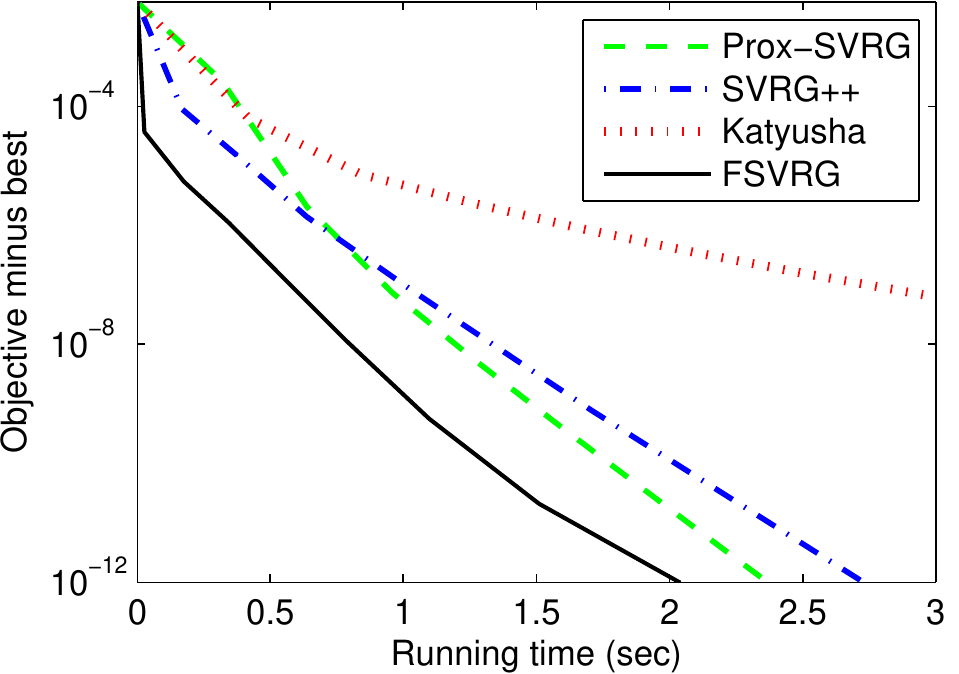}}
\subfigure[$\lambda_{1}\!=\!10^{-5}$ and $\lambda_{2}\!=\!10^{-6}$]{\includegraphics[width=0.243\columnwidth]{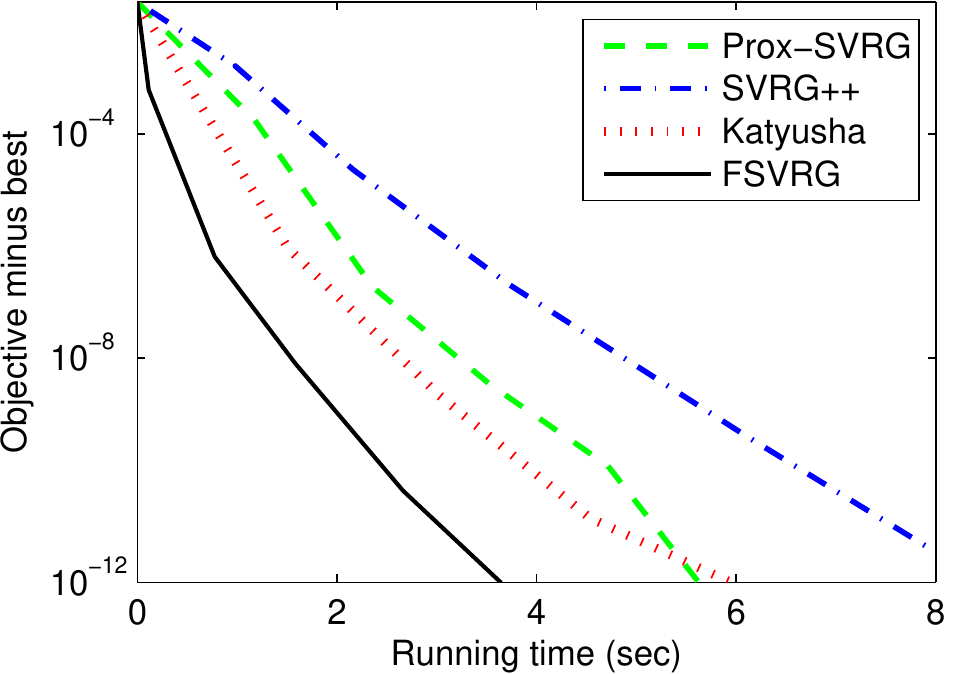}}
\subfigure[$\lambda_{1}\!=\!10^{-5}$ and $\lambda_{2}\!=\!10^{-6}$]{\includegraphics[width=0.243\columnwidth]{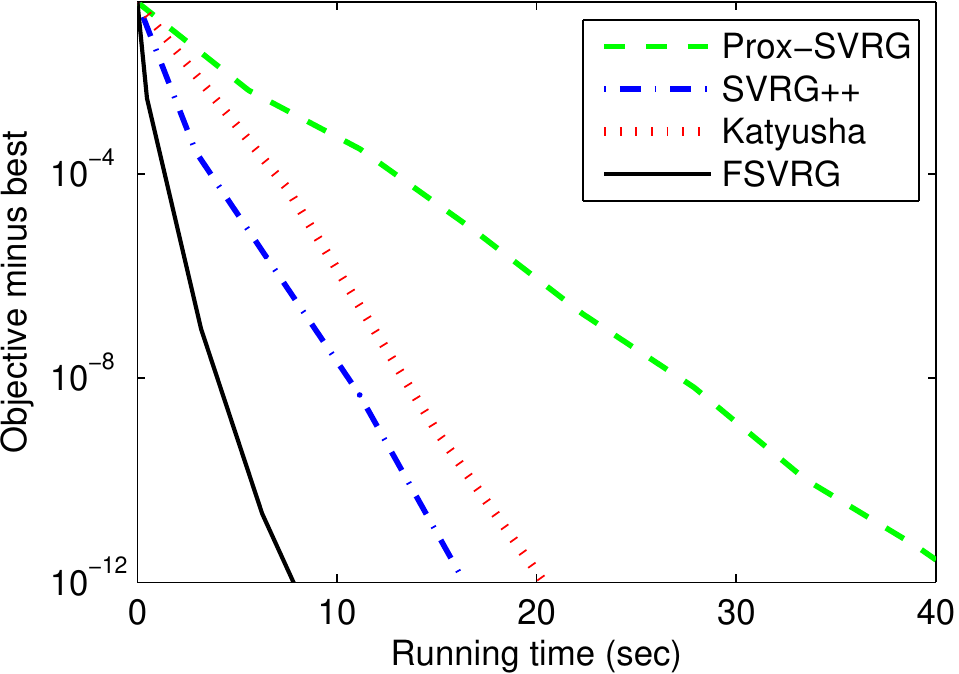}}

\subfigure[$\lambda_{1}\!=\!10^{-5}$ and $\lambda_{2}\!=\!10^{-6}$]{\includegraphics[width=0.243\columnwidth]{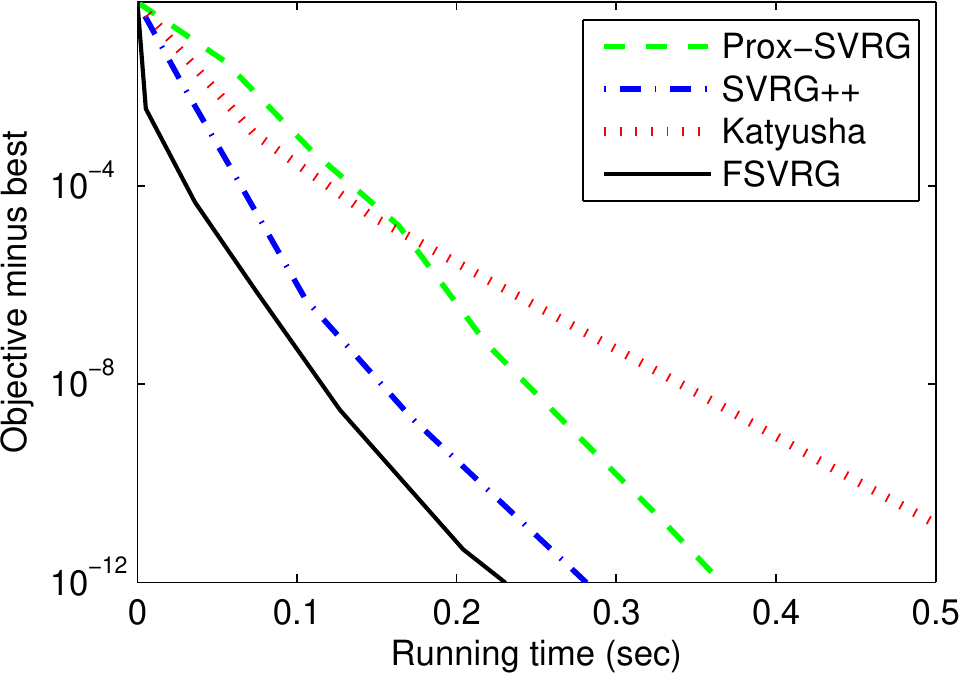}}
\subfigure[$\lambda_{1}\!=\!10^{-5}$ and $\lambda_{2}\!=\!10^{-6}$]{\includegraphics[width=0.243\columnwidth]{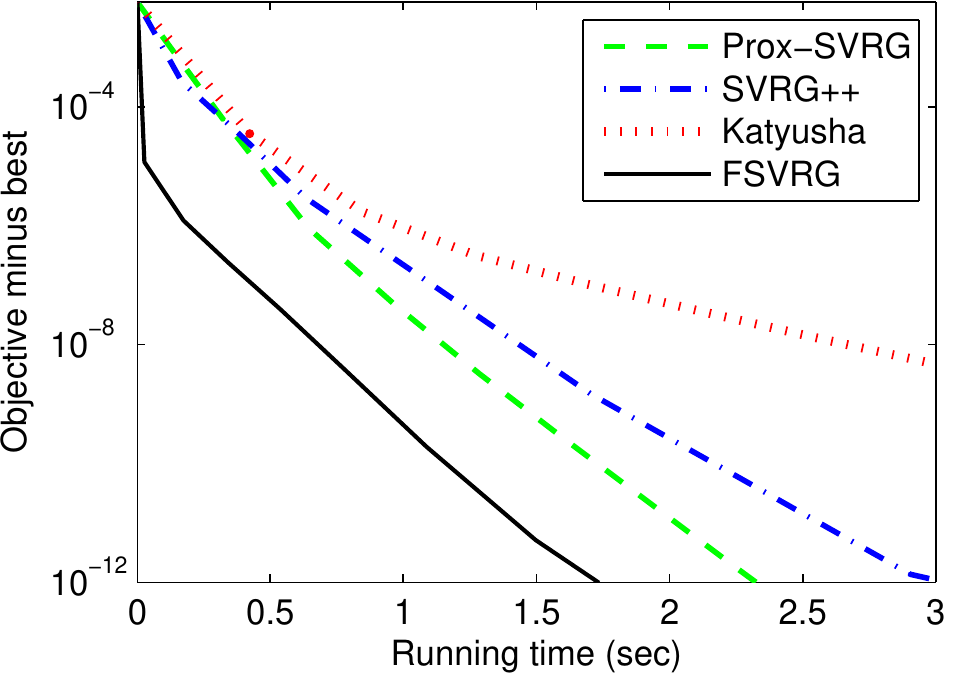}}
\subfigure[$\lambda_{1}\!=\!10^{-6}$ and $\lambda_{2}\!=\!10^{-6}$]{\includegraphics[width=0.243\columnwidth]{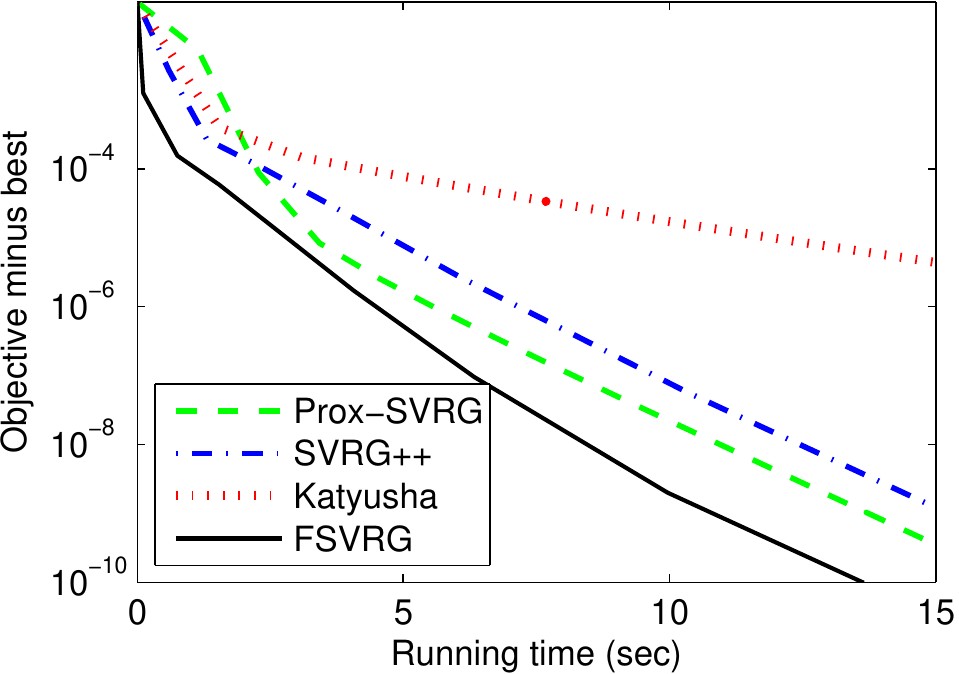}}
\subfigure[$\lambda_{1}\!=\!10^{-6}$ and $\lambda_{2}\!=\!10^{-6}$]{\includegraphics[width=0.243\columnwidth]{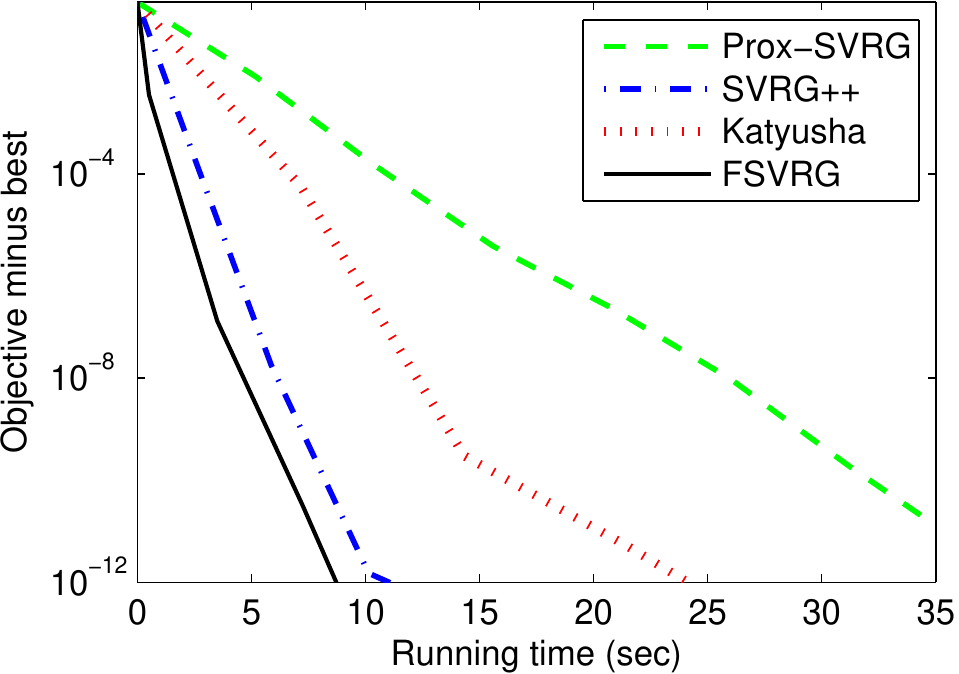}}
\caption{Comparison of Prox-SVRG~\cite{xiao:prox-svrg}, SVRG++~\cite{zhu:vrnc}, Katyusha~\cite{zhu:Katyusha}, and FSVRG for solving elastic net (i.e., $\lambda_{1}\|x\|^{2}\!+\!\lambda_{2}\|x\|_{1}$) regularized logistic regression problems on the four data sets: IJCNN (the first column), Protein (the second column), Covtype (the third column), and SUSY (the fourth column). Note that the $y$-axis represents the objective value minus the minimum, and the $x$-axis corresponds to the running time (seconds).}
\label{figs10}
\end{figure}

\begin{figure}[!th]
\centering
\subfigure[IJCNN: $\lambda\!=\!10^{-3}$]{\includegraphics[width=0.243\columnwidth]{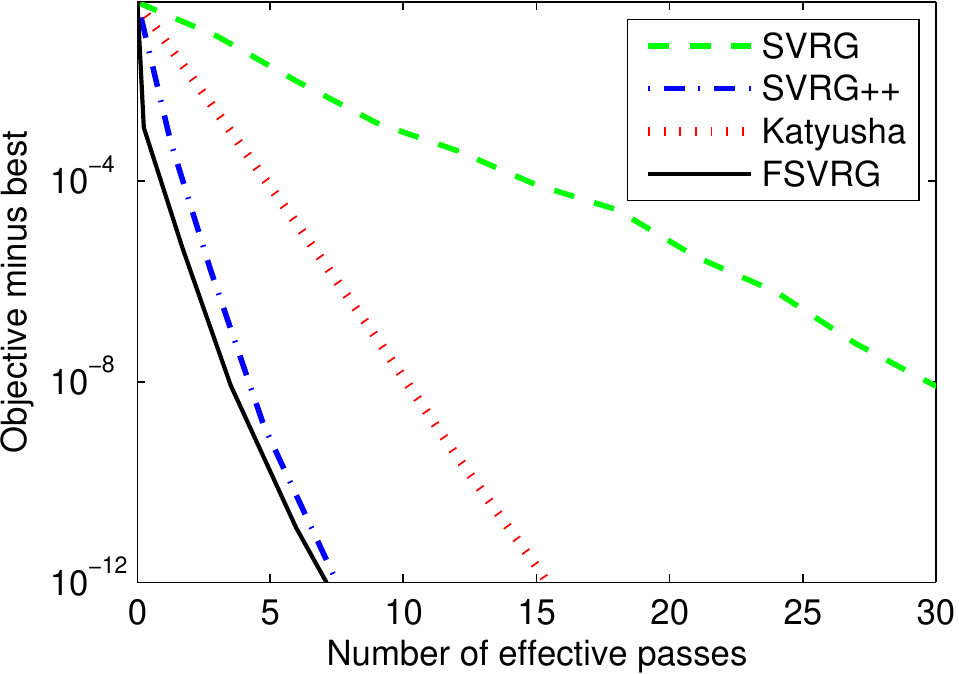}}
\subfigure[Protein: $\lambda\!=\!10^{-3}$]{\includegraphics[width=0.243\columnwidth]{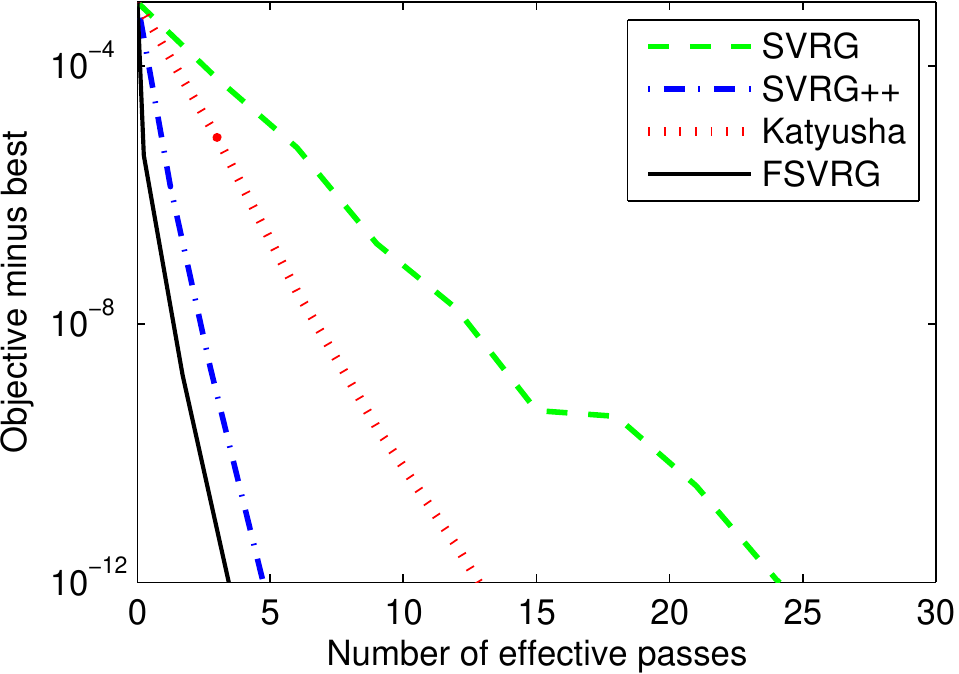}}
\subfigure[Covtype: $\lambda\!=\!10^{-4}$]{\includegraphics[width=0.243\columnwidth]{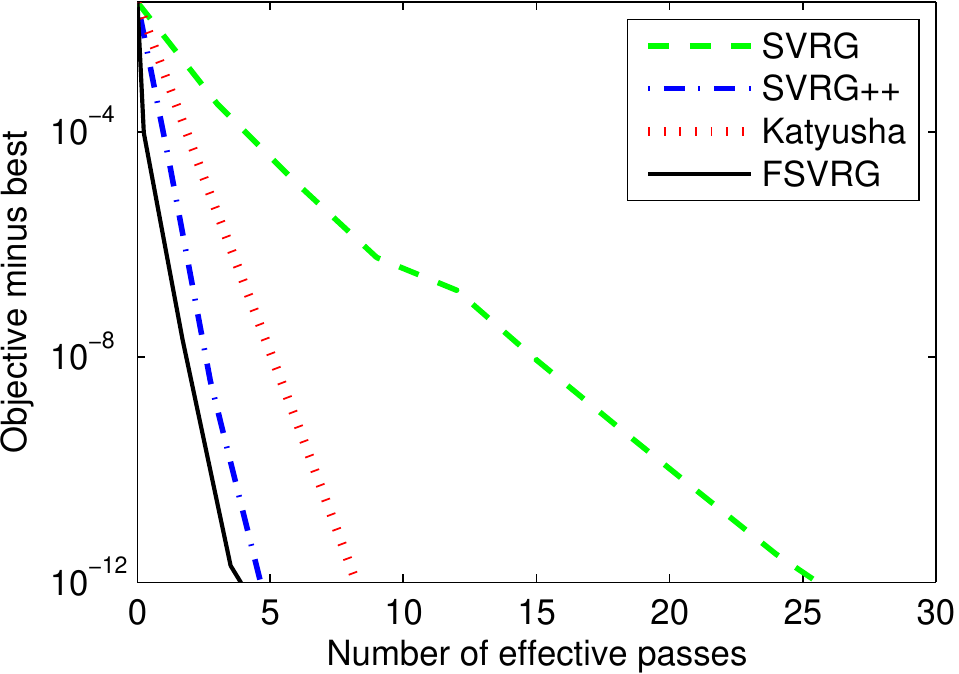}}
\subfigure[SUSY: $\lambda\!=\!10^{-4}$]{\includegraphics[width=0.243\columnwidth]{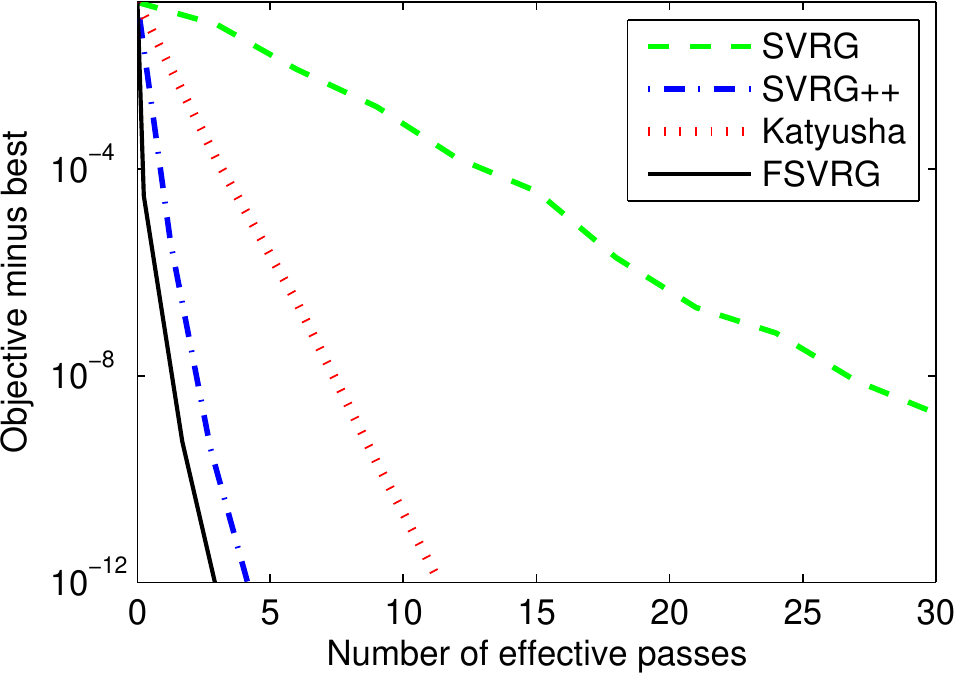}}

\subfigure[IJCNN: $\lambda\!=\!10^{-4}$]{\includegraphics[width=0.243\columnwidth]{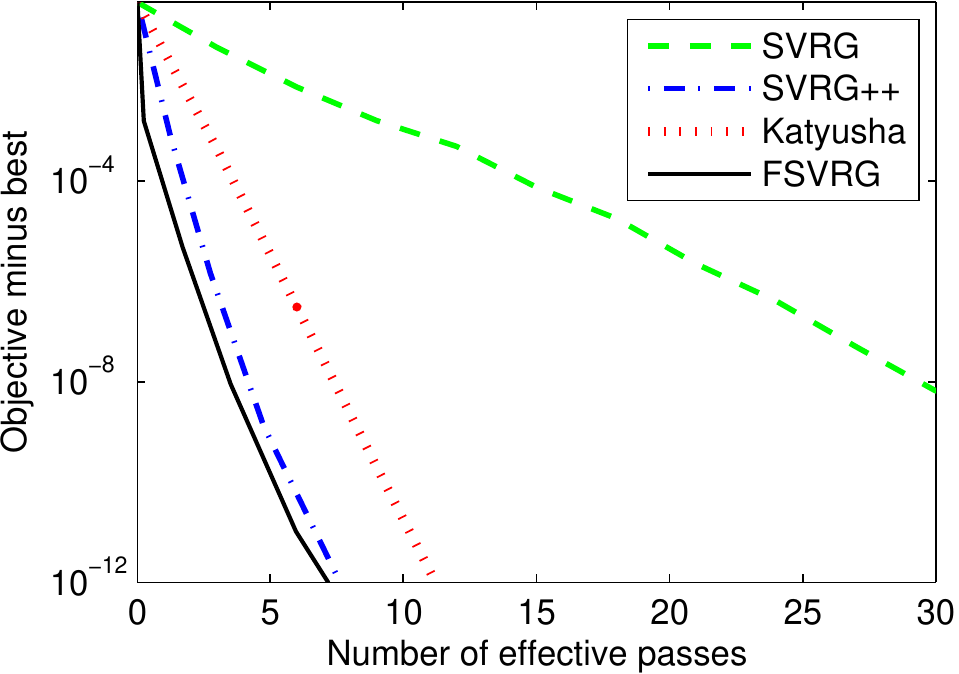}}
\subfigure[Protein: $\lambda\!=\!10^{-4}$]{\includegraphics[width=0.243\columnwidth]{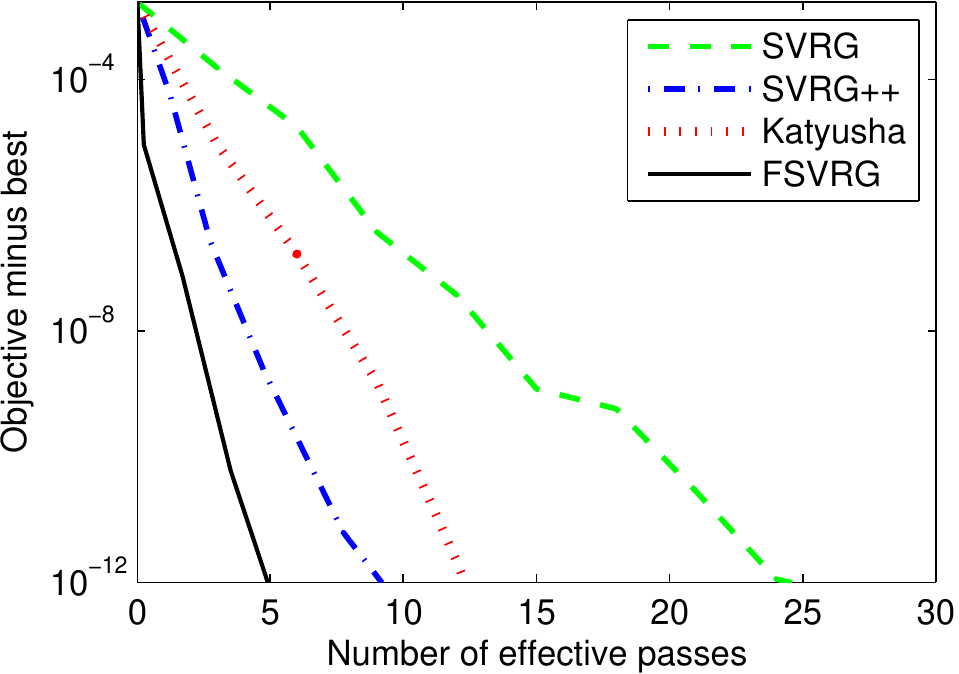}}
\subfigure[Covtype: $\lambda\!=\!10^{-5}$]{\includegraphics[width=0.243\columnwidth]{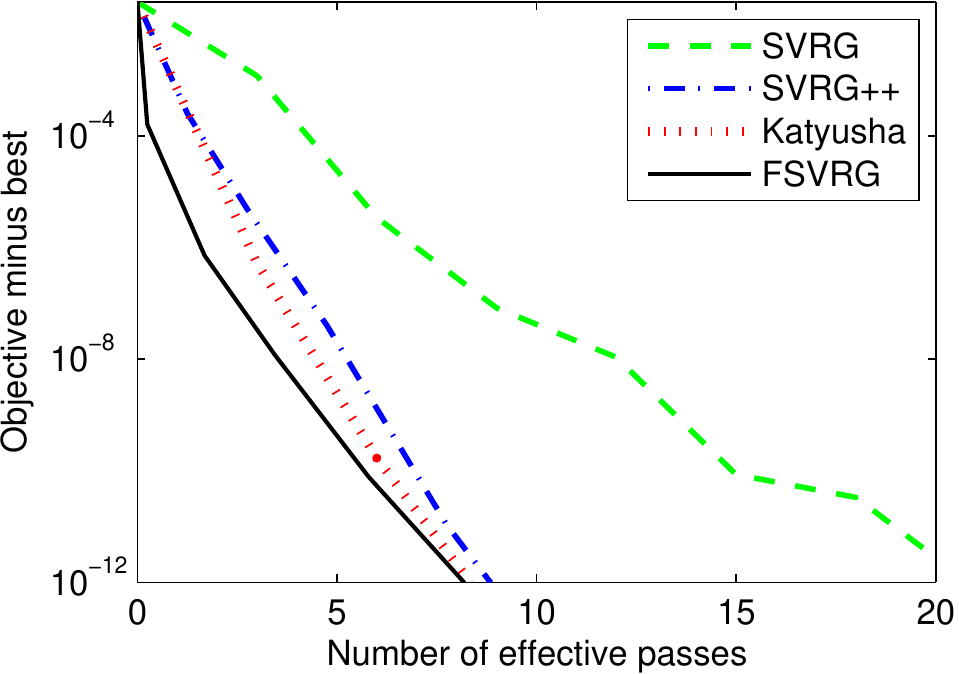}}
\subfigure[SUSY: $\lambda\!=\!10^{-5}$]{\includegraphics[width=0.243\columnwidth]{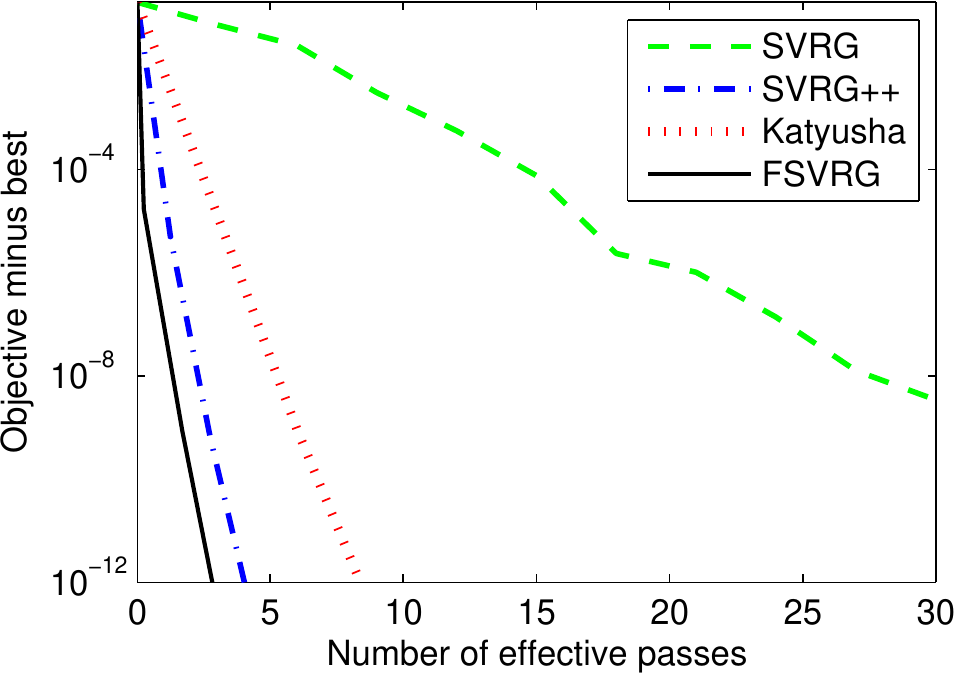}}

\subfigure[IJCNN: $\lambda\!=\!10^{-5}$]{\includegraphics[width=0.243\columnwidth]{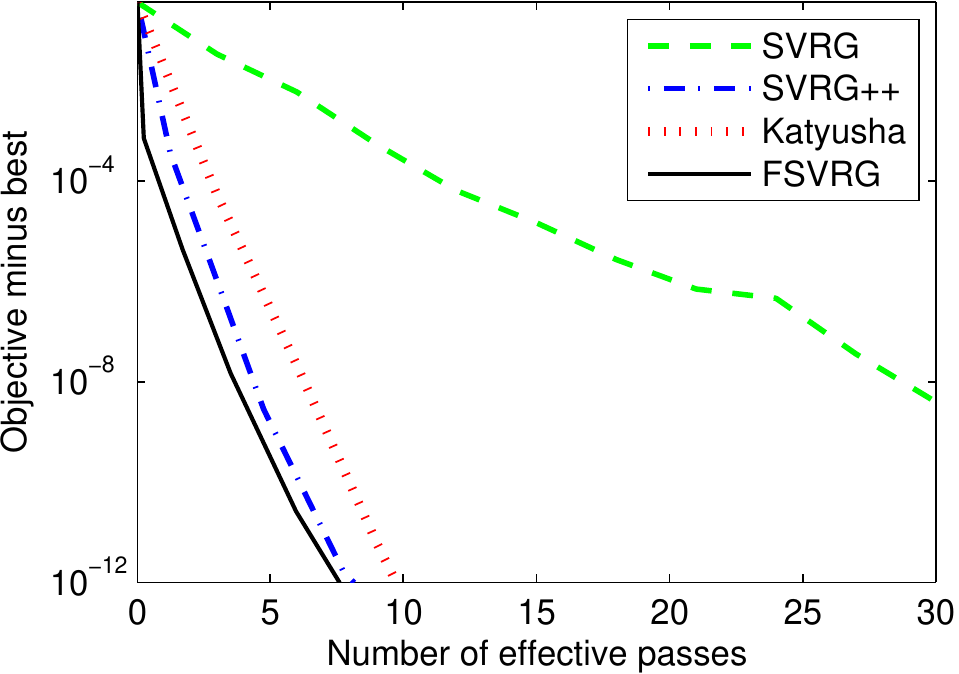}}
\subfigure[Protein: $\lambda\!=\!10^{-5}$]{\includegraphics[width=0.243\columnwidth]{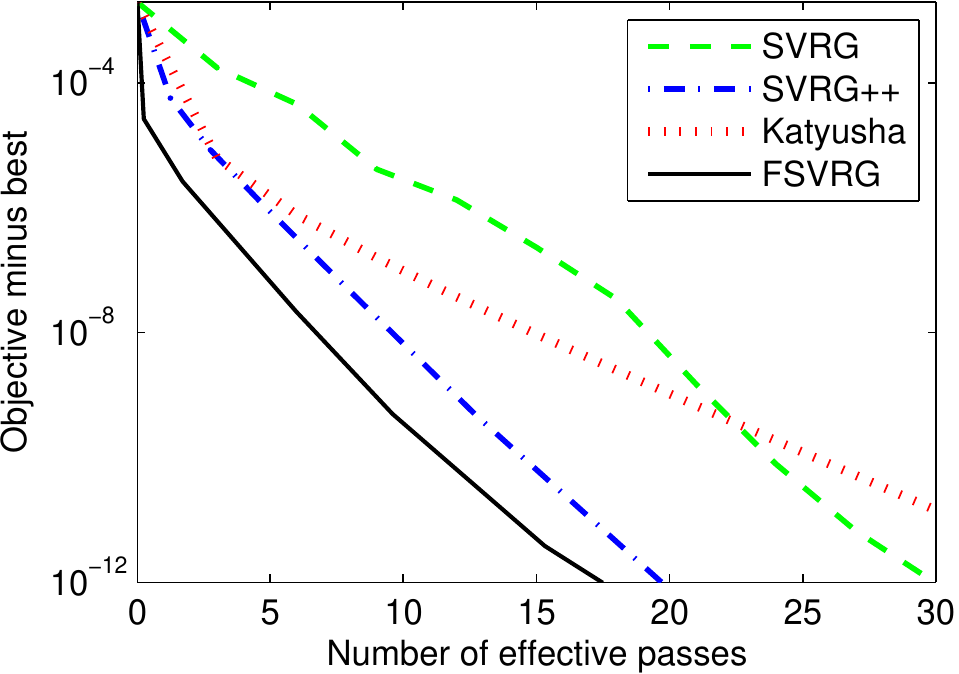}}
\subfigure[Covtype: $\lambda\!=\!10^{-6}$]{\includegraphics[width=0.243\columnwidth]{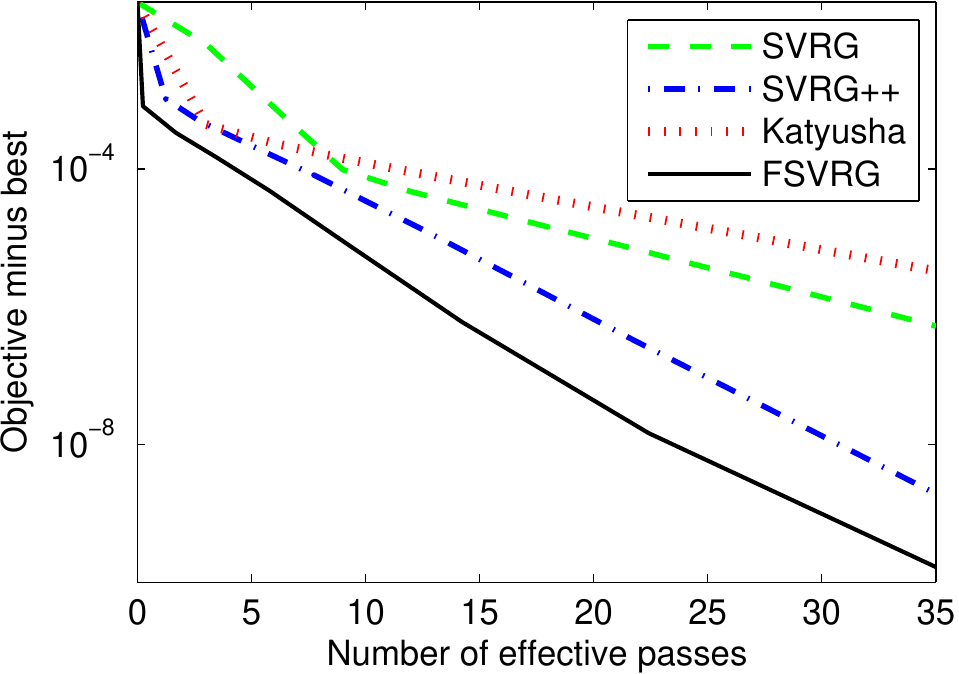}}
\subfigure[SUSY: $\lambda\!=\!10^{-6}$]{\includegraphics[width=0.243\columnwidth]{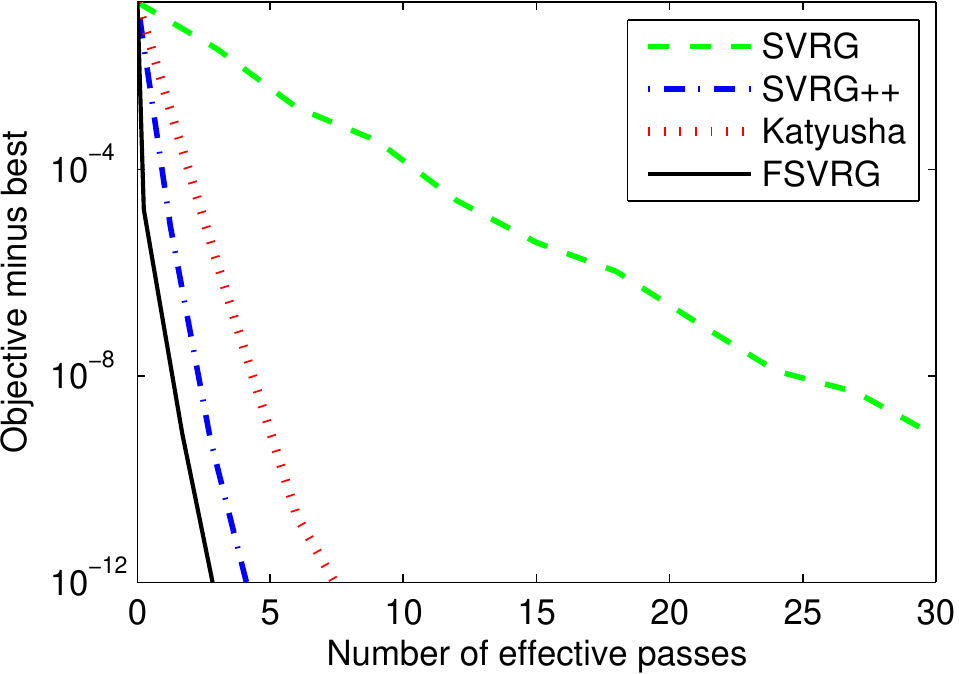}}

\subfigure[IJCNN: $\lambda\!=\!10^{-6}$]{\includegraphics[width=0.243\columnwidth]{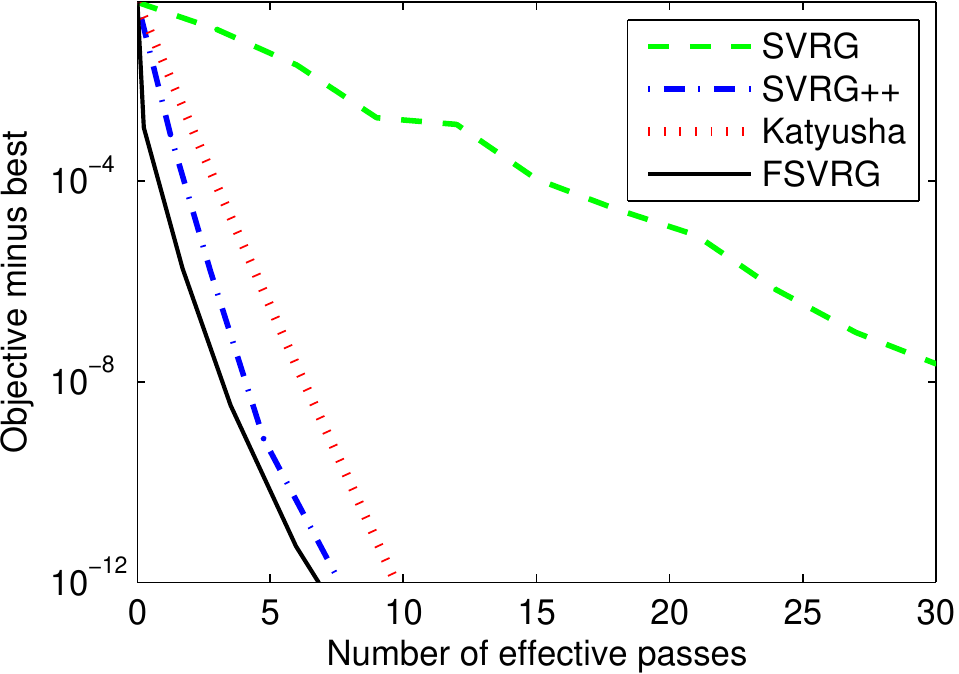}}
\subfigure[Protein: $\lambda\!=\!10^{-6}$]{\includegraphics[width=0.243\columnwidth]{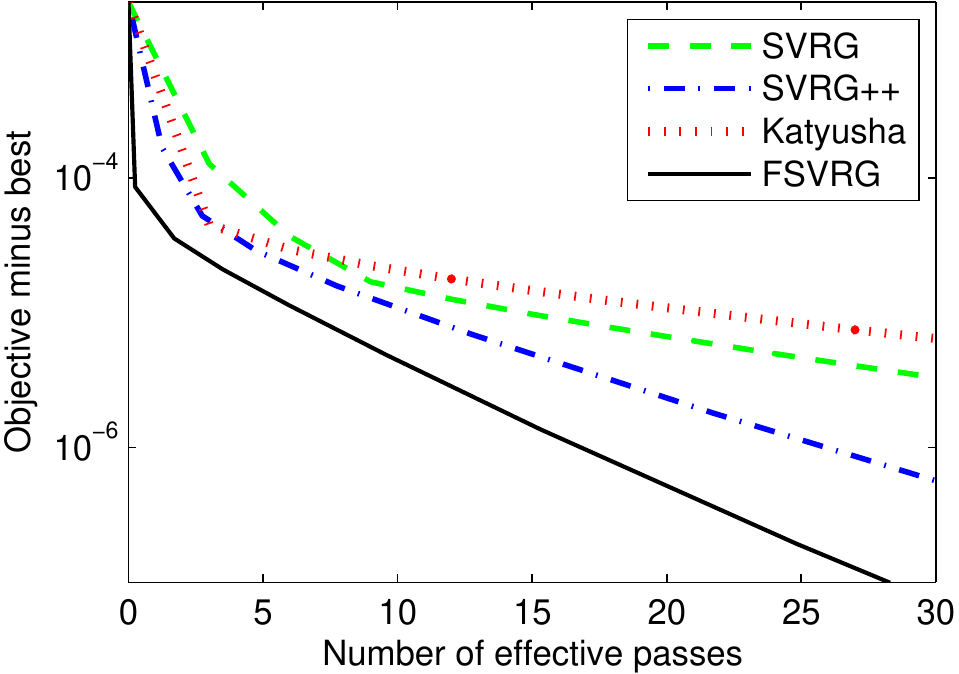}}
\subfigure[Covtype: $\lambda\!=\!10^{-7}$]{\includegraphics[width=0.243\columnwidth]{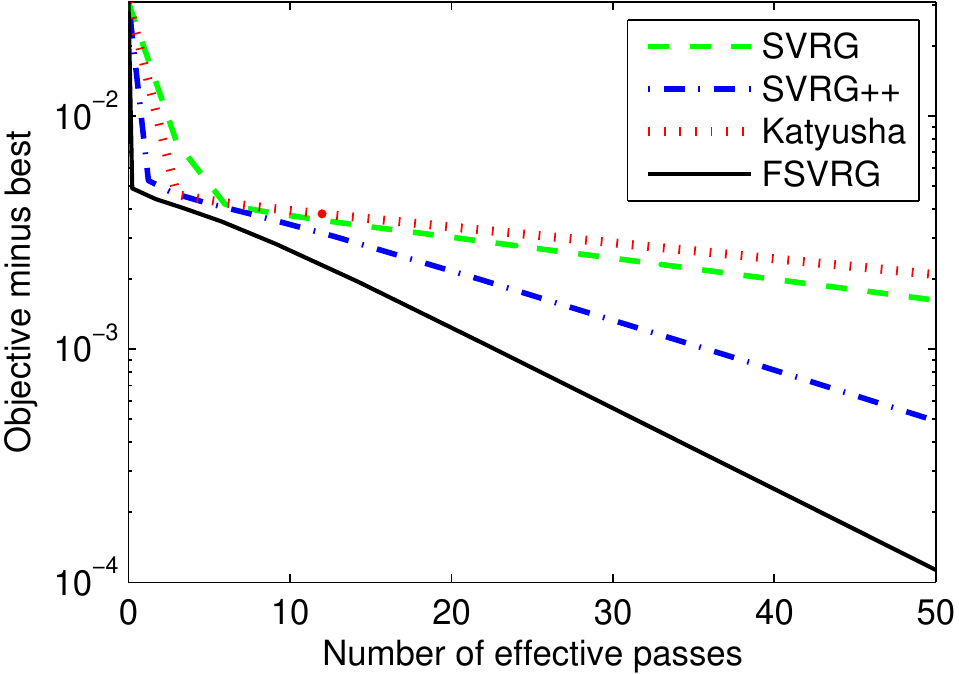}}
\subfigure[SUSY: $\lambda\!=\!10^{-7}$]{\includegraphics[width=0.243\columnwidth]{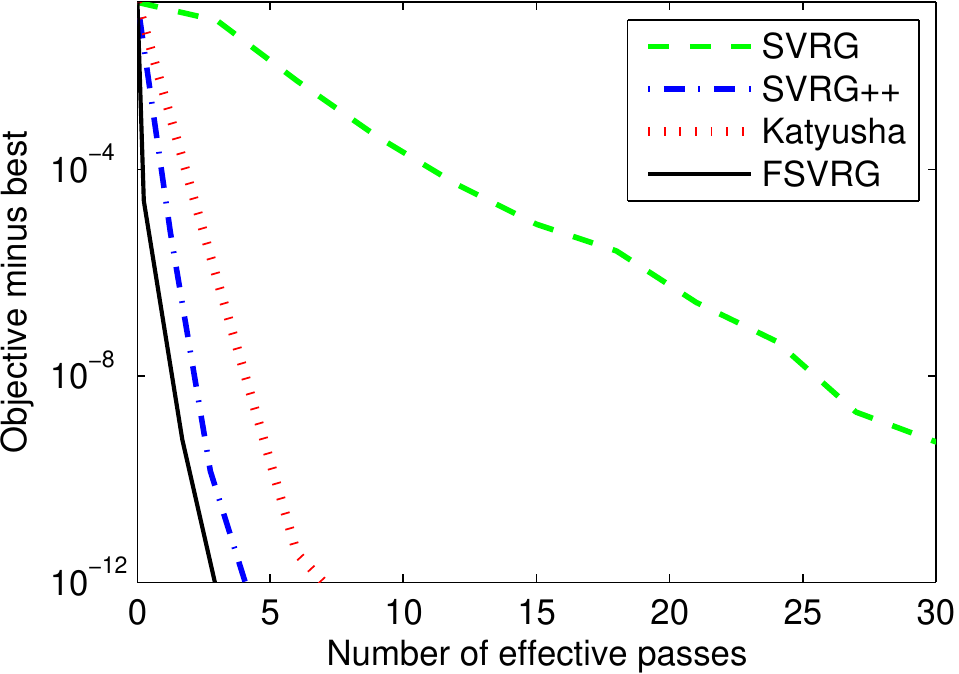}}
\caption{Comparison of SVRG~\cite{johnson:svrg}, SVRG++~\cite{zhu:vrnc}, Katyusha~\cite{zhu:Katyusha}, and FSVRG for solving ridge regression problems with different regularization parameters. The $y$-axis represents the objective value minus the minimum, and the $x$-axis corresponds to the number of effective passes.}
\label{figs11}
\end{figure}

\begin{figure}[!th]
\centering
\subfigure[IJCNN: $\lambda\!=\!10^{-3}$]{\includegraphics[width=0.243\columnwidth]{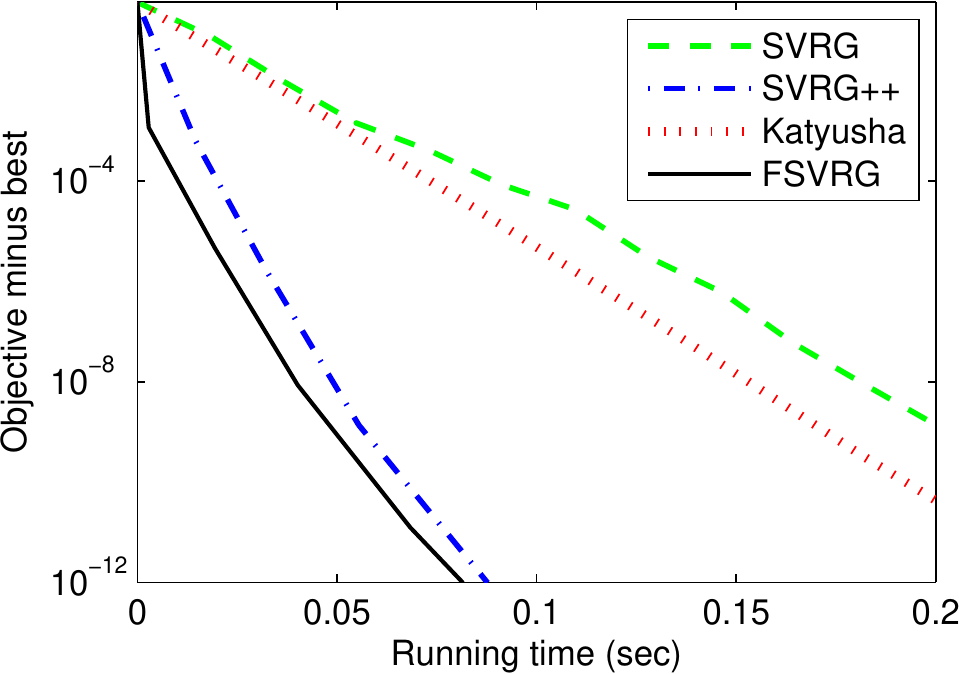}}
\subfigure[Protein: $\lambda\!=\!10^{-3}$]{\includegraphics[width=0.243\columnwidth]{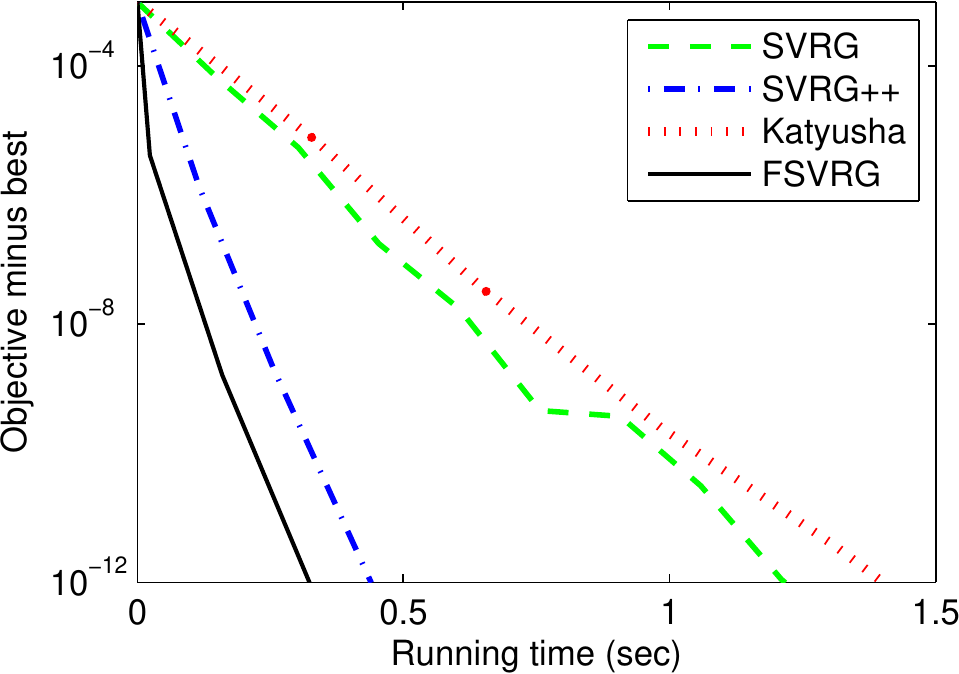}}
\subfigure[Covtype: $\lambda\!=\!10^{-4}$]{\includegraphics[width=0.243\columnwidth]{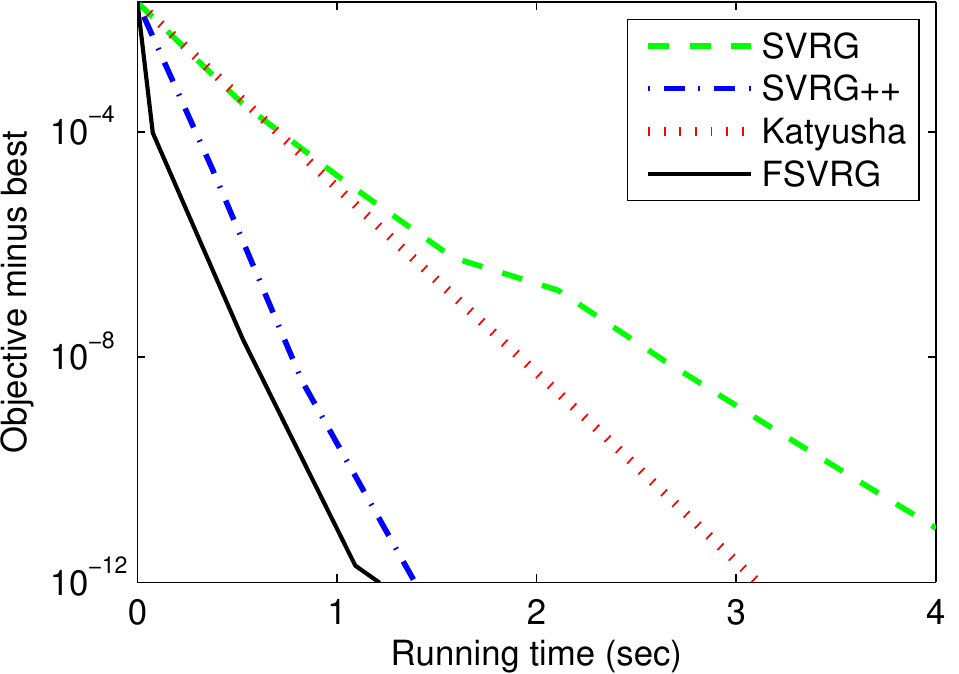}}
\subfigure[SUSY: $\lambda\!=\!10^{-4}$]{\includegraphics[width=0.243\columnwidth]{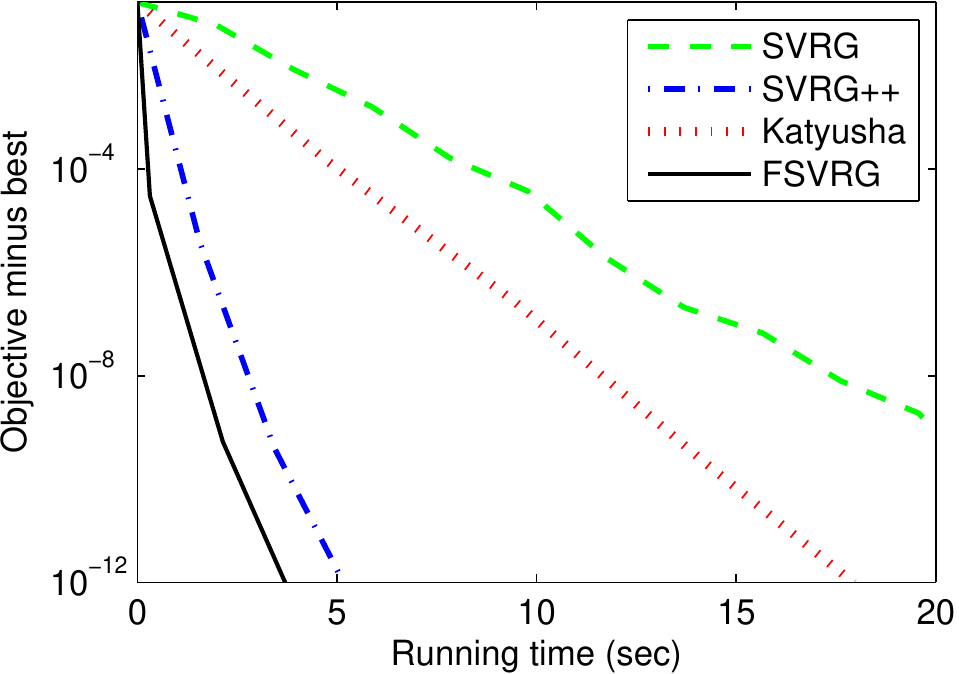}}

\subfigure[IJCNN: $\lambda\!=\!10^{-4}$]{\includegraphics[width=0.243\columnwidth]{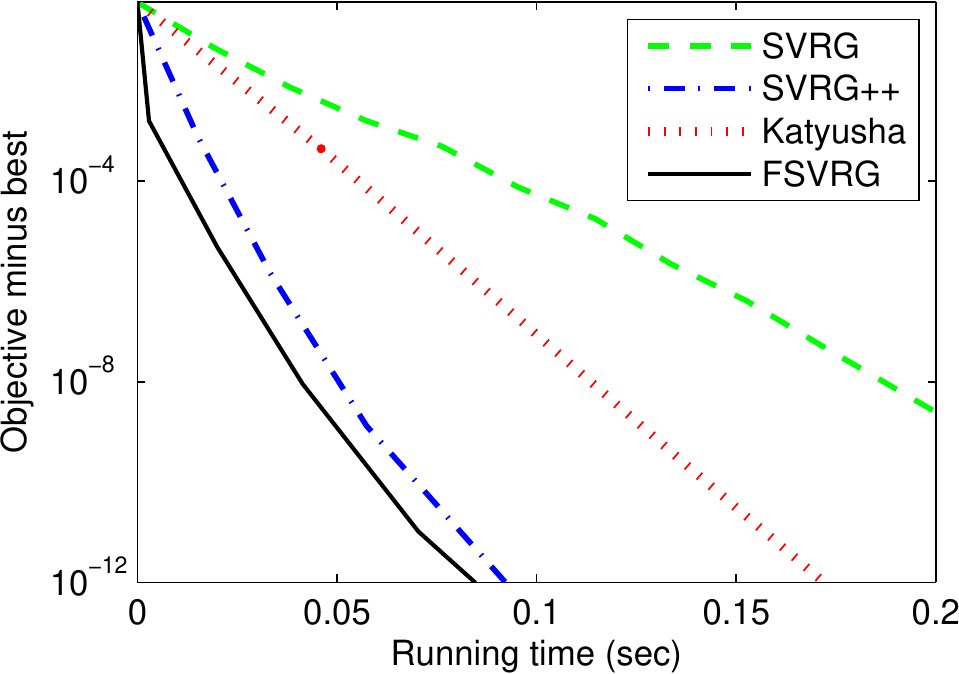}}
\subfigure[Protein: $\lambda\!=\!10^{-4}$]{\includegraphics[width=0.243\columnwidth]{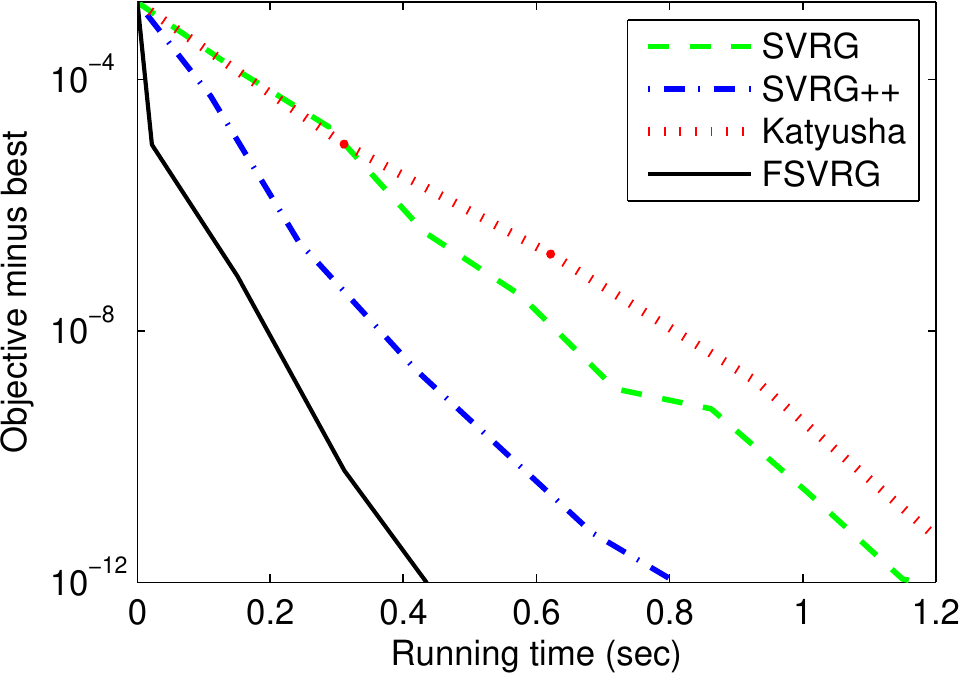}}
\subfigure[Covtype: $\lambda\!=\!10^{-5}$]{\includegraphics[width=0.243\columnwidth]{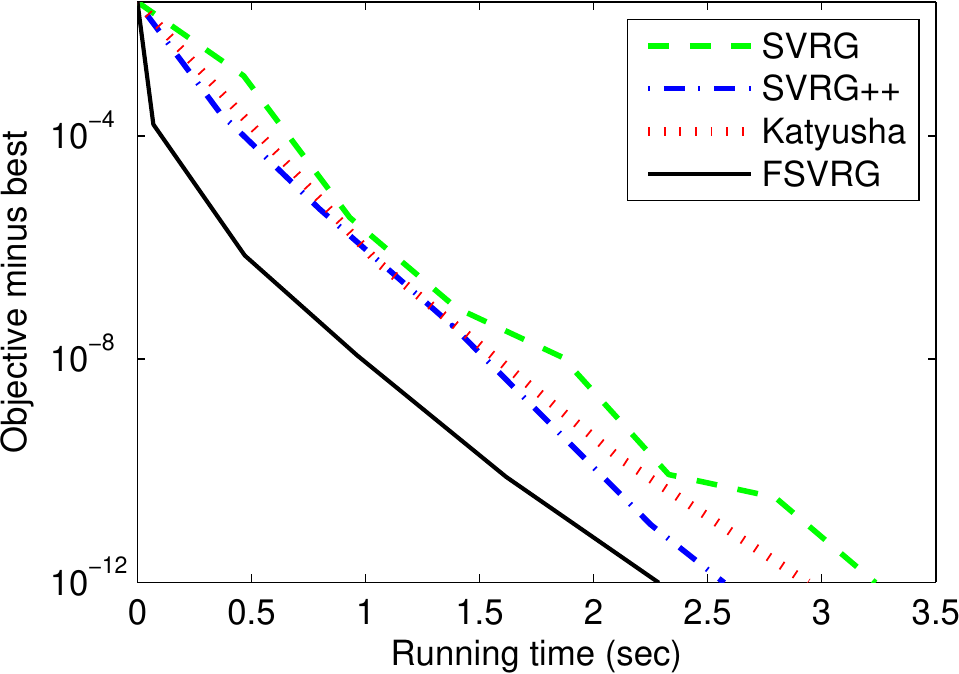}}
\subfigure[SUSY: $\lambda\!=\!10^{-5}$]{\includegraphics[width=0.243\columnwidth]{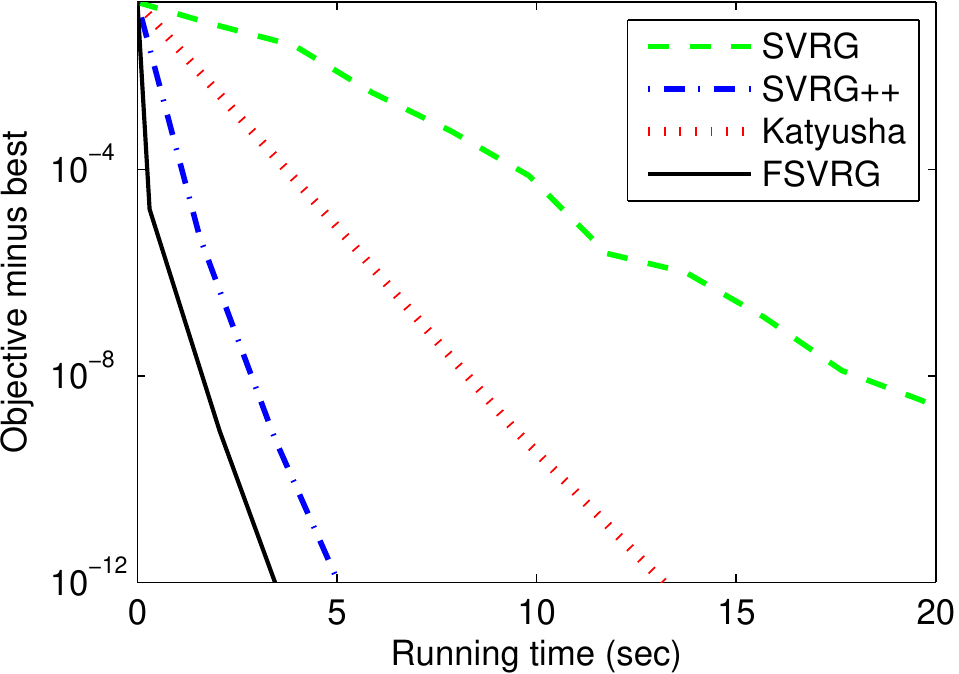}}

\subfigure[IJCNN: $\lambda\!=\!10^{-5}$]{\includegraphics[width=0.243\columnwidth]{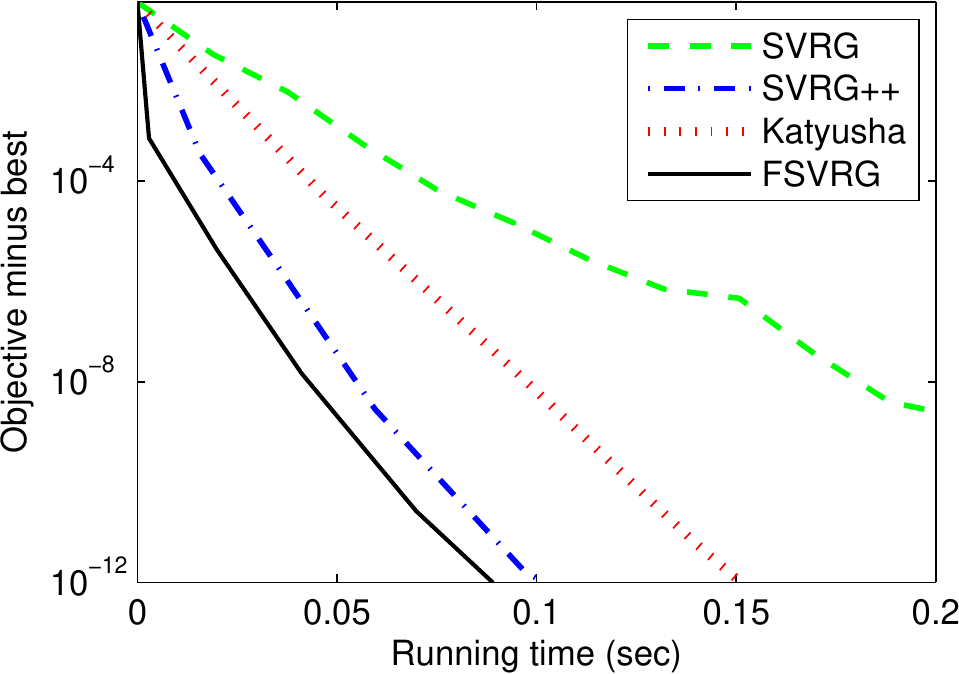}}
\subfigure[Protein: $\lambda\!=\!10^{-5}$]{\includegraphics[width=0.243\columnwidth]{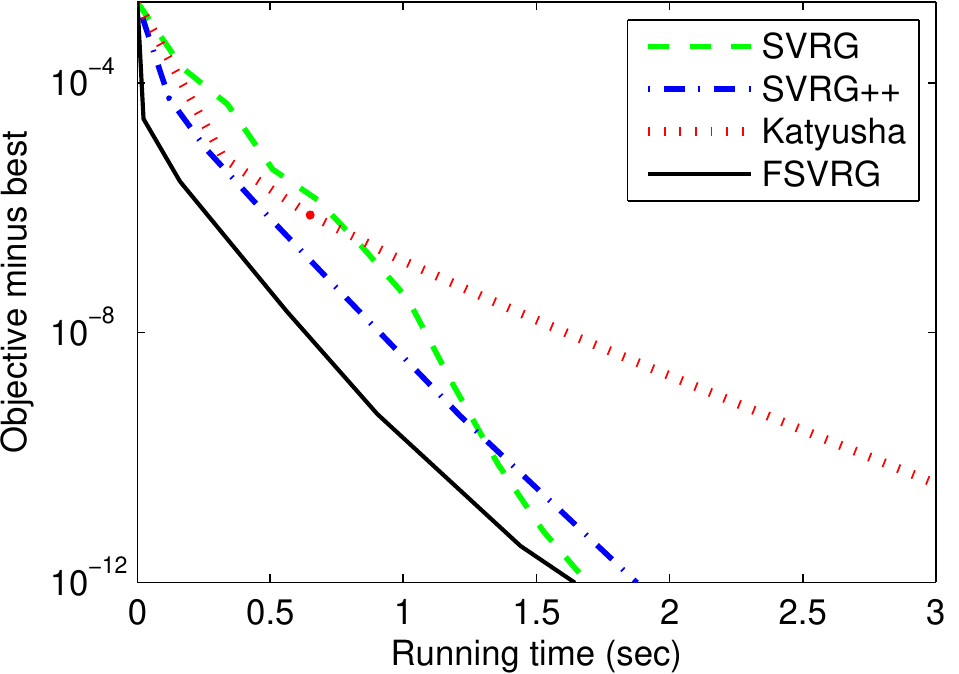}}
\subfigure[Covtype: $\lambda\!=\!10^{-6}$]{\includegraphics[width=0.243\columnwidth]{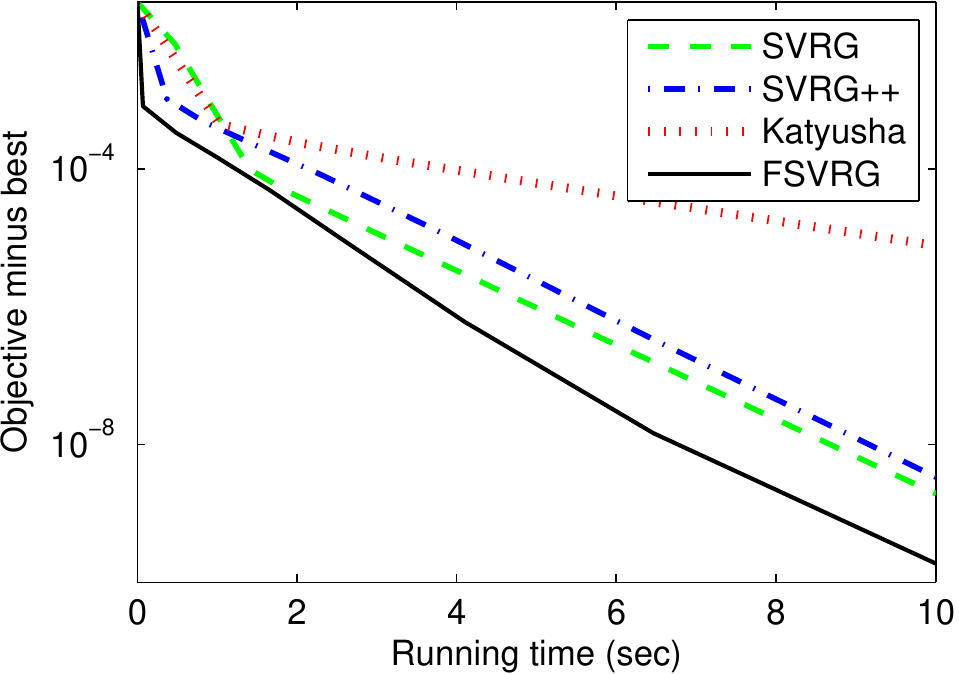}}
\subfigure[SUSY: $\lambda\!=\!10^{-6}$]{\includegraphics[width=0.243\columnwidth]{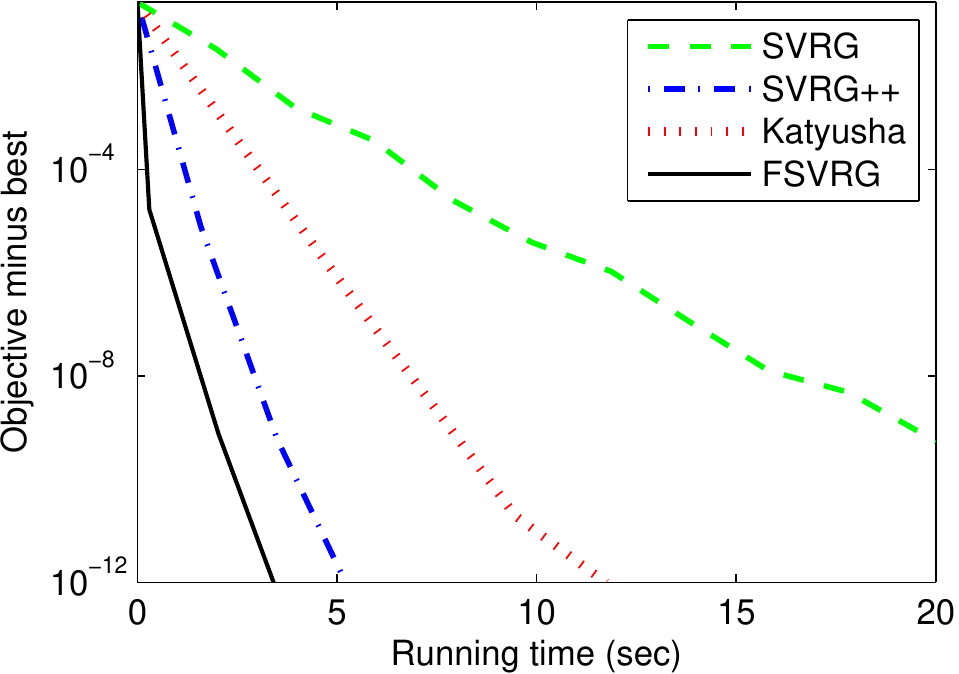}}

\subfigure[IJCNN: $\lambda\!=\!10^{-6}$]{\includegraphics[width=0.243\columnwidth]{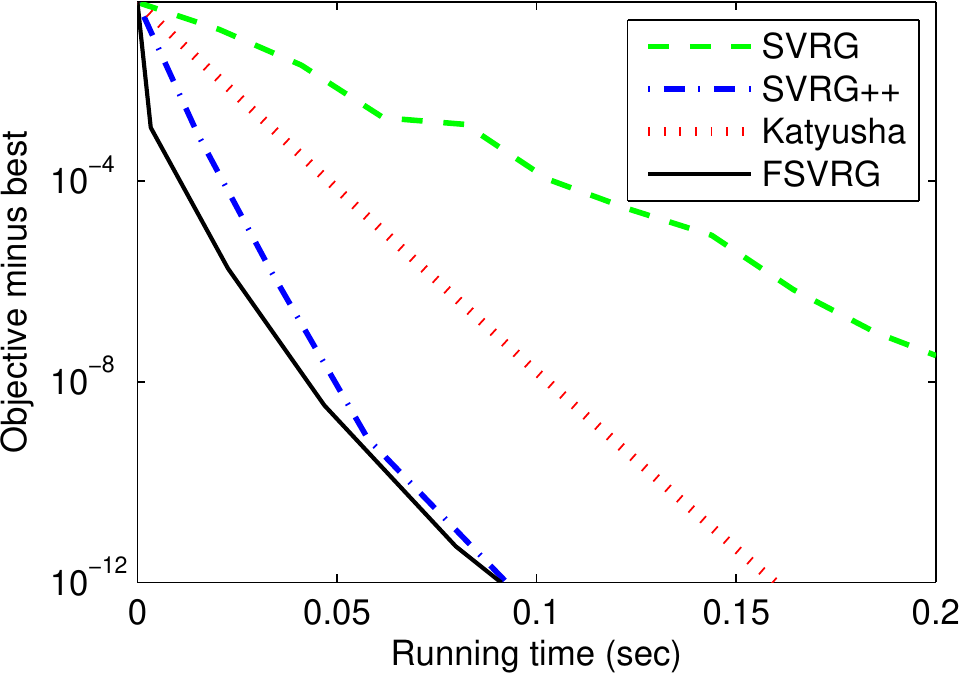}}
\subfigure[Protein: $\lambda\!=\!10^{-6}$]{\includegraphics[width=0.243\columnwidth]{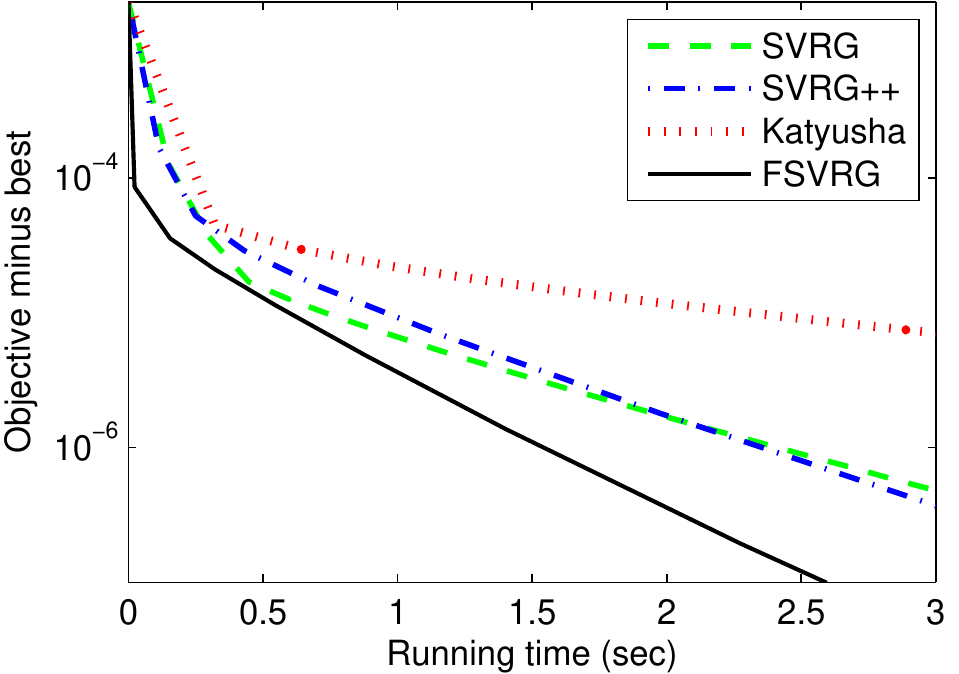}}
\subfigure[Covtype: $\lambda\!=\!10^{-7}$]{\includegraphics[width=0.243\columnwidth]{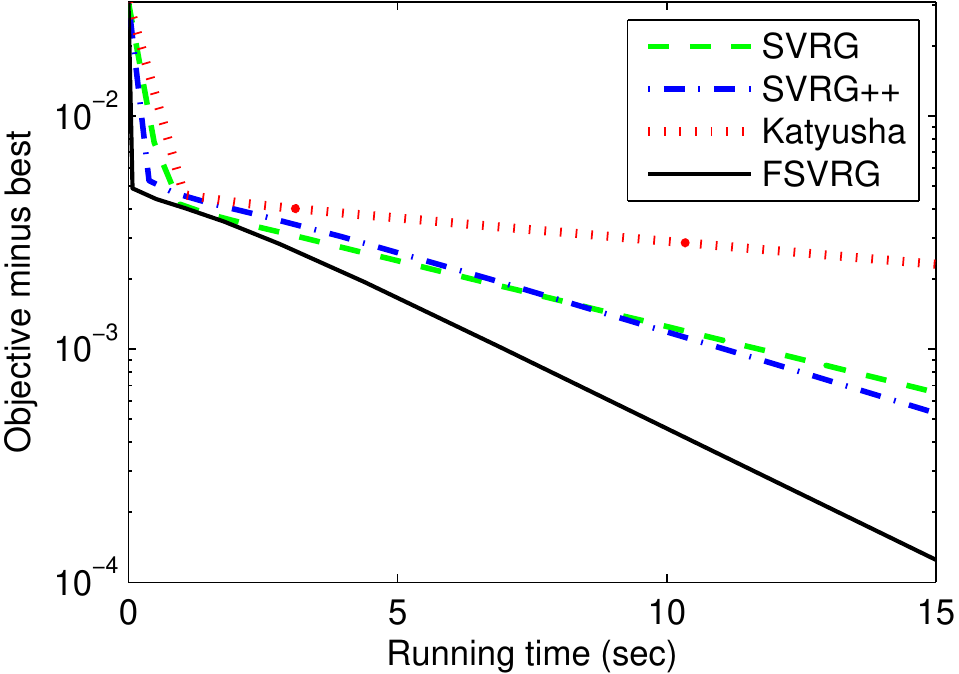}}
\subfigure[SUSY: $\lambda\!=\!10^{-7}$]{\includegraphics[width=0.243\columnwidth]{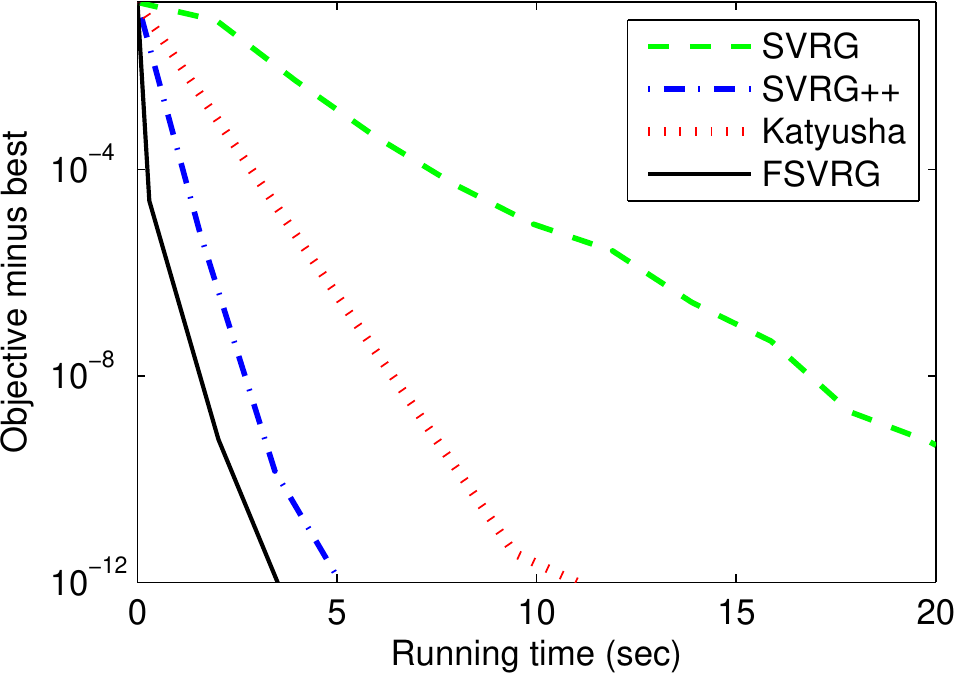}}
\caption{Comparison of SVRG~\cite{johnson:svrg}, SVRG++~\cite{zhu:vrnc}, Katyusha~\cite{zhu:Katyusha}, and FSVRG for solving ridge regression problems with different regularization parameters. The $y$-axis represents the objective value minus the minimum, and the $x$-axis corresponds to the running time (seconds).}
\label{figs12}
\end{figure}

\begin{figure}[!th]
\centering
\subfigure[IJCNN: $\lambda\!=\!10^{-4}$]{\includegraphics[width=0.243\columnwidth]{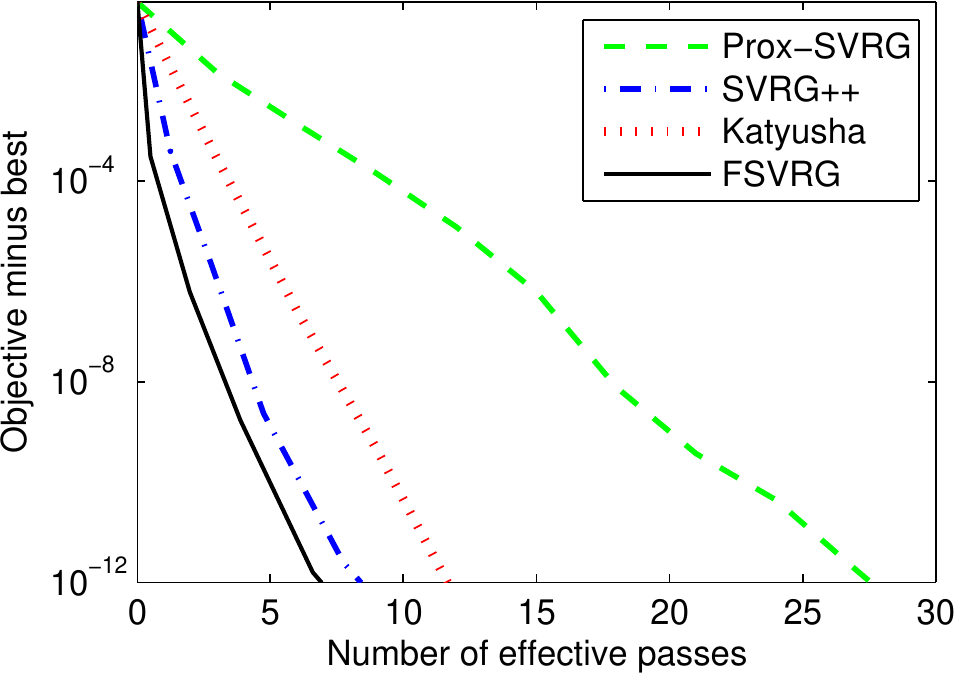}}
\subfigure[Protein: $\lambda\!=\!10^{-3}$]{\includegraphics[width=0.243\columnwidth]{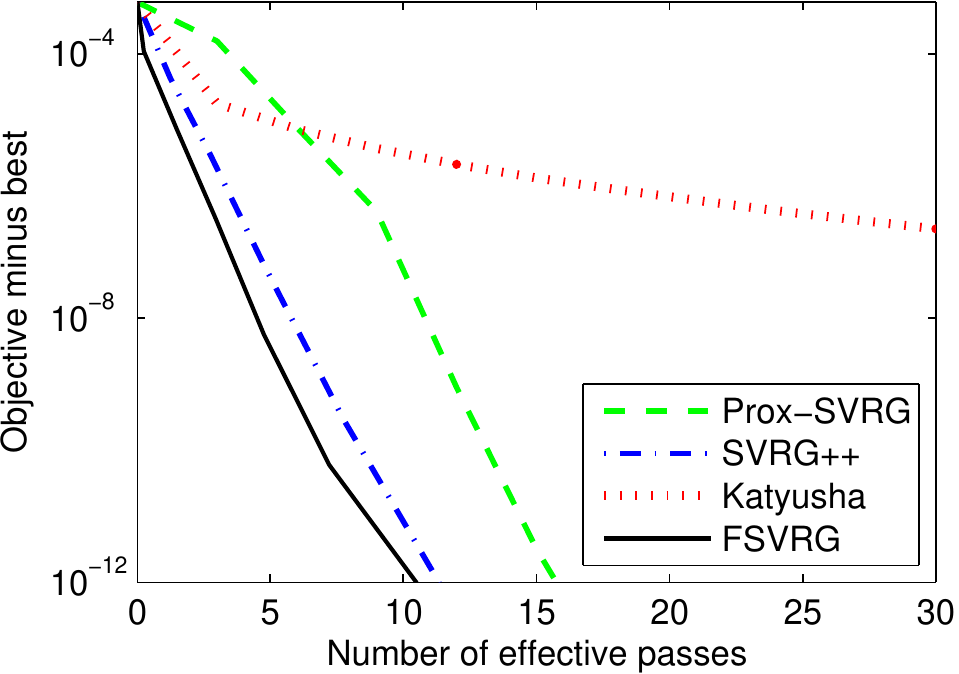}}
\subfigure[Covtype: $\lambda\!=\!10^{-3}$]{\includegraphics[width=0.243\columnwidth]{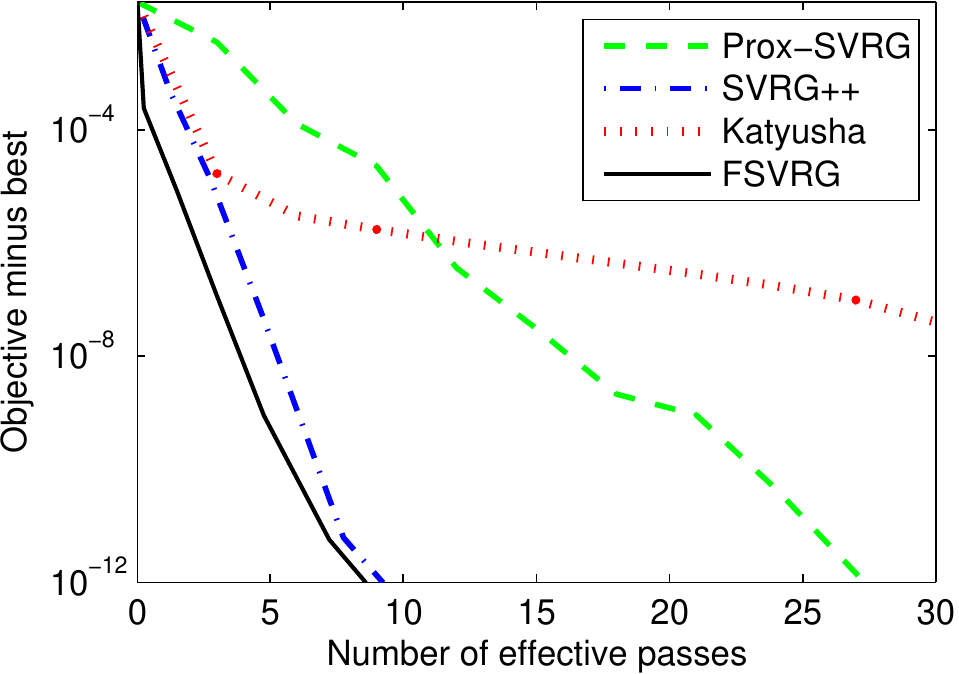}}
\subfigure[SUSY: $\lambda\!=\!10^{-4}$]{\includegraphics[width=0.243\columnwidth]{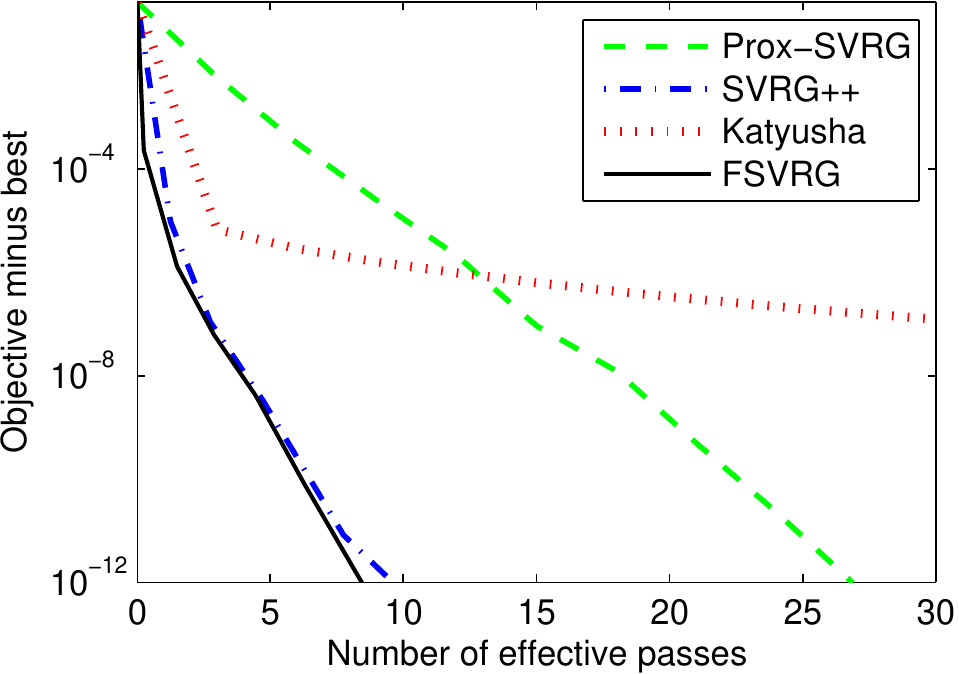}}

\subfigure[IJCNN: $\lambda\!=\!10^{-5}$]{\includegraphics[width=0.243\columnwidth]{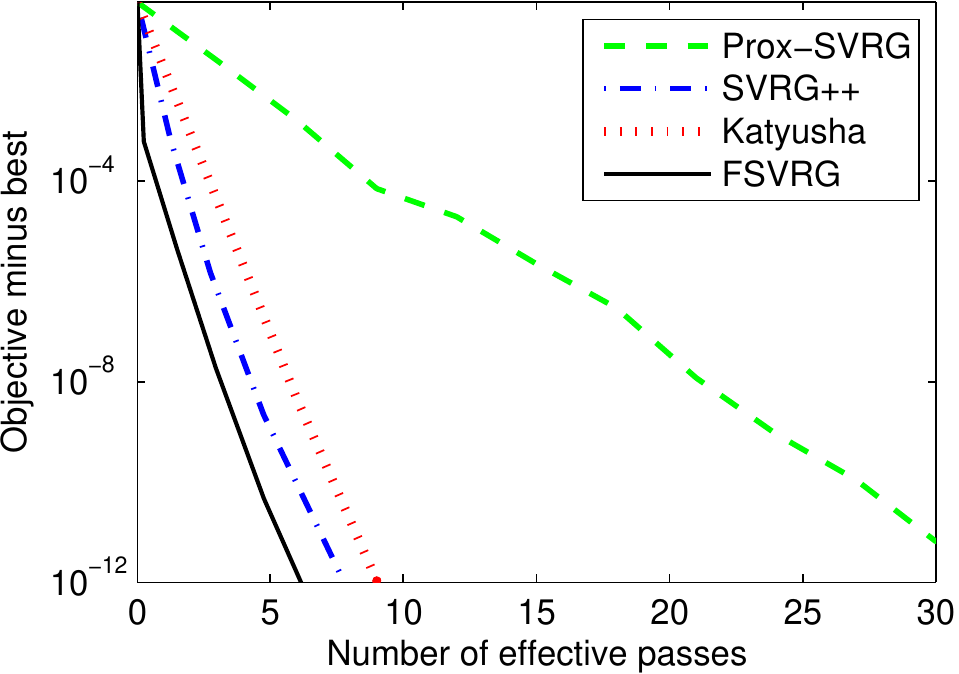}}
\subfigure[Protein: $\lambda\!=\!10^{-4}$]{\includegraphics[width=0.243\columnwidth]{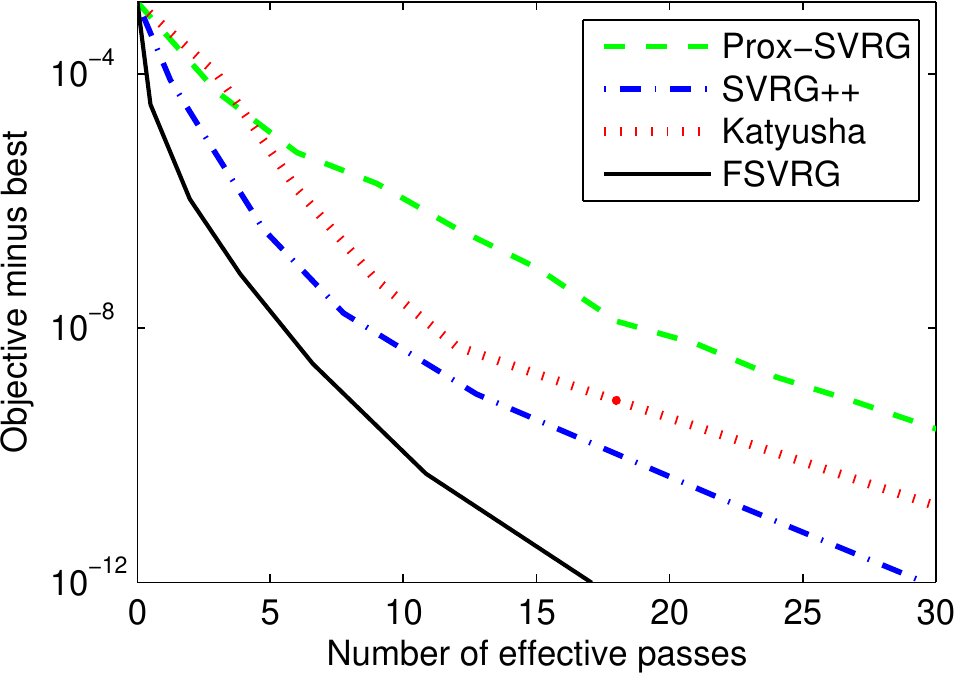}}
\subfigure[Covtype: $\lambda\!=\!10^{-4}$]{\includegraphics[width=0.243\columnwidth]{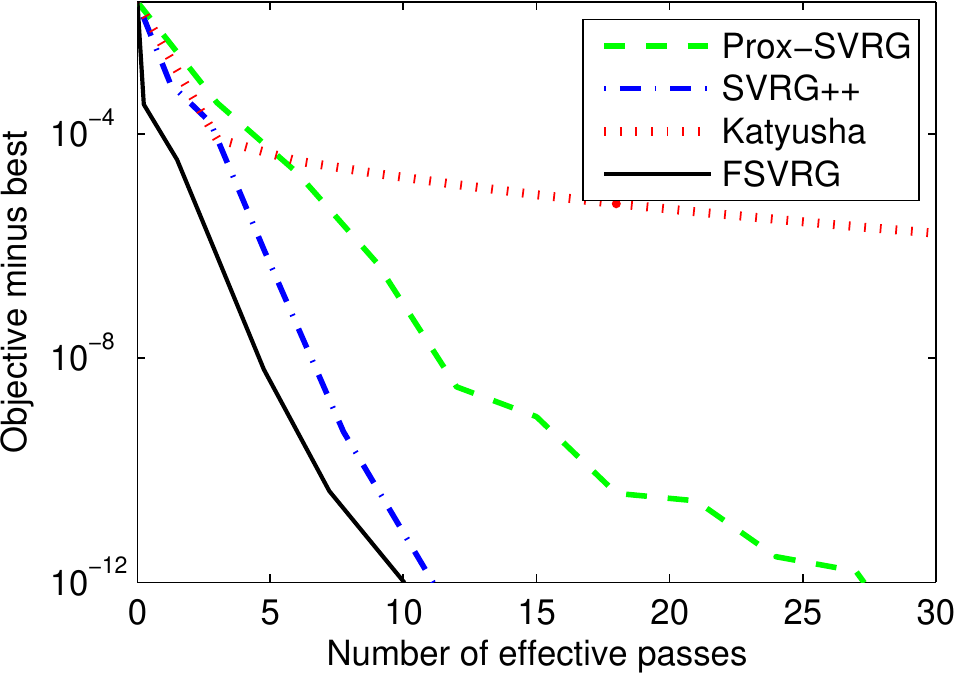}}
\subfigure[SUSY: $\lambda\!=\!10^{-5}$]{\includegraphics[width=0.243\columnwidth]{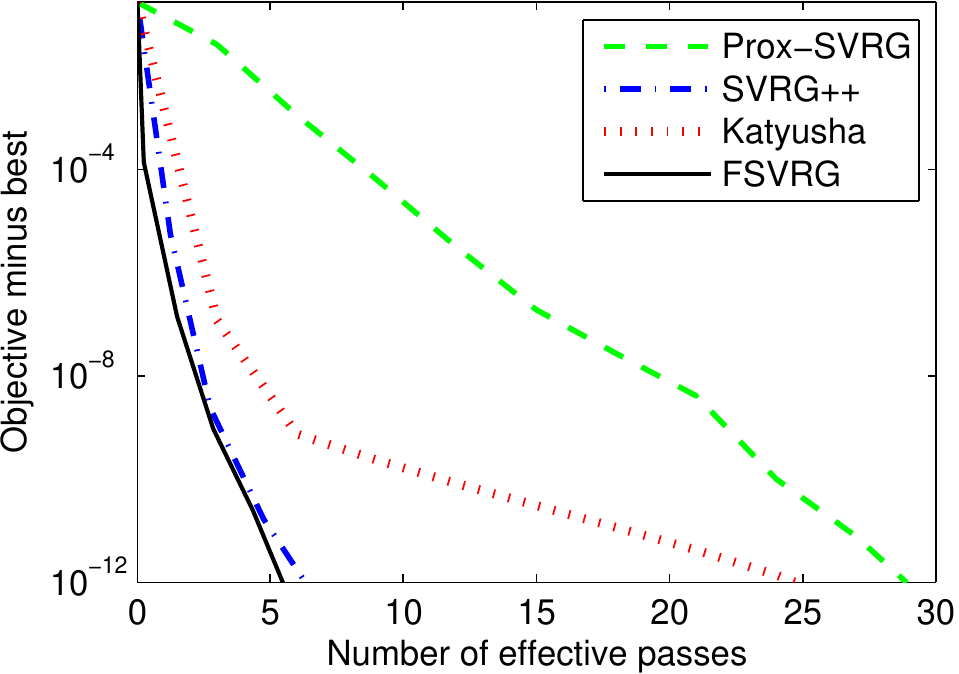}}

\subfigure[IJCNN: $\lambda\!=\!10^{-6}$]{\includegraphics[width=0.243\columnwidth]{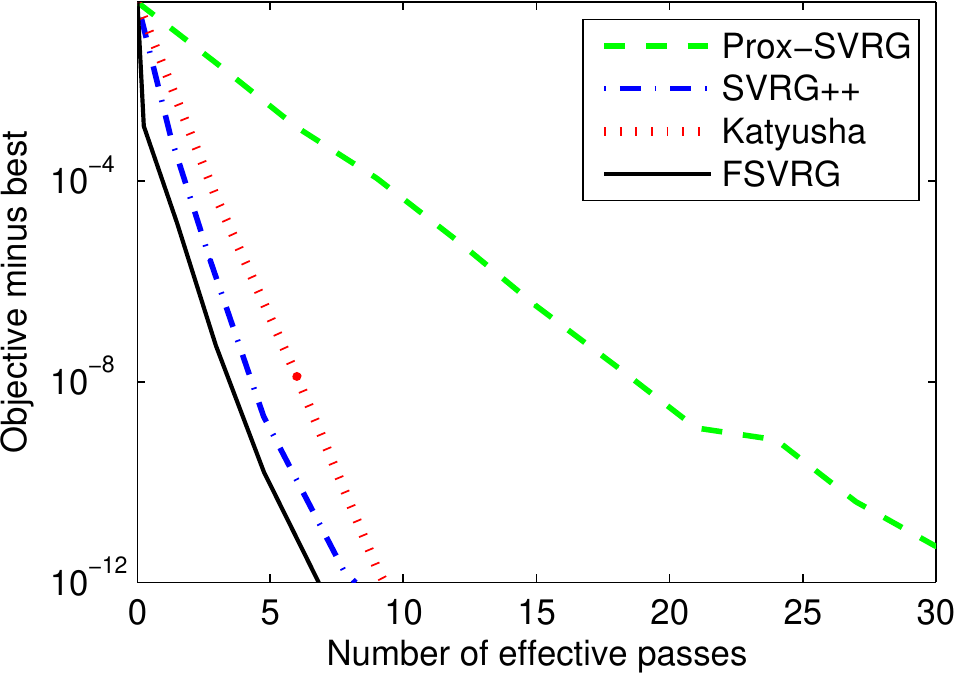}}
\subfigure[Protein: $\lambda\!=\!10^{-5}$]{\includegraphics[width=0.243\columnwidth]{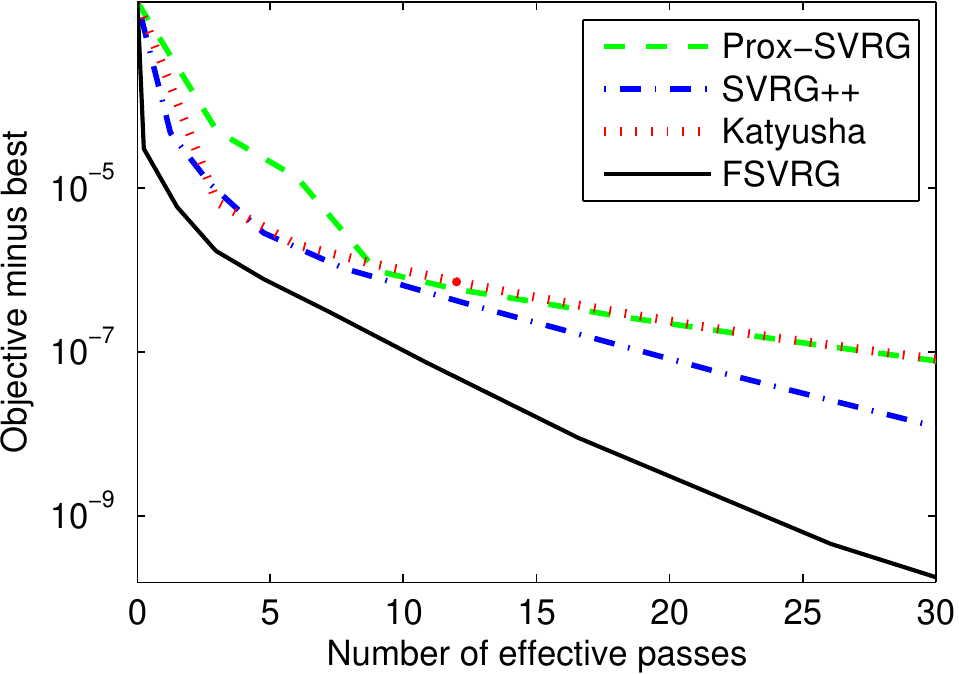}}
\subfigure[Covtype: $\lambda\!=\!10^{-5}$]{\includegraphics[width=0.243\columnwidth]{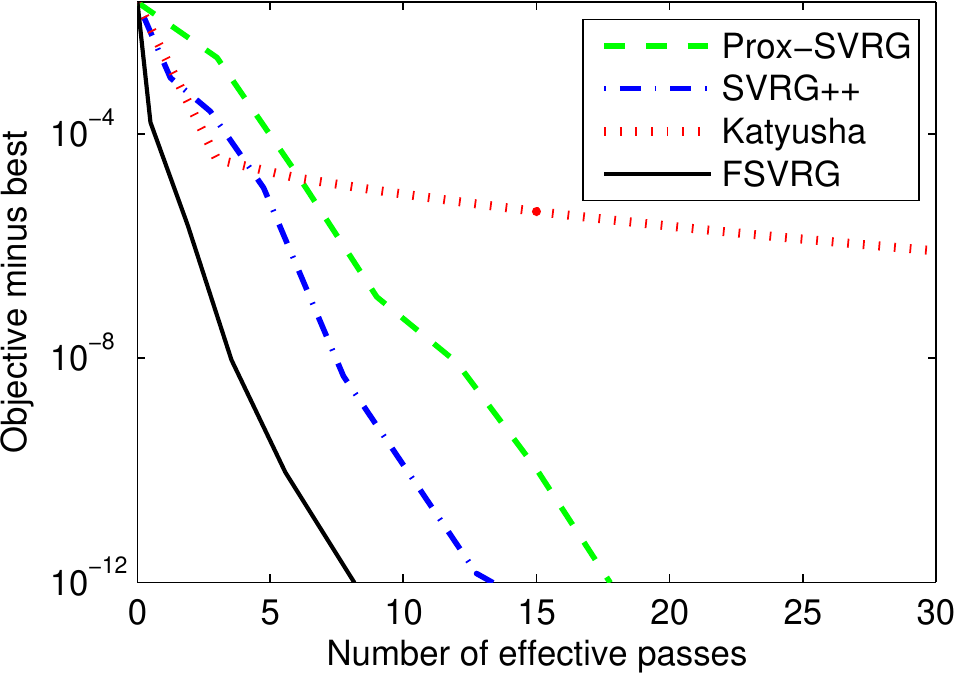}}
\subfigure[SUSY: $\lambda\!=\!10^{-6}$]{\includegraphics[width=0.243\columnwidth]{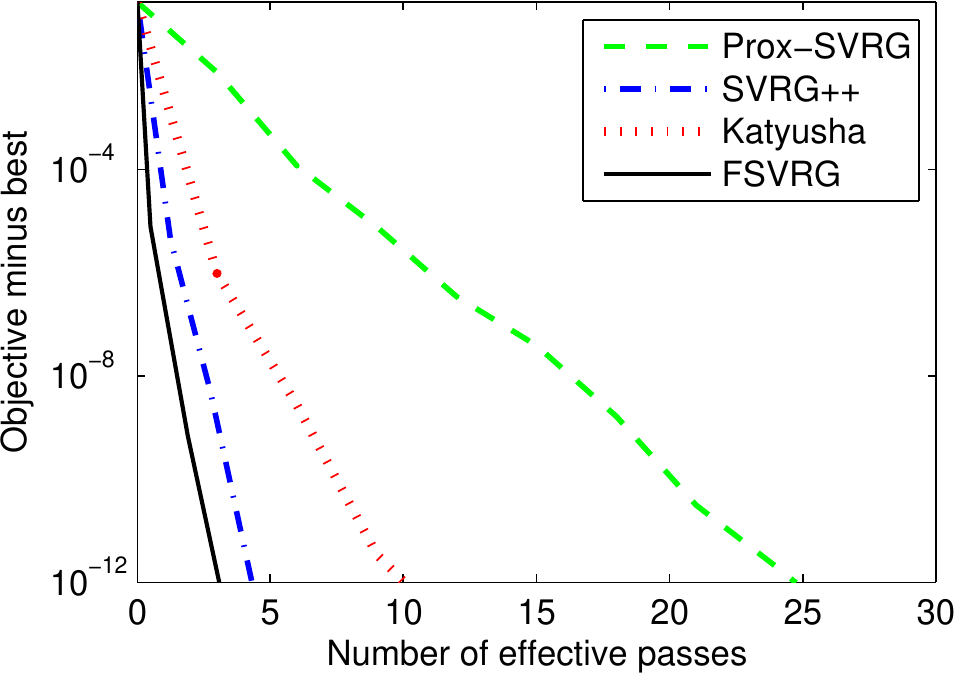}}

\subfigure[IJCNN: $\lambda\!=\!10^{-7}$]{\includegraphics[width=0.243\columnwidth]{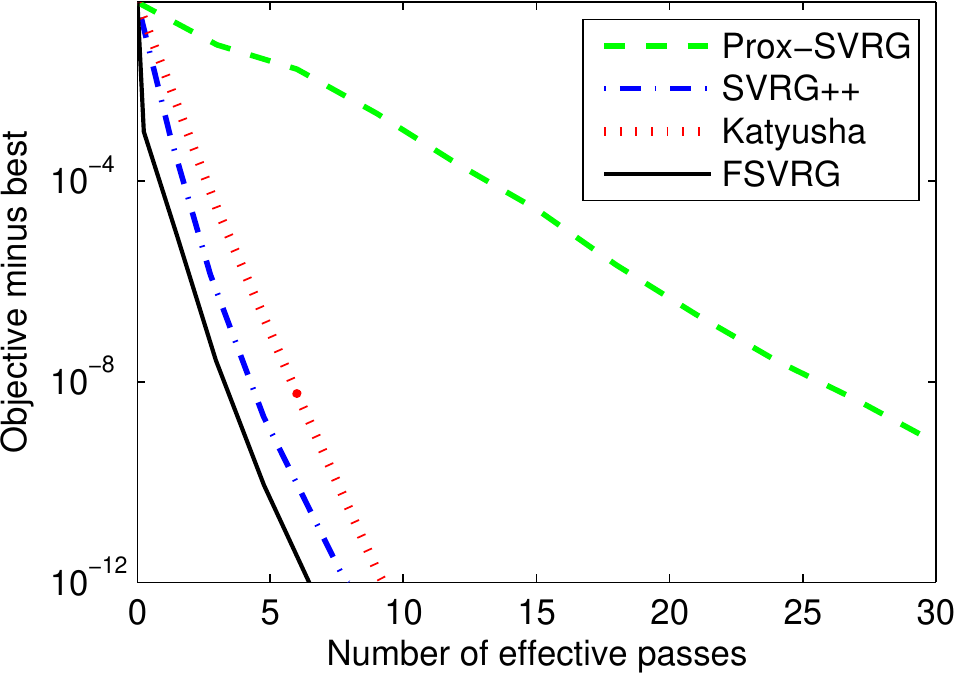}}
\subfigure[Protein: $\lambda\!=\!10^{-6}$]{\includegraphics[width=0.243\columnwidth]{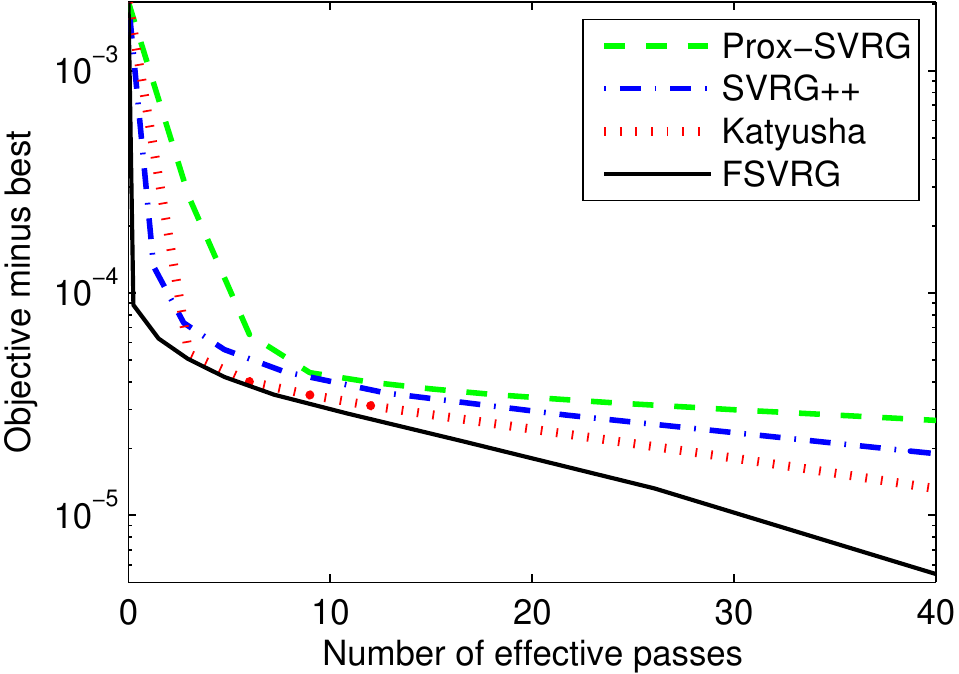}}
\subfigure[Covtype: $\lambda\!=\!10^{-6}$]{\includegraphics[width=0.243\columnwidth]{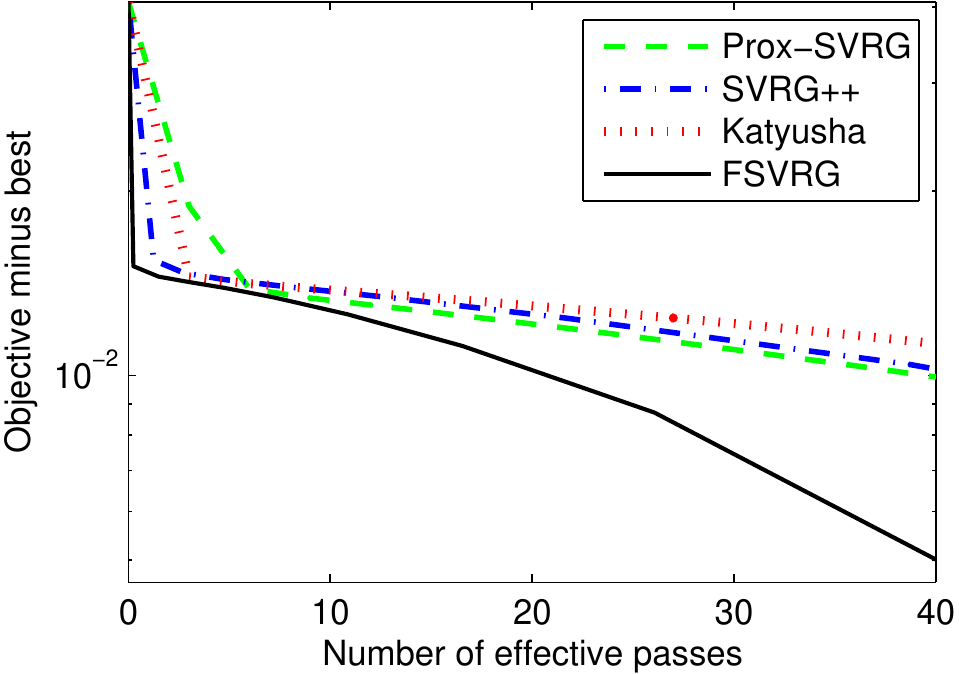}}
\subfigure[SUSY: $\lambda\!=\!10^{-7}$]{\includegraphics[width=0.243\columnwidth]{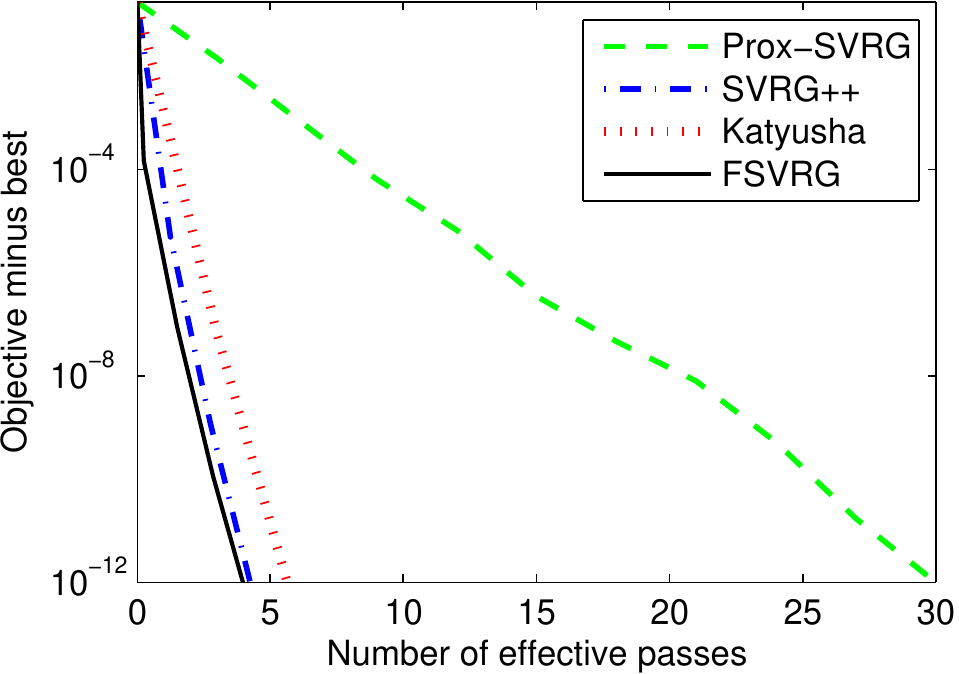}}
\caption{Comparison of Prox-SVRG~\cite{xiao:prox-svrg}, SVRG++~\cite{zhu:vrnc}, Katyusha~\cite{zhu:Katyusha}, and FSVRG for solving Lasso problems on the four data sets: IJCNN, Protein, Covtype, and SUSY. Note that the $y$-axis represents the objective value minus the minimum, and the $x$-axis corresponds to the number of effective passes.}
\label{figs13}
\end{figure}

\begin{figure}[!th]
\centering
\subfigure[IJCNN: $\lambda\!=\!10^{-4}$]{\includegraphics[width=0.243\columnwidth]{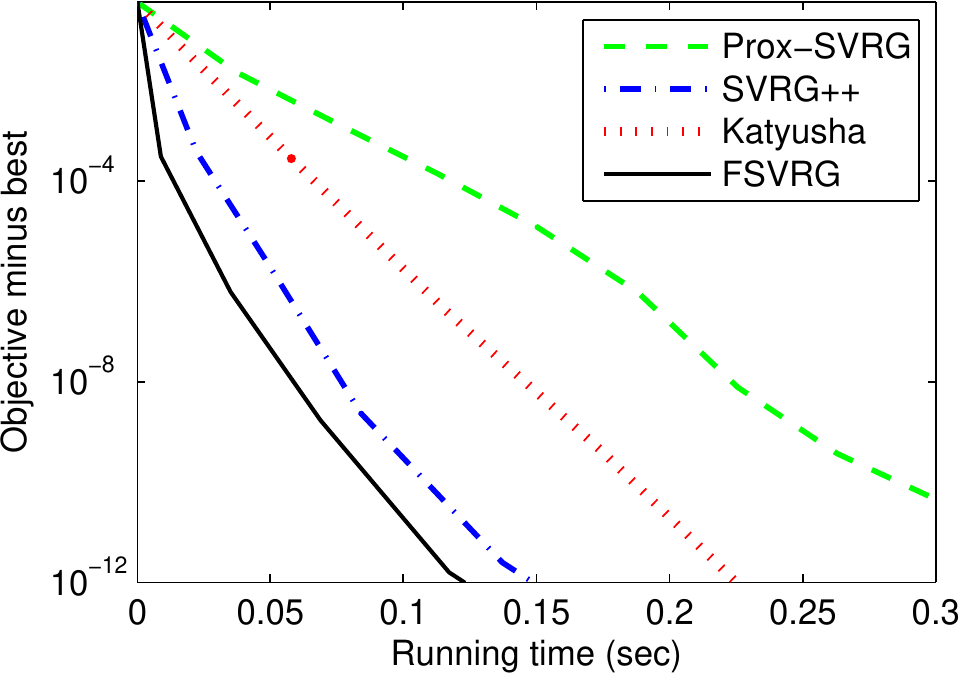}}
\subfigure[Protein: $\lambda\!=\!10^{-3}$]{\includegraphics[width=0.243\columnwidth]{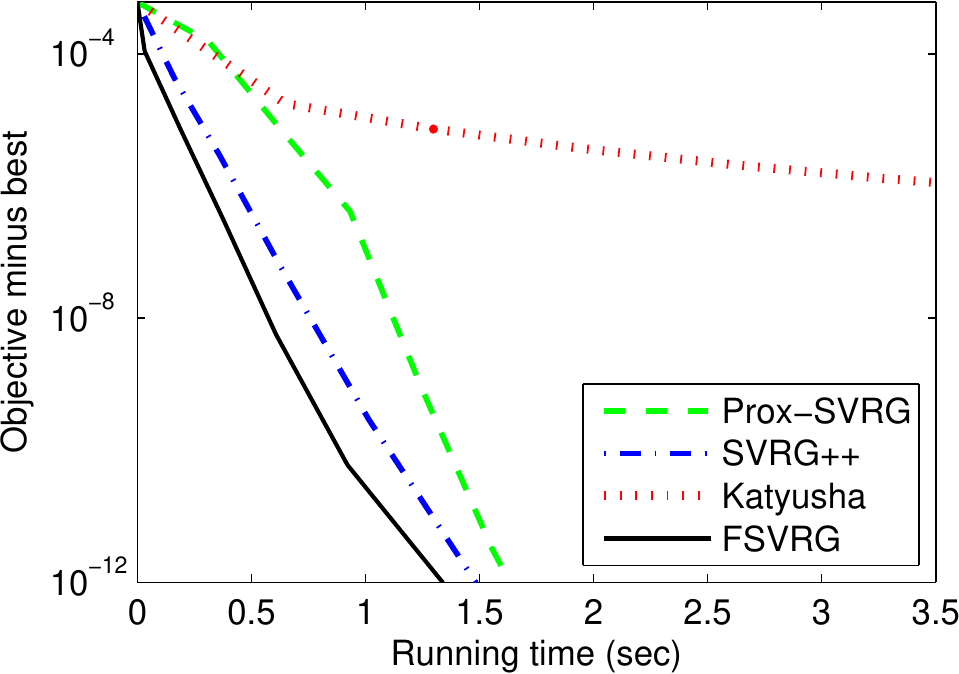}}
\subfigure[Covtype: $\lambda\!=\!10^{-3}$]{\includegraphics[width=0.243\columnwidth]{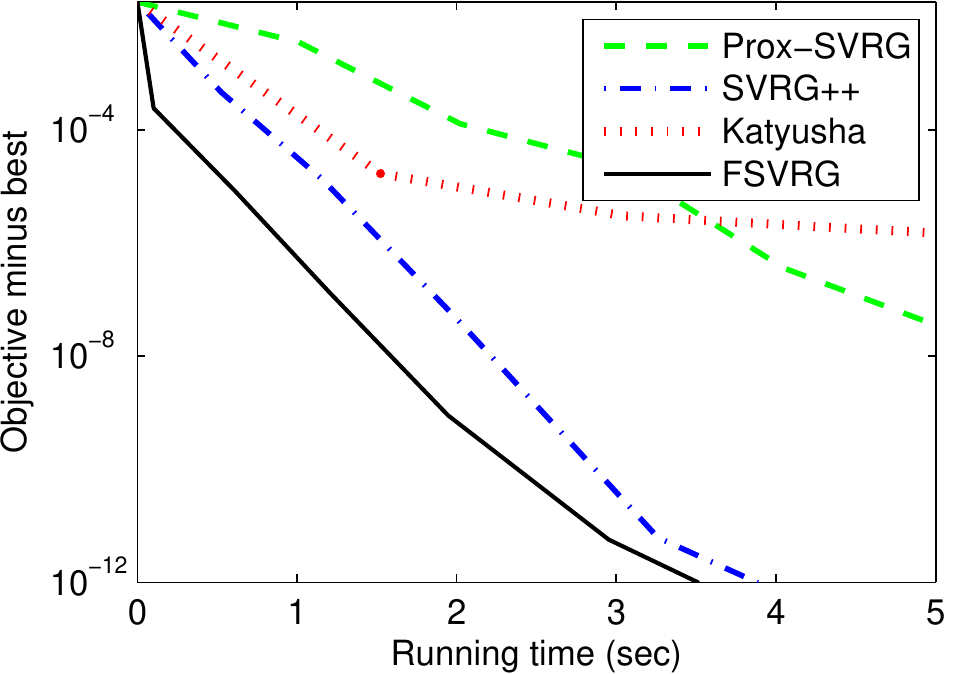}}
\subfigure[SUSY: $\lambda\!=\!10^{-4}$]{\includegraphics[width=0.243\columnwidth]{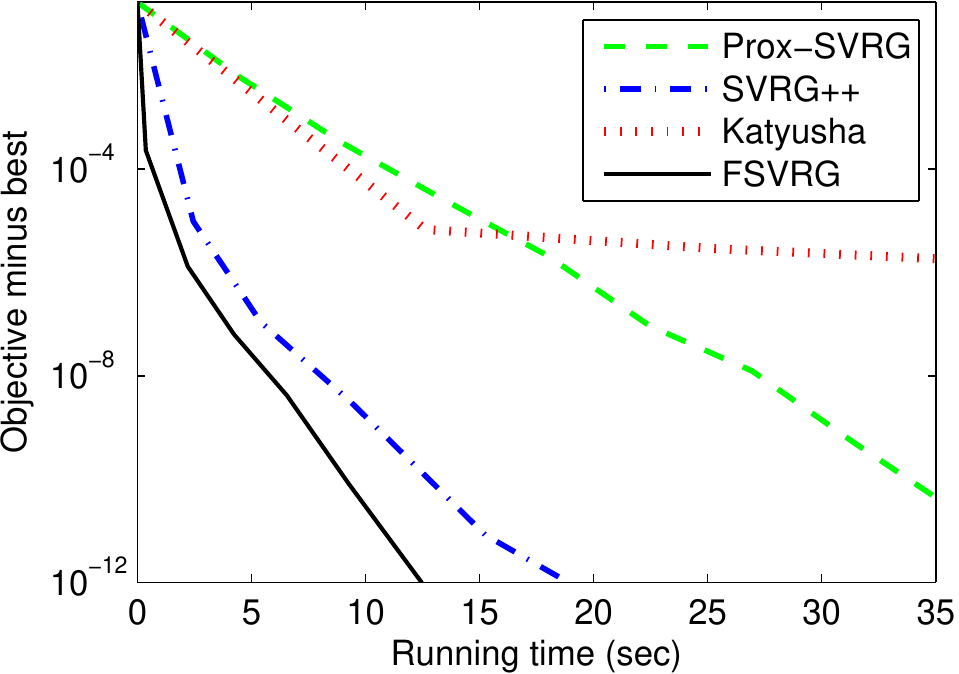}}

\subfigure[IJCNN: $\lambda\!=\!10^{-5}$]{\includegraphics[width=0.243\columnwidth]{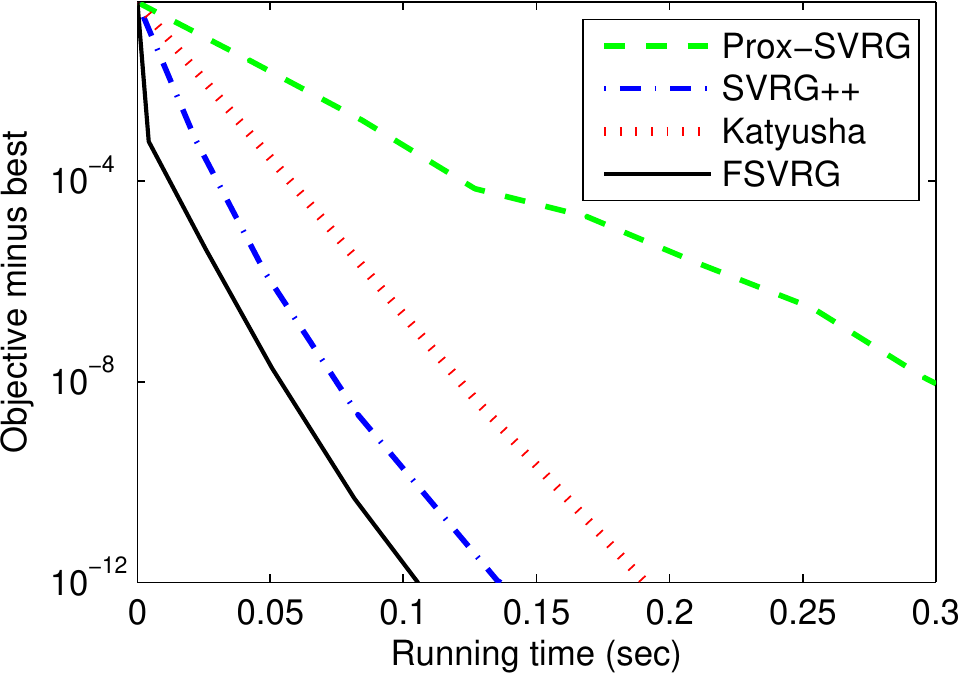}}
\subfigure[Protein: $\lambda\!=\!10^{-4}$]{\includegraphics[width=0.243\columnwidth]{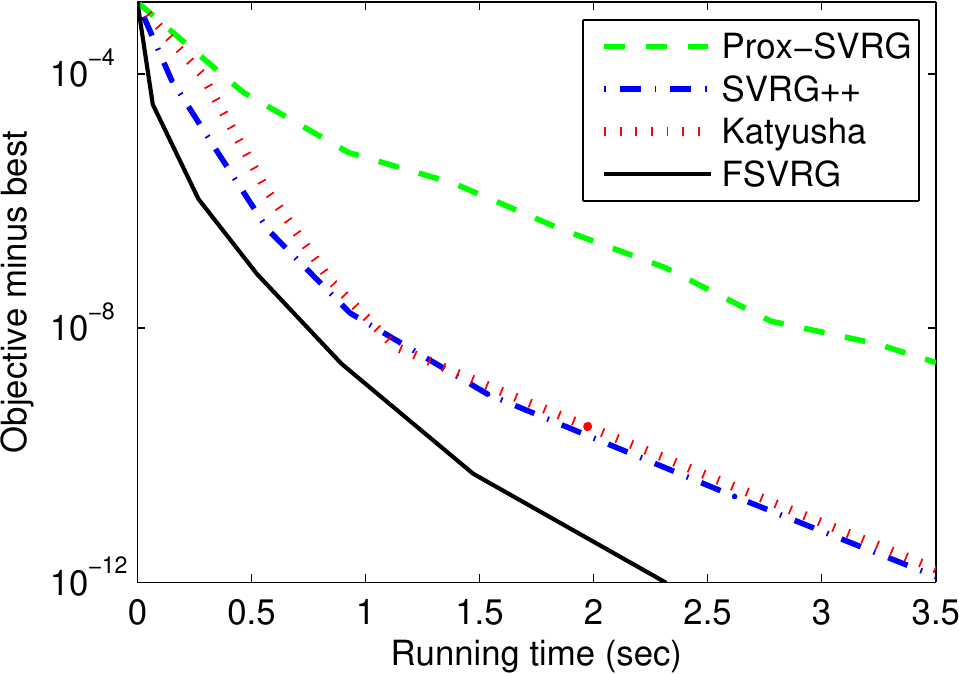}}
\subfigure[Covtype: $\lambda\!=\!10^{-4}$]{\includegraphics[width=0.243\columnwidth]{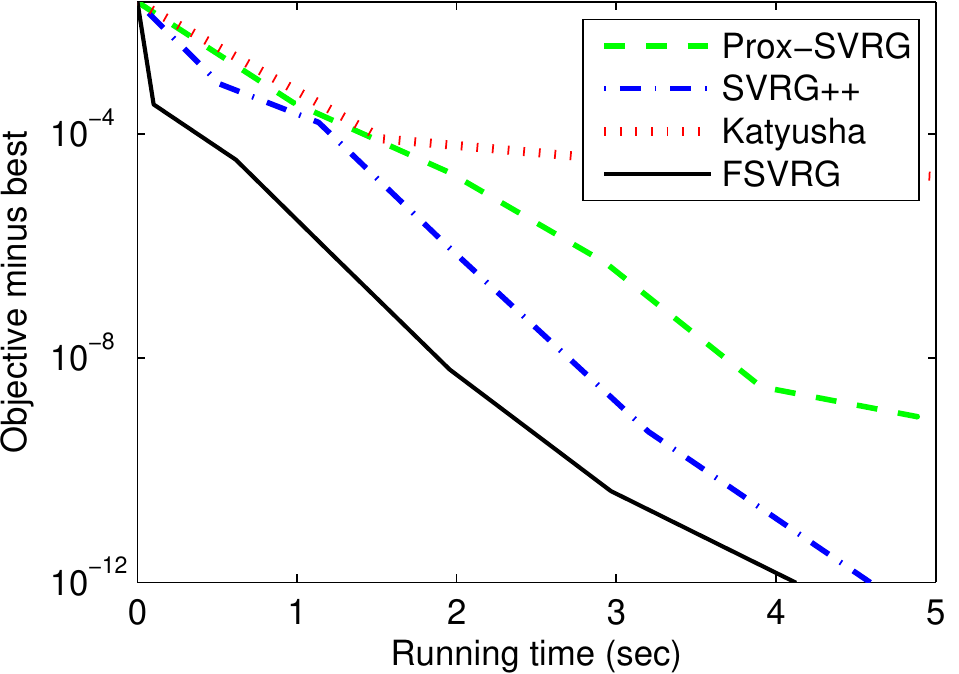}}
\subfigure[SUSY: $\lambda\!=\!10^{-5}$]{\includegraphics[width=0.243\columnwidth]{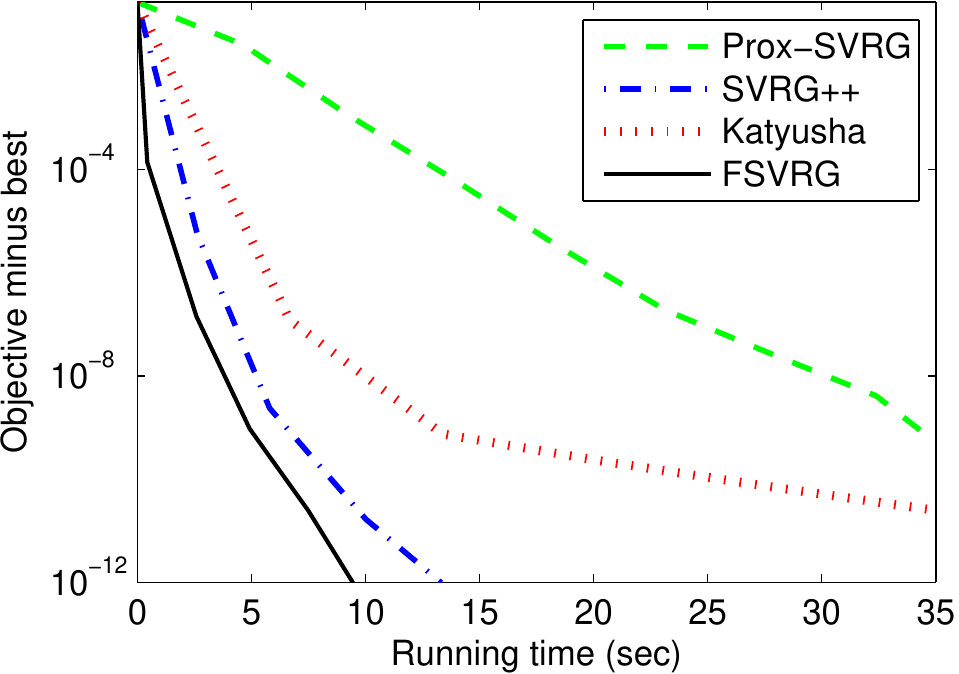}}

\subfigure[IJCNN: $\lambda\!=\!10^{-6}$]{\includegraphics[width=0.243\columnwidth]{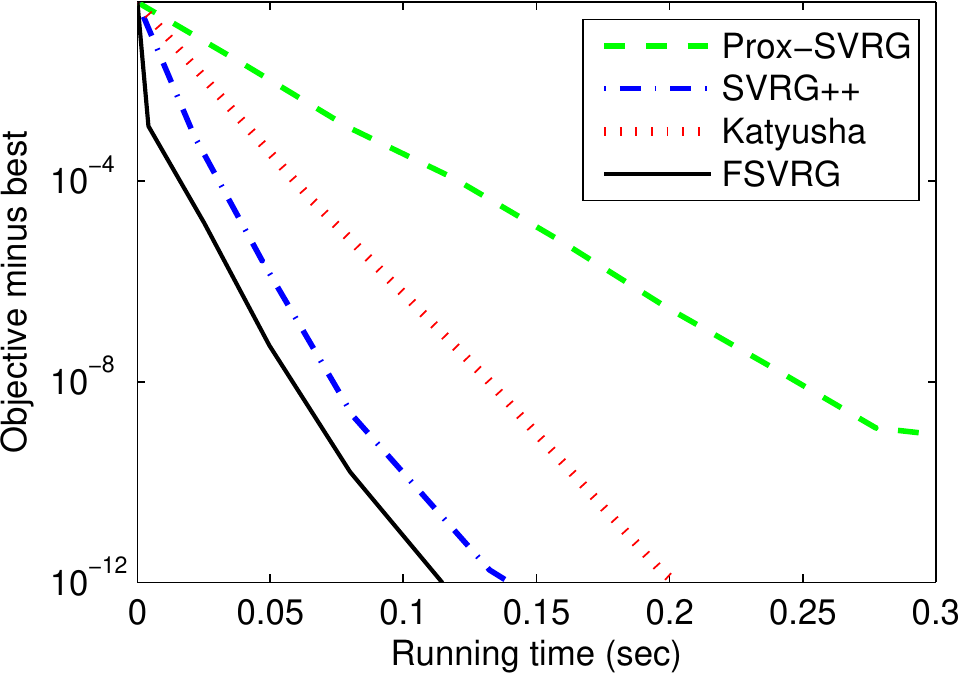}}
\subfigure[Protein: $\lambda\!=\!10^{-5}$]{\includegraphics[width=0.243\columnwidth]{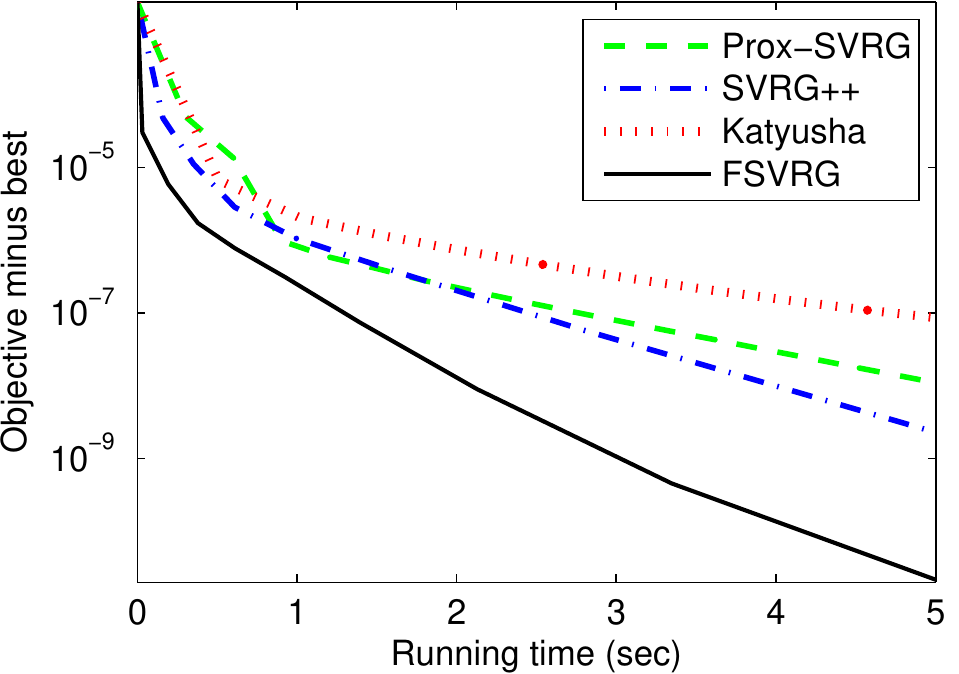}}
\subfigure[Covtype: $\lambda\!=\!10^{-5}$]{\includegraphics[width=0.243\columnwidth]{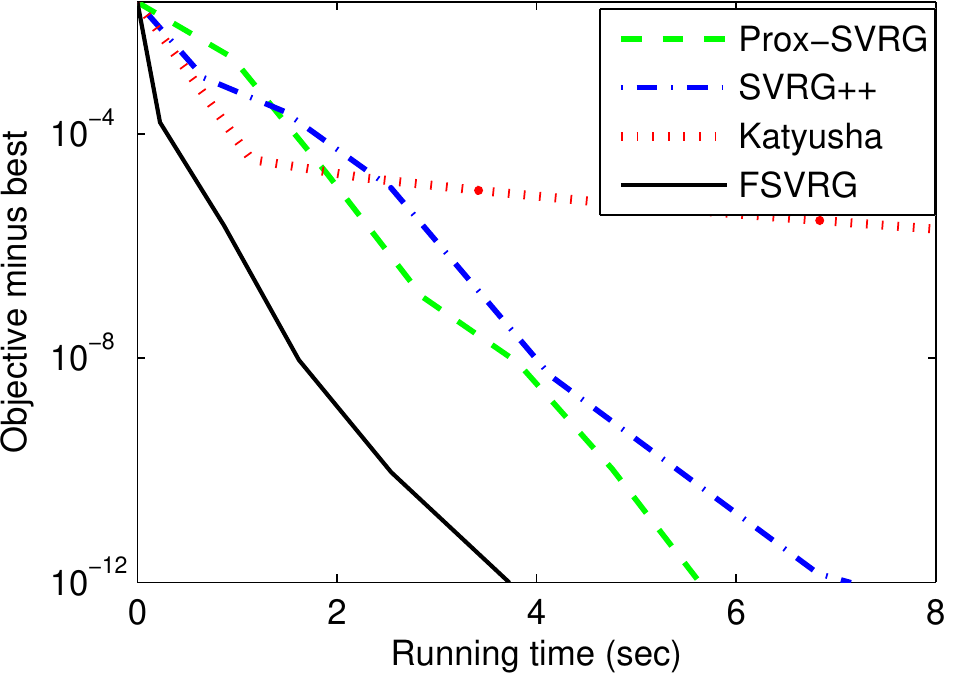}}
\subfigure[SUSY: $\lambda\!=\!10^{-6}$]{\includegraphics[width=0.243\columnwidth]{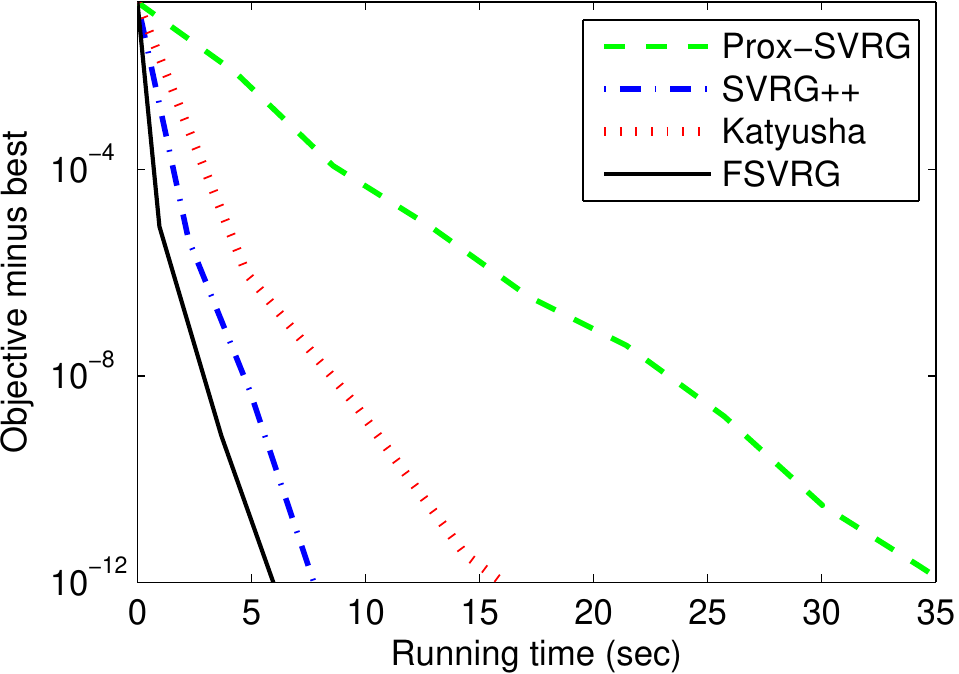}}

\subfigure[IJCNN: $\lambda\!=\!10^{-7}$]{\includegraphics[width=0.243\columnwidth]{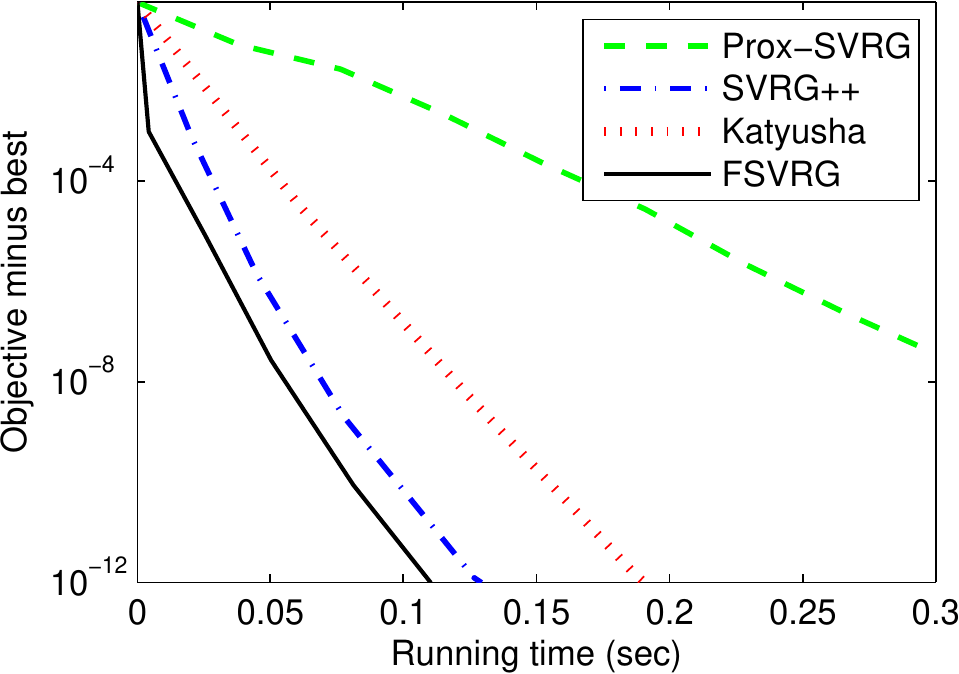}}
\subfigure[Protein: $\lambda\!=\!10^{-6}$]{\includegraphics[width=0.243\columnwidth]{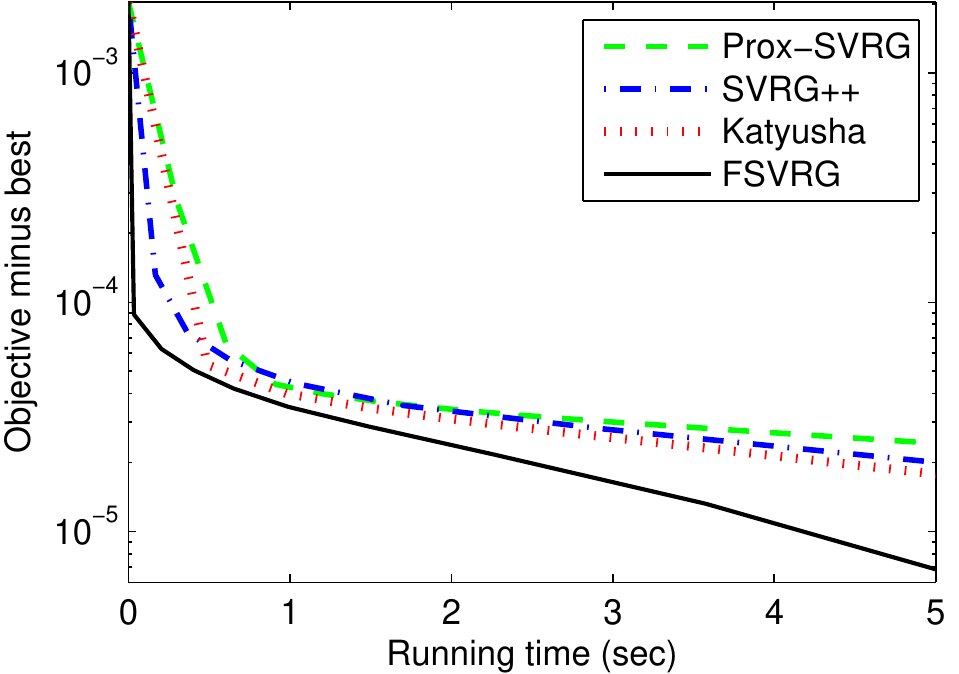}}
\subfigure[Covtype: $\lambda\!=\!10^{-6}$]{\includegraphics[width=0.243\columnwidth]{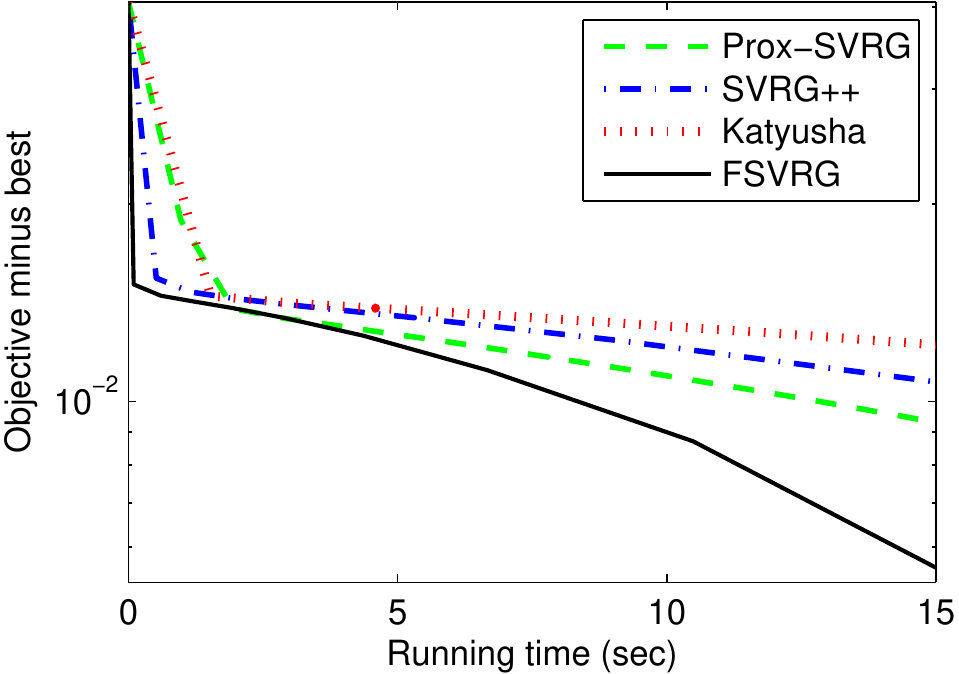}}
\subfigure[SUSY: $\lambda\!=\!10^{-7}$]{\includegraphics[width=0.243\columnwidth]{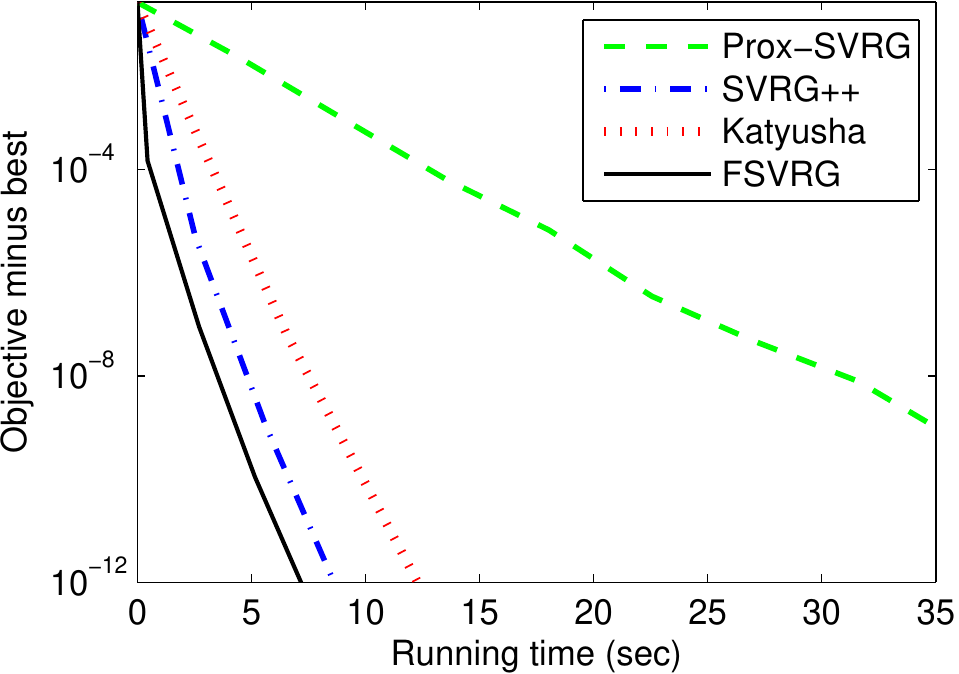}}
\caption{Comparison of Prox-SVRG~\cite{xiao:prox-svrg}, SVRG++~\cite{zhu:vrnc}, Katyusha~\cite{zhu:Katyusha}, and FSVRG for solving Lasso problems on the four data sets: IJCNN, Protein, Covtype, and SUSY. Note that the $y$-axis represents the objective value minus the minimum, and the $x$-axis corresponds to the running time (seconds).}
\label{figs14}
\end{figure}

\end{document}